\documentclass[11pt]{report}
\usepackage[
  dissertation
 ,final
 ,raggedbottom
]{USCthesis}

\usepackage[export]{adjustbox} 
\usepackage{amsmath}
\usepackage{amssymb}
\usepackage{array}
\usepackage[utf8]{inputenc} 
\usepackage[english]{babel}
\usepackage[
  backend     = biber,
  doi         = true,
  hyperref    = true,
  maxbibnames = 99,
  sortlocale  = en_US,
  style       = numeric,
]{biblatex}
\usepackage{booktabs}
\usepackage{color, colortbl}
\usepackage{csquotes}
\usepackage{efbox}
\usepackage[shortcuts]{extdash} 
\usepackage[tt=false]{libertine} 
\usepackage[T1]{fontenc} 
\usepackage[symbol]{footmisc}
\usepackage[
  showframe = false,
  pass      = true, 
]{geometry}
\usepackage{graphicx}
\usepackage{hyphenat}
\usepackage{ifthen}
\usepackage{lipsum}
\usepackage{multirow}
\usepackage{parnotes}
\usepackage{pdflscape} 
\usepackage{pifont}
\usepackage{ragged2e}
\usepackage{seqsplit}
\usepackage{siunitx}
\usepackage{subcaption}
\usepackage{tabularx}
\usepackage{xcolor}
\usepackage{xspace}
\usepackage{url}

\usepackage[
  breaklinks    = true,
  colorlinks    = true,
  hypertexnames = false,
  pdfpagelabels = false,
  citecolor     = {blue!80!black},
  linkcolor     = {blue!80!black},
  urlcolor      = {blue!80!black},
]{hyperref} 

\AtEveryBibitem{\iffieldundef{doi}{}{\clearfield{url}}} 

\addto\extrasenglish{%

}

\usepackage{etoolbox}
\newtoggle{draft}
\settoggle{draft}{false} 

\iftoggle{draft}{
  \overfullrule=10pt                       
}{
  \PassOptionsToPackage{final}{showlabels} 
}

\newtoggle{thesis}
\settoggle{thesis}{true}

\usepackage[inline]{showlabels}

\usepackage{graphics} 
\usepackage{epsfig} 
\usepackage{mathptmx} 
\usepackage{times} 
\usepackage{amsmath} 
\usepackage{amssymb}  
\usepackage{amsthm}
\usepackage{algorithm}
\usepackage[noend]{algpseudocode}
\usepackage{caption}
\usepackage{subcaption}

\usepackage[shortlabels]{enumitem}
\usepackage{multirow}
\usepackage{booktabs}
\usepackage[nocomma]{optidef}
\usepackage{cancel}
\usepackage{color}
\usepackage[normalem]{ulem}
\usepackage{comment}
\usepackage{sidecap}
\usepackage{wrapfig}
\usepackage{siunitx}

\renewcommand{\min}[1]{\underset{#1}{\operatorname{min}}\;}
\newcommand{\minimize}[1]{\underset{#1}{\t{minimize}}\;}
\renewcommand{\max}[1]{\underset{#1}{\operatorname{max}}\;}
\renewcommand{\lim}[1]{\underset{#1}{\operatorname{lim}}\;}
\newcommand{\subjectto}{\text{subject to}}

\newcommand{\etal}{\textit{et al}. }

\usepackage{soul}

\newcommand{\new}[1]{#1}

\newcommand{\smplong}{Planning on Sequenced Manifolds}
\newcommand{\smp}{PSM$^*$}

\renewcommand{\v}{\boldsymbol}
\renewcommand{\t}[1]{{\textrm{#1}}}

\newcommand{\R}{\mathbb{R}}

\newcommand{\opt}{^\star}

\ifdefined\T
\renewcommand{\T}{^\top}
\else
\newcommand{\T}{^\top}
\fi

\newcommand{\mat}[3][.9]{
  \renewcommand{\arraystretch}{#1}{\scriptscriptstyle{\left(
    \hspace*{-1ex}\begin{array}{#2}#3\end{array}\hspace*{-1ex}
  \right)}}\renewcommand{\arraystretch}{1.2}
}

\newcommand{\specialcell}[2][c]{%
	\begin{tabular}[#1]{@{}c@{}}#2\end{tabular}}

\newcommand{\inv}{^{-1}}

\newcommand{\comma}{~,}
\newcommand{\period}{~.}

\newcommand{\ga}{\alpha}
\newcommand{\gb}{\beta}

\renewcommand{\ge}{\epsilon}
\renewcommand{\gg}{\gamma}

\newcommand{\gl}{\lambda}

\newcommand{\gr}{\rho}

\newtheorem{lemma}{Lemma}

\newtheorem{theorem}{Theorem}
\newtheorem{definition}{Definition}

\DeclarePairedDelimiter\abs{\lvert}{\rvert}%
\DeclarePairedDelimiter\norm{\lVert}{\rVert}%

\makeatletter
\let\oldabs\abs
\def\abs{\@ifstar{\oldabs}{\oldabs*}}
\let\oldnorm\norm
\def\norm{\@ifstar{\oldnorm}{\oldnorm*}}
\makeatother

\newcommand{\ecmnnlong}{Equality Constraint Manifold Neural Network}
\newcommand{\ecmnn}{ECoMaNN}
\newcommand{\configspace}{\mathcal{C}}
\newcommand{\dimambient}{n}
\newcommand{\dimconstraint}{l}

\newcommand{\jacobian}{\mathbf{J}}
\newcommand{\jointposition}{\mathbf{q}}

\newcommand{\eye}{\mathbf{I}}

\newcommand{\constraintmanifold}{M}

\newcommand{\constraintfunction}{{h}_{\constraintmanifold}}
\newcommand{\constraintmanifoldjacobian}{\jacobian_\constraintmanifold}
\newcommand{\onconstraintconfigspace}{{\configspace}_{\constraintmanifold}}
\newcommand{\offconstraintconfigspace}{{\configspace}_{\backslash {\constraintmanifold}}}
\newcommand{\numnearestneighbor}{K}
\newcommand{\idxnearestneighbor}{k}
\newcommand{\KNNset}{\mathcal{K}}
\newcommand{\origKNN}{\hat{\KNNset}}
\newcommand{\recenteredKNN}{\tilde{\KNNset}}
\newcommand{\nearestneighborjointposition}{\hat{\jointposition}}
\newcommand{\recenterednearestneighborjointposition}{\tilde{\jointposition}}
\newcommand{\coordframe}{\mathcal{F}}
\newcommand{\orthobasisvec}{\vct{b}}
\newcommand{\orthobasismat}{\mathbf{B}}
\newcommand{\flippedorthobasismat}{\bar{\orthobasismat}}
\newcommand{\randunitvec}{\vct{u}}
\newcommand{\randscalarweight}{w}
\newcommand{\designmatrix}{\mathbf{X}}
\newcommand{\samplecovariancematrix}{\mathbf{S}}

\newcommand{\diagsingularvalues}{\boldsymbol{\Sigma}}
\newcommand{\eigval}{\lambda}
\newcommand{\righteigmat}{\mathbf{V}}
\newcommand{\covdiagsingularvalues}{\diagsingularvalues}
\newcommand{\coveigval}{\eigval}

\newcommand{\covrighteigmat}{\righteigmat}
\newcommand{\lpcacoordframe}{\coordframe_L}
\newcommand{\normalspaceid}{N}
\newcommand{\normalspaceatjointposition}{\normalspaceid_{\jointposition}\constraintmanifold}
\newcommand{\lnormalcoordframe}{\coordframe_\normalspaceid}
\newcommand{\coveigvec}{\vct{v}}
\newcommand{\constrjaceigvec}{\vct{e}}
\newcommand{\nullspaceid}{\text{N}}
\newcommand{\covnullspaceeigmat}{\righteigmat_{\nullspaceid}}
\newcommand{\covnullspaceolmat}{\mathbf{W}_{\nullspaceid}}
\newcommand{\flippedrighteigmat}{\bar{\righteigmat}}
\newcommand{\flippedcovnullspaceeigmat}{\flippedrighteigmat_{\nullspaceid}}
\newcommand{\constrjacrighteigmat}{\mathbf{E}}
\newcommand{\constrjacnullspaceeigmat}{\constrjacrighteigmat_{\nullspaceid}}
\newcommand{\augjointposition}{\check{\jointposition}}
\newcommand{\augidx}{i}
\newcommand{\augmagnitude}{\epsilon}

\newcommand{\loss}{\mathcal{L}}
\newcommand{\normloss}{\loss_\text{norm}}
\newcommand{\siamreflectionloss}{\loss_\text{reflection}}
\newcommand{\siamfracloss}{\loss_\text{fraction}}
\newcommand{\siamsimilarloss}{\loss_\text{similar}}
\newcommand{\osaloss}{\loss_\text{osa}}
\newcommand{\rotmat}{\mathbf{R}}
\newcommand{\diffson}{\rotmat_{\normalspaceid}}
\newcommand{\globalalignmentrotmat}{\rotmat_{G}}
\newcommand{\skewsymmmat}{\mathbf{L}}
\newcommand{\mstnumnearestneighbor}{H}
\newcommand{\nnsparsegraph}{\mathcal{G}}
\newcommand{\mst}{\mathcal{T}}
\newcommand{\dagedges}{\mathcal{E}}
\newcommand{\rootid}{r}
\newcommand{\rootjointposition}{\jointposition_{\rootid}}
\newcommand{\directededge}{\vct{e}}
\newcommand{\alignedcovnullspaceeigmat}{\covnullspaceeigmat^{aligned}}
\newcommand{\currentnodeid}{d}
\newcommand{\parentnodeid}{p}

\newcommand{\argmin}{\operatorname*{arg\:min}}
\newcommand{\argmax}{\operatorname*{arg\:max}}

\renewcommand{\Re}{\mathbb{R}}

\newcommand{\vct}[1]{\boldsymbol{#1}} 

\renewcommand{\ge}{\epsilon}
\renewcommand{\gg}{\gamma}
\renewcommand{\t}[1]{{\textrm{#1}}}

\newcommand{\gsutast}{\bgroup\markoverwith{\textcolor{green}{\rule[0.5ex]{2pt}{0.4pt}}}\ULon}
\newcommand{\ragst}{\bgroup\markoverwith{\textcolor{magenta}{\rule[0.5ex]{2pt}{0.4pt}}}\ULon}

\usepackage{color}
\usepackage{graphicx}
\usepackage{graphbox}
\usepackage{fixmath}

\usepackage{subcaption}
\usepackage{ifthen}
\usepackage{xspace}
\PassOptionsToPackage{hyphens}{url}
\usepackage{cleveref}

\newcommand{\gtlocations}{\mathbold{G^\#}}
\newcommand{\gtsensedlocations}{\mathbold{X^\#}}
\newcommand{\gtsensedvalues}{\mathbold{Y^\#}}

\newcommand{\sensedlocations}{\mathbold{X}}
\newcommand{\sensedvalues}{\mathbold{Y}}
\newcommand{\location}{g}
\newcommand{\sensedlocation}{x}

\newcommand{\quantiles}{Q}
\newcommand{\quantile}{q}
\newcommand{\estimatedquantilevalues}{\tilde{V}}
\newcommand{\estimatedquantilevalue}{\tilde{v}}
\newcommand{\quantilevalues}{V}

\newcommand{\quantilespatiallocations}{\mathcal{Q}}
\newcommand{\estimatedquantilespatiallocations}{\tilde{\quantilespatiallocations}}

\newcommand{\score}{l_s}
\newcommand{\numtiles}{|\quantiles|}

\newcommand{\objectivefunction}{f}

\renewcommand{\t}[1]{{\textrm{#1}}}

\newcommand{\mjci}{Maritz-Jarrett }

\newcommand{\thirdcolumn}{.32\columnwidth}

\newcommand{\threeboxplot}{0.26\columnwidth}
\newcommand{\fourboxplot}{.45\columnwidth}

\newcommand{\boxplotheight}{6.2cm}

\newcommand{\nrobots}{N}
\newcommand{\spread}{\alpha}

\newcommand{\estimatedquantilevaluesfinal}{\estimatedquantilevalues_{\textrm{final}}}

\newcommand{\boxheight}{0.5\linewidth}
\newcommand{\boxwidth}{0.49\linewidth}

\newcommand{\psig}{p \leq 5\textrm{e}{-2}} 

\newcommand{\pppsig}{p \leq 1\textrm{e}{-3}}

\newcommand{\wpres}[2]{T {=} #1,~p {\leq} #2}

\newcommand{\ego}{i}
\newcommand{\other}{j}
\newcommand{\othermeasurements}{M^j}
\newcommand{\msg}{\mathcal{M}}
\newcommand{\utility}{u}
\newcommand{\aggutility}{U}
\newcommand{\utilthresh}{T}
\newcommand{\successprob}{p_{\textrm{success}}}
\newcommand{\newobjectivevalues}{\mathcal{Y}}

\begin{document}

\title{ADVANCING ROBOT AUTONOMY FOR LONG-HORIZON TASKS}

\author{Isabel M. Rayas Fern\'andez}

\majorfield{COMPUTER SCIENCE}

\submitdate{August 2023}

\begin{preface}
  \prefacesection{Acknowledgements}
  On its surface, this dissertation appears to be the culmination of 5 years of my work at USC, but it is impossible to overlook the huge community of people throughout my life who have helped me to get to this point. 
I hope to acknowledge a small number of them here.

Firstly, I'd like to thank Gaurav, my PhD advisor, for achieving the perfect balance of hands-on and hands-off mentoring which has allowed me to learn from his deep wisdom and yet grow on my own. His is a voice of calm reason, which was often exactly what I needed. I am grateful to have had a PhD experience where, though it certainly had its stressful aspects, my advisor was never one of them. Thank you for believing in me, and telling me so, when I felt I couldn't possibly do it. 

I'd also like to thank the other members of my PhD dissertation committee, David Caron and Stefanos Nikolaidis, who have been kind and supportive, and always had insightful questions for me.

I am fortunate to have received support from the National Science Foundation Graduate Research Fellowship Program as well as the USC Annenberg and USC Kunzel Fellowships, which together have funded my research and granted me stability and flexibility during my PhD.

My experience would have been immeasurably less pleasant without Lizsl De Leon as my academic advisor and liaison to all things administrative at USC. Her care for me as a student was remarkable and her ability to resolve administrative tasks seemingly instantaneously was a gift to my sanity.

I am grateful for Maja Matarić's mentorship my first year at USC and for her unwavering support when I decided to make the huge move to several desks away. She was instrumental in my decision to come to USC, and for that I am thankful.

It is not an overstatement to say I would have not pursued graduate school at all without Julie Shah, my academic advisor and role model at MIT. She saw potential in me that I did not know was there, and without her gentle but insistent prodding, I would not have sought out research as an undergrad with her bright and talented student, Lindsay Sanneman. My positive experience with Lindsay in Julie's Interactive Robotics Group was key in leading me to where I am today. I'm so thankful to have learned from both of them.

In addition to Julie, there are so many faculty and staff I had the privilege of interacting with and learning from at MIT. To name a few: Marie Stuppard, Joyce Light, Prof. Balakrishnan, Prof. Radovitzky, Prof. Darmofal, Prof. Karaman, Prof. Williams,  Prof. How, Prof. Stirling, Prof. Miller, Prof. Spakovszky, and the entire AeroAstro department; Prof. Gibson, Prof. Gabrieli, Prof. Desimone, Prof. Fee, Prof. Hogan, Prof. Tenenbaum, and the Brain and Cognitive Sciences department; Prof. Minicozzi; Prof. Staffilani; and Prof. Kaelbling. My experiences at MIT were an indispensable foundation for everything following those four years.

I have also been fortunate to have been a part of several student groups throughout my many years as a student. 
MIT Women in Aerospace Engineering gave me a sense of community and friendship as well as an opportunity to develop leadership skills, including allowing me the joys of planning social events.
With such a positive experience with MIT WAE, it was no surprise that I found comfort in being a part of the committee for the USC PhD Women and Gender Minorities in Computing Club as well.

My education has been enhanced by the multiple internships I've been able to have. 
At NASA JPL, I'd like to thank Celina Clewans for her compassionate mentorship, intelligence, and patience, and Rich Rieber for his infectious enthusiasm and willingness to spend seemingly endless time teaching his curious intern. 
Anne Dorsey at Waymo has been a wonderful mentor, teacher, and friend over the past years, and I am lucky to have had such a positive experience even dealing with the realities of remote internships. With her and Jenny Iglesias, I always felt like a valued member of the team, and I'm thankful for their support and constant guidance both in technical and life questions.
I am also thankful to all the other coworkers and mentors that I've had the chance to interact with at both JPL and Waymo. You all made these experiences some of the highlights of the past five years.

If my future mentors are welcoming of even 1\% of the amount of questions that I asked Mr. Skerbitz in my high school Differential Equations \& Linear Algebra class, I will be unbelievably lucky. Thank you for not just tolerating but encouraging my truly incessant questions, even at 7:00am.
When asking for help feels daunting, I lean on the belief that it is not too much or too burdensome to ask, and my sense of curiosity is due in part to Mr. Skerbitz' extraordinary teaching.

I'm thankful for Mr. Kilkelly, another high school math teacher, whose joy in the \textit{aha} moments of understanding made me eager to keep learning and gave me determination to keep trying, even when I struggled. I'm also grateful for all my other teachers  and mentors at Wayzata and St. Louis Park -- Ms. Hagen, Mr. French, Mr. Bisanz, Ms. Speers, Mr. Gulsvig, Ms. Gohman, Ms. Brown, Mr. Cole, Mr. Ellingson, Mr. Heebink, Mr. Bailer, Mary Kay Williams, Mr. Vrudny, Mrs. Westman, Mrs. Anderson, Mrs. Blaeser, Sr. Franco, and Sr. Bleske are just a tiny few -- for giving me the foundations of my education and fostering a safe and nourishing environment for me to grow up in. 

My music teachers throughout the years were instrumental in encouraging creativity and inspiring me to see and hear beauty in every part of life. To Mr. Gitch and Ms. P-O, thank you. 

Pottery and ceramics have been a constant source of happiness and restoration for me. Thank you to Mr. Braun for introducing me to the art, and to Jay, Darrell, Julia, Tom, and Ki for their gentle direction and helpful instruction which have allowed me to experiment and continue building skills.

I've been shaped by my varied experiences abroad, which would not have been possible without the instruction and inspiration from all my many language teachers over the years, as well as my high school advisor Mr. Nagel's trust in me to complete a semester of courses abroad in Germany, and the support and many opportunities abroad that the MIT MISTI office offered me. To Mme. Magallanes, Mme. Janssen, M. Tuura, Mme. Levet, Mme. Showrai, Frau Gut, Frau Jaeger, Herr Suß, Frau Crocker, Nilma Dominique, Ikeda Sensei, and Matsumoto Sensei: Merci, danke, obrigada, arigatō gozaimasu for opening countless doors for me in every corner of the world and for sharing your love of languages with me. 

I am so fortunate to have many homes away from home. Thank you to Vivi and Astrid, Ann, Elli, and everyone else who has welcomed me when I have been in a new place. 

I am honored to have been a part of the Robotic Embedded Systems Lab at USC. My labmates in RESL have been supportive, engaging, and fun; they have challenged me and pushed me to be better and to learn new things; they have been excellent companions and collaborators. I am so grateful to have had the chance to work with such talented roboticists and good people. 
The work presented in this thesis would not have been possible without my collaborators. 
Peter Englert helped me navigate my first end-to-end research project and was a skillful mentor,
Ragesh Ramachandran provided indispensable technical knowledge, 
and Giovanni Sutanto brought new, interesting ideas and was excited to collaborate. 
My day-to-day research would have been undoubtedly less enjoyable without Chris Denniston to work with; our meetings were always rife with informative and interesting discussions, and I'm grateful that he convinced me early on that working on environmental robotics was very cool and fun.

To Micaela, my oldest friend -- our lasting friendship has been a great happiness to me, and knowing you are always there for me is a comfort I feel privileged to have.
Paige, your tremendous compassion and fantastic sense of humor have pulled me through some of my toughest challenges. You are a gift.  
To Kegan, thank you for your optimism and joy with the small things in life. I am lucky to call someone so kind and intelligent a close friend. 
To David, you made difficult parts of the PhD not just bearable but enjoyable, and I am grateful for your technical help and discussions, and your patience, support, and care over the years. 
Rachel, you are a star, and you inspire me every day. It's not often someone is talented in so many ways and still a sweet, caring, generous, and funny person besides the fact.
My dear friend Justin, you never fail to make me laugh, and I'm so thankful for our friendship. My life is better with you in it. Thank you for making me feel at home, wherever in the world we are.
Lilly, I couldn't have done half the things I've done without your constant, unconditional support, although you would say that of course I could have. Thank you for always understanding me even and especially when you're the only one in the world who does. 

Finally, nothing would have been possible without my family. My parents have shown me love, encouragement, courage, discipline, balance, joy, and strength. They've believed in me from the start, and I feel incredibly fortunate to have them and their support in my life. To my siblings, Wisi, Atzín, Agie, and Eli -- even when you're the worst, you're still the best. Thank you for bringing Gabbie, Angie, and CJ into my life, and Luca, Félix, and Mateo.

There are so many more people who have impacted me positively, both directly and indirectly, and I am so grateful to all of them. Thank you.

  {
  \hypersetup{hidelinks} 
  \tableofcontents
  \listoftables
  \listoffigures
  }

  \prefacesection{Abstract}
  Autonomous robots have real-world applications in diverse fields, such as mobile manipulation and environmental exploration, and many such tasks benefit from a hands-off approach in terms of human user involvement over a long task horizon. However, the level of autonomy achievable by a deployment is limited in part by the problem definition or task specification required by the system. Task specifications often require technical, low-level information that is unintuitive to describe and may result in generic solutions, burdening the user technically both before and after task completion. In this thesis, we aim to advance task specification abstraction toward the goal of increasing robot autonomy in real-world scenarios. We do so by tackling problems that address several different angles of this goal. First, we develop a way for the automatic discovery of optimal transition points between subtasks in the context of constrained mobile manipulation, removing the need for the human to hand-specify these in the task specification. We further propose a way to automatically describe constraints on robot motion by using demonstrated data as opposed to manually-defined constraints. Then, within the context of environmental exploration, we propose a flexible task specification framework, requiring just a set of quantiles of interest from the user that allows the robot to directly suggest locations in the environment for the user to study. We next systematically study the effect of including a robot team in the task specification and show that multirobot teams have the ability to improve performance under certain specification conditions, including enabling inter-robot communication. Finally, we propose methods for a communication protocol that autonomously selects useful but limited information to share with the other robots.
\end{preface}

\chapter{Introduction}
  \label{ch:introduction}

Many robotic tasks benefit from a hands-off approach in terms of human user involvement over a long task horizon.
However, robot systems continue to face barriers to deployment for autonomous real-world, widespread use.
One of these barriers stems from the problem descriptions, or task specifications, that robots require from their users. 
These task specifications often require a high degree of specific, technical input, which is not necessarily intuitive to describe,
or they suffer from a lack of customizability or tailoring to the user's end goal, producing outputs that require further post-processing by the user.
In contrast, we would like to be able to specify a robot task in an interpretable way, that both abstracts the specification away from technical details
and flexibly targets solutions that the user cares about.

In this thesis, we propose solutions to five problems, which advance toward this goal. 
\begin{enumerate}
    \item In long-horizon, sequenced tasks, transitions between subtasks can have a large impact on the quality of solution, but humans are not well-equipped to specify the optimal transitions to use. How can we remove the need for hand-specified transition points?
    \item Constraints on robot motions during task execution can be difficult to specify analytically but more accessible to show via demonstration. How can we automatically learn a representation of a motion constraint from data?
    \item In long-horizon environment exploration tasks, task specifications typically target finding a predefined type of location, but different application domains require different types of interest. How can we describe these distinct tasks in a unified high-level, interpretable task specification that directly leads to automatic discovery of those locations?
    \item Different users and domains will also have different resources available to them, including number of robots and inter-robot communication infrastructure, but the effect of task specifications with these conditions on performance remains understudied in a systematic way. How can different robot team setups be effectively used during environment exploration?
    \item Inter-robot communication aids collaborative task performance, but real robot networks are resource-constrained and cannot support arbitrary message loads, and not every message is equally informative. How can we autonomously decide which information is most valuable to send to other robots?
\end{enumerate}


To ground these problems, we propose solutions in two application domains. 
We consider the first two problems in the context of constrained manipulation and mobility applications and provide necessary background for this context in \Cref{ch:motion-planning-background}.

We address (1) in \Cref{ch:psm}, where we tackle the problem of motion planning for sequential tasks, where each subtask therein is defined by the constraints on the motion allowed during it. 
Our solution results in the automatic discovery of optimal transition points between subtasks, removing the need for the human to hand-specify these in the task specification.
In particular, 
we address the problem of planning robot motions in constrained configuration spaces where the constraints change throughout the motion.
The problem is formulated as a fixed sequence of intersecting manifolds, which the robot needs to traverse in order to solve the task.
We propose a class of sequential motion planning problems that fulfill a particular property of the change in the free configuration space when transitioning between manifolds, and
for this problem class, we develop the algorithm \smplong~(\smp) which searches for optimal intersection points between manifolds by using rapidly exploring random trees (RRT$^*$) in an inner loop with a novel steering strategy. 
We provide a theoretical analysis regarding \smp s probabilistic completeness and asymptotic optimality. 
We also evaluate its planning performance on multi-robot object transportation tasks and compare it to other methods.

We address (2) in \Cref{ch:ecomann}, where we propose a way to automatically describe constraints on robot motion by using demonstrated data as opposed to manually-defined constraints.
Specifically, in this work, we consider the problem of learning representations of motion constraints from demonstrations with a deep neural network, which we call {\ecmnnlong} ({\ecmnn}). 
To do so, we propose learning a level-set function of the constraint by aligning subspaces in the network with subspaces of the data such that it can be integrated into a constrained sampling-based motion planner, such as \smp.
We combine both learned constraints and analytically described constraints into the planner and use a projection-based strategy to find valid points. 
We evaluate {\ecmnn} on its representation capabilities of constraint manifolds, the impact of its individual loss terms, and the motions produced with it.

The last three problems we consider in the context of robotics in environmental sensing and monitoring, inspired by our collaboration with a team of biologists using robots to help explore and monitor harmful algae bloom growth in aquatic environments.
\Cref{ch:ipp-background} provides relevant background for this context.

We address (3) in \Cref{ch:quantiles}, where we propose a flexible, abstract task specification, requiring just a set of quantiles of interest from the user that allows the robot to directly suggest locations in the environment to the user that will contain information that they care about.
To this end, we propose to choose locations for scientific analysis by using a robot to perform an informative path planning survey, in contrast to relying on expert heuristics to choose them. 
The survey results in a list of locations that correspond to the quantile values of the phenomenon of interest.
We develop a robot planner using novel objective functions to improve the estimates of the quantile values over time and an approach to find locations which correspond to the quantile values.
We test our approach in four different environments using previously collected aquatic data and compare it to a baseline approach which attempts to maximize spatial coverage.
Additionally, we validate our work in a real field trial.

We address (4) in \Cref{ch:multirobot}, where we systematically study the effect of including a robot team in the task specification for the quantile estimation task introduced in \Cref{ch:quantiles} and show that multirobot teams have the ability to improve accuracy under certain specification conditions, one of which is when they are communication-enabled.
This work aims to understand effects of different multirobot setups, since
a multirobot team can be difficult to practically bring together and coordinate, especially when robot communication is involved.
To this end, we present a study across several axes of the impact of using multiple robots to estimate quantiles of a distribution of interest using an informative path planning formulation. We measure quantile estimation accuracy with increasing team size to understand what benefits result from a multirobot approach in a drone exploration task of analyzing the algae concentration in lakes. 
We additionally perform an analysis on several parameters, including the spread of robot initial positions, the planning budget, and inter-robot communication, and
find that while using more robots generally results in lower estimation error, this benefit is achieved under certain conditions. 
We present our findings in the context of real field robotic applications and discuss the implications of the results.

We address (5) in \Cref{ch:comms}, where we propose methods for a communication protocol which autonomously selects useful but limited information to share with the other robots.
These methods determine the utility of sharing information across a resource-constrained communication network, in which case na\"ive, constant information-sharing is infeasible,
and we show that limiting the load on the network is possible without sacrificing performance on the estimation task.
In particular, we propose online, locally computable metrics for determining the utility of transmitting a given message to the other team members and a decision-theoretic approach that chooses to transmit only the most useful messages, using a decentralized and independent framework for maintaining beliefs of other teammates.
We validate our approach in simulation on a real-world aquatic dataset, 
and show that restricting communication via a utility estimation method based on the expected impact of a message on future teammate behavior results in a decrease in network load while simultaneously maintaining good or improved task performance.

Finally, we conclude the thesis and summarize the contributions in \Cref{ch:conclusion}.


\chapter{Background on Constrained Motion Planning}
\label{ch:motion-planning-background}


In this and the next two chapters of this thesis, we place ourselves in the context of constrained mobile manipulation tasks, and in particular, sequenced tasks.
These sequenced tasks are composed of smaller individual subtasks, which can together be as short as a picking task, or as complex as a multirobot, multi-handover object manipulation task. 
Viewed through this lens, such problems are long-horizon in that their subtasks and associated motion constraints can be arbitrarily numerous. 
As we will discuss, task specifications for these problems have typically required technical descriptions of the subtasks, their constraints, and how to transfer between them.
This chapter gives relevant background for this application domain, laying the foundation for \Cref{ch:psm,ch:ecomann} which address aspects of task abstraction toward increased robot autonomy in these types of problems.
In these chapters, we will consider offline planning in observable environments.

\label{sec:psm_related_work}
\section{Sampling-Based Motion Planning}
Sampling-based motion planning (SBMP) is a broad field which tackles the problem of motion planning by using randomized sampling techniques to build a tree or graph of configurations (also called samples), which can then be used to plan paths between configurations. 
A PRM (probabilistic roadmap) is a path planner that builds a graph in the free configuration space that can be used for multiple queries \cite{amato1996randomized, kavraki1994randomized, overmars1992random}. 
Kavraki \etal \cite{kavraki1994randomized} describe the method as a two-step procedure. First, a roadmap is built by sampling collision-free nodes and edges. Second, a path is found from a start to a goal state by using a graph search algorithm. 
This technique is multi-query, as it 
can be reused for many planning problems with the same system.
Tree methods such as RRTs (rapidly-exploring random trees) are generally single-query, taking a specific start state from which a tree of feasible states is grown toward a specific goal state or region \cite{lavalle1998rapidly}. 
Many extensions to RRT exist, such as bidirectional trees and goal biasing \cite{kuffner2000rrt, lavalle2006planning}. Optimal variants RRT$^*$ and PRM$^*$ find paths that minimize a cost function and guarantee asymptotic optimality by using a rewiring procedure on the edges in the graphs \cite{karaman2011sampling}. 
All these techniques consider the problem of planning without motion constraints; that is, they plan in the free configuration space. 

\section{Constrained Sampling-Based Motion Planning}
A more challenging and realistic motion planning task is that of constrained
SBMP, where there are motion constraints beyond just obstacle avoidance which lead to a free configuration space manifold 
$\onconstraintconfigspace \in
\Re^\dimconstraint$ of lower dimension than the ambient configuration space $\configspace \in
\Re^\dimambient$. 
Many practical tasks require planning with constraints. 
Constrained SBMP algorithms \cite{stilman2010global, berenson2011task, jaillet2013asymptotically, kim2016tangent, kingston2019ijrr, csucan2012motion, cortes2004sampling} extend SBMP to these types of constrained configuration spaces (see \cite{kingston2018sampling} for a review). 
Randomized sampling in a $\dimambient$-dimensional space
will produce samples on an $\dimconstraint$-dimensional space with probability $0$ where  $\dimconstraint< \dimambient$;
these spaces usually cannot be sampled directly. 
Thus, a large focus of these methods is how to generate valid samples and steer the robot while fulfilling the constraints.
Many SBMP algorithms derive from rapidly-exploring random trees (RRT),  probabilistic roadmaps (PRM), or their optimal counterparts RRT$^\star$ and PRM$^\star$.


One family of approaches are projection-based strategies \cite{berenson2011task, stilman2010global, csucan2012motion, kaiser2012constellation} that first sample a configuration from the ambient configuration space and then project it using an iterative gradient descent strategy to a nearby configuration that satisfies the constraint. 
%
Berenson \etal \cite{berenson2011task} proposed CBiRRT (constrained bidirectional RRT) that uses projections to find configurations that fulfill constraints. The constraints are described by task space regions, which are a representation of pose constraints. Their method can also be used for multiple constraints over a single path. However, their approach requires each constraint's active domain to be defined prior to planning, or configurations are simply projected to the nearest manifold rather than respecting a sequential order. 

An alternative is to approximate the constraint surface by a set of local models and use this approximation throughout the planning problem for sampling or steering operations \cite{jaillet2013asymptotically, jaillet2013path, kim2016tangent,suh2011tangent, stilman2010global, bordalba2018randomized}.
For example, AtlasRRT \cite{jaillet2013path} builds an approximation of the constraint consisting of local charts defined in the tangent space of the manifold. This representation is used to generate samples that are close to the constraint. 
Similarly, Tangent Bundle RRT \cite{kim2016tangent, suh2011tangent} builds a bidirectional RRT by sampling a point on a tangent plane, extending this point to produce a new point, and if it exceeds a certain distance threshold from the center of the plane, projecting it on the manifold to create a new tangent plane. 

Kingston \etal \cite{kingston2019ijrr} presented the implicit manifold configuration space (IMACS) framework that decouples two parts of a geometrically constrained motion planning problem: the motion planning algorithm and the method for constraint adherence. With this approach, IMACS acts as a representative layer between the configuration space and the planner. 
Many of the previously mentioned techniques fit into this framework. They present examples with both projection-based and approximation-based methods for constraint adherence in combination with various motion planning algorithms. 

Previous research has also investigated incorporating learned constraints or manifolds into planning frameworks. 
These include performing planning in learned latent spaces
\cite{ichter2019robot}, learning a better sampling distribution in order to take
advantage of the structure of valid configurations rather than blindly sample
uniformly in the search space \cite{ichter2018learning, madaanlearning}, and
attempting to approximate the manifold (both explicitly and implicitly) of valid
points with graphs in order to plan on them more effectively
\cite{phillips2012graphs, havoutis2009motion, zha2018learning}. 

\section{Planning with Sequential Tasks}
One approach to plan sequential motions is task and motion planning (TAMP), which requires semantic reasoning on selecting and ordering actions to complete a higher-level task \cite{kaelbling2013integrated, dantam2016incremental, konidaris2018skills, toussaint2018differentiable, dantam2018task, barry2013hierarchical, cambon2009hybrid, kingston2020informing, pflueger2015multi}. 
Broadly speaking, TAMP is more general and difficult to solve compared to SBMP due to scaling and a more complicated problem definition. 
Though we do not address TAMP in this work, it is an interesting future direction that shares many characteristics with SBMP. 
\cite{hauser2011randomized} proposed the multi-modal motion planning algorithm Random-MMP, which plans motions over multiple mode switches that describe changes in the planning domain (e.g., contacts). 
Their planner builds a tree in a hybrid state using an SBMP that consists of the continuous robot configuration and a discrete mode, which changes based on the domain. 
%
Other previous work \cite{vega2016asymptotically, kingston2020informing, hauser2011randomized, simeon2004manipulation, simeon2000visibility, mirabel2016hpp, schmitt2019modeling} addressed this problem in related ways. They generally assume sampling of the modes and their boundaries, and build graphs that connect these to each other.

Finally, trajectory optimization \cite{stryk1992direct, schulman2013finding, toussaint2017newton, ratliff2015understanding, pavone2019rss} is an approach to solve sequential motion planning problems where an optimization problem over a trajectory is defined that minimizes costs subject to constraints. 
However, trajectory optimization often depends on having a good initial trajectory and can suffer from poor local minima. 
The time points in the trajectory at which costs and constraints are active also need to be specified precisely, which can be challenging since it is difficult to know in advance how long a specific part of a task takes.

\section{Manifold Theory}
\label{ssec:manifold_theory}
We include a brief background on manifold theory (see \cite{Boothby:107707} for a rigorous treatment). 
An important idea in differential geometry is the concept of a \textit{manifold} -- a surface which can be well-approximated locally using an open set of a Euclidean space. 
In general, manifolds are represented using a collection of local regions called \textit{charts} and a continuous map associated with each chart such that the charts can be continuously deformed to an open subset of a Euclidean space. 
An alternate representation of manifolds, which is useful from a computational perspective, is to express them as zero level sets of continuous functions defined on a Euclidean space. 
Since the latter representation is a direct result of the implicit function theorem, it is referred to as \textit{implicit representation} of the manifold. 
For example, a unit sphere embedded in $\R^3$ can be represented implicitly as $\{x \in \R^3 ~|~ \|x\|-1 = 0\}$. 
An implicit manifold is said to be smooth if the implicit function associated with it is smooth. 
The set of all tangent vectors at a point on a manifold is a vector space called the \textit{tangent space} of the manifold.
The \textit{null space} of the Jacobian of the implicit function at a point is isomorphic to the tangent space of the corresponding manifold at that point. 
Since the tangent spaces of a manifold are vector spaces, we can equip them with an inner product structure which enables the computation of the length of curves traced on the manifold. 
A manifold endowed with an inner product structure is called a Riemannian manifold \cite{lee2006riemannian}. 
In this thesis, we will only consider smooth Riemannian manifolds. 

\section{Manifold Learning}
\label{ssec:manifold-learning}
Manifold learning is applicable to many fields and thus has garnered a wide
variety of methods over the years. Linear methods include PCA and LDA
\cite{duda2001pattern}, and while they are simple, they lack the complexity to
describe nonlinearities. Nonlinear methods include multidimensional scaling
(MDS) \cite{kruskal1978multidimensional},
locally linear embedding (LLE) \cite{roweis2000nonlinear},
Isomap \cite{tenenbaum2000global},
and local tangent space alignment (LTSA) \cite{zhang2004principal}.
These approaches use techniques such as eigenvalue decomposition, nearest neighbor reconstructions, and local-structure-preserving graphs to visualize and represent manifolds.
MDS uses eigenvalue decomposition to produce data that preserve distances between points.
LLE preserves local structure by reconstructing each point with the same set of nearest neighbors in both the ambient and the embedding spaces. 
Isomap approximates geodesic distances on a manifold by constructing a graph from the data, and then computing all shortest paths.
In LTSA, the local tangent space information of each point is aligned to create a global representation of the manifold. 
We refer the reader to \cite{ma2011manifold} for a thorough description of these and other methods.
Recent work in manifold learning additionally takes advantage of (deep) neural architectures. 
Some approaches use autoencoder-like models \cite{holden2015learning,
chen2016dynamic} or deep neural networks \cite{nguyen2019neural} to learn
manifolds, e.g. of human motion. 
Others use classical methods combined with neural networks, for example as a
loss function for control \cite{sutanto2019learning} or as structure for the
network \cite{wang2014generalized}. Locally Smooth Manifold Learning (LSML)
\cite{dollar2007non} defines and learns a function which describes the tangent
space of the manifold, allowing randomly sampled points to be projected onto it.

\section{Learning from Demonstration}
Learning from demonstration (LfD) techniques learn a task representation from data which is usable to generate robot motions that imitate the demonstrated behavior.
One approach to LfD is inverse optimal control (IOC), which aims to find a cost function that describes the demonstrated behavior \cite{RatliffBZ06,ziebart2008maximum, ratliff2009learch, levine2012cioc}. 
Recently, IOC has been extended to also extract constraints from demonstrations \cite{puydupin2012convex, englert2017ijrr}. 
There, a cost function as well as equality and inequality constraints are extracted from demonstrations, which are useful to describe behavior like contacts or collision avoidance.

A more direct approach to LfD is based on learning parameterized motion representations \cite{schaal2003computational,Paraschos2013pmp,Pastor_RAIIC_2011}. They represent the demonstrations in a parameterized form such as Dynamic Movement Primitives \cite{Ijspeert_NC_2013}. Here, learning is easier because it is often linear regression; however, the ability to generalize to new situations is more limited. 
The advantage of these approaches lies in dynamic tasks like throwing a ball where a fixed order of states must be followed. A main advantage of using a planner compared to primitives is that planners are better at avoiding collisions between the robot and its environment.
Other approaches to LfD include learning of task space \cite{14-jetchev-AuRo} and deep learning approaches \cite{finn2016guided}. We refer the reader to the survey \cite{argall2009survey} for a broad overview on LfD.

\chapter{\smp: Planning on Sequenced Manifolds}
  \label{ch:psm}
  We begin motivated by the desire to assign the task of selecting subtask transitions to the robot, as opposed to requiring the user to preselect exactly how to transfer between subtasks.
Our solution to this problem, presented in this chapter based on the work published in \cite{englert2020sampling}, results in autonomous optimal subtask intersection point selection and motion plans which solve long-horizon problems with an arbitrary number of subtasks.

  \section{Introduction}
\label{sec:psm_introduction}
Sampling-based motion planning (SBMP) considers the problem of finding a collision-free path from a start configuration to a goal configuration. Probabilistic roadmaps \cite{kavraki1994randomized} or rapidly exploring random trees \cite{lavalle1998rapidly} generate plans for such motions while providing theoretical guarantees regarding probabilistic completeness. Optimal algorithms like RRT$^*$ \cite{karaman2011sampling} additionally minimize a cost function and achieve asymptotic optimality. However, many tasks in robotics 
require the incorporation of additional objectives like subgoals or constraints.
%
SBMP methods are difficult to directly apply to such problems because they require splitting the overall task into multiple planning problems that fit into the SBMP structure.
A disadvantage of this approach is that a goal in the configuration space needs to be specified for each subtask, and each selected subgoal will affect the future subtasks. 

We consider a 
formulation for sequential motion planning where a problem is represented as a fixed sequence of manifolds.
\begin{figure}[t]
	\centering
 	\includegraphics[width=.4\textwidth]{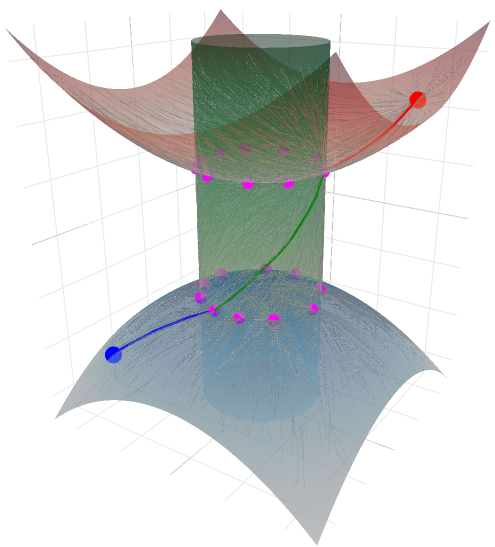}
	\caption{3D Point on Geometric Constraints -- The surfaces visualize the level sets of the three constraints. The task is to move from the start point (red dot) to the goal point (blue dot) while fulfilling the constraints. The line shows a solution found by \smp~that fulfills these constraints and the magenta points are the found intersection vertices.}
	\label{fig:psm_hourglass}
\end{figure}
We propose \smplong~(\smp), an algorithm that searches for an optimal path that starts at an initial configuration, traverses the manifold sequence, and converges when the final manifold is reached. 
Our solution is to grow a single tree over the manifold sequence.
This tree consists of multiple subtrees that originate at the intersections between pairs of manifolds. We propose a novel steering strategy that guides the robot towards these manifold intersections. After an intersection is reached, a new subtree is initialized with the found intersection points and their costs. 
We use dynamic programming over optimal intersection points which scales well to long-horizon tasks since a new subtree is initialized for every manifold.
The algorithm is applicable to a class of problems we call \emph{intersection point independent}, specified by a property which restricts how the free configuration space changes across subtasks (see \Cref{sec:problem_formulation}).

A running example in this chapter is the task of using a robot arm to transport a mug from one table to another while keeping the orientation of the mug upright, a task involving multiple phases and constraints described informally as follows.
First, the arm moves to pick up the mug. In this subtask, the arm can move freely in space and only needs to avoid collisions with obstacles. The second subtask is to grasp the mug. The third subtask is for the arm to transfer the mug with the constraint of always keeping the mug upright. The final subtask is to place and release the mug which requires the base of the mug to be near the table. 
In \Cref{sec:psm_experiments}, we demonstrate the performance of \smp~on similar sequential kinematic planning problems that involve multiple robots and compare it to alternative planning strategies.


\subsection{Background}
Our proposed method builds on the methods introduced in \Cref{ch:motion-planning-background} and extends them towards sequential tasks where the active constraints change during the motion.
We use a modified version of RRT$^*$ in the inner loop of \smp~that can handle goals defined in terms of equality constraints instead of a goal configuration. 
Restricting the problem class to intersection point independent problems (\Cref{ssec:ipip}) allows us to grow a single tree over a sequence of manifolds on which rewiring operations can still be performed. 
We show that \smp~inherits the probabilistic completeness and asymptotic optimality guarantees of RRT$^*$ (\Cref{ssec:comp-optimal}).

Our approach is perhaps closest to the CBiRRT algorithm. 
A main difference is that our method assumes a different problem formulation where the task is given in terms of a fixed sequence of manifolds where the intersections between manifolds describe subgoals that the robot should reach. 
This formulation allows us to define a more structured steering strategy that guides the robot towards the next subgoals. 
Another difference is that CBiRRT does not search for optimal paths, while we employ RRT$^*$ to minimize path lengths. We compare our method \smp~to CBiRRT in \Cref{sec:psm_experiments}.

Our work also shares similarities with others in the task and motion planning (TAMP) problem area, in that we consider planning over multiple manifolds where changes in the configuration space occur due to picking or placing objects. 
However, our work differs in several key ways. First, TAMP approaches assume the task sequence is unknown in advance, and thus selecting the transition to the next constraint is part of the planning problem. Here we specify our problem formulation over a fixed and known sequence of constraints, focusing our algorithm on optimizing the transition points between constraints to find an optimal path. 
Additionally, we do not assume direct sampling of modes or switches is possible; rather, our algorithm is designed to find the mode-switching configurations during exploration toward manifold intersections. 
Importantly, because we build one search tree, these configurations are guaranteed to be connected to the start. This eliminates sampling of configurations that are disconnected from the viable solution space.
Further, we differ by applying \smp~in the domain of intersection point independent problems, 
which we specify for sequential planning problems based on the free space transformation at each manifold intersection (\Cref{ssec:ipip}).
These problems do not include foliated manifolds and are efficiently solvable by growing a single tree.

Finally, in contrast to trajectory optimization, our approach only requires the sequential order of tasks and does not make any assumption of specific time points or durations of subtasks.

\section{Problem Formulation and Application Domain}
\label{sec:problem_formulation}
We consider kinematic motion planning problems in a configuration space $C\subseteq \R^k$. A configuration $q\in C$ describes the state of one or multiple robots with $k$ degrees of freedom in total. 
We represent a manifold $M$ as an equality constraint $h_{M}(q)=0$ where $$h_{M}(q) : \R^{k} \to \R^{l}$$
and $l$ is the intrinsic dimensionality of the constraint ($l \leq k$). The set of robot configurations that are on a manifold $M$ is given by $$C_{M} = \{q\in C ~|~ h_{M}(q) = 0\}\period$$
We define a projection operator $$q_\t{proj} =  \textit{Project}(q, M)$$
that takes a robot configuration $q\in C$ and a manifold $M$ as inputs and maps $q$ to a nearby configuration on the manifold $q_\t{proj}\in C_M$. 
We use an iterative optimization method, similar to \cite{kingston2019ijrr, berenson2011task, stilman2007task}, that iterates 
$$q_{n+1} = q_n - J_{M}(q_n)^+ h_{M}(q_n)$$
until a fixed point on the manifold is reached, which is checked by the condition $||h_M(q_\t{proj})|| \leq \ge$. The matrix $J_{M}(q)^+$ is the pseudo-inverse of the constraint Jacobian $$J_{M}(q) = \frac{\partial }{\partial q} h_{M} (q)\period$$

We are interested in solving tasks that are defined by a sequence of $n+1$ such manifolds 
$\mathcal{M} = \{M_1, M_2, \dots, M_{n+1}\}$
\new{where the number and order of manifolds is given,} and an initial configuration $q_\t{start}\in C_{M_1}$ that is on the first manifold.
The goal is to find a path from $q_\t{start}$ that traverses the manifold sequence $\mathcal{M}$ and reaches a configuration on the goal manifold $M_{n+1}$.
A path on the $i$-th manifold is defined as $$\tau_i : [0, 1] \to C_{M_i}$$
and $J(\tau_i)$ is a cost function of a path
$${J} : \mathcal{T} \to \R_{\geq 0}$$
where $\mathcal{T}$ is the set of all non-trivial paths. 

In the following problem formulation, we consider the scenario where the free configuration space $C_\t{free}$ changes during the task (e.g., when picking or placing objects).
We define an operator $\Upsilon$ on the free configuration space $C_\t{free}$, which describes how $C_\t{free}$ changes as manifold intersections are traversed. 
$\Upsilon$ takes as input the \new{endpoint of a} path $\tau_{i-1}(1)$ on the previous manifold $M_{i-1}$ and its associated $C_\t{free}$, which we denote as $C_{\t{free}, i-1}$. $\Upsilon$ outputs an updated free configuration space 
$$C_{\t{free}, i} = \Upsilon(C_{\t{free}, i-1}, \tau_{i-1}(1))$$
that accounts for the geometric changes due to transitioning to a new manifold. 
\new{Here, we assume these changes only occur at the transition between two manifolds.}
\new{We define $\Upsilon$ in this way rather than with the configuration spaces of the robot and the objects in the environment because our formulation does not have a notion of ``objects"; any object in the world is either an obstacle, or it becomes part the robot system itself.}

We assume access to a collision check routine
$$\textit{CollisionFree}(q_a, q_b, C_{\t{free}, i}) \to \{0,1\}$$ that returns $1$ if the path between two configurations on the manifold $q_a, q_b \in C_M$ is collision-free, $0$ otherwise. Note that here, we use the straight-line path in ambient space between two nearby points. 
However, other implementations for the collision check are possible, such as a path on the manifold.

Returning to our illustrative example, we can now describe one of its constraints more precisely. A simple grasp constraint can be described with \mbox{$h_M(q) = x_g - f_{\t{pos}, e}(q)$} where $f_{\t{pos}, e}$ is the forward kinematics function of the robot end effector point $e$ and $x_g$ is the grasp location on the mug. 
This constraint can be fulfilled for multiple robot configurations $q$, which correspond to different hand orientations that will affect the free configuration space for the subsequent tasks. For example, a mug grasped from the side will have a different free space during the transport phase than a mug grasped from the top. 

\subsection{Problem Formulation}
We formulate an optimization problem over a set of paths $\v\tau = (\tau_1, \dots, \tau_{n})$ that minimizes the sum of path costs under the constraints of traversing $\mathcal{M}$ and of being collision-free. The \emph{planning on sequenced manifolds problem} is
\begin{align} %
	\begin{alignedat}{2} %
	\label{eq:smp_problem}
	\v\tau\opt =~ &\argmin{\v\tau} \sum_{i=1}^n J(\tau_i) &&\\
	\text{s.t.}\quad &\tau_1 (0) = q_\t{start} &\\
	&\tau_i(1) = \tau_{i+1}(0)  &&\forall_{i=1,\dots,n-1} \\ 
	&C_{\t{free}, i+1} = \Upsilon(C_{\t{free}, i}, \tau_{i}(1))~~~~ &&\forall_{i=1,\dots,n}\\
	&\tau_i(s) \in C_{M_i} \cap C_{\t{free}, i} &&\forall_{i=1,\dots,n}~ \forall_{s \in [0, 1]}\\
	&\tau_{n}(1) \in C_{M_{n+1}} \cap C_{\t{free}, n+1} &&
	\end{alignedat} %
\end{align}
The second constraint ensures continuity such that the endpoint of a path $\tau_i(1)$ is the start point of the next path $\tau_{i+1}(0)$. The third constraint captures the change in the collision-free space defined by the operator $\Upsilon$. The last two constraints ensure that the path is collision free and on the corresponding manifolds. The endpoint of $\tau_{n}$ must be on the goal manifold $M_{n+1}$, which denotes the successful completion of the task.
\new{Note that this problem can be seen as a special case of the multi-modal motion planning problem presented in \cite{hauser2011randomized} where the manifolds are provided in their sequential order.}

\new{There are several advantages to formulating the problem with manifolds in this way.} 
One is that it is not necessary to choose a specific target configuration in $C$ and thus a wider range of goals 
can be described in the form of manifolds. 
Another is that it is possible to describe complex sequential tasks in a single planning problem, not requiring the specification of subgoal configurations. 
The algorithm proposed in \Cref{sec:method} searches for a path that solves this optimization problem for the problem class described in the next section.

\begin{figure}[t]
    \centering
    \includegraphics[width=1\linewidth]{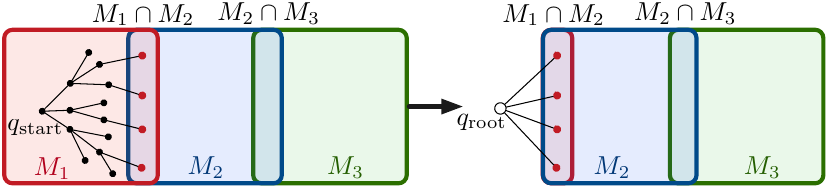}
    \caption{Initialization of a new subtree (\cref{alg:newtreeinit_start} -- \cref{alg:newtreeinit_end} of \Cref{alg:smp}). The reached goal nodes at the intersection $M_1\cap M_2$ are used to initialize the next tree where $q_\t{root}$ is a synthetic root node that maintains the tree structure.}
    \label{fig:tree_init}
\end{figure}

\subsection{Intersection Point Independent Problems}
\label{ssec:ipip}
We now use the problem formulation in 
\Cref{eq:smp_problem} to describe robotic manipulation planning problems in which the manifolds describe subtasks for the robot to complete (e.g., picking up objects). The solution is an end-to-end path across multiple manifolds.
%


For a certain class of sequential manifold planning problems, the following property holds: For each manifold intersection \mbox{${M}_{i} \cap {M}_{i+1}$}, the free space output by $\Upsilon$ is set-equivalent for every possible path $\tau_{i}$. In other words, the precise action taken to move the configuration from one constraint to the next does not affect the feasible planning space for the subsequent subtask. When this property holds for all intersections, we call the problem \emph{intersection point independent}. The condition for this class of problems is
\begin{align}
	\begin{alignedat}{2} %
		\label{eq:ip_independence}
		&\forall i \in [0, n] ~\forall \tau_i, \tau_i' \in \mathcal{T} : \tau(1), \tau'(1) \in C_{M_i} \cap C_{M_{i+1}} \\
		&\Rightarrow\Upsilon(C_{\t{free}, i}, \tau(1)) \equiv \Upsilon(C_{\t{free}, i}, \tau'(1))
	\end{alignedat}
\end{align}
where $\equiv$ denotes set-equivalence. In this work, we focus on the intersection point independent class of motion planning problems, which encompass a wide range of common problems.
For grasping constraints, a notion of object symmetry about the grasp locations results in intersection point independent problems. If the object to be grasped is a cylindrical can, for example, allowing grasps to occur at any point around the circumference but at a fixed height would be intersection point independent. Any two grasps with the same relative orientation of the gripper result in the same free configuration space of the system (robot + can).
However, suppose the grasps can occur at any height on the can. Now, a grasp near the top of the can and a grasp in the middle of the can result in different free configuration spaces, and thus this would be an intersection point dependent problem.

Focusing on intersection point independent problems allows us to define an efficient algorithm that grows a single tree over a sequence of manifolds. The more general intersection point dependent problem covers a wider class of problems. However, they are more difficult to solve because they require handling foliated manifolds \cite{kim2014randomized} (e.g., every grasp leads to a different manifold). 
Our illustrative example is one such problem, since the grasp on the handle and the grasp from the top of the mug result in different free spaces. We provide some insights in \Cref{sec:psm_conclusion} on how the proposed algorithm could be used to tackle them.

\section{\smplong}
\label{sec:method}

\begin{table}[t]
	\centering
 \caption{\smp~algorithm parameters and their description.}
\begin{tabular}{|c|l|}
	\hline
	\textbf{Symbol} & \textbf{Description}\\
	\hline
	$\ga$ & Max step size of an edge in the tree\\
	$\gb$ & Probability of SteerConstraint step\\
	$\gr$ & Min distance between intersection nodes\\
	$\ge$ & Threshold for a point to be on a manifold\\
	$r$ & Max manifold intersection projection distance\\
	\hline
\end{tabular}
\label{tab:parameter}
\end{table}

We now present the algorithm \smp~that solves the problem formulated in \Cref{eq:smp_problem} for tasks that fulfill the intersection point independent property defined in \Cref{eq:ip_independence}.
The algorithm searches for an optimal solution to the constrained optimization problem, which is a sequence of paths $\v\tau = (\tau_1, \dots, \tau_{n})$ where each $\tau_i$ is a collision-free path on the corresponding manifold $M_i$. 

The steps of \smp~are outlined in \Cref{alg:smp}.
The input to the algorithm is a sequence of manifolds $\mathcal{M}$, an initial configuration \mbox{$q_\t{start}\in M_1$} on the first manifold, 
\new{and the hyperparameters of the algorithm $(\ga, \gb, \gr, \ge, r, m)$ (see  \Cref{tab:exp_parameter}).
$V, E$ are the vertices and edges, respectively, in the tree that the algorithm builds, while $C$ is the robot configuration space and $C_{\t{free}, 1}$ is the initial free configuration space.}

The overall problem is divided into $n$ subproblems, \new{which correspond to the number of manifold intersections in $\mathcal{M}$.} Each subproblem \new{can be broadly thought of as growing a tree between two consecutive manifold intersections.} We give a high-level summary here.
In the inner loop (\cref{alg:innerloop_start} -- \cref{alg:innerloop_end}; see \Cref{ssec:phase12}), a tree is iteratively grown from a set of initial nodes toward the intersection with the next manifold.
First, a new constraint-adhering configuration $q_\t{new}$ is found by sampling a point, steering, and projecting it (\cref{alg:steer_start} -- \cref{alg:steer_project_step}).
Second, in \cref{alg:check_proj} -- \cref{alg:extend_end}, the point is checked for validity before being added to the tree as well as the set of intersection points (if applicable).
\new{After the inner loop completes}, the algorithm initializes the next subproblem with the found intersection points in \cref{alg:newtreeinit_start} -- \cref{alg:next_freespace} (see \Cref{ssec:phase3}). It returns the optimal path that traverses all manifolds in $\mathcal{M}$.


\subsection{Inner Loop: Growing a Tree to the Next Manifold}
\label{ssec:phase12}
The first phase of the inner loop focuses on producing a new candidate configuration $q_\t{new}$ which lies on the constraint manifold $M_i$ and adding it to the tree. 
The \smp \_STEER routine in \Cref{alg:steer_project} computes $q_\t{new}$.
Instead of targeting a single goal configuration as in general SBMP, we propose two novel steering strategies that steer toward the intersection between manifolds $M_i\cap M_{i+1}$. 
In \cref{alg:steer_start} -- \cref{alg:find_nearest} of \Cref{alg:smp}, a new target point $q_\t{rand}$ in the configuration space is sampled and its nearest neighbor $q_\t{near}$ in the tree is computed. Next, in \Cref{alg:steer_project} one of the following two steering strategies is selected to find a direction $d$ in which to extend the tree:
\begin{enumerate}
    \item \new{$\t{\textbf{SteerPoint}}(q_\text{near}, q_\text{rand}, M_i)$} extends the tree from \mbox{$q_\text{near}\in C_{M_i}$} towards $q_\text{rand}\in C$ while staying on the current manifold $M_i$.
    \item \new{$\t{\textbf{SteerConstraint}}(q_\text{near}, M_i, M_{i+1})$} extends the tree from $q_\text{near}\in C_{M_i}$ towards the intersection of the current and next manifold $M_i\cap M_{i+1}$.
\end{enumerate}

\subsubsection{\new{$\t{SteerPoint}(q_\t{near}, q_\t{rand}, M_i)$}}
\label{sec:steer_point}
In this extension step, the robot is at $q_\t{near}\in C_{M_i}$ and should step towards the target configuration $q_\t{rand}\in C$ while staying on the manifold $M_i$. We formulate this problem as a constrained optimization problem of finding a curve $\gg ~:~ [0,1]\to C$ that
\begin{align}
    \begin{alignedat}{2}
    	&\minimize{\gg} &&||\gg(1) - q_\t{rand}||^2\\
		&\subjectto~~&&\gg(0) = q_\t{near},\quad \int_{0}^1 ||\dot{\gg}(t)|| ~\mathrm{d}s \leq \ga\\
		& & &h_{M_i}(\gg(s)) = 0 \quad\qquad \forall~ s\in [0,1]
	\label{eq:steer_point}
	\end{alignedat}
\end{align}
This problem is hard to solve due to the nonlinear constraints. Since the steering operations are called many times in the inner loop of the algorithm, we choose a simple curve representation and only compute an approximate solution to this problem. We parameterize the curve as a straight line $\gg(s) = q_\t{near} + s\ga \frac{d}{||d||}$ with length $\ga$
\new{where the direction $d$ is chosen as orthogonal projection of $q_\t{rand}-q_\t{near}$ onto the tangent space of the manifold at $q_\t{near}$.}
We apply a first-order Taylor expansion of the manifold constraint $h_{M_i}(q_\t{near} + d) \approx h(q_\t{near}) + J_{M_i}(q_\t{near}) d$, which reduces the problem to
\begin{align}
	\begin{alignedat}{3}
	&\minimize{d} &&\frac{1}{2} ||d  - (q_\t{rand}-q_\t{near})||^2\\
	&\subjectto~~ &&J_{M_i}(q_\t{near}) d = 0
	\end{alignedat}
\end{align}
The optimal solution of this problem is
\begin{align}
	\begin{alignedat}{2}
	d &=(I -  J_{M_i}\T (J_{M_i} J_{M_i}\T)\inv J_{M_i}) (q_\t{rand}-q_\t{near})\\
	&= V_{\bot} V_{\bot}\T (q_\t{rand}-q_\t{near})
	\label{eq:steer_point_sol}
	\end{alignedat}
\end{align}
where $V_{\bot}$ contains the singular vectors that span the right nullspace of $J_{M_i}$ \cite{ratliff2014multivariate}. We normalize $d$ later in the algorithm, and thus do not include the constraint $||d||\leq \alpha$ in the reduced optimization problem. The new configuration $q_\t{near} + \alpha \frac{d}{||d||}$ will be on the tangent space of the manifold at configuration $q_\t{near}$, so that only few projection steps will be necessary before it can be added to the tree. Note that projection may fail to converge, and if this occurs the sample is discarded.

\begin{algorithm}[t]
\begin{algorithmic}[1]
	\algrenewcommand\algorithmicindent{1.5em}%
	\State $V_1 = \{q_{\t{start}}\}$; $E_1 = \emptyset$; $n=\text{len}(\mathcal{M})-1$
	\For{$i=1$ to $n$}
		\State $V_\t{goal} = \emptyset$
		\For{$k = 1$ to $m$} \label{alg:innerloop_start}
			\State $q_\t{rand} \leftarrow \t{Sample}(C)$  \label{alg:steer_start}
			\State $q_\t{near} \leftarrow \t{Nearest}(V_i, q_\t{rand})$ \label{alg:find_nearest}
			\State{$q_\t{new} \leftarrow \t{\smp\_STEER} (\alpha, \beta, r, q_\t{near}, M_i, M_{i+1}$)} \label{alg:steer_project_step}
	\If{\new{$||h_{M_{i}} (q_\t{new})|| > \ge$}} \label{alg:check_proj}
	    \State continue
	\EndIf
			\If{RRT$^*$\_EXTEND$(V_i, E_i, q_\text{near}, q_\text{new}, C_{\text{free},i})$} \label{alg:extend_start}
				\If{$||h_{M_{i+1}} (q_\t{new})|| < \ge$ \label{alg:next_manifold} \textbf{and}\\ $\hspace{5.5em} ||\t{Nearest}(V_\t{goal}, q_\t{new}) - q_\t{new}|| \geq \gr$} \label{alg:avoid_duplicates}
					\State $V_\t{goal} \leftarrow V_\t{goal} \cup q_\t{new}$
				\EndIf
			\EndIf \label{alg:extend_end}
        \EndFor \label{alg:innerloop_end}
		\State{// initialize next tree with the intersection nodes}
		\State $q_\t{root}=\t{null}, V_{i+1} = \{q_\t{root}\}; E_{i+1}=\emptyset$ \label{alg:newtreeinit_start}
		\For{$q \in V_\t{goal}$}
			\State $V_{i+1} \leftarrow V_{i+1} \cup \{q \}; E_{i+1} \leftarrow E_{i+1} \cup \{(q_\t{root}, q )\}$
		\EndFor\label{alg:newtreeinit_end}
		\State{\new{$C_{\t{free},i+1} \leftarrow \Upsilon(C_{\t{free},i}, V_\t{goal}[0])$}} \label{alg:next_freespace}
	\EndFor
	\State \textbf{return} $\mathrm{OptimalPath}(V_{1:n}, E_{1:n}, q_\t{start}, M_{n+1})$
\end{algorithmic}
\caption{\smp~$(\mathcal{M}, q_{\t{start}}, \ga, \gb, \ge, \rho, r, m)$}
\label{alg:smp}
\end{algorithm}

\begin{algorithm}[t]
\begin{algorithmic}[1]
	\algrenewcommand\algorithmicindent{1.0em}%
	\If{\new{$\t{Sample}(\mathcal{U}(0,1)) < \gb$}}
		\State $d \leftarrow \t{SteerConstraint}(q_\t{near}, M_i, M_{i+1})$ 
	\Else
		\State $d \leftarrow \t{SteerPoint}(q_\t{near}, q_\t{rand}, M_i)$
	\EndIf
	\State $q_\t{new} \leftarrow q_\t{near} + \alpha \frac{d}{||d||}$ \label{alg:steer_end}
	\If{$||h_{M_{i+1}} (q_\t{new})|| <$ \new{$ \t{Sample}(\mathcal{U}(0,r))$}} \label{alg:project_start}
		\State $q_\t{new} \leftarrow \t{Project}(q_\t{new}, M_{i} \cap M_{i+1})$
	\Else
		\State $q_\t{new} \leftarrow \t{Project}(q_\t{new}, M_{i})$
	\EndIf \label{alg:project_end}
	\State \textbf{return} $q_\t{new}$
\end{algorithmic}
\caption{$\t{\smp\_STEER} (\alpha, \beta, r, q_{\t{near}}, M_i, M_{i+1})$}
\label{alg:steer_project}
\end{algorithm}

\subsubsection{\new{$\t{SteerConstraint}(q_\t{near}, M_i, M_{i+1})$}}
\label{sec:steer_constraint}
This steering step extends the tree from $q_\t{near}\in C_{M_i}$ towards the intersection of the current and next manifold $M_i\cap M_{i+1}$, which can be expressed as the optimization problem
\begin{align}
\begin{alignedat}{2}
	&\minimize{\gg} &&||h_{M_{i+1}}(\gg(1))||^2 \\
	&\subjectto\quad &&\gg(0) = q_\t{near},\quad \int_{0}^1 ||\dot{\gg}(t)|| ~\mathrm{d}s \leq \ga\\ 
	&&&h_{M_i}(\gg(s)) = 0 \quad\qquad \forall~ s\in [0,1]
	\end{alignedat}
\end{align}
The difference from \Cref{eq:steer_point} is that the loss is now specified in terms of the distance to the next manifold $h_{M_{i+1}}(\gg(1))$. This cost pulls the robot towards the manifold intersection. Again, we approximate the curve with a line $\gamma(s)$ and apply a first-order Taylor expansion to the nonlinear terms, which results in the simplified problem
\begin{align}
\begin{alignedat}{2}
	&\minimize{d} &&\frac{1}{2}||h_{M_{i+1}}(q_\t{near}) + J_{M_{i+1}}(q_\t{near})d||^2\\
	&\subjectto~~&&J_{M_i}(q_\t{near}) d = 0\period
	\end{alignedat}
\end{align}
A solution $d$ can be obtained by solving the linear system
\begin{align}
    \label{eq:steer_constraint_sol}
    \mat{c c}{J_{M_{i+1}}\T J_{M_{i+1}} & J_{M_{i+1}}\T \\ J_{M_{i}} & 0} \mat{c}{d\\ \gl} = \mat{c}{-J_{M_{i+1}}\T h_{M_{i+1}}\\ 0}
\end{align}
where $\gl$ are the Lagrange variables. The solution is in the same direction as the steepest descent direction of the loss projected onto the tangent space of $M_i$.
Similar to the goal bias in RRT, a parameter $\gb\in[0,1]$ specifies the probability of selecting the SteerConstraint step rather than SteerPoint. 

After the steering strategy is determined and $d$ is computed, $q_\t{new}$ is projected either on $M_i$ or on the intersection manifold $M_i \cap M_{i+1}$ depending on the distance to the intersection of the manifolds measured by $||h_{M_{i+1}} (q_\t{new})||$ (\cref{alg:project_start} of \Cref{alg:steer_project}).
The threshold for this condition is sampled uniformly between $0$ and $r$, where the parameter $r\in \R_{>0}$ describes the closeness required by a point around $M_{i+1}$ to be projected onto $M_i \cap M_{i+1}$. 
This randomization is necessary in order to achieve probabilistic completeness that is discussed in \Cref{ssec:comp-optimal}, which also gives a formal definition of $r$. 

At this point, $q_\t{new}$ has been produced and the algorithm attempts to add it to the tree using the extend routine from RRT${}^*$ \cite{karaman2011sampling}, which we outline in \Cref{sec:rrt_extend}.


The final step in each inner iteration of \Cref{alg:smp} is to check two necessary conditions to determine if the point should also be added to the set of intersection nodes $V_\t{goal}$: 
\begin{enumerate}
    \item The point $q_\t{new}$ has to be on the next manifold $h_{M_{i+1}}$ (\cref{alg:next_manifold}); 
    \item $V_\t{goal}$ does not already contain the point $q_\t{new}$ or a point in its vicinity (\cref{alg:avoid_duplicates}).
\end{enumerate}

For the second condition, we introduce the parameter $\gr \in \R_{\geq 0}$ that is the minimum distance between two intersection points. As we will discuss in the theoretical analysis, $\gr$ must be $0$ in order to achieve probabilistic completeness and asymptotic optimality. In practice, we usually select $\gr > 0$, which results in better performance since it avoids having a large number of nearby and duplicate intersection points.

\subsection{Outer Loop: Initializing the Next Subproblem}
\label{ssec:phase3}
After the inner loop of \smp~completes, a new tree is initialized in \cref{alg:newtreeinit_start} -- \cref{alg:newtreeinit_end} with all the intersection nodes in $V_\t{goal}$ and their costs so far. Initializing a new tree only with the intersection nodes has the advantage that subsequent planning steps focus on these relevant nodes, which results in a lower computation time and more goal-directed growing of the tree towards the next manifold. To keep the tree structure, we add a synthetic root node $q_\t{root}$ as parent node for all intersection nodes (\Cref{fig:tree_init}). In \cref{alg:next_freespace}, the free configuration space is updated based on the reached intersection node. Convergence of the algorithm can be defined in various ways. We typically set an upper limit to the number of nodes or provide a time limit for the inner loop. After reaching the convergence criteria, the algorithm returns the path with the lowest cost that reached the goal manifold $M_{n+1}$.

\subsection{Completeness and Optimality}
\label{ssec:comp-optimal}
Under the assumption that $\rho$ is 0, \smp~is provably probabilistically complete and asymptotically optimal. Here, we outline our approach to proving these properties. Probabilistic completeness is proved in two steps. First we prove probabilistic completeness of \smp~on a single manifold. We claim subsequently that almost surely the tree grown on a manifold can be expanded onto the next manifold as the number of samples tends to infinity. The calculations related to the first step are based on \cite[Theorem 1]{Kleinbort2019}. Theorem 1 in \Cref{appendix:psm} proves the subsequent claim. 

Proving the asymptotic optimality of \smp~relies on three lemmas (Lemma 2, 3 and 4 in \Cref{appendix:psm}). Lemma 2 shows that for any weak clearance path \cite{karaman2011sampling} there exist a sequence of paths with strong clearance \cite{karaman2011sampling} that converges to a path with weak clearance. Our analysis assumes that the optimal path has weak clearance. Hence, there exist a sequence of strong clearance paths that converges to the optimal path.  Using Lemma 3 and Lemma 4, we prove that with an appropriate choice of the parameter $\gr$, the \smp~tree grown on the sequenced manifolds contains the paths which are arbitrarily close to any strong clearance path in the sequence of paths that converges to the optimal path. Finally, using continuity of the cost function, we prove that \smp~is asymptotically optimal.

Detailed  proofs of the probabilistic completeness and asymptotic optimality can be found in \Cref{sec:Probabilistic completeness} and \Cref{sec:Asymptotic optimality}.

\section{Evaluation}
\label{sec:psm_experiments}
In the following experiments, we solve kinematic motion planning problems where the cost function measures path length. We compare the following methods with each other:

\begin{itemize}[leftmargin=*]
    \item \textbf{\smp}: The method proposed here -- \Cref{alg:smp}.
    \item \textbf{\smp~(Greedy)}: \Cref{alg:smp} with the modification that only the node with the lowest cost in $V_\t{goal}$ is selected to initialize the next tree (\cref{alg:newtreeinit_start} -- \cref{alg:newtreeinit_end} of \Cref{alg:smp}).
    \item \textbf{\smp~(Single Tree)}: This method grows a single tree over the manifold sequence without splitting it into individual subtrees. The algorithm uses $\rho=0$. It is explained in detail in Algorithm 3 in \Cref{sec:smp_single_tree}.
    \item {\textbf{RRT$^*$+IK}}: This method consists of a two-step procedure that is applied to every manifold in the sequence. First, a goal point on the manifold intersection is generated via inverse kinematics by randomly sampling a point in $C$ and projecting it onto the manifold intersection. Next, RRT$^*$ is applied to compute a path on the current manifold towards this node \cite{kingston2019ijrr}. This procedure is repeated until a point on the goal manifold is reached.
    \item {\textbf{Random MMP}}: The Randomized Multi-Modal Motion Planning algorithm \cite{hauser2011randomized} was originally developed for the more general class of motion planning problems where the sequence of manifolds is unknown. Here, we apply it to the simpler problem where the sequence is given in advance and use RRT$^*$ to connect points between mode transitions.
    \item {\textbf{CBiRRT}: The Constrained Bidirectional Rapidly-Exploring Random Tree algorithm \cite{berenson2009manipulation, berenson2011task} grows two trees towards each other with the RRT Connect strategy \cite{kuffner2000rrt}. In each steering step, the sampled configuration is projected onto the manifold with the lowest level set function value. Since CBiRRT does not optimize an objective function, we apply a short cutting algorithm as a post-processing step.}
\end{itemize}
We compare these methods on the following criteria:
\begin{itemize}[leftmargin=*]
    \item \textbf{Path length} -- The length of the found path in $C$ space.
    \item \textbf{Success rate} -- The number of times a collision-free path to the goal manifold is found for different random seeds.
    \item \textbf{Computation time} -- Time taken to compute a path. All experiments are run on a 2.2 GHz Quad-Core Intel Core i7.
\end{itemize}
\new{The parameter values $\alpha, \beta, \epsilon, \rho, r, m$ for the individual experiments were chosen experimentally and are summarized in \Cref{tab:exp_parameter}.}

\begin{table}[t]
\caption{Parameters of the 3D Point on Geometric Constraints and Multi-Robot Object Transport experiments.}
    \label{tab:exp_parameter}
\centering
\renewcommand{\tabcolsep}{4pt}
    \begin{tabular}{|c|c|c|c|}
        \hline
        \textbf{Parameter name} & \textbf{Symbol} & \specialcell{\textbf{3D point}\\ (\Cref{ssec:geom_exp})} & \specialcell{\textbf{Robot transport}\\ (\Cref{ssec:robot_exp})}\\
        \hline
        max step size & $\alpha$ & $1.0$ & $1.0$\\
        goal bias probability & $\beta$ & $0.1$ & $0.3$\\
        constraint threshold & $\epsilon$ & $0.01$ & 1e-5\\
        duplicate threshold & $\rho$ & $0.1$ & $0.5$\\
        projection distance & $r$ & $1.5$ & $0.5$\\
        number of samples & $m$ & $1200$ & $2000$\\
        \hline
    \end{tabular}
\end{table}

\subsection{3D Point on Geometric Constraints}
\label{ssec:geom_exp}
\new{We demonstrate the planning performance and properties of the individual methods on a simple point that can move in a 3D space.} The point needs to traverse three constraints defined by geometric primitives. The configuration space is limited to $[-6, 6]$ in all three dimensions.
The initial state is \mbox{$q_\t{start}=(3.5, 3.5, 4.45)$} and the sequence of manifolds is
\begin{enumerate}
	\item Paraboloid: $$h_{M_1}(q) = 0.1 q_1^2 + 0.1 q_2^2 + 2 - q_3$$
	\item Cylinder: $$h_{M_2}(q) = 0.25 q_1^2 + 0.25 q_2^2 - 1.0$$
	\item Paraboloid: $$h_{M_3}(q) = -0.1 q_1^2 - 0.1 q_2^2 - 2 - q_3$$
	\item Goal point:  $$h_{M_4}(q) = q - q_\t{goal}$$ with the goal configuration \mbox{$q_\t{goal}=(-3.5, -3.5, -4.45)$}.
\end{enumerate}
We evaluate the algorithms on two variants of this problem: an obstacle-free variant (\Cref{fig:psm_hourglass}), and a variant that contains four box obstacles placed at the intersections between the manifolds (\Cref{fig:psm_hourglass_paths}). In \Cref{fig:psm_hourglass}, $q_\t{start}$ is drawn as red point and $q_\t{goal}$ as blue point. The intersection nodes $V_\t{goal}$ are shown as magenta points and a solution path from \smp~is visualized as a line. 

The results on success rate, path length and computation time are given in \Cref{tab:comparison_results}. All the methods are consistently able to find a path for all $10$ random seeds. \smp~and \smp~(Single Tree) consistently achieve the lowest cost in both scenarios. \smp~(Single Tree) has a higher computation time, because it keeps all nodes in a single tree whereas \smp~splits them into individual trees per manifold. CBiRRT does not optimize an objective function and immediately converges when it finds a feasible path, which achieves the overall lowest computation time. However, the mean cost and standard deviations are higher compared to \smp. RRT${}^*$+IK and Random MMP only optimize over the individual paths, but do not optimize the intersection point selection, which results in higher costs and standard deviations. 
\smp~(Greedy) is better in terms of computation time since it only takes one intersection point as initial point for the next tree, but achieves an overall lower performance compared to \smp. 
The results also show that only \smp~and \smp~(Single Tree) are consistently able to find the intersection regions between the obstacles in which the optimal path lies. The other methods mainly select the intersection regions that are discovered first in the exploration. \Cref{fig:psm_hourglass_paths} shows a set of found paths on the variant with obstacles.

In \Cref{fig:psm_hourglass_rho_cost}, the path costs $J(\tau)$ of \smp~are compared for various values of $\gr$. This parameter specifies the minimum distance between two intersection points, which influences the number of intersection points created during planning (\cref{alg:avoid_duplicates} in \Cref{alg:smp}). As a reference, we visualize the path costs of \smp~(Greedy) and \smp~(Single Tree). The results show that the computed paths of \smp~with a small $\gr$ value are very similar to the ones of \smp~(Single Tree) whereas for larger $\gr$ values, \smp~converges to the same performance as \smp~(Greedy) since only a single intersection point is considered. Therefore, $\gr$ can be seen as a trade-off between planning faster and achieving lower path costs. \Cref{fig:psm_hourglass_samples_cost} compares the path costs for different samples $m$. All methods improve with an increasing amount of samples and \smp~and \smp~(Single Tree) converge to similar cost values.

\begin{figure}[t]
	\centering
	\includegraphics[width=.45\textwidth]{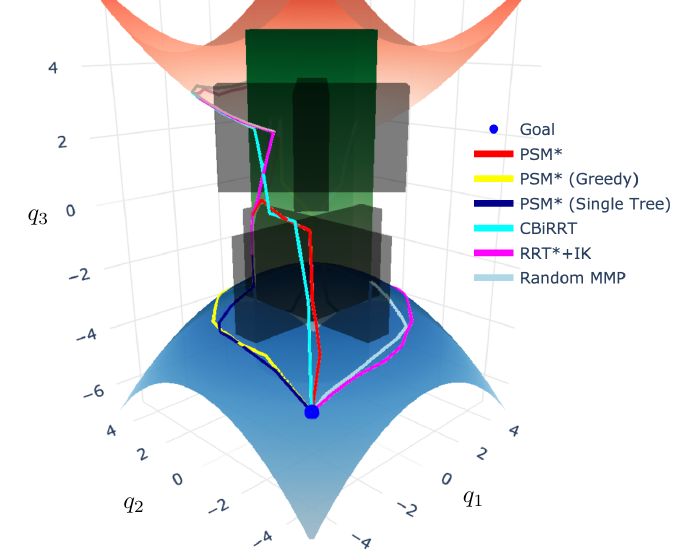}
	\caption{Samples of found paths on the 3D point on with obstacles problem (\Cref{ssec:geom_exp}).}
	\label{fig:psm_hourglass_paths}
\end{figure}

\begin{table}[t]
\centering
\caption{Results of the 3D point and robot transport problems. We report the mean and one unit standard deviation over $10$ runs with different random seeds.}
    \label{tab:comparison_results}
\renewcommand{\tabcolsep}{4pt}
    \begin{tabular}{|l|c|c|c|}
\hline
                                 & \textbf{Success} & \textbf{Path length} & \textbf{Comp. time [s]} \\
 \hline
 \textbf{3D point w/o obstacles} &                  &                      &                     \\
 PSM${}^*$                       & $10 / 10$        & $\mathbf{14.47\pm 0.04}$      & $10.64\pm 0.16$     \\
 PSM${}^*$ (Single Tree)         & $10 / 10$        & $14.47\pm 0.05$      & $13.72\pm 0.18$     \\
 PSM${}^*$ (Greedy)              & $10 / 10$        & $16.20\pm 0.05$      & $10.36\pm 0.10$     \\
 RRT${}^*$+IK                    & $10 / 10$        & $17.84\pm 2.23$      & $28.35\pm 13.50$    \\
 Random MMP                      & $10 / 10$        & $17.33\pm 1.28$      & $34.75\pm 17.21$    \\
 CBiRRT                          & $10 / 10$        & $14.70\pm 0.71$      & $\mathbf{5.04\pm 0.30}$      \\
 \hline
 \textbf{3D point w/ obstacles}  &                  &                      &                     \\
 PSM${}^*$                       & $10 / 10$        & $15.95\pm 0.13$      & $13.74\pm 0.51$     \\
 PSM${}^*$ (Single Tree)         & $10 / 10$        & $\mathbf{15.87\pm 0.18}$      & $20.42\pm 0.86$     \\
 PSM${}^*$ (Greedy)              & $10 / 10$        & $19.69\pm 0.27$      & $12.89\pm 0.51$      \\
 RRT${}^*$+IK                    & $10 / 10$        & $21.56\pm 3.05$      & $30.21\pm 7.55$     \\
 Random MMP                      & $10 / 10$        & $22.15\pm 2.20$      & $42.09\pm 17.22$     \\
 CBiRRT                          & $10 / 10$        & $16.66\pm 1.34$      & $\mathbf{3.34\pm 0.25}$      \\
 \hline
 \textbf{Robot transport A}      &                  &                      &                     \\
 PSM${}^*$                       & $10 / 10$        & $\mathbf{7.76\pm 0.84}$       & $7.22\pm 0.81$      \\
 PSM${}^*$ (Single Tree)         & $9 / 10$         & $8.53\pm 1.19$       & $11.56\pm 0.41$     \\
 PSM${}^*$ (Greedy)              & $10 / 10$        & $8.18\pm 0.91$       & $\mathbf{7.03\pm 0.09}$      \\
 RRT${}^*$+IK                    & $10 / 10$        & $11.83\pm 3.23$      & $74.65\pm 28.86$    \\
 Random MMP                      & $10 / 10$        & $11.62\pm 2.36$      & $85.45\pm 25.66$    \\
 \hline
 \textbf{Robot transport B}      &                  &                      &                     \\
 PSM${}^*$                       & $10 / 10$        & $\mathbf{14.73\pm 1.27}$      & $\mathbf{14.19\pm 0.76}$     \\
 PSM${}^*$ (Single Tree)         & $0 / 10$         & --                   & --                  \\
 PSM${}^*$ (Greedy)              & $9 / 10$         & $15.41\pm 2.38$      & $14.75\pm 0.53$     \\
 RRT${}^*$+IK                    & $2 / 10$         & $45.14\pm 0.58$      & $89.84\pm 11.00$    \\
 Random MMP                      & $10 / 10$        & $16.96\pm 4.59$      & $163.26\pm 37.80$   \\
 \hline
 \textbf{Robot transport C}      &                  &                      &                     \\
 PSM${}^*$                       & $10 / 10$         & $\mathbf{27.07 \pm 2.58}$      & $\mathbf{275.83 \pm 19.73}$   \\
 PSM${}^*$ (Single Tree)         & $0 / 10$         & --                   & --                  \\
 PSM${}^*$ (Greedy)              & $10 / 10$         & $31.75 \pm 2.51$                   & $311.81\pm 9.79$                  \\
 RRT${}^*$+IK                    & $0 / 10$         & --                   & --                  \\
 Random MMP              & $0 / 10$         & --                   & --                  \\
 \hline
\end{tabular}
\end{table}

\subsection{Multi-Robot Object Transport Tasks}
In this experiment, we evaluate \smp~on various object transportation tasks involving multiple robots. The overall objective is to transport an object from an initial to a goal location. We consider three variations of this task:
\begin{itemize}[leftmargin=*]
    \item \textbf{Task A}: A single robot arm mounted on a table with $k=6$ degrees of freedom. The task is to transport an object from an initial location on the table to a target location. This task is described by $n=3$ manifolds.
    \item \textbf{Task B}: This task consists of two robot arms and a mobile base consisting of $k=14$ degrees of freedom. The task is defined such that the first robot arm picks the object and places it on the mobile base. Then, the mobile base brings it to the second robot arm that picks it up and places it on the table. This procedure is described with $n=5$ manifolds.
    \item \textbf{Task C}: In this task, four robots are used to transport two objects between two tables. Three arms are mounted on the tables and another arm with a tray is mounted on a mobile base. Besides transporting the objects, the orientation of the two objects needs to be kept upright during the whole motion. This task is described with $n=12$ manifolds and the configuration space has $k=26$ degrees of freedom.
\end{itemize}
The initial states of the three tasks are visualized in \Cref{fig:robot_tasks} where the target locations of the objects are shown in green. The geometries of the objects were chosen such that the tasks are intersection point independent (\Cref{ssec:ipip}). Three types of constraints are used to describe the tasks. Picking up an object is defined with the constraint
$h_M(q) = x_g - f_{\t{pos}, e}(q)$
where $x_g\in\R^3$ is the location of the object and $f_{\t{pos}, e}(q)$ is the forward kinematics function to a point $e\in\R^3$ on the robot end effector. The handover of an object between two robots is described by
$h_M(q) = f_{\t{pos}, e_1}(q) - f_{\t{pos}, e_2}(q)$
where $f_{\t{pos}, e_1}(q)$ is the forward kinematics function to the end effector of the first robot and $f_{\t{pos}, e_2}(q)$ is that of the second robot.
The orientation constraint is given by an alignment constraint
$h_M(q) = f_{\t{rot}, z}(q) \T e_z - 1$
where $f_{\t{rot}, z}(q)$ is a unit vector attached to the robot end effector that should be aligned with the vector $e_z=(0,0,1)$ to point upwards.
These constraints are sufficient to describe the multi-robot transportation tasks. The parameters of the algorithms are summarized in \Cref{tab:exp_parameter}.

The costs of the algorithms are reported in \Cref{tab:comparison_results}. \new{CBiRRT could not be applied to this task since it requires a goal configuration at which the backward tree originates, which is not available for these tasks.} On task A, nearly all methods robustly find solutions. However, when task complexity increases, RRT$^*$+IK, \smp~(Single Tree), and Random MMP have difficulties to solve to the problem. Only \smp~and \smp~(Greedy) were able to find solutions of task C with the given parameters. The cost and computation time of \smp~is lower compared to \smp~(Greedy). A found solution of \smp~for Task C is visualized in \Cref{fig:robot_task_c}.

\subsection{Pick-and-Pour on Panda Robot}
\label{ssec:robot_exp}
In this experiment, we demonstrate \smp~on a real Panda robot arm with seven degrees of freedom. The task description consists of $n=7$ manifolds, which describe the individual grasp, transport, and pouring motions. We assume the position of the objects in the scene are known. The parameters are the same as in the robot transport experiments (\Cref{tab:exp_parameter}) and the planning time was 131.65 seconds. \Cref{fig:panda_pour} shows a found path executed on the real robot.

\begin{figure*}
	\centering
	\subfloat[\label{fig:psm_hourglass_rho_cost}]{\includegraphics[width=.5\textwidth]{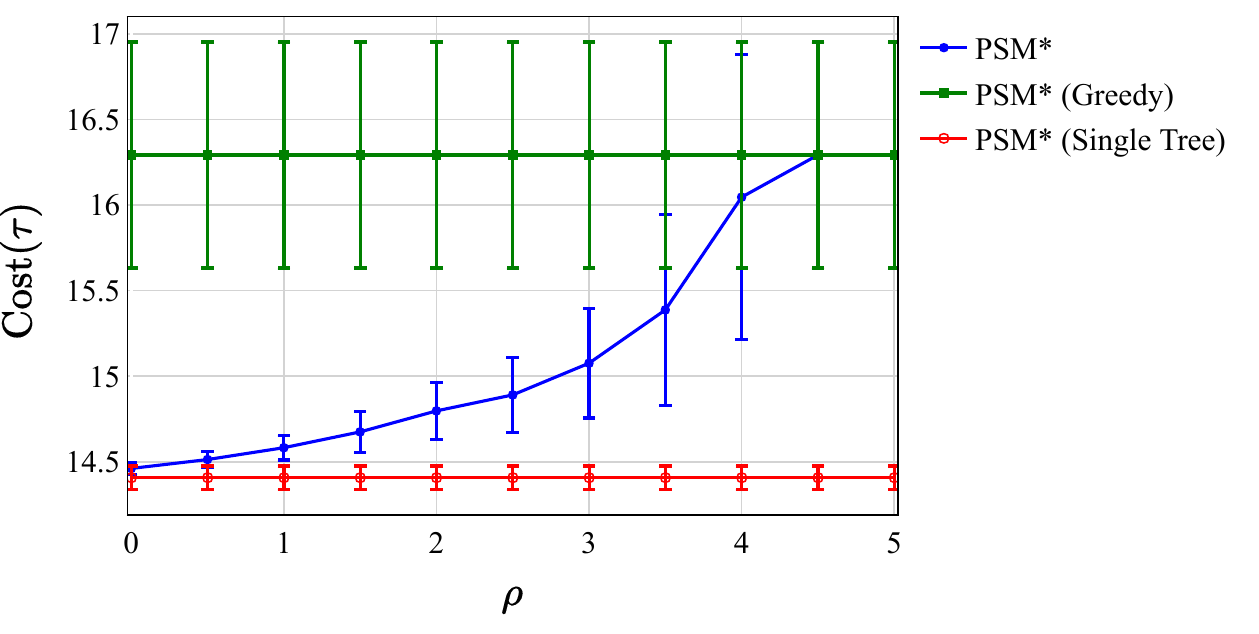}}
	\subfloat[\label{fig:psm_hourglass_samples_cost}]{\includegraphics[width=.5\textwidth]{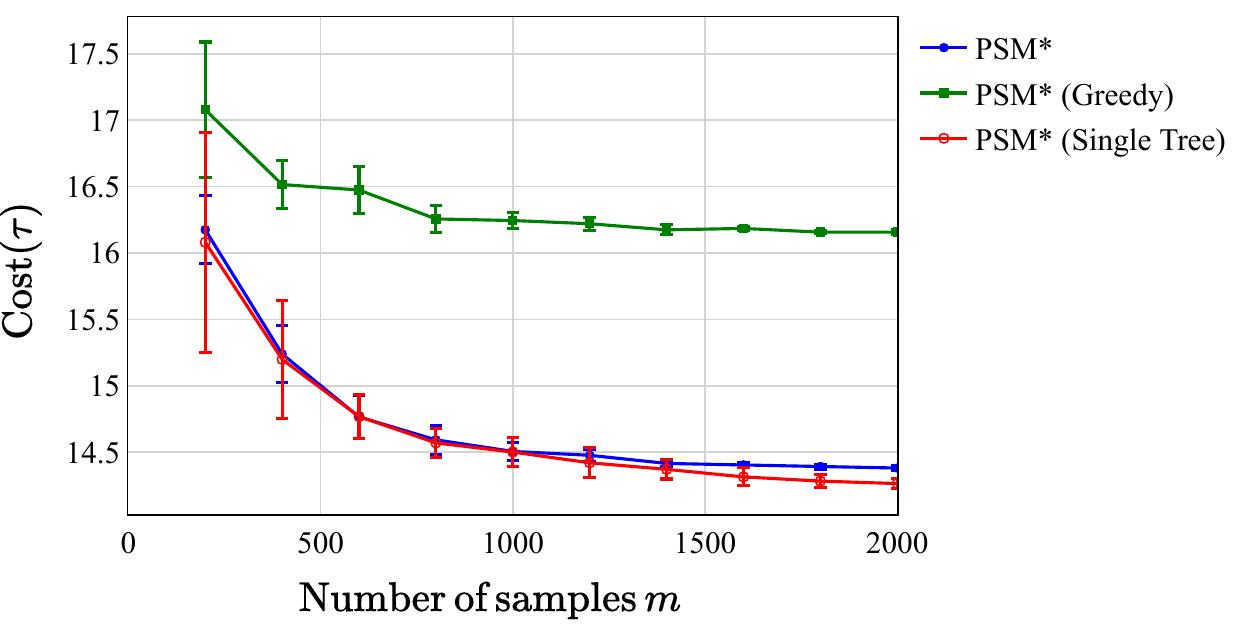}}
	\caption{Path costs over variations of the parameter $\gr$ (left) and the number of samples $m$ (right) on the geometric constraints w/o obstacles problem (\Cref{ssec:geom_exp}). The graphs show the mean and unit standard deviation over $10$ trials. Figure (a) shows the costs increase for higher values of $\gr$, meaning fewer intersection points are considered during planning. 
	In (b), the performance of all methods improves with larger $m$ values where \smp~and \smp~(Single Tree) converge to similar path costs.}
	\label{fig:psm_hourglass_eval}
\end{figure*}

\begin{figure*}
    \centering
    \subfloat[Task A]{\includegraphics[height=0.205\textwidth]{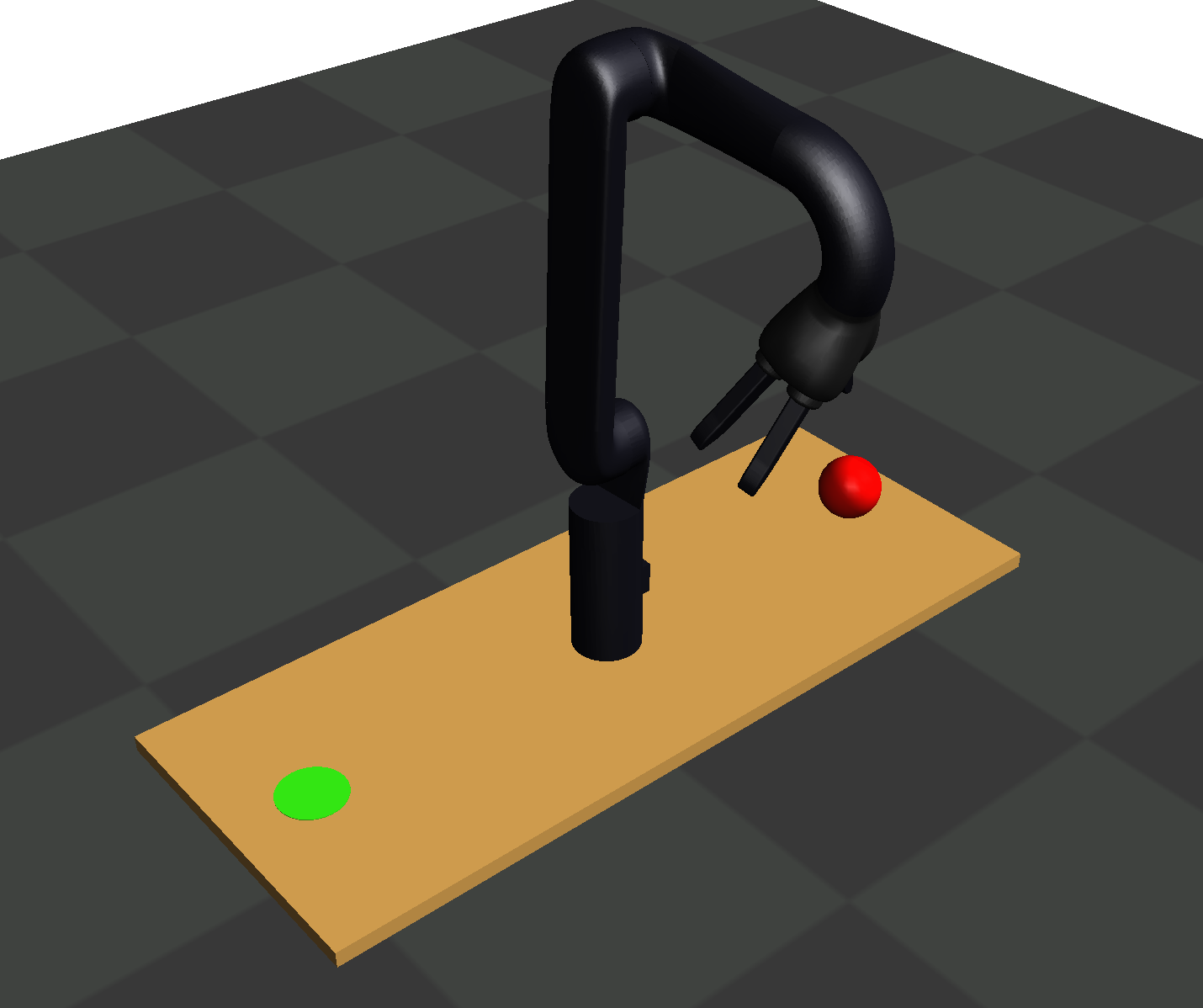}}
    \hspace{0.0in}
    \subfloat[Task B]{\includegraphics[height=0.205\textwidth]{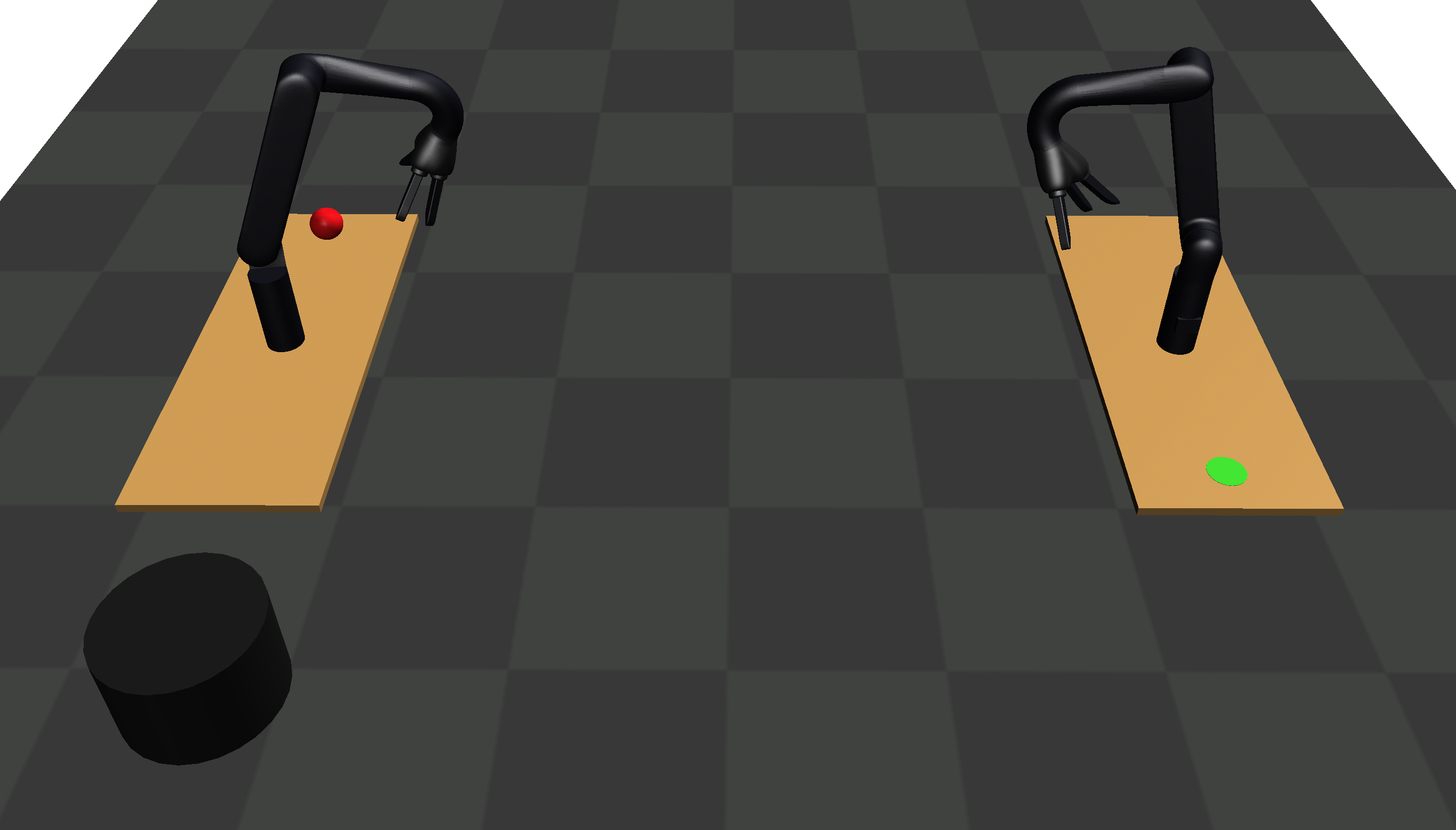}}
    \hspace{0.0in}
    \subfloat[Task C]{\includegraphics[height=0.205\textwidth]{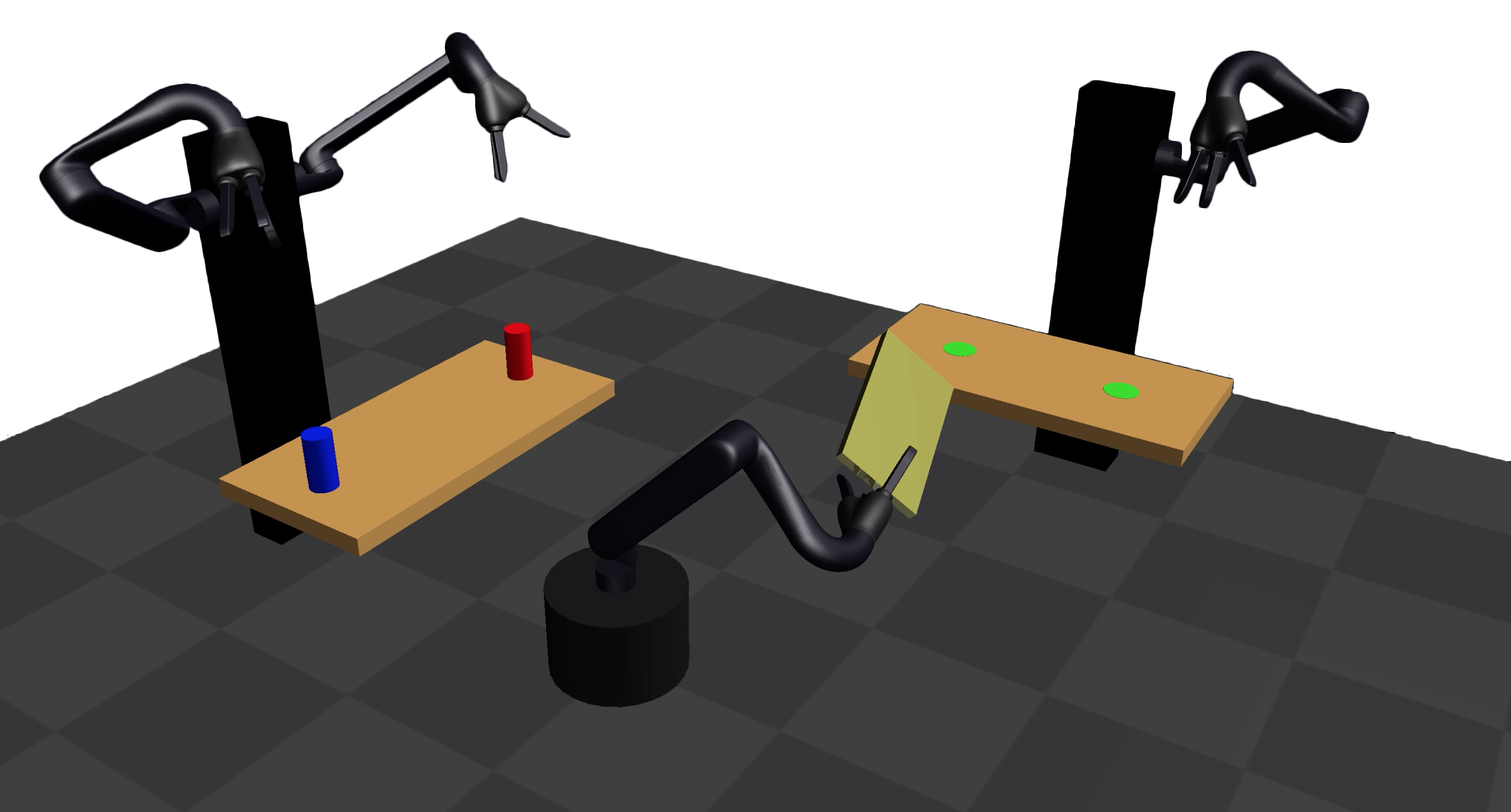}}
    \caption{Start states of the object transport tasks where the goal is to place the red and blue object to the green target locations.}
    \label{fig:robot_tasks}
\end{figure*}

\begin{figure*}
    \centering
    \includegraphics[width=1.0\textwidth]{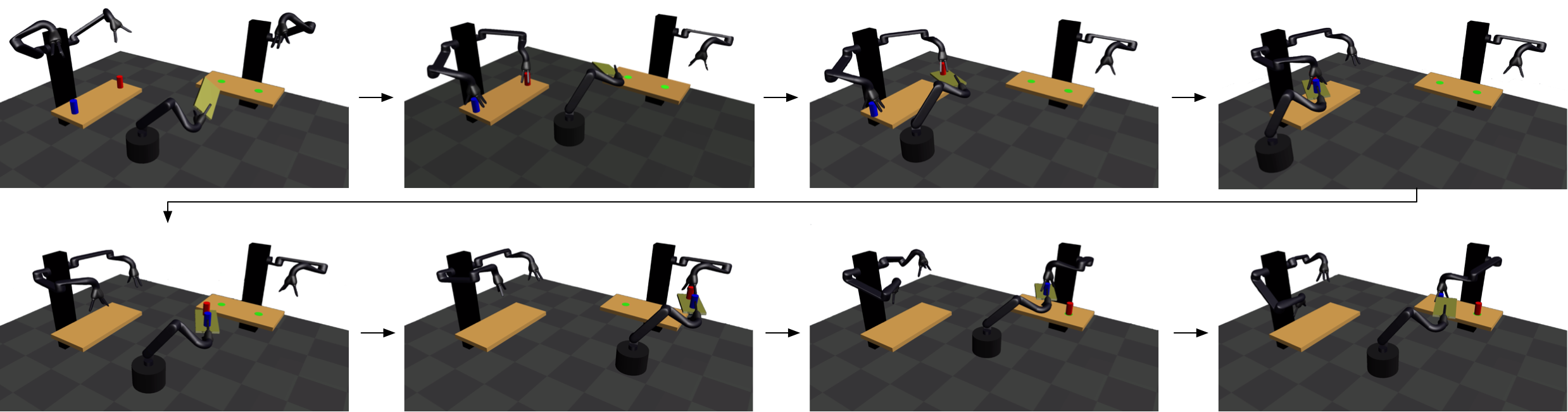}
\caption{Snapshots of the resulting motion that \smp~found for Task C.}
    \label{fig:robot_task_c}
\end{figure*}
\begin{figure*}
    \centering
    \includegraphics[width=\textwidth]{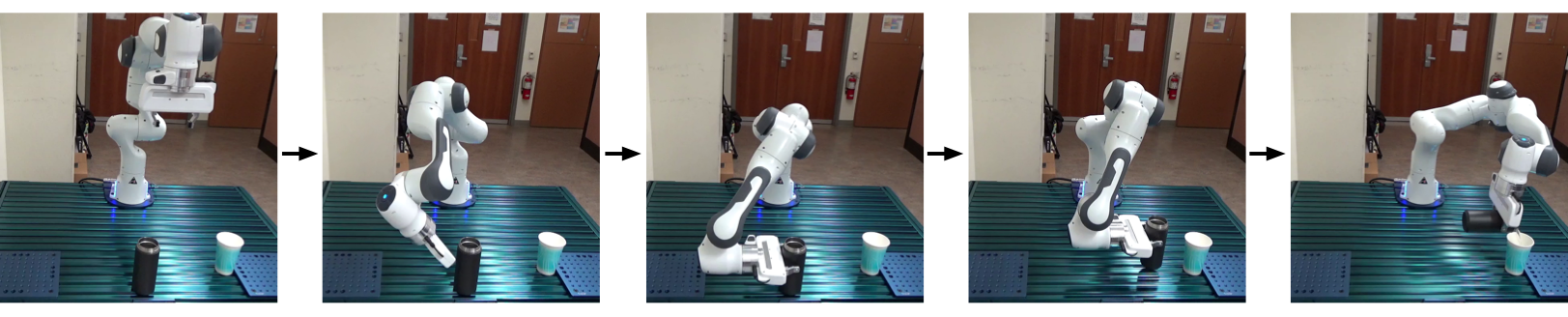}
    \caption{Snapshots of the pouring task motion.}
    \label{fig:panda_pour}
\end{figure*}

\section{Discussion}
\label{sec:psm_conclusion}
We proposed the algorithm \smp~to solve sequential motion planning 
posed as a constrained optimization problem, where the goal is to find a collision-free path that minimizes a cost function and fulfills a given sequence of manifold constraints.
The algorithm is applicable to a certain problem class that is specified by the intersection point independent property, which says that the change in free configuration space is independent of the selected intersection point between two manifolds. 
\smp~uses RRT$^*$'s extend method with a novel steering strategy that is able to discover intersection points between manifolds.
We proved that \smp~is probabilistically complete and asymptotically optimal and demonstrated it on multi-robot transportation tasks.

Restricting the problem class to intersection point independent planning problems allowed us to develop efficient solution strategies like growing a single tree over a sequence of manifolds. An interesting question for future research is how to extend the strengths of \smp~to a larger problem class; specifically, how \smp~can be used for problems that do not fulfill the intersection point independent property. For such problems, the choice of intersection points influences the future parts of the planning problem, which results in a more complex problem. 
An interesting aspect is the effect of an intersection point on subsequent manifolds and how solutions for one intersection point can be reused and transferred to other intersection points without replanning from scratch.
Further, it may be possible to reduce such problems to the simpler intersection point independent problem class, for example, by morphing object geometries into simpler shapes such that the intersection point independent property is satisfied. \smp~could be applied to the reduced problem and provide a good initial guess for solving the original problem.


\chapter{\ecmnn: Learning Equality Constraints for Motion Planning on Manifolds}
\label{ch:ecomann}
With \Cref{ch:psm}, we have a motion planner for constrained motions, where we assumed the user has the knowledge and capability to write the constraint manifolds for the entire task sequence. However, this may not always be the case, or it may be more convenient to give demonstrations of a task such that the constraints can be automatically extracted for future planning.
In this chapter, we explore the integration of learning techniques into the motion planner, specifically, neural network models that are able to automatically learn manifold constraints from robot demonstrations instead of requiring hand-specified constraints.
This chapter contains work published in \cite{sutanto2020learning}.

\section{Introduction}
Robots must be able to plan motions that follow various constraints in order for them to be useful in real-world environments.
Constraints such as holding an object, maintaining an orientation, or staying within a certain distance of an object of interest are just some examples of possible restrictions on a robot's motion.
In general, two approaches to many robotics problems can be described. One is the traditional approach of using handwritten models to
capture environments, physics, and other aspects of the problem mathematically or analytically, and then solving or optimizing these to find a solution. 
The other, popularized more recently, involves the use of machine learning to replace, enhance, or simplify these hand-built parts.
Both have challenges: Acquiring training data for learning can be difficult and expensive, while describing precise models analytically can range from tedious to impossible.
Here, we approach the problem from a machine learning perspective and propose a solution to learn constraints from demonstrations. The learned constraints can be used alongside analytical solutions within a motion planning framework.

In this work, we propose a new learning-based method for describing motion constraints, called {\ecmnnlong} (\ecmnn). {\ecmnn} learns a function which evaluates a robot configuration on whether or not it meets the constraint, and for configurations near the constraint, on how far away it is. 
We train {\ecmnn} with datasets consisting of configurations that adhere to constraints, and show results for tasks learned from demonstrations of robot tasks. 
We use a sequential motion planning framework 
to solve motion planning problems that are both constrained and sequential in nature, 
and incorporate the learned manifolds into it. We evaluate the constraints learned by \mbox{\ecmnn} with various datasets on their representation quality. Further, we investigate the usability of learned constraints in sequential motion planning problems.

\subsection{Background}
The work described in this chapter combines ideas from many fields, including constrained sampling-based motion planning (SBMP), manifold learning, and learning from demonstrations, as described in \Cref{ch:motion-planning-background}.

As opposed to other constrained SBMP methods, our method differs in that {\ecmnn} learns an implicit description of a constraint manifold via a level set function, and during planning, we assume this representation for each task. We note that our method could be combined with others, e.g. learned sampling distributions, to further improve planning results.

Our work is related to many of the manifold learning approaches described in \Cref{ssec:manifold-learning}; in particular, the tangent space alignment in LTSA is an idea that {\ecmnn} uses extensively. 
Similar to the ideas presented in this chapter, the work in \cite{osa2020learning} delineates an approach to solve motion planning problems by learning the solution manifold of an optimization problem. 
In contrast to others, our work focuses on learning implicit functions of equality constraint manifolds, which is a generalization of the learning representations of Signed Distance Fields (SDF) \cite{park2019deepsdf, mahler2015gp}, up to a scale, but for higher-dimensional manifolds.

Further, our work can be seen as a special case of inverse optimal control (IOC) where the task is only represented in form of constraints. Instead of using the extracted constraints in optimal control methods, we integrate them into 
sampling-based motion planning methods, which are not parameterized by time and do not suffer from poor initializations.

\subsection{Motion Planning on Manifolds}
\label{subsec:smp}
In this work, we aim to integrate learned constraint manifolds into 
the \smp~framework from \Cref{ch:psm}, using the same problem formulation and notation as described in \Cref{sec:problem_formulation}.
We employ data-driven learning methods to learn individual
manifolds $M$ from demonstrations with the goal to integrate them with
analytically defined manifolds into this framework. 


\section{{\ecmnnlong} (\ecmnn)}
\label{ssec:ecmnn}

We propose a novel neural network structure, called \emph{\ecmnnlong} (\ecmnn),
which is a (global) equality constraint manifold learning representation that
enforces the alignment of the (local) tangent spaces and normal spaces with the
information extracted from the Local Principal Component Analysis (Local PCA)
\cite{Kambhatla_LocalPCA} of the data. 
{\ecmnn} takes a configuration $\jointposition$ as input and outputs the
prediction of the implicit function $\constraintfunction(\jointposition)$.
We train {\ecmnn} in a supervised manner, from demonstrations. 
One of the challenges is that the supervised training dataset is collected only
from demonstrations of data points which are on the equality constraint manifold
$\onconstraintconfigspace$, called the \emph{on-manifold} dataset.
Collecting both the on-manifold $\onconstraintconfigspace$ and off-manifold
$\offconstraintconfigspace = \{ \jointposition \in \configspace \mid
\constraintfunction(\jointposition) \neq \vct{0} \}$ datasets for supervised
training would be tedious because the implicit function $\constraintfunction$
values of points in $\offconstraintconfigspace$ are typically unknown and hard
to label.
We show that, though our approach is only given data on
$\onconstraintconfigspace$, it can still learn a useful and sufficient
representation of the manifold for use in planning.

\begin{figure}[h]
    \centering
    \includegraphics[width=0.5\textwidth]{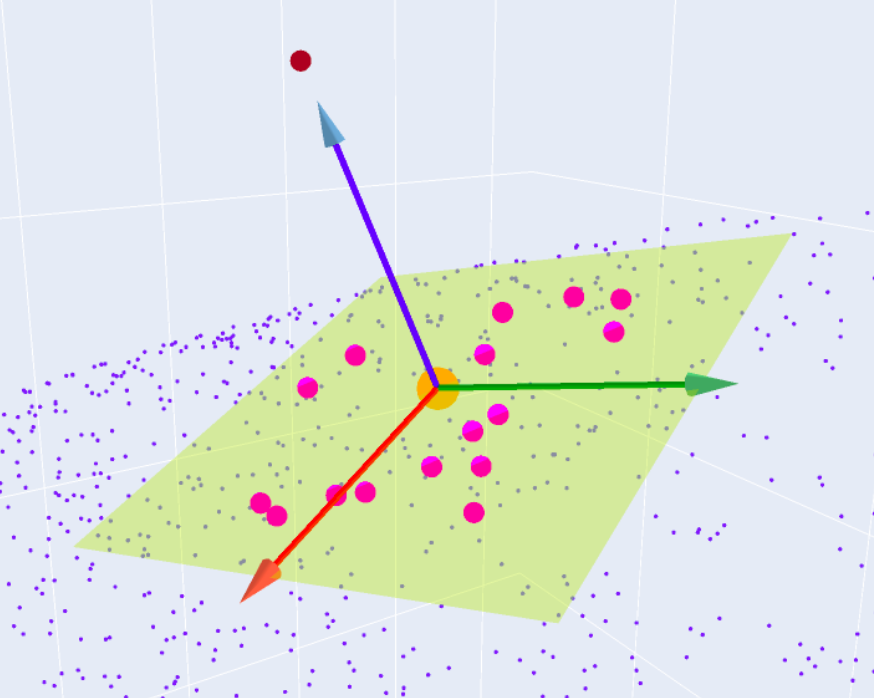}
    \caption{A visualization of data augmentation along the 1D normal space of a
    point $\jointposition$ in 3D space. Here, purple points are the dataset,
    pink points are the $\numnearestneighbor$NN of $\jointposition$, and the
    dark red point is $\augjointposition$. $\jointposition$ is at the axes
    origin, and the green plane is the approximated tangent space at that
    point.}
    \label{fig:localpca}
\end{figure}

Our goal is to learn a single global representation of the constraint manifold $M$ in the form of a neural network. 
Our approach leverages local information on the manifold in the form of the
tangent and normal spaces \cite{Deutsch2015_TensorVotingGraph}.
With regard to $\constraintfunction$, the tangent and normal spaces are
equivalent to the null and row space, respectively, of the Jacobian matrix 
$\constraintmanifoldjacobian(\acute{\jointposition}) = \left. \frac{\partial
\constraintfunction({\jointposition})}{\partial
{\jointposition}}\right|_{{\jointposition} = \acute{\jointposition}}$, 
and valid in a small neighborhood around the point $\acute{\jointposition}$.
Using on-manifold data, the local information of the manifold can be analyzed
using Local PCA. For each data point $\jointposition$ in the on-manifold
dataset, we establish a local neighborhood using $\numnearestneighbor$-nearest
neighbors ($\numnearestneighbor$NN) $\origKNN =
\{\nearestneighborjointposition_1, \nearestneighborjointposition_2, \dots,
\nearestneighborjointposition_\numnearestneighbor\}$, with $\numnearestneighbor
\geq \dimambient$. 
After a change of coordinates, $\jointposition$ becomes the origin of a new
local coordinate frame $\lpcacoordframe$, and the $\numnearestneighbor$NN
becomes $\recenteredKNN = \{\recenterednearestneighborjointposition_1,
\recenterednearestneighborjointposition_2, \dots,
\recenterednearestneighborjointposition_\numnearestneighbor\}$ with
$\recenterednearestneighborjointposition_\idxnearestneighbor =
\nearestneighborjointposition_\idxnearestneighbor - \jointposition$ for all
values of $\idxnearestneighbor$. Defining the matrix 
$\designmatrix = 
\begin{bmatrix} 
\recenterednearestneighborjointposition_1 & \recenterednearestneighborjointposition_2 & \hdots & \recenterednearestneighborjointposition_\numnearestneighbor \\
\end{bmatrix}\T \in \Re^{\numnearestneighbor \times \dimambient}
$, we can compute the covariance matrix $\samplecovariancematrix = \frac{1}{\numnearestneighbor-1} \designmatrix\T \designmatrix \in \Re^{\dimambient \times \dimambient}$.
The eigendecomposition of $\samplecovariancematrix = \covrighteigmat
\covdiagsingularvalues \covrighteigmat\T$ gives us the Local PCA. 
The matrix $\covrighteigmat$ contains the eigenvectors of
$\samplecovariancematrix$ as its columns in decreasing order w.r.t.\ the
corresponding eigenvalues in the diagonal matrix $\covdiagsingularvalues$. These
eigenvectors form the basis of the coordinate frame $\lpcacoordframe$.

This local coordinate frame $\lpcacoordframe$ is tightly related to the tangent
and normal spaces of the manifold at $\jointposition$. That is, the
$(\dimambient - \dimconstraint)$ eigenvectors corresponding to the $(\dimambient
- \dimconstraint)$ biggest eigenvalues of $\covdiagsingularvalues$ form a basis
of the tangent space, while the remaining $\dimconstraint$ eigenvectors form the
basis of the normal space. 
Furthermore, 
the $\dimconstraint$ smallest eigenvalues of $\covdiagsingularvalues$ will be
close to zero, resulting in the $\dimconstraint$ eigenvectors associated with
them forming the basis of the null space of $\samplecovariancematrix$. On the
other hand, the remaining $(\dimambient - \dimconstraint)$ eigenvectors form the
basis of the row space of $\samplecovariancematrix$. 
We follow the same technique as \cite{Deutsch2015_TensorVotingGraph} for
automatically determining the number of constraints $\dimconstraint$ from data,
which is also the number of outputs of {\ecmnn}\footnote{Here we assume that the
intrinsic dimensionality of the manifold at each point remains constant.}. 
Suppose the eigenvalues of $\samplecovariancematrix$ are $\{\coveigval_1,
\coveigval_2, \dots, \coveigval_\dimambient\}$ (in decreasing order w.r.t.
magnitude). Then the number of constraints can be determined as $\dimconstraint
= \argmax{\left(\left[\coveigval_1 - \coveigval_2, \coveigval_2 - \coveigval_3,
\dots, \coveigval_{\dimambient-1} - \coveigval_{\dimambient}
\right]\right)}$.

We now present several methods to define and train \ecmnn.

\subsection{Alignment of Local Tangent and Normal Spaces}
{\ecmnn} aims to align the following:
\begin{enumerate}
    \item the null space of $\constraintmanifoldjacobian$ and the row space of
        $\samplecovariancematrix$ (which must be equivalent to the tangent space)
    \item the row space of $\constraintmanifoldjacobian$ and the null space of
        $\samplecovariancematrix$ (which must be equivalent to the normal space)
\end{enumerate}
for the local neighborhood of each point $\jointposition$ in the on-manifold dataset. 
Suppose the eigenvectors of $\samplecovariancematrix$ are $\{\coveigvec_1, \coveigvec_2, \dots, \coveigvec_\dimambient\}$ and the singular vectors of $\constraintmanifoldjacobian$ are $\{\constrjaceigvec_1, \constrjaceigvec_2, \dots, \constrjaceigvec_\dimambient\}$, where the indices indicate decreasing order w.r.t. the eigenvalue/singular value magnitude. The null spaces of $\samplecovariancematrix$ and $\constraintmanifoldjacobian$ are spanned by $\{\coveigvec_{\dimambient-\dimconstraint+1}, \dots, \coveigvec_\dimambient\}$ and $\{\constrjaceigvec_{\dimconstraint+1}, \dots, \constrjaceigvec_\dimambient\}$, respectively. The two conditions above imply that the projection of the null space eigenvectors of $\constraintmanifoldjacobian$ into the null space of $\samplecovariancematrix$ should be $\vct{0}$, and similarly for the row spaces. 
Hence, we achieve this by training {\ecmnn} to minimize projection errors $\norm{\covnullspaceeigmat \covnullspaceeigmat\T \constrjacnullspaceeigmat}_2^2$ and $\norm{ \constrjacnullspaceeigmat \constrjacnullspaceeigmat\T \covnullspaceeigmat}_2^2$ with 
$\covnullspaceeigmat = 
\begin{bmatrix}
    \coveigvec_{\dimambient-\dimconstraint+1} & \dots & \coveigvec_\dimambient
\end{bmatrix}$
and 
$\constrjacnullspaceeigmat =  
\begin{bmatrix}
    \constrjaceigvec_{\dimconstraint+1} & \dots & \constrjaceigvec_\dimambient
\end{bmatrix}$.
%

\subsection{Data Augmentation with Off-Manifold Data}
The training dataset is on-manifold, i.e., each point $\jointposition$ in the dataset satisfies $\constraintfunction(\jointposition) = \vct{0}$. 
Through Local PCA on each of these points, we know the data-driven approximation of the normal space of the manifold at $\jointposition$. 
Hence, we know the directions where the violation of the equality constraint increases, i.e., the same direction as any vector picked from the approximate normal space. 
Since our future use of the learned constraint manifold on motion planning does not require the acquisition of the near-ground-truth value of $\constraintfunction(\jointposition) \neq \vct{0}$, we can set this off-manifold valuation of $\constraintfunction$ arbitrarily, as long as it does not interfere with the utility for projecting an off-manifold point onto the manifold. 
Therefore, we can augment our dataset with off-manifold data to achieve a more robust learning of {\ecmnn}. 
For each point $\jointposition$ in the on-manifold dataset, and for each random unit vector $\randunitvec$ picked from the normal space at $\jointposition$, we can add an off-manifold point $\augjointposition = \jointposition + \augidx \augmagnitude \randunitvec$ with a positive integer $\augidx$ and 
a small positive scalar $\augmagnitude$ (see \Cref{fig:localpca} for a visualization). 
However, if the
closest on-manifold data point to an augmented point
$\augjointposition = \jointposition + \augidx \augmagnitude \randunitvec$ 
is not $\jointposition$, we reject it.
This prevents situations like augmenting a point on a sphere beyond the center of the sphere. 
With this data augmentation, we now define several losses used to train {\ecmnn}.
\subsubsection{Training Losses}
\label{sssec:losses}
\textbf{Loss Based on the Norm of {\ecmnn} Output.} 
For the augmented data point $\augjointposition = \jointposition + \augidx \augmagnitude \randunitvec$, we set the label satisfying $\norm{\constraintfunction(\augjointposition)}_2 = \augidx \augmagnitude$. During training, we minimize the norm prediction error: 
$$\normloss = \norm{(\norm{\constraintfunction(\augjointposition)}_2 - \augidx \augmagnitude)}_2^2$$ 
for each augmented point $\augjointposition$.

\textbf{Loss for Reflection Pairs.}
For the augmented data point $\augjointposition = \jointposition + \augidx \augmagnitude \randunitvec$ and its reflection pair $\jointposition - \augidx \augmagnitude \randunitvec$, we can expect that $\constraintfunction(\jointposition + \augidx \augmagnitude \randunitvec) = -\constraintfunction(\jointposition - \augidx \augmagnitude \randunitvec)$. Therefore, during training we also try to minimize the following pairs loss:
$$\siamreflectionloss = \norm{\constraintfunction(\jointposition + \augidx \augmagnitude \randunitvec) + \constraintfunction(\jointposition - \augidx \augmagnitude \randunitvec)}_2^2$$

\textbf{Loss for Augmentation Fraction Pairs.}
Similarly, between the pair $\augjointposition = \jointposition + \augidx \augmagnitude \randunitvec$ and $\jointposition + \frac{a}{b} \augidx \augmagnitude \randunitvec$, where $a$ and $b$ are positive integers satisfying $0 < \frac{a}{b} < 1$, we can expect that $\frac{\constraintfunction(\jointposition + \augidx \augmagnitude \randunitvec)}{\norm{\constraintfunction(\jointposition + \augidx \augmagnitude \randunitvec)}_2} = \frac{\constraintfunction(\jointposition + \frac{a}{b} \augidx \augmagnitude \randunitvec)}{\norm{\constraintfunction(\jointposition + \frac{a}{b} \augidx \augmagnitude \randunitvec)}_2}$. Hence, during training we also try to minimize the pairs loss:
$$\siamfracloss = \norm{\frac{\constraintfunction(\jointposition + \augidx \augmagnitude \randunitvec)}{\norm{\constraintfunction(\jointposition + \augidx \augmagnitude \randunitvec)}_2} - \frac{\constraintfunction(\jointposition + \frac{a}{b} \augidx \augmagnitude \randunitvec)}{\norm{\constraintfunction(\jointposition + \frac{a}{b} \augidx \augmagnitude \randunitvec)}_2}}_2^2$$

\textbf{Loss for Similar Augmentation Pairs.}
Suppose for nearby on-manifold data points $\jointposition_a$ and $\jointposition_c$, their approximate normal spaces $\normalspaceid_{\jointposition_a}\constraintmanifold$ and $\normalspaceid_{\jointposition_c}\constraintmanifold$ are spanned by eigenvector bases $\lnormalcoordframe^a = \{\coveigvec^a_{\dimambient-\dimconstraint+1}, \dots, \coveigvec^a_\dimambient\}$ and $\lnormalcoordframe^c = \{\coveigvec^c_{\dimambient-\dimconstraint+1}, \dots, \coveigvec^c_\dimambient\}$, respectively\footnote{$\normalspaceatjointposition$ is the normal space at point $\jointposition$ on manifold $\constraintmanifold$.}. 
If $\lnormalcoordframe^a$ and $\lnormalcoordframe^c$ are closely aligned, the random unit vectors
$\randunitvec_a$ from $\lnormalcoordframe^a$ and $\randunitvec_c$ from $\lnormalcoordframe^c$ can be obtained by 
$\randunitvec_a = \frac{\sum_{j=\dimambient-\dimconstraint+1}^\dimambient \randscalarweight_j \coveigvec^a_j}{\norm{\sum_{j=\dimambient-\dimconstraint+1}^\dimambient \randscalarweight_j \coveigvec^a_j}_2}$ 
and $\randunitvec_c = \frac{\sum_{j=\dimambient-\dimconstraint+1}^\dimambient \randscalarweight_j \coveigvec^c_j}{\norm{\sum_{j=\dimambient-\dimconstraint+1}^\dimambient \randscalarweight_j \coveigvec^c_j}_2}$, 
where $\{\randscalarweight_{\dimambient-\dimconstraint+1}, \dots, \randscalarweight_\dimambient\}$ 
are random scalar weights from a standard normal distribution common to both the bases of $\lnormalcoordframe^a$ and $\lnormalcoordframe^c$. This will ensure that $\randunitvec_a$ and $\randunitvec_c$ are aligned as well, and we can expect that  $\constraintfunction(\jointposition_a + \augidx \augmagnitude \randunitvec_a) = \constraintfunction(\jointposition_c + \augidx \augmagnitude \randunitvec_c)$. Therefore, during training we also try to minimize the pairs loss:
$$\siamsimilarloss = \norm{\constraintfunction(\jointposition_a + \augidx \augmagnitude \randunitvec_a) - \constraintfunction(\jointposition_c + \augidx \augmagnitude \randunitvec_c)}_2^2$$
In general, the alignment of $\lnormalcoordframe^a$ and $\lnormalcoordframe^c$ is not guaranteed, 
for example due to the numerical sensitivity of singular value/eigen decomposition. 
Therefore, we introduce an algorithm for Orthogonal Subspace Alignment (OSA) in \Cref{sec:osa} to ensure that this assumption is satisfied.

While $\normloss$ governs only the norm of {\ecmnn}'s output, the other three losses $\siamreflectionloss$, $\siamfracloss$, and $\siamsimilarloss$ constrain the (vector) outputs of {\ecmnn} based on pairwise input data points without explicitly hand-coding the desired output itself. We avoid the hand-coding of the desired output because this is difficult for high-dimensional manifolds, except when we have prior knowledge about the manifold, such as in the case of Signed Distance Fields (SDF) manifolds.



\section{Experiments}
\label{sec:ecomann_experiments}
\begin{figure}
\centering
\begin{subfigure}{0.4\textwidth}
\includegraphics[trim={1cm 1cm 1cm 1cm}, clip, height=0.7\textwidth]
{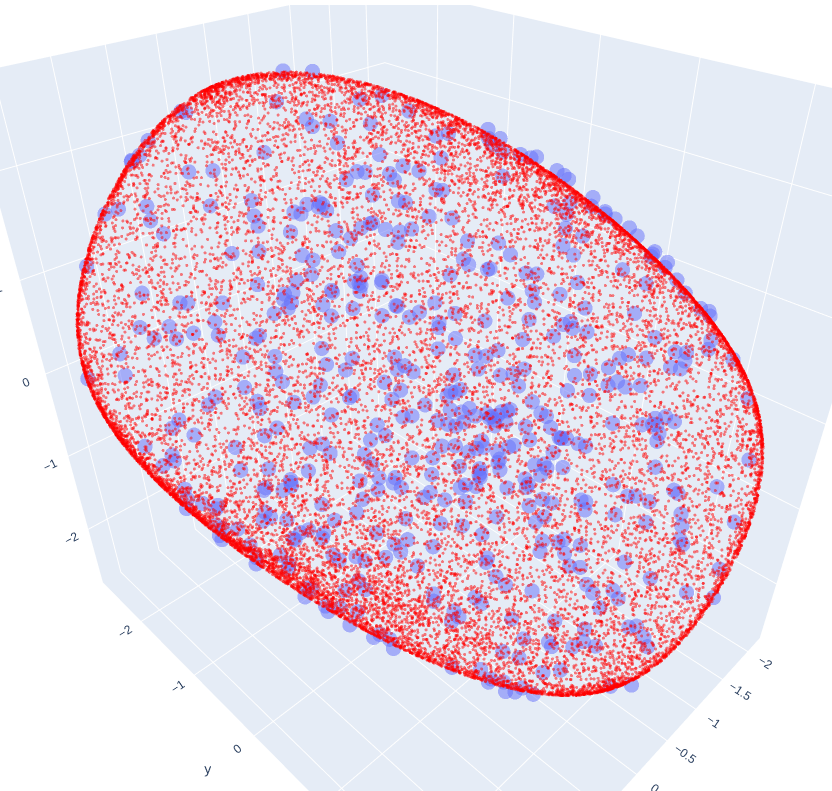}
\caption{Projected samples from {\ecmnn}}
\label{fig:ecomann_3dof_sample}
\end{subfigure}
\hspace{2em}
\begin{subfigure}{0.4\textwidth}
\includegraphics[trim={1cm 1cm 1cm 1cm}, clip, height=0.7\textwidth]
{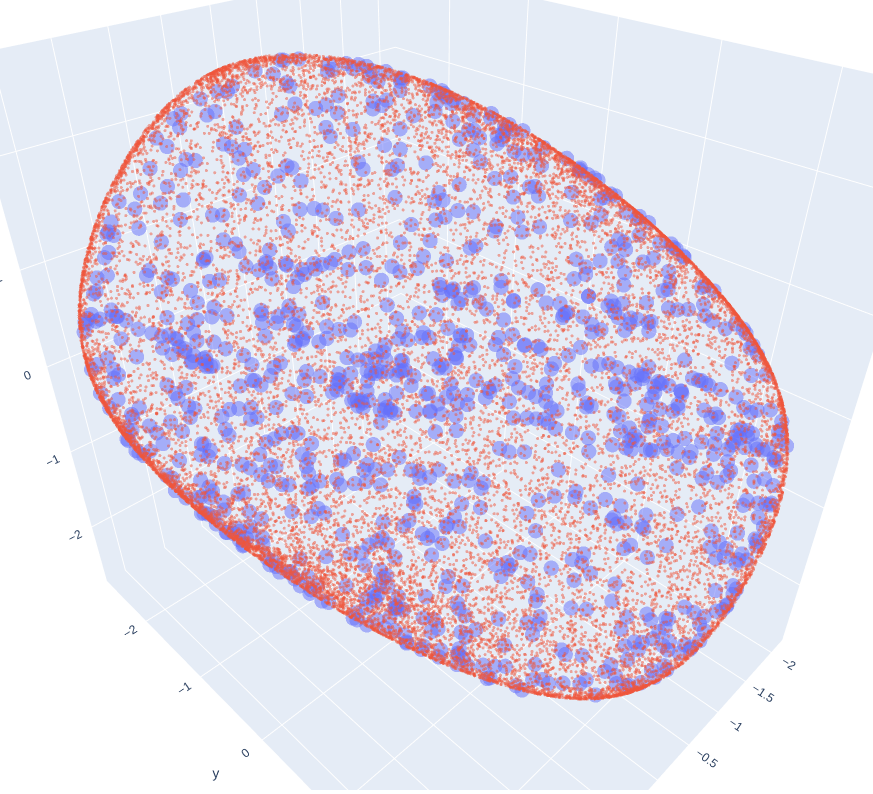}
\caption{New samples from VAE}
\label{fig:vae_3dof_sample}
\end{subfigure}
\caption{Visualization for \Cref{sec:experiment1}, Plane dataset. Red points are the training dataset and blue points are samples generated from the learned manifolds.}
\label{fig:samples}
\end{figure}
%
We use the robot simulator MuJoCo \cite{todorov2012mujoco} to generate four
datasets. For each dataset, we define a ground truth constraint $\bar{h}_M$,
randomly sample points in the configuration (joint) space, and use a constrained
motion planner to find robot configurations in $\onconstraintconfigspace$ that
produce the on-manifold datasets:
\begin{itemize}
    \item \textbf{Sphere}: 3D point that has to stay on the surface of a sphere. $N=5000,~ d=3,~ l=1$.
    \item \textbf{3D Circle}: A point that has to stay on a circle in 3D space. $N=1000,~ d=3,~ l=2$.
    \item 
    \textbf{Plane}: Robot arm with 3 rotational DoFs where the end effector has to be on a plane. $N=20000,~ d=3,~ l=1$.
    \item 
    \textbf{Orient}: Robot arm with 6 rotational DoFs that has to keep its orientation upright (e.g., transporting a cup). $N=21153,~ d=6,~ l=2$.
\end{itemize}

\subsection{Accuracy and Precision of Learned Manifolds} 
\label{sec:experiment1}
We compare the accuracy and precision of the manifolds learned by {\ecmnn}
with those learned by a variational autoencoder (VAE) \cite{kingma2013AutoEncodingVB}. 
VAEs are a popular generative model that embeds data points as a distribution in a learned latent space, and as such new latent vectors
can be sampled and decoded into new examples which fit the distribution of
the training data. 
We use two metrics: First, the distance $\mu_{\bar{h}_M}$ which measures how far a point is away from the ground-truth manifold $\bar{h}_M$ and which we evaluate for both the training data points and randomly sampled points, and second,
the percent $P_{\bar{h}_M}$ of randomly sampled points that are on the manifold $\bar{h}_M$.
We use a distance threshold of $0.1$ to determine success when calculating $P_{\bar{h}_M}$.
For {\ecmnn}, randomly sampled points are projected using gradient descent with the learned implicit function until convergence. 
For the VAE, latent points are sampled from $\mathcal{N}(0,1)$ and decoded into new configurations.

We sample 1000 points for each of these comparisons. 
We report results in \Cref{table:projection} and a visualization of the test phase in \Cref{fig:samples}. 
We also plot the level set and the normal space eigenvector field of the {\ecmnn} trained on the sphere and plane constraint dataset in \Cref{fig:contourplot_vecfield}.
In all experiments, we set the value of the augmentation magnitude $\augmagnitude$ to the square root of the mean eigenvalues of the approximate tangent space, which we found to work well experimentally.
With the exceptions of the embedding size and the input size, which are set to the same dimensionality $l$ 
as the constraint learned by {\ecmnn} 
and the ambient space dimensionality $d$ of the dataset, respectively,
the VAE has the same parameters for each dataset: 
4 hidden layers with 128, 64, 32, and 16 units in the encoder and the same but reversed in the decoder; the weight of the KL divergence loss $\beta = 0.01$; using batch normalization; and trained for 100 epochs.

Our results show that for every dataset except Orient, {\ecmnn} out performs the VAE in both metrics. 
{\ecmnn} additionally outperforms the VAE with the Orient dataset in the testing phase, which suggests more robustness of the learned model. 
We find that though the VAE also performs relatively well in most cases, it cannot learn a good representation of the 3D Circle constraint and fails to produce any valid sampled points. 
{\ecmnn}, on the other hand, can learn to represent all four constraints well.

\begin{table}[H]
  \caption{Accuracy and precision of learned manifolds averaged across 3 instances. 
 ``Train" indicates results on the on-manifold training set;  ``test" indicates $N=1000$ projected (\ecmnn) or sampled (VAE) points.}
  \label{table:projection}
  \centering
  \resizebox{\columnwidth}{!}{%
\begin{tabular}{c||ccc|ccc}
    \toprule
          & \multicolumn{6}{c}{Method} \\
     Dataset & \multicolumn{3}{c}{\ecmnn} & \multicolumn{3}{c}{VAE} \\
     \cmidrule{2-7}
              & $\mu_{\bar{h}_M}$ (train) & $\mu_{\bar{h}_M}$ (test) & $P_{\bar{h}_M}$
              & $\mu_{\bar{h}_M}$ (train) & $\mu_{\bar{h}_M}$ (test) & $P_{\bar{h}_M}$
                \\
\midrule
      Sphere 
      & $0.024\pm 0.009$ & $0.023 \pm 0.009$ & $100.0 \pm 0.0$ 
      & $0.105\pm 0.088$ & $0.161 \pm 0.165$ & $46.867 \pm 18.008$
      \\
      3D Circle     
      & $0.029 \pm 0.011$ & $0.030 \pm 0.011$ & $78.0 \pm 22.0$ 
        & $0.894 \pm 0.074$ & $0.902 \pm 0.069$ & $0.0 \pm 0.0$   \\
      Plane  
      & $0.020 \pm 0.005$ & $0.020 \pm 0.005$ & $88.5 \pm 10.5$ 
        & $0.053 \pm 0.075$ & $0.112 \pm 0.216$  & $77.733 \pm 7.721$ \\
      Orient     & $0.090\pm 0.009$ & $0.090 \pm 0.009$ & $73.5 \pm 6.5$
        & $0.010 \pm 0.037$ & $0.085 \pm 0.237$ & $85.9 \pm 1.068$   \\
      \bottomrule
  \end{tabular}
  }
\end{table}

\begin{figure}
    \centering
    \begin{subfigure}{0.45\textwidth}
        \centering
        \includegraphics[trim={0cm 0cm 0cm 0cm}, clip, height=0.8\textwidth]{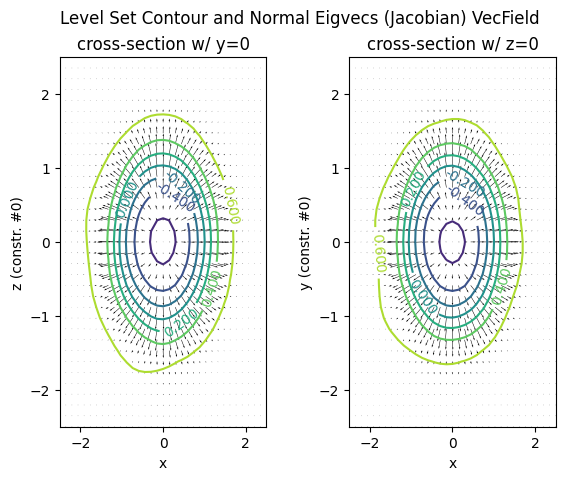}
        \label{sfig:contourplot_vecfield_sphere}
    \end{subfigure}
    \hfill
    \begin{subfigure}{0.45\textwidth}
        \centering
        \includegraphics[trim={0cm 0cm 0cm 0cm}, clip, height=0.8\textwidth]{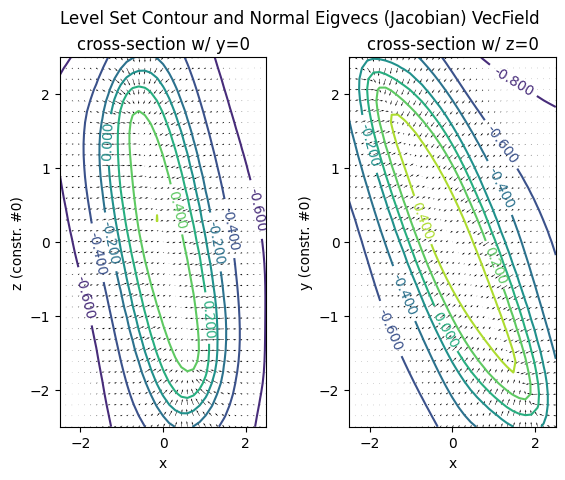}
        \label{sfig:contourplot_vecfield_3dof_traj}
    \end{subfigure}
    \caption{Trained {\ecmnn}’s level set contour plot and the normal space eigenvector field, after training on the sphere constraint dataset (left) and plane constraint dataset (right).}
    \label{fig:contourplot_vecfield}
\end{figure}


\begin{figure}[H]
    \centering
    \begin{subfigure}{0.45\textwidth}
        \centering
        \includegraphics[height=0.9\textwidth]{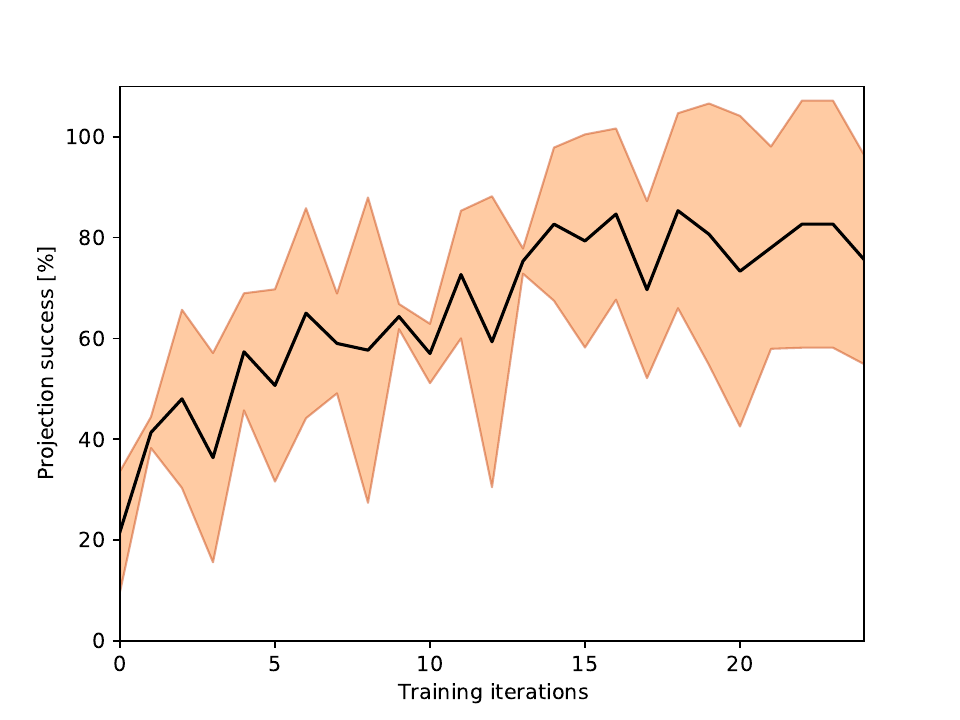}
        \caption{}
        \label{fig:3dof_success_rate}
    \end{subfigure}
    \hspace{2em}
    \begin{subfigure}{0.45\textwidth}
        \centering
        \includegraphics[height=0.9\textwidth]{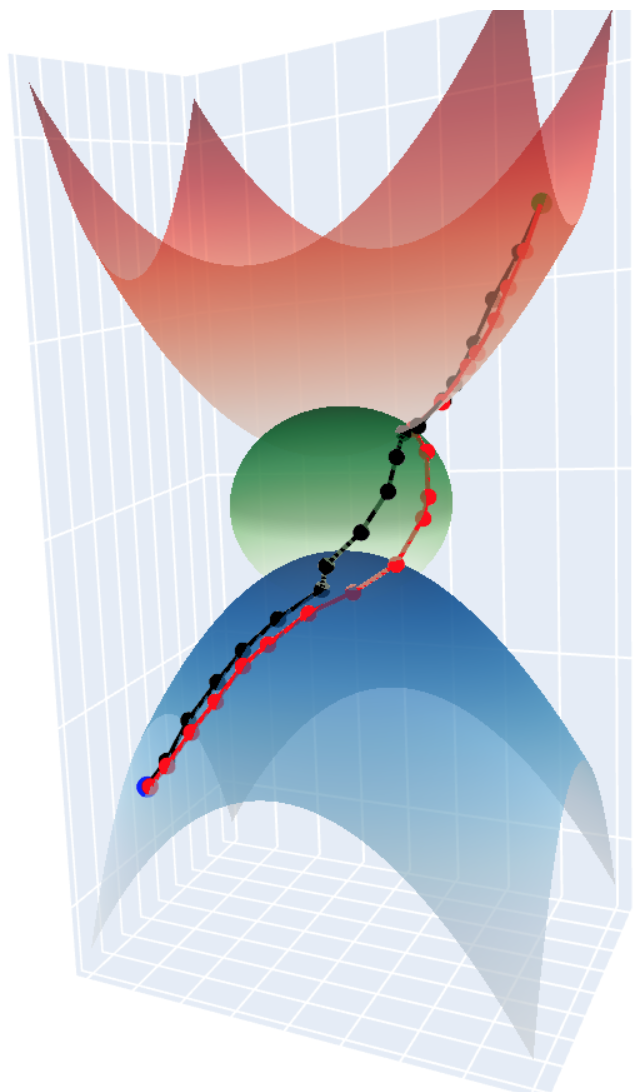} 
        \caption{}
        \label{fig:ecomann_hourglass}
    \end{subfigure}
    \caption{(\subref{fig:3dof_success_rate}): {\ecmnn} success rate on the Orient dataset over number of training iterations. (\subref{fig:ecomann_hourglass}): A path planned on the learned manifold (red) and path on the ground truth manifold (black) are visualized.}
\end{figure}

\subsection{Ablation Study of \ecmnn}
In the ablation study, we compare $P_{\bar{h}_M}$ across 4 different {\ecmnn} setups: 
\begin{enumerate}
    \item no ablation;
    \item without data augmentation;
    \item without pairs losses during training;
    \item and without orthogonal subspace alignment (OSA) during data augmentation.
\end{enumerate}
Results are reported in \Cref{tab:ablation_table}. 
For this ablation test we use {\ecmnn} with 4 hidden layers of size 36, 24, 18, and 10.
The data suggest that all three parts of the training process are essential for a high success rate during projection. 
Of the three features tested, data augmentation appears to have the most impact.
This makes sense because without augmented data to train on, any configuration that does not already lie on the manifold will have undefined output when evaluated with {\ecmnn}.

\begin{table}[ht]
    \centering
    \caption{Percentage of projection success rate for a variety of ablation on {\ecmnn} components.}
    \begin{tabular}{ |c||c|c|c| }
        \hline
        Ablation Type         & Sphere                  & 3D Circle               & Plane \\
        \hline
        \hline
        No Ablation           & (98.67 $\pm$ 1.89) \% & (100.00 $\pm$  0.00) \% & (92.33 $\pm$ 10.84) \% \\
        w/o Data Augmentation & ( 9.00 $\pm$ 2.45) \% & ( 16.67 $\pm$  4.64) \% & ( 3.67 $\pm$  3.30) \% \\
        w/o Pairs Losses    & (17.33 $\pm$ 3.40) \% & ( 65.67 $\pm$ 26.03) \% & (35.67 $\pm$  5.56) \% \\
        w/o Pairs Loss $\siamreflectionloss$ & (92.67 $\pm$ 4.03) \% & (  9.67 $\pm$  2.05) \% & (38.00 $\pm$ 16.87) \% \\
        w/o Pairs Loss $\siamfracloss$       & (88.33 $\pm$ 8.34) \% & ( 99.67 $\pm$  0.47) \% & (85.33 $\pm$ 16.68) \% \\
        w/o Pairs Loss $\siamsimilarloss$    & (83.00 $\pm$ 5.10) \% & ( 70.67 $\pm$ 21.00) \% & (64.33 $\pm$  2.62) \% \\
        w/o OSA               & (64.33 $\pm$ 9.03) \% & ( 33.33 $\pm$ 17.13) \% & (61.00 $\pm$  6.16) \% \\
        \hline
    \end{tabular}
    \label{tab:ablation_table}
\end{table}

\subsection{Learning {\ecmnn} on Noisy Data}
We also evaluate {\ecmnn} learning on noisy data. We generate a noisy unit sphere dataset and a noisy 3D unit circle with additive Gaussian noise of zero mean and standard deviation 0.01. After we train {\ecmnn} on this noisy sphere and 3D circle datasets, we evaluate the model and obtain (82.00 $\pm$ 12.83) \% and (89.33 $\pm$ 11.61) \%, respectively, as the $P_{\bar{h}_M}$ metric.\footnote{The small positive scalar $\augmagnitude$ needs to be chosen sufficiently large as compared to the noise level, so that the data augmentation will not create inconsistent data w.r.t. the noise.}

\subsection{Relationship Between the Number of Augmentation Levels and the Projection Success Rate}
Moreover, we also perform an experiment to study the relationship between the number of augmentation levels -- i.e.,  the maximum value of the positive integer $\augidx$ in the off-manifold points $\augjointposition = \jointposition + \augidx \augmagnitude \randunitvec$ -- and the projection success rate. As we vary the maximum value of $\augidx$ at 1, 2, 3, and 7 on the sphere dataset, the projection success rates are (5.00 $\pm$ 2.83) \%, (12.33 $\pm$ 9.03) \%, (83.67 $\pm$ 19.01) \%, and (97.33 $\pm$ 3.77) \%, respectively, showing that the projection success rate improves as the number of augmentation levels are increased.

\subsection{Motion Planning on Learned Manifolds}
In the final experiment, we integrate {\ecmnn} into the sequential motion planning framework described in \Cref{subsec:smp}. 
We mix the learned constraints with analytically defined constraints and evaluate it for two tasks. 
The first one is a geometric task, visualized in \Cref{fig:ecomann_hourglass}, where a point starting on a paraboloid in 3D space must find a path to a goal state on another paraboloid. 
The paraboloids are connected by a sphere, and the point is constrained to stay on the surfaces at all times.
In this case, we give the paraboloids analytically to the planner, and use {\ecmnn} to learn the connecting constraint using the Sphere dataset.
\Cref{fig:ecomann_hourglass} shows the resulting path where the sphere is represented by a learned manifold (red line) and where it is represented by the ground-truth manifold (black line). While the paths do not match exactly, both paths are on the manifold and lead to similar solutions in terms of path lengths. The task was solved in \SI{27.09}{\s} on a 2.2 GHz Intel Core i7 processor. The tree explored $1117$ nodes and the found path consists of $24$ nodes.

The second task is a robot pick-and-place task with the additional constraint that the transported object needs to be oriented upwards throughout the whole motion. 
For this, we use the Orient dataset to learn the manifold for the transport phase and combine it with other manifolds that describe the pick and place operation. The planning time was \SI{42.97}{\s}, the tree contained $1421$ nodes and the optimal path had $22$ nodes. Images of the resulting path are shown in \Cref{fig:pickandplace}. 

\begin{figure}[t]
    \centering
    \includegraphics[width=0.19\textwidth]{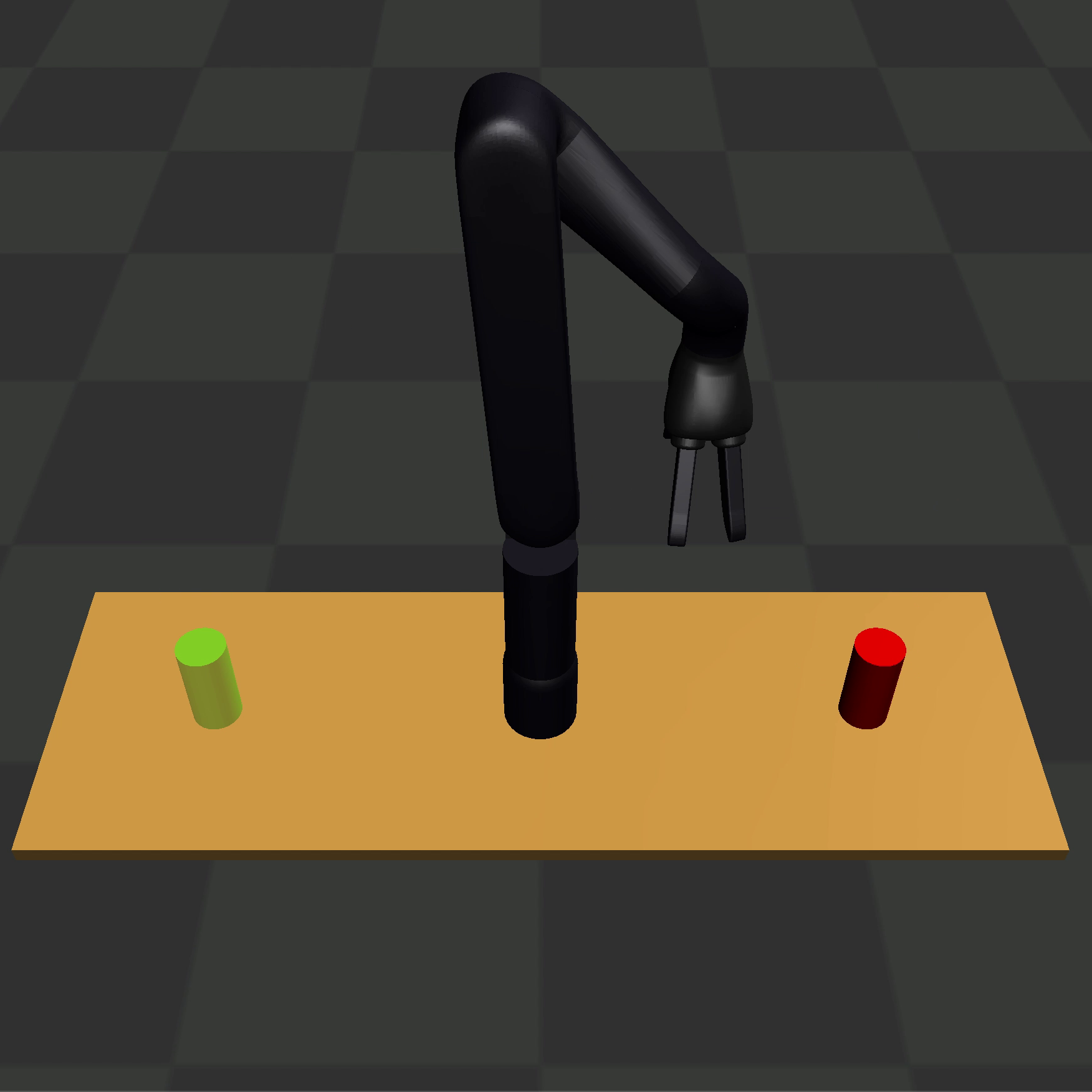}
    \includegraphics[width=0.19\textwidth]{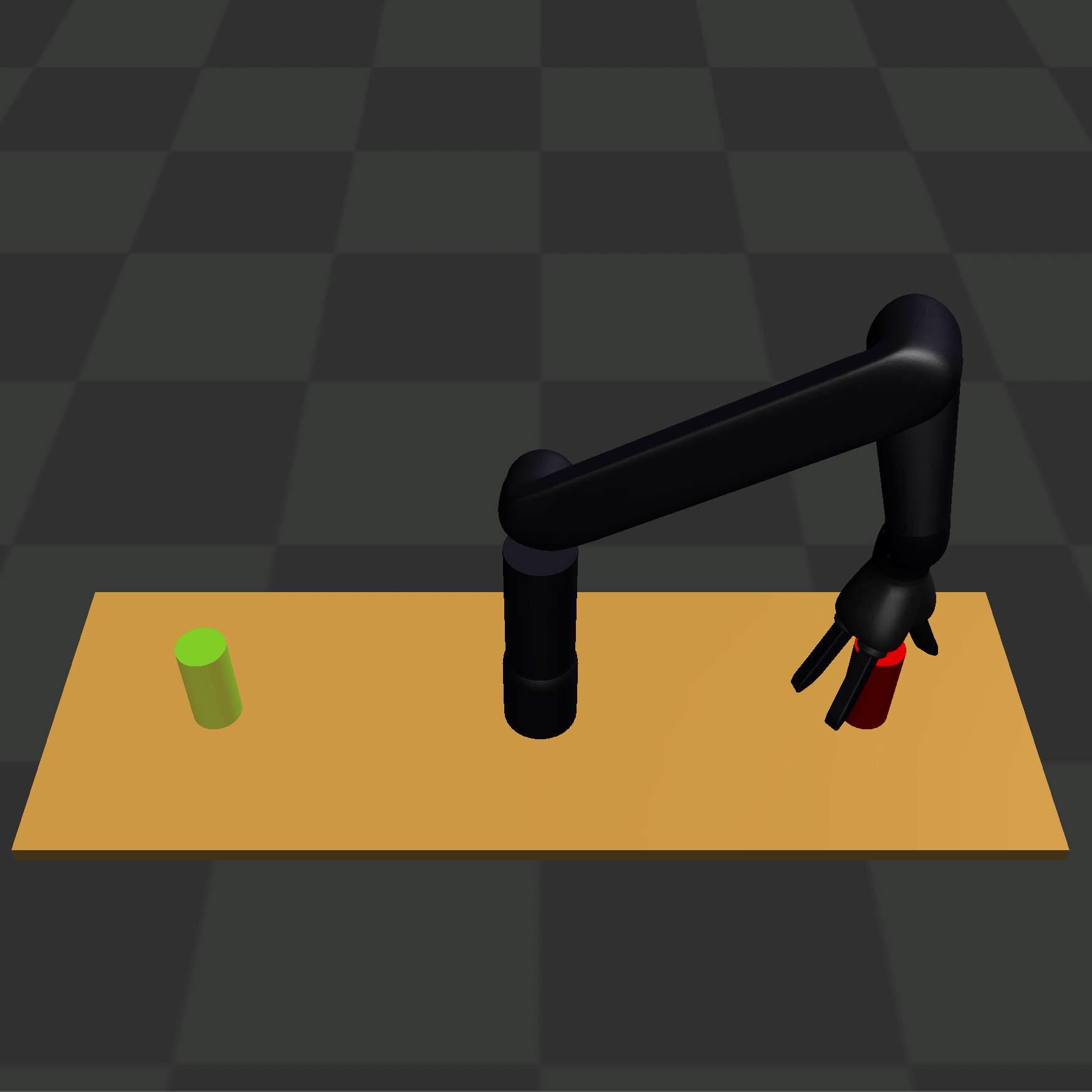}
    \includegraphics[width=0.19\textwidth]{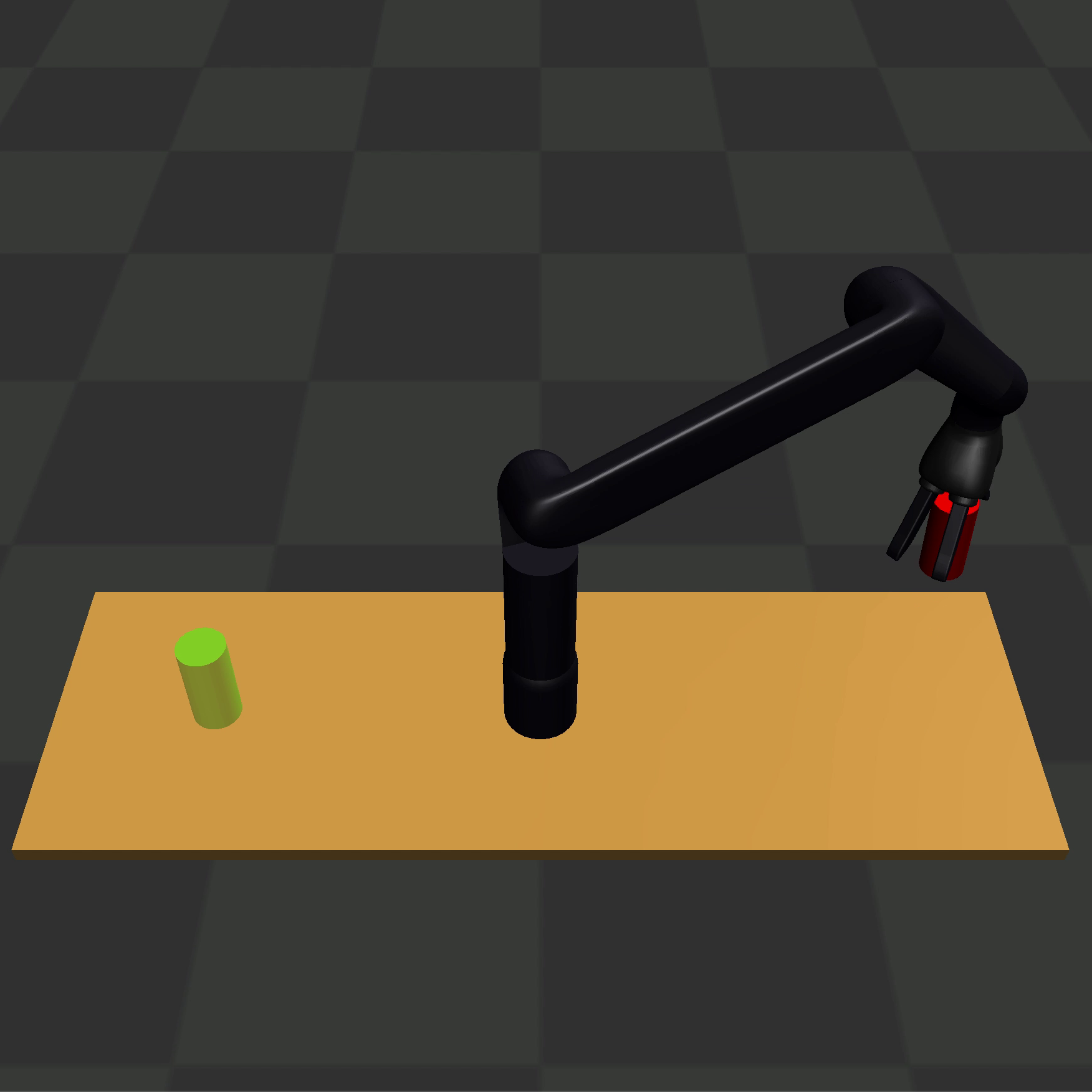}
    \includegraphics[width=0.19\textwidth]{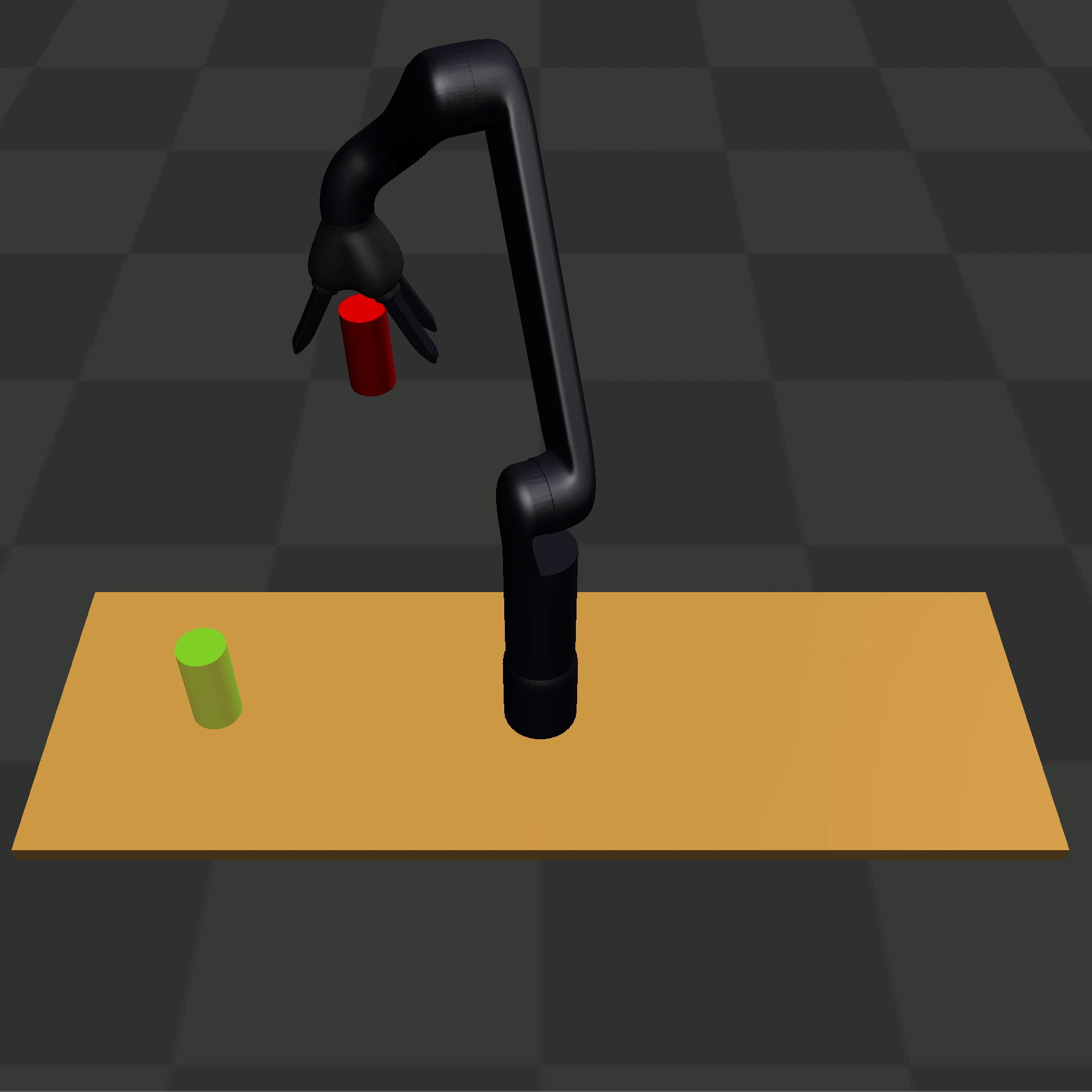}
    \includegraphics[width=0.19\textwidth]{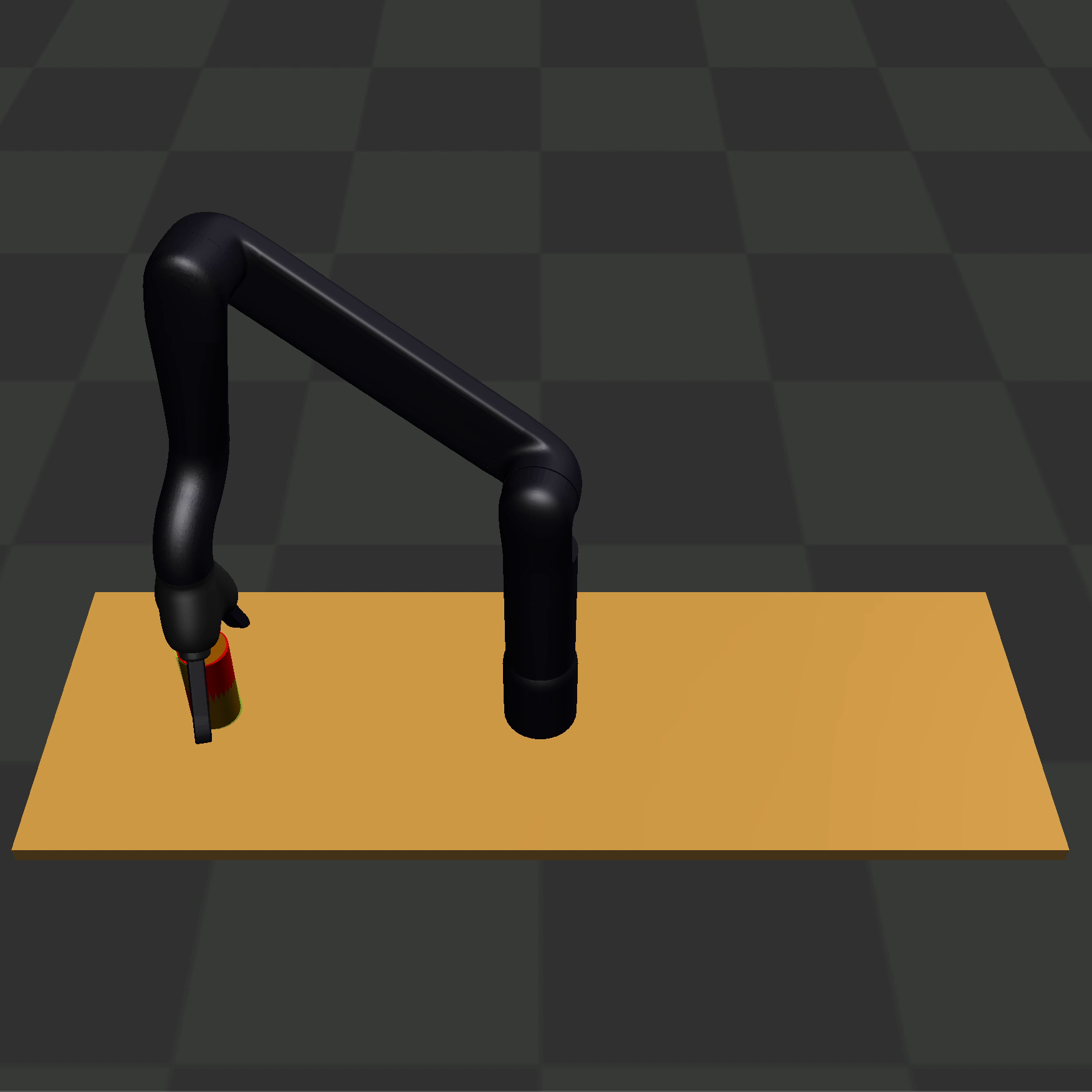}
    \caption{Visualization of a path that was planned on a learned orientation manifold.}
    \label{fig:pickandplace}
\end{figure}


\section{Discussion and Conclusion}
\label{sec:ecomann_conclusion}
In this chapter, we presented a novel method called {\ecmnnlong} for learning equality constraint manifolds from data.
{\ecmnn} works by aligning the row and null spaces of the local PCA and network Jacobian, which results in approximate learned normal and tangent spaces of the underlying manifold, suitable for use within a constrained sampling-based motion planner.
In addition, we introduced a method for augmenting a purely on-manifold dataset to include off-manifold points and several loss functions for training.
This improves the robustness of the learned method while avoiding hand-coding the labels for the augmented points. We also showed that the learned manifolds can be used in a sequential motion planning framework for constrained robot tasks.

While our experiments show success in learning a variety of manifolds, there are some limitations to our method.
First, {\ecmnn} by design can only learn equality constraints. Although many tasks can be specified with such constraints, inequality constraints are also an important part of many robot planning problems. 
Additionally, because of inherent limitations in learning from data, {\ecmnn} does not guarantee that a randomly sampled point in configuration space will be projected successfully onto the learned manifold.
This presents challenges in designing asymptotically complete motion planning algorithms, and is an important area of research.


\chapter{Background on Informative Path Planning in Natural Environments}
\label{ch:ipp-background}
Until now, we have considered offline planning in fully observable environments.
We next consider situations in which these assumptions do not necessarily hold, and the next three chapters address aspects of task abstraction in one such context. 
Inspired by our collaboration with a team of microbiologists, the robot planning problem in these chapters aims to explore an unknown environment in an online fashion, in this case, natural aquatic environments to study harmful algae blooms.
Though in a different domain, we are faced with similar concerns with task specification abstraction; these problems have historically only accepted generic specifications, and expanding these to complex multirobot systems has not been deeply investigated. 
In \Cref{ch:quantiles,ch:multirobot,,ch:comms}, these aspects of task abstractions are addressed,
and this chapter gives the background and motivation for this domain.

\section{Informative Path Planning (IPP)}
Informative path planning (IPP) is a common planning framework often used in an online manner and consists of alternating between planning and taking an action.
It uses knowledge of the environment, typically in the form of an internal model, to inform the next action that is taken, after which new information is gained and incorporated into the model, and the process repeats.
IPP uses objective functions to determine which actions are the most informative at each step.
These objective functions depend on the application, but some common functions include entropy and mutual information \cite{kemna_pilot_2018,guestrin_near-optimal_2005}.
The robot produces a planned trajectory $p$ at each planning step which maximizes (or minimizes) the objective function $\objectivefunction$, and these trajectories together create the complete planned path $P$ through the environment. 
Typically there is a planning step limit, or budget, $B$ on the problem which defines the maximum cost $c$ of the robot path \cite{denniston_icra_2020,Hollinger2014}, so the plan and act steps iterate until that budget is reached. 
This problem can be described as
\begin{equation}
P^* = \argmax_{P \in \Phi} f(P) ~ | ~  c(P) \leq B    
\end{equation}
where $\Phi$ is the space of full trajectories and $P^*$ is the optimal trajectory. 

For planning for taking measurements, Partially Observable Markov Decision Processes (POMDPs) provide a useful formulation.
POMDPs are a framework that can determine optimal actions when the environment is not fully observable or there is uncertainty or noise in the environment, resulting in their widespread applicability in planning problems.
Formulating IPP as a partially observable Markov decision process (POMDP) has been explored~\cite{borenstein_bayesian_2014}; 
previous works have used this for finding maxima~\cite{Marchanta} and to adapt the parameters of a rollout-based POMDP solver online to improve its efficiency~\cite{denniston_icra_2020}.

Gaussian Processes (GPs) are non-parametric models with uncertainty quantification which are widely used to represent the belief distribution from observations in POMDP formulations of sequential Bayesian optimization and IPP~\cite{borenstein_bayesian_2014,Marchanta,denniston_icra_2020,kemna_pilot_2018}.
They approximate an unknown function from its known outputs by computing the similarity between points from a kernel function, in our case the squared exponential kernel~\cite{Rasmussen2006}, and can incorporate spatial dependency of the distribution into its predictions of unmeasured locations.
The GP model produces an estimate of the mean value $\mu(\sensedlocation)$ and variance $\sigma^2(\sensedlocation)$ at any specific location $\sensedlocation$.

Using observations as measurements from the environment and a GP to represent the belief distribution, sequential Bayesian optimization can be formulated as a Bayesian search game~\cite{borenstein_bayesian_2014}. 
Measurements from the environment $\sensedvalues_t=GT(\sensedlocations_t)$ constitute observations which are partially observable components of the overall ground truth environment $GT$.
To adapt the Bayesian search game formulation to IPP, the belief state is augmented with the state of the robot $g_t$ and the actions the planner is allowed to take are restricted to local movements which are feasible for the robot~\cite{Marchanta}.
A complete description of the POMDP for IPP can be seen in \Cref{tab:BayesianSearchAndPOMDPs}.
Solving the POMDP allows action selection to take measurements that maximize the objective function.

\begin{table}[]
\caption{Informative Path Planning as a POMDP. After~\cite{borenstein_bayesian_2014,denniston_icra_2020}.}
\centering
\begin{tabular}{c|c}
\textbf{POMDP} & \textbf{Informative Path Planning}        \\ \hline
States         & Robot position $\location_t$, Underlying unknown function $GT$ \\ \hline
Actions        & Neighboring search points                     \\ \hline
Observations   & Robot position $\location_t$, Measured location(s) $\sensedlocations_t = o(\location_t)$   \\ 
               & Measured value(s) $\sensedvalues_t = GT(\sensedlocations_t) $ \\\hline
Belief         & GP conditioned on previously measured      \\ 
& locations  $(\sensedlocations_{0:t-1})$ and values $(\sensedvalues_{0:t-1})$ \\ \hline
Rewards        & $\objectivefunction(\sensedlocations_t)$      
\end{tabular}
\label{tab:BayesianSearchAndPOMDPs}
\end{table}

\section{Objective Functions for IPP}
Objective functions are used to specify the mission objective to maximize during planning. We describe some history here and show that there is a gap in the current literature for estimating quantile values.

Often, the objective in IPP is to generate good coverage of a distributed phenomenon using an information theoretic objective such as entropy~\cite{kemna_pilot_2018} or mutual information~\cite{guestrin_near-optimal_2005}. 
Other common objectives include finding the highest concentration location~\cite{Marchanta} or measuring near hotspots~\cite{mccammon_topological_2018}.
Sequential Bayesian optimization objectives are used to locate extreme values or areas of high concentration and have shown success in doing so in simulation and in real field scenarios~\cite{Marchanta,blanchard_informative_2021,souza_bayesian_2014}.
These specialized objective functions are typically tuned to locate a single point, often the extrema.
%
There has also been work on extending IPP missions to domains such as underwater inspection by modifying the GP and objective functions~\cite{Hollinger2013ActivePF}.
Specialized objective functions have also been developed for the task of modeling continuous and discrete variables, as well as including sensor-specific characteristics such as the increase in uncertainty due to camera distance to the subject~\cite{popovic_informative_2020}.

Other objective functions have been proposed for level set estimation which seek to classify points into sets above and below a pre-specified concentration, or a fraction of the concentration~\cite{level_set}.


\section{Physical Specimen Collection} 
Physical specimen collection is the process of collecting portions of the environment for later analysis.
The choice of physical specimen collection locations is driven by the spatial heterogeneity of the studied phenomena~\cite{ysi_whitepaper},
e.g., freshwater algal growth or crop health.
We are primarily motivated by the study of algal growth and distribution in freshwater and marine ecosystems. 
Spatial heterogeneity of algal and cyanobacterial blooms is often obvious as conspicuous accumulations
at down-wind and down-current locations at the surface of lakes, and vertical heterogeneity can
occur due to water column stratification, differential growth at different depths, or active vertical migration by some planktonic species~\cite{Seegers2015,Hozumi2020}. 
Such heterogeneity thwarts studies designed to quantify the spatial distribution of algal biomass, and the concentrations of algal toxins that may be produced by some algae and cyanobacteria.
Nonetheless, accurately characterizing this heterogeneity is fundamental for investigating average conditions (indicated by median quantiles) for investigating trends across many lakes or large geographical areas~\cite{Ho2019,Mantzouki2018}. 
Assessing heterogeneity is also essential for assessing the worst case scenarios for exposure of animals and humans to algal or cyanobacterial toxins (specifically the highest quantiles). 
Such information contributes to ecotoxicological studies used to develop thresholds that constitute significant exposure to these toxins~\cite{Mehinto2021}.

Robotic approaches for quantifying algal biomass across quantiles, often using chlorophyll or chlorophyll fluorescence as a proxy for algal biomass, have become a mainstay for quantitatively documenting heterogeneity in natural ecosystems~\cite{Zhang2020,Sharp2021}.
Adaptive surveys for selecting collection locations have typically focused on finding high-concentration areas from which to measure, and perform the specimen collection onboard the robot or by a following robot~\cite{das_data-driven_2015, other_onboard_sampling}.

\section{Quantiles and Quantile Standard Error Estimation}
\label{sec:quantile-estimation}
\textit{Quantile estimation} refers to acquiring the value of a given quantile in a distribution. For example, the median algae concentration would be given by the value of the 0.5 quantile. 
Statistical techniques have been proposed to estimate quantiles and quantile standard error of a distribution.
Quantile values can be estimated from measurements using $\estimatedquantilevalue = x_{\lfloor h \rfloor} + (h - \lfloor h \rfloor)(x_{\lceil h \rceil} -x_{\lfloor h \rfloor})$
where $h = (n - 1)\quantile + 1$, $n$ is the number of measurements, and $\quantile$ is the quantile (Equation 7 in~\cite{hyndman_sample_1996}), which is the default in the \texttt{numpy} Python package.
Standard error of a quantile estimate is a measure of the uncertainty, and is typically used in the construction of confidence intervals of a quantile estimate.
Given a probability density function $p$, the standard error of the $q$th quantile is 
$\sqrt{q(1-q)} / (\sqrt{n}p(\estimatedquantilevalue))$.
Typically $p$ is not known, but can be estimated from samples using a density estimator~\cite{WILCOX2012103}.

\section{Continuous Non-Convex Gradient-Free Optimization} 
General optimization problems can be solved using continuous, non-convex, gradient-free methods.
The cross-entropy (CE) method is a popular approach which works by maintaining an estimate about the distribution of good solutions, and iteratively updating the distribution parameters based on the quality of the solutions sampled at each step, where the quality is determined via some function of the solution configuration~\cite{de2005tutorial}.
More precisely, a prior $Pr$ and a posterior $Po$ set of samples are maintained.
$Po$ contains configurations sampled from $\mathcal{N}(\vec{\mu}$, $\vec{\sigma})$.
$\vec{\mu}$ and $\vec{\sigma}$ are computed from $Pr$, which is the best $\eta\%$ configurations from $Po$.
This continues iteratively, minimizing the cross-entropy between the maintained and target distributions~\cite{de2005tutorial}.

Simulated annealing (SA) \cite{kirkpatrick1987optimization}  is an iterative algorithm inspired by statistical mechanics. SA optimizes an energy function similar to a loss function in other optimization schemes, as configurations with lower energy are preferred. SA begins with a temperature $T = T_\t{max}$. At each iteration, $T$ is decreased exponentially, and the configuration is slightly perturbed randomly. The perturbed state is either accepted or rejected probabilistically based on its energy and $T$. 
The algorithm terminates when $T$ reaches a given threshold.

Bayesian optimization (BO) is a method for selecting parameters for difficult optimization problems such as hyperparameter selection~\cite{bohyper}.
BO is similar to IPP and sequential Bayesian optimization~\cite{Marchanta} in that these methods build a model using a GP and select points which maximally improve this model. 
BO uses acquisition functions, such as expected improvement, and iteratively selects the best point to sample from~\cite{jones_efficient_1998, qin_improving_2017}. 

Other methods include genetic or evolutionary algorithms \cite{srinivas1994genetic}.
These algorithms are inspired by biology and generally incorporate some method for crossover between two solutions or mutations to encourage exploration.

\section{Multirobot Studies}
\label{sec:multirobot-background}
Multirobot systems research has a rich history; many well-known problems have been extended to more than one robot, such as multirobot SLAM \cite{howard2006multi}, multirobot exploration \cite{howard2006experiments}, and multirobot learning \cite{sartoretti2019distributed}.
Many previous studies concerning the effect of robot team size on performance have focused on human factors such as mental demand and operator workload \cite{wang2009search, velagapudi2008scaling},
while those that have studied autonomous group performance have shown inconclusive results regarding correlation between performance and team size \cite{rosenfeld2008study}, and that multirobot problems tend to follow the Law of Marginal Returns \cite{rosenfeld2004study, stachniss2008efficient}, which states that as more resources are added to a problem, smaller returns are generated. 
This property, also known as submodularity, can be seen both in number of robots as well as measurements taken, and previous work has exploited it to approximately solve challenging planning problems \cite{Corah-2019-120007}.

\subsection{Multirobot IPP}
Multirobot informative path planning uses a team of robots to collect information about an underlying spatial distribution. 
Previous works have used objective functions like entropy or mutual information to explore phenomena such as temperature fields, plankton density, salinity, and chlorophyll.
\cite{cao2013multi} exploits spatial correlation in the concentration to solve the planning problem while also improving the tradeoff between performance and efficiency.
\cite{dutta2019multi} addresses the problem of multirobot planning with continuous connectivity constraints. In that work, bipartite graphs are used to determine the next points for each robot to travel to. Though the main experimental focus is on varying the communication radius and the amount of collected information used to make future decisions with, the authors also find a linear relationship between the number of robots and the improvement using an entropy-based objective function.
Bipartite graph matching has also been used to iteratively plan paths for each robot to the most informative points found in the environment \cite{ma2018multi}. 
A similar problem is addressed in \cite{dutta2020multi} where there is intermittent communication in cluttered environments. There, the region is partitioned into Voronoi cells to better balance the workload between robots at every point when they come into communication range with each other. 
\cite{kemna2017multi} also uses repeated Voronoi partitioning combined with limited information-sharing between underwater robots.
Non-constant connectivity has also been studied in \cite{hollinger2010multi} where multirobot search with periodic connectivity is resolved using implicit coordination to solve the problem in a scalable manner.
Reinforcement learning has been used to learn a planning policy prioritizing exploring hotspots and robust to robot failures in the environment \cite{pan2022marlas}. 

\subsection{Bandwidth-Constrained Multirobot Communication}
\label{ssec:multirobot-comms-background}
Work that has considered the problem
of deciding what or when to communicate is largely motivated by the same desire
to respect bandwidth limitations that real networks must face and to reduce
communication costs.  However, particular methods 
differ in how they approach the solution and in what context the problem is
situated.  Some of the varied contexts in which utility-based decision-making for
inter-robot communication has been investigated, include months-long data
collection to be routed to a base station~\cite{padhy2010utility},
non-cooperative multi-agent path finding~\cite{ma2021learning}, and
those with request-reply protocols~\cite{ma2021learning,zhang2019efficient}.
In addition, some metrics used to determine utility include message category or
intrinsic data features~\cite{padhy2010utility}, difference in expected reward
after incorporating a message~\cite{simonetto2014distributed,
marcotte2020optimizing, unhelkar2016contact}, and the expected effect of the
message on robot actions~\cite{ma2021learning, zhang2019efficient}, 

Previous work using a measure of utility based on expected reward have focused on scenarios where the end goal is to reach some goal locations, assigned to each robot, and the environment is fully or partially unknown \cite {marcotte2020optimizing,unhelkar2016contact}. 
Robot teams in these works are decentralized in terms of planning and reward calculation, and 
may include forward-simulation of robot teammates during planning.
Markov decision processes may be used to underlie robot planning.
Communication network models in these works make several simplifying assumptions, including that communication transmission is deterministic, so reception is guaranteed, and that the environment is perfectly measurable.

\chapter{Informative Path Planning to Estimate Quantiles for Environmental Analysis}
\label{ch:quantiles}
Motivated to be able to describe task specification as a high-level, interpretable scientific goal, this chapter describes a method that addresses this in the context of analysis of natural environments. 
In this work, based on \cite{rayas2022informative}, we move away from generic task specifications that result in users burdened by post-processing the planning survey results, and instead automate physical location selection such that scientists can directly use the output to collect useful specimens.

\section{Introduction}
\begin{figure}[h!]
\centering
    \includegraphics[width=\columnwidth]{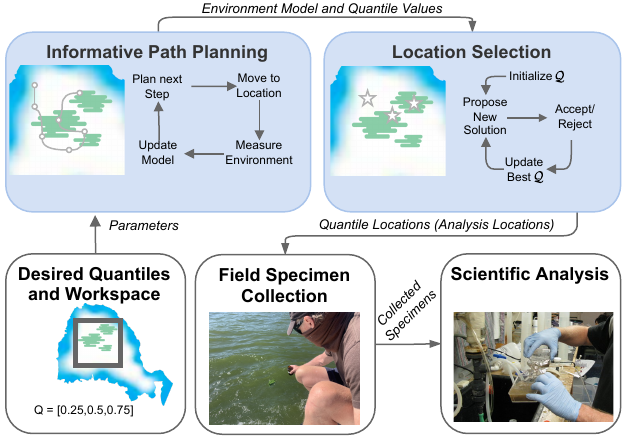}
    \caption{Full system.
    First, the parameters for the robotic survey are chosen, such as the area bounds and the quantiles for specimen collection.
    The robot performs informative path planning using our proposed objective functions, creating an environment model and an estimate of the quantile values.
    The quantile locations are then selected by minimizing our proposed loss function. These quantile locations correspond to physical locations which, when measured, have the estimated quantile values.
    After the quantile locations are chosen, humans go to them to collect field specimens which they later analyze in a laboratory setting. 
    This work focuses on the steps shown in the blue shaded boxes.
    }
    \label{fig:quantiles-hero}
\end{figure}


In order to understand biological phenomena, scientists take a small number of physical specimens for later analyses to characterize the biological community and the contextual environmental conditions at the site.
A marine biologist may capture water in a container or filter media at a location for later analysis, and an agricultural scientist may capture a small portion of a plant.
Scientists later analyze these captured portions of the environment in a detailed laboratory setting.
It is often expensive to collect these physical specimens as the scientist must go to the location and use a physical reagent, such as a container or filter media, which are limited in the field. 
Traditionally, expert heuristics underpin the selection of locations for scientific analysis. 
These expert heuristics generally attempt to spread out the analysis locations in the phenomena of interest in order to take specimens with differing concentrations for characterization of heterogeneity~\cite{ysi_whitepaper}.
The heuristics may make use of robotic surveys, but currently most robotic surveys use either pre-programmed paths~\cite{ysi_whitepaper}, or autonomy which seeks maxima~\cite{Marchanta} or to cover the area spatially~\cite{kemna_pilot_2018,guestrin_near-optimal_2005}, and do not directly plan for the specimen collection tasks.
In contrast, we propose performing an adaptive robotic survey to find locations of interest for scientific analysis.
To specify these locations, we propose calculating the quantiles of the distribution of interest so that scientists can capture specimens at varied locations in the environment based on differing values of some phenomenon of interest to them. 
For instance, if a marine biologist is interested in taking 9 physical water specimens that are spread over a range of chlorophyll concentrations, they may choose to use the deciles of the concentration. 
If only a small number of locations can be analyzed and the upper extrema values are of interest, they may choose to perform analysis at the locations of the $(0.90, 0.95, 0.99)$ quantiles.
The quantiles of interest are a flexible way to describe the objective of specimen collection that is largely invariant to the exact phenomena being measured.

Our goal is to select locations for detailed analysis by scientists in two steps. 
We use a robot to perform an adaptive survey, and then find locations to suggest for specimen collection based on the measurements it takes.
Specifically, we first aim to find the desired quantile values of the measurement distribution by adaptively selecting robot measurement locations that maximize an objective function designed to estimate quantile values, and then we produce suggested locations for specimen collection that are likely to contain these values.
\Cref{fig:quantiles-hero} summarizes the approach, showing how our objective function for informative path planning and our loss function for location selection fit into an overall physical specimen collection pipeline.

Our contributions are:

\begin{itemize}
    \item A planner with custom objective functions that adaptively improve the quality of estimated quantile values;
    \item A general loss function, which can be used with any optimizer, for selecting spatial locations for physical specimen collection which have the values estimated for the quantiles;
    \item Quantitative evaluation of our method on point and camera sensors with previously collected aquatic datasets;
    \item A demonstration of our method on a robot in a real-world crop health estimation task.
\end{itemize}

\section{Background}
We pose our problem as an IPP problem, where the overall goal is to minimize error in the estimate of a given set of quantiles of a distribution in the environment, and we use a GP to model the environment. 
The GP environment model estimates the value $\mu(\sensedlocation)$ and variance $\sigma^2(\sensedlocation)$ at a specific location $\sensedlocation$.
Similar to previous work, we use the POMDP formulation, and approximately solve this POMDP to select actions to take measurements for maximizing the objective function.

Our works differs from others primarily in the overall goal that the task specification permits. 
Specifically, as discussed in \Cref{ch:ipp-background}, previous work has typically aimed for coverage of the given area, or to find the highest concentration of the distribution. 
In contrast, our goal is to directly address the user's scientific goal by instead focusing on improving estimates for any arbitrary specific quantiles instead of broad understanding of the spatial makeup of a distribution.

In this work, we use a Gaussian kernel density estimator for quantile standard error when planning.
Other estimation methods include the \mjci method~\cite{maritz_note_1978}, and the bootstrap or jackknife methods which involve calculating the quantiles over repeatedly sampled subsets of the data. 
These last two methods can be slow for large datasets and require many iterations to converge.

The final output of this algorithm is to select the analysis locations by minimizing our proposed loss function once the IPP survey is completed (see \Cref{sec:location}), and for this,
we use continuous non-convex gradient-free optimization methods.

\section{Formulation}
We use a grid-based representation of the planning space, $\gtlocations$, which defines the set of locations that the robot could visit.
For a robot that moves in $\mathbb{R}^d$, $\gtlocations \subset \mathbb{R}^d$.
$\gtsensedlocations$ is the set of locations the robot could measure. 
If the robot sensor has finer resolution than $\gtlocations$,
e.g. if the robot uses a camera sensor or takes measurements while traversing between grid points,
then $|\gtsensedlocations| > |\gtlocations|$.
We define $\sensedlocations_{0:t}$ as the locations the robot has measured up to time $t$ and $\gtsensedvalues$ and $\sensedvalues_{0:t}$ as the values at all possible measured locations, and the values the robot has measured up to time $t$, respectively, such that $|\gtsensedlocations| = |\gtsensedvalues|$ and $|\sensedlocations| = |\sensedvalues|$.

We define the ground truth quantile values as 
$\quantilevalues = 
quantiles(\gtsensedvalues,\quantiles)$
where $quantiles$ is the function described in \Cref{sec:quantile-estimation} which computes the values $V$ of the quantiles $\quantiles$ of a set of measurements, in this case $\gtsensedvalues$.
To define the robot's estimated quantile values, we compute 
$\estimatedquantilevalues = quantiles(\mu(\gtsensedlocations),\quantiles)$
that is, the quantile values of the predicted values from the robot's current GP for all locations the robot could sense.
This is done  to prevent the number of measurements from which the quantile values are estimated from changing as the robot explores (instead of, e.g. using $\mu(\sensedlocations)$).
By doing this, we ensure we always estimate the quantiles across the entire measurable area.
During planning, we aim to minimize this error by taking actions which maximize an objective function $\objectivefunction$ that minimizes the error in the quantile value estimate.

To suggest locations for the quantile values, we aim to find a set of $\numtiles$ locations $\quantilespatiallocations$ in the continuous space whose values at those locations are equal to the quantile ground-truth values. 
A set of locations is defined as $\quantilespatiallocations \in \quantilespatiallocations^\#$ where $\quantilespatiallocations^\# \subset \mathbb{R}^{d \times |Q|}$ and $\quantilespatiallocations^\#$ is continuous over the space of $\gtsensedlocations$.
In practice, the robot only has access to $\estimatedquantilevalues$ (not $\quantilevalues$) during the selection process, so the problem of finding the estimated quantile spatial locations $\estimatedquantilespatiallocations$ can be stated as shown in \Cref{eq:point_selection} with some selection loss function $\score$ (see \Cref{eq:ps_loss}).
\begin{equation}\label{eq:point_selection}
    \estimatedquantilespatiallocations = \argmin_{\quantilespatiallocations' \subset \quantilespatiallocations^\#}  \score (\estimatedquantilevalues, \quantilespatiallocations')
\end{equation} 

\section{Approach}
\Cref{fig:quantiles-hero} illustrates our method.
We separate our approach into two steps: the survey, and the suggestion of locations.
When performing the survey using IPP (\Cref{sec:planning}), the robot takes measurements of the environment to improve its estimate of the quantiles.
After the survey has concluded, location selection (\Cref{sec:location}) produces locations for scientists to visit to perform specimen collection.

\subsection{Informative Path Planning}\label{sec:planning}
To plan which locations to measure, the robot uses a POMDP formulation of IPP. 
In order to generate a policy, we use the partially observable Monte Carlo planner (POMCP)~\cite{silver2010pomcp}.
POMCP uses Monte Carlo tree search to create a policy tree. 
To expand the tree and estimate rewards, the tree is traversed until a leaf node is reached.
From the leaf node, the reward conditioned on that action is estimated by a random policy rollout which is executed until the discounted reward is smaller than some value $\epsilon$.
We modify the rollout reward to be fixed horizon, giving a reward of zero after a certain number of random policy steps.
We adopt the t-test heuristic for taking multiple steps from a POMCP plan for IPP to improve performance of the planner with fewer rollouts~\cite{denniston_icra_2020}.
We define $GP_{i-1} = GP(\sensedlocations_{0:i-1},\sensedvalues_{0:i-1}; \theta )$
as a GP conditioned on the previous locations ($\sensedlocations_{0:i-1}$) and measurements ($\sensedvalues_{0:i-1}$) before measuring a proposed value, and 
$GP_i = GP(\sensedlocations_{0:i-1} \cup \sensedlocations_i, \sensedvalues_{0:i-1} \cup \sensedvalues_i; \theta )$ 
as a GP conditioned on the previous and proposed locations ($\sensedlocations_i$) and measurements ($\sensedvalues_i$), where $\theta$ are GP parameters and $\sensedvalues_i = GP_{i-1}(\sensedlocations_i)$.
Because the observations $GT(x)$ for unseen locations are not known during planning, the predicted mean from $GP_{i-1}$ is used~\cite{Marchanta}.
%

\textbf{Objective Functions for Quantile Estimation}
We develop two novel objective functions to improve the quality of quantile value estimates.
Both compare a measure of the quality of the quantiles estimated by the GP before and after adding a measurement to the GP,
and include an exploration term $c_{plan}\sigma^2(\sensedlocation_i)$ inspired by the upper confidence bound objective function~\cite{Marchanta}, where $c_{plan}$ is a chosen constant.
For both proposed objective functions we use \Cref{eq:general_objective}, where $\delta$ is defined by the objective:
\begin{equation}\label{eq:general_objective}
    \objectivefunction(\sensedlocations_i) = \frac{\delta(\sensedlocations_i)}{\numtiles} +  \sum_{\sensedlocation_j \in \sensedlocations_i} c_{plan}\sigma^2(\sensedlocation_j), 
\end{equation}

The first objective function, which we call \textit{quantile change}, is based on the idea of seeking out values which change the estimate of the quantile values by directly comparing the estimated quantiles before and after adding the measured values to the GP.
The idea behind this is that a measurement which changes the estimate of the quantiles indicates that the quantiles are over- or under-estimated.
This can be seen in in \Cref{eq:quantile_change}:
\begin{align}
\begin{split}\label{eq:quantile_change}
  \delta_\t{qc}(\sensedlocations_i) = \|&quantile(\mu_{GP_{i-1}}(\gtsensedlocations),\quantiles) - \\ &quantile(\mu_{GP_{i}}(\gtsensedlocations),\quantiles) \|_{1} \\
\end{split}
\end{align}
The second objective function we develop, which we call \textit{quantile standard error}, is based on the change in the estimate of the standard error for the estimated quantiles.
It draws from the same idea that if the uncertainty in the quantile estimate changes after observing a measured value, then it will change the estimate of the quantile values, shown in \Cref{eq:standard_error}:
\begin{align}\label{eq:standard_error}
\begin{split} 
   \delta_\t{se}(\sensedlocations_i) &= \|se(\mu_{GP_{i-1}}(\gtsensedlocations),\quantiles) - se(\mu_{GP_{i}}(\gtsensedlocations),\quantiles) \|_{1} \\
\end{split}
\end{align}
$se$ is an estimate of the standard error of the quantile estimate for quantiles $\quantiles$. 
$se$ uses a Gaussian kernel density estimate. 
To compute the objective function over a set of measured points, we average the objective function at each point.

\textbf{Baseline Objective Functions}
We compare against two baselines, one which maximizes spatial coverage of a phenomena and another which seeks maximal areas. 

\textit{Entropy} is a common objective function for IPP when only good spatial coverage of the environment is desired~\cite{kemna_pilot_2018,guestrin_near-optimal_2005,denniston_comparison_2019}. 
It provides a good baseline as it is often used when the specific values of the underlying concentration are unknown.
Entropy is defined \Cref{eq:entropy}:

\begin{equation}\label{eq:entropy}
   \objectivefunction_\t{en}(\sensedlocations_i) = \sum_{\sensedlocation_j \in \sensedlocations_i}\frac{1}{2} \log(2 \pi e \sigma^2(\sensedlocation_j))
\end{equation}

Another objective function we compare against is \textit{expected improvement} (EI), which is widely used in Bayesian optimization and sequential Bayesian optimization for finding maxima~\cite{jones_efficient_1998, qin_improving_2017}.
EI favors actions that offer the best improvement over the current maximal value, with an added exploration term $\xi$ to encourage diverse exploration.
The EI objective function is defined according to \Cref{eq:ei}:

\begin{equation}\label{eq:ei}
     f_{ei}(\sensedlocations_i) = \sum_{\sensedlocation_j \in \sensedlocations_i} I\Phi(Z) + \sigma(\sensedlocation_j)\phi(Z) \\
\end{equation}
where $Z =\frac{I}{\sigma^2(\sensedlocation_j)}$,  $I = \mu(\sensedlocation_j) - \max(\mu(\gtsensedlocations)) - \xi$, and $\Phi$ and $\phi$ are the CDF and PDF of the normal distribution, respectively.

\subsection{Location Selection}\label{sec:location}
Our final goal is to produce a set of $\numtiles$ locations $\quantilespatiallocations$, at which the concentration values will be equal to $\quantilevalues$, the values of the quantiles $\quantiles$. 
The selection process can be done offline as it does not affect planning.
Finding locations that represent $\quantiles$ is difficult because the objective function over arbitrary phenomena in natural environments will likely be non-convex, and in a real-world deployment, the robot will only have an estimate $\estimatedquantilevalues$ of the $\quantilevalues$ it searches for.

With the location selection problem formulation as in \Cref{eq:point_selection}, we propose the loss function 
\begin{equation}\label{eq:ps_loss}
\score(\estimatedquantilevalues, \estimatedquantilespatiallocations) = \|\estimatedquantilevalues - \mu(\estimatedquantilespatiallocations)\|_{2} + c_{select} \sigma^2(\estimatedquantilespatiallocations),
\end{equation}
where $c_{select} \sigma^2(\estimatedquantilespatiallocations)$ is an added penalty for choosing points that the GP of collected measurements is not confident about\footnote{The parameters $c_{plan}$ and $c_{select}$ are distinct.}, and $\score$ can be used in any optimization scheme.
During optimization, \Cref{eq:ps_loss} is evaluated using $\estimatedquantilevalues$ and returns the suggested specimen collection locations $\estimatedquantilespatiallocations$. 
We compare three optimization methods in our experiments to determine which best minimizes our selection loss function (\Cref{eq:ps_loss}).
A strength of these types of optimization methods is that the formulation allows for suggesting points that may be spatially far from locations the robot was able to measure if they have values closer to the quantile values of interest.

\begin{figure}[t]
    \centering
        \begin{subfigure}{\threeboxplot}
                \centering
                \includegraphics[width=\textwidth,height=\boxplotheight,trim={.8cm 0 .5cm 0},clip]{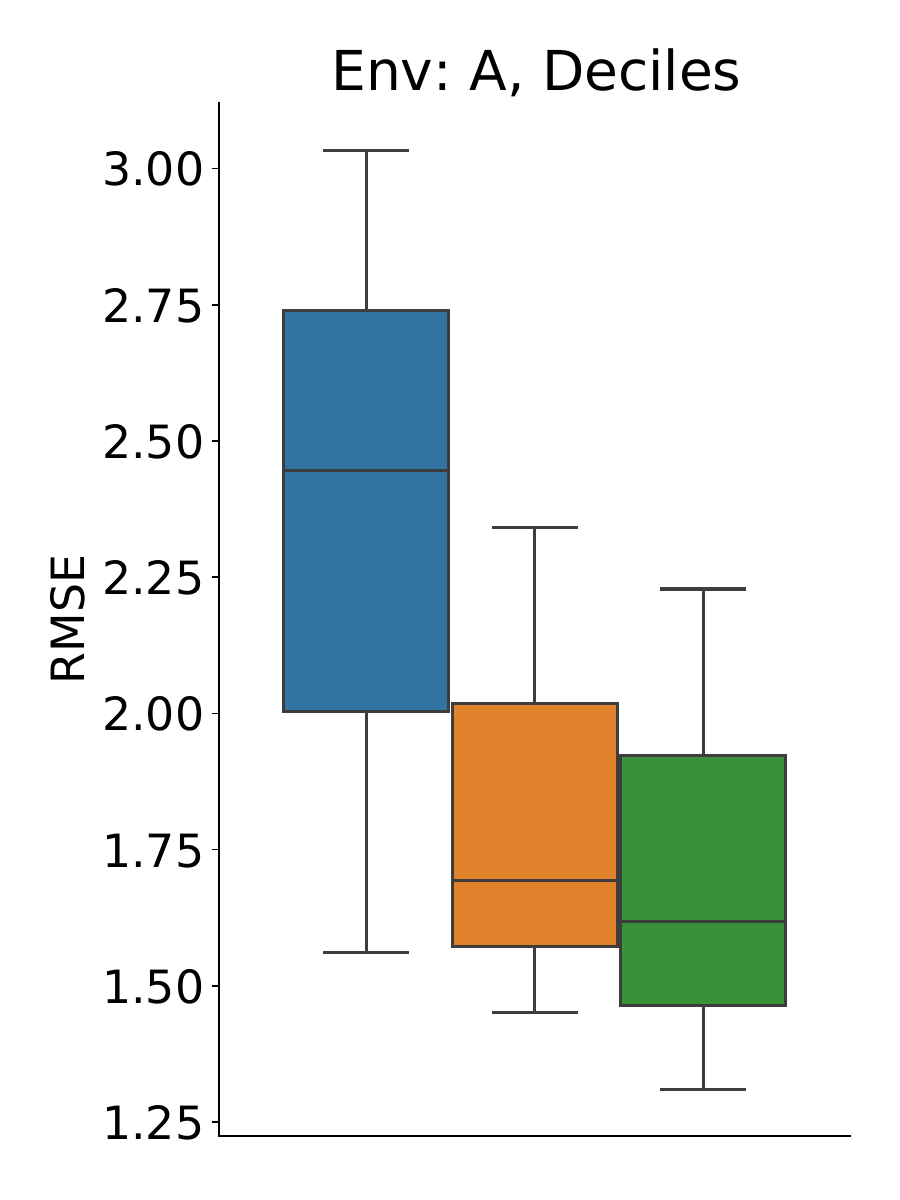}
        \end{subfigure}
        \begin{subfigure}{\threeboxplot}
                \centering
                \includegraphics[width=\textwidth,height=\boxplotheight,trim={.8cm 0 .5cm 0},clip]{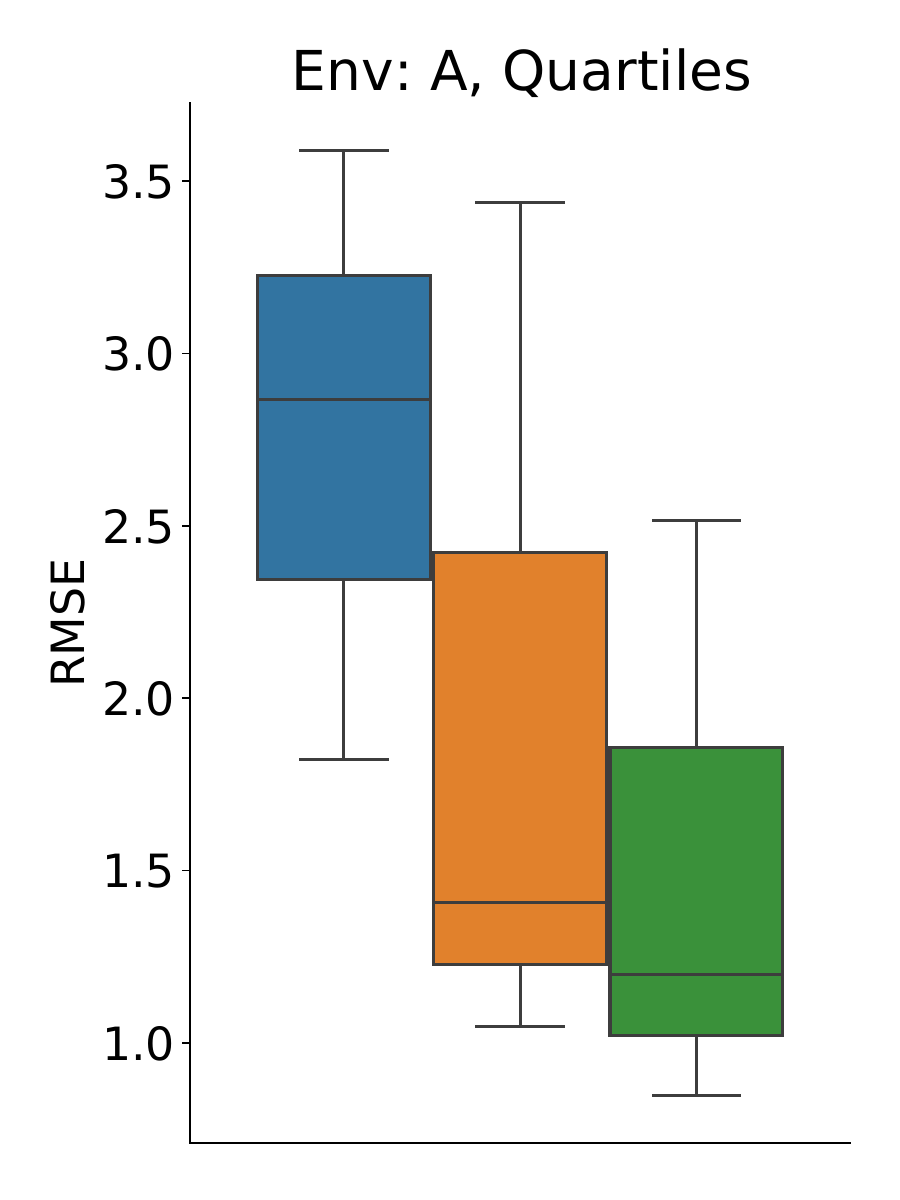}
        \end{subfigure}
        \begin{subfigure}{\fourboxplot}
                \centering
                \includegraphics[width=\textwidth,height=\boxplotheight,trim={.8cm 0 .5cm 0},clip]{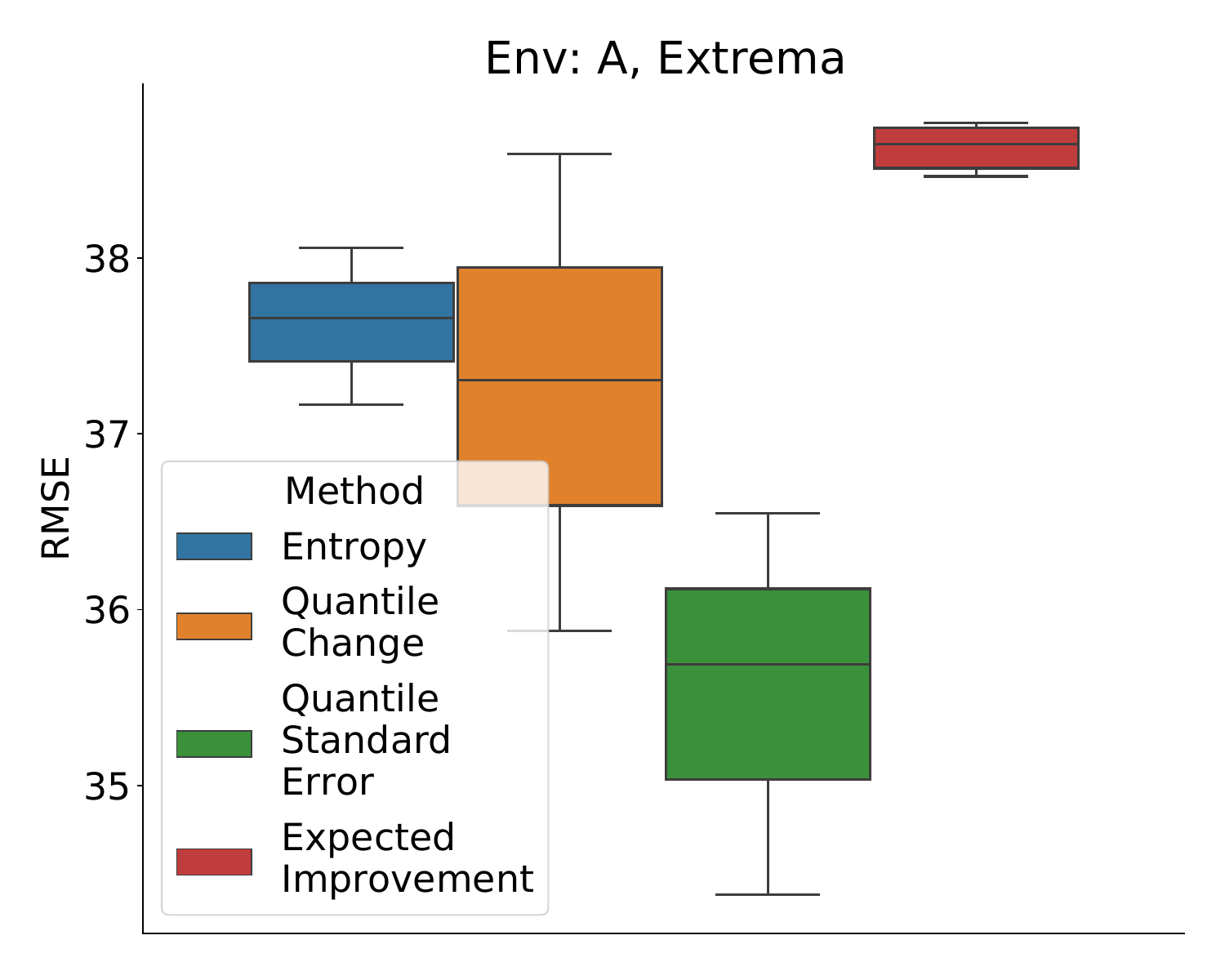}
        \end{subfigure}
        \begin{subfigure}{\threeboxplot}
                \centering
                \includegraphics[width=\textwidth,height=\boxplotheight,trim={.8cm 0 .5cm 0},clip]{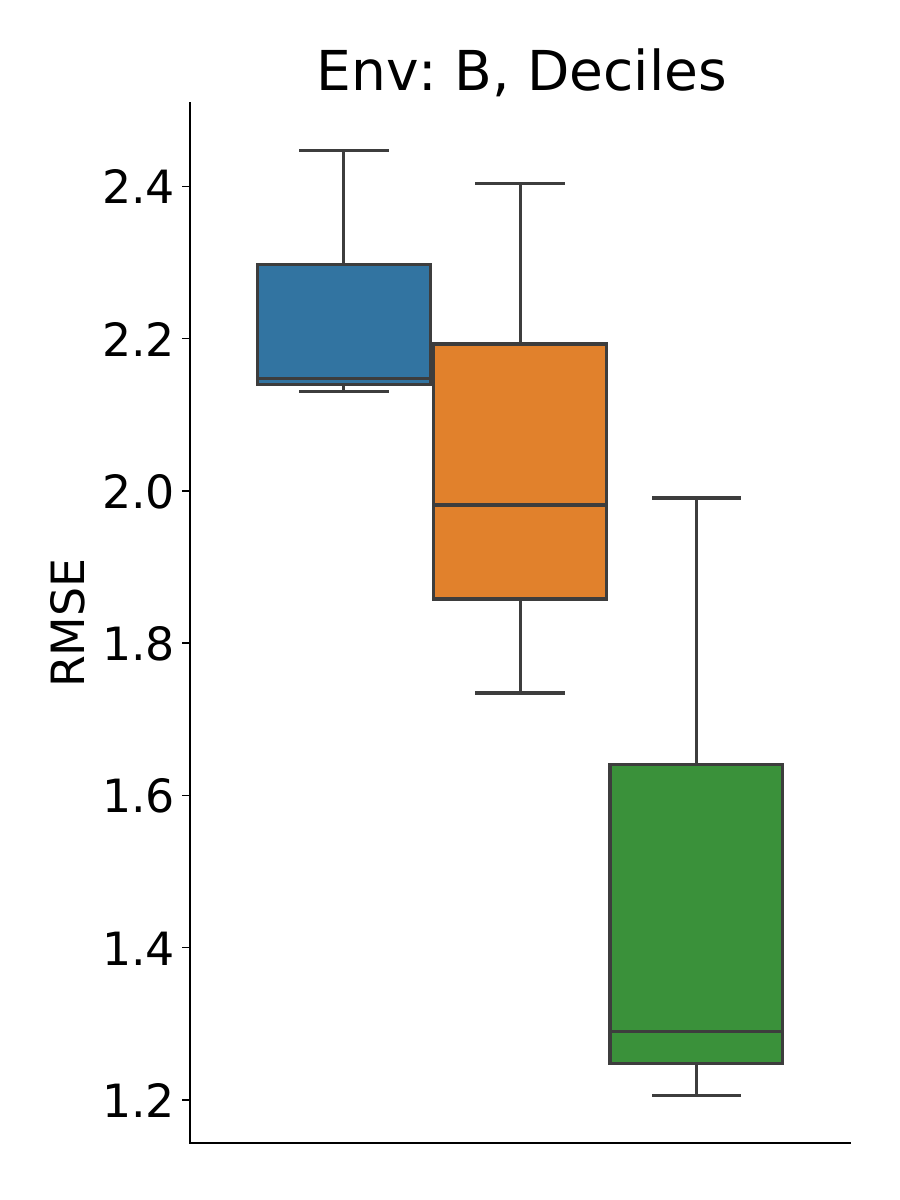}
                
        \end{subfigure}
        \begin{subfigure}{\threeboxplot}
                \centering
                \includegraphics[width=\textwidth,height=\boxplotheight,trim={.8cm 0 .5cm 0},clip]{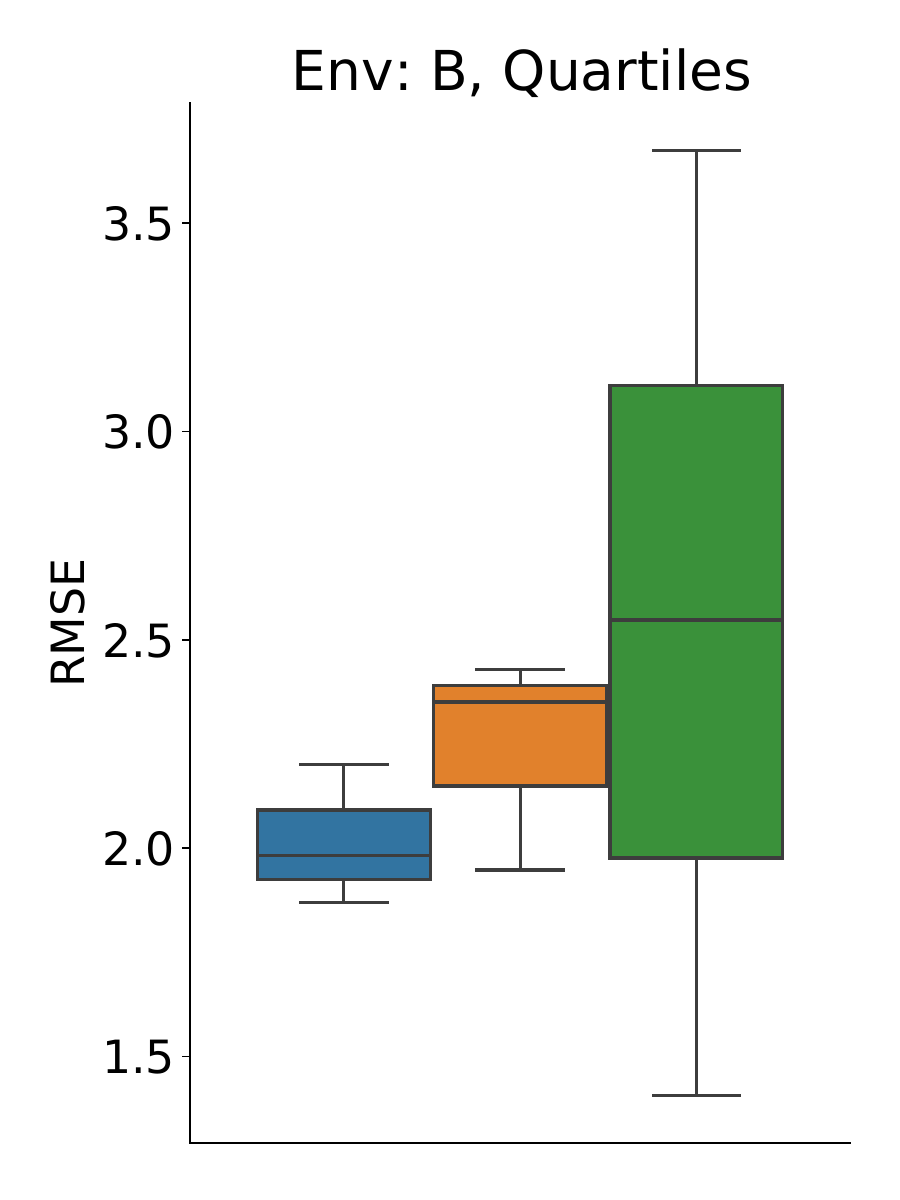}
        \end{subfigure}
        \begin{subfigure}{\fourboxplot}
                \centering \includegraphics[width=\textwidth,height=\boxplotheight,trim={.8cm 0 .5cm 0},clip]{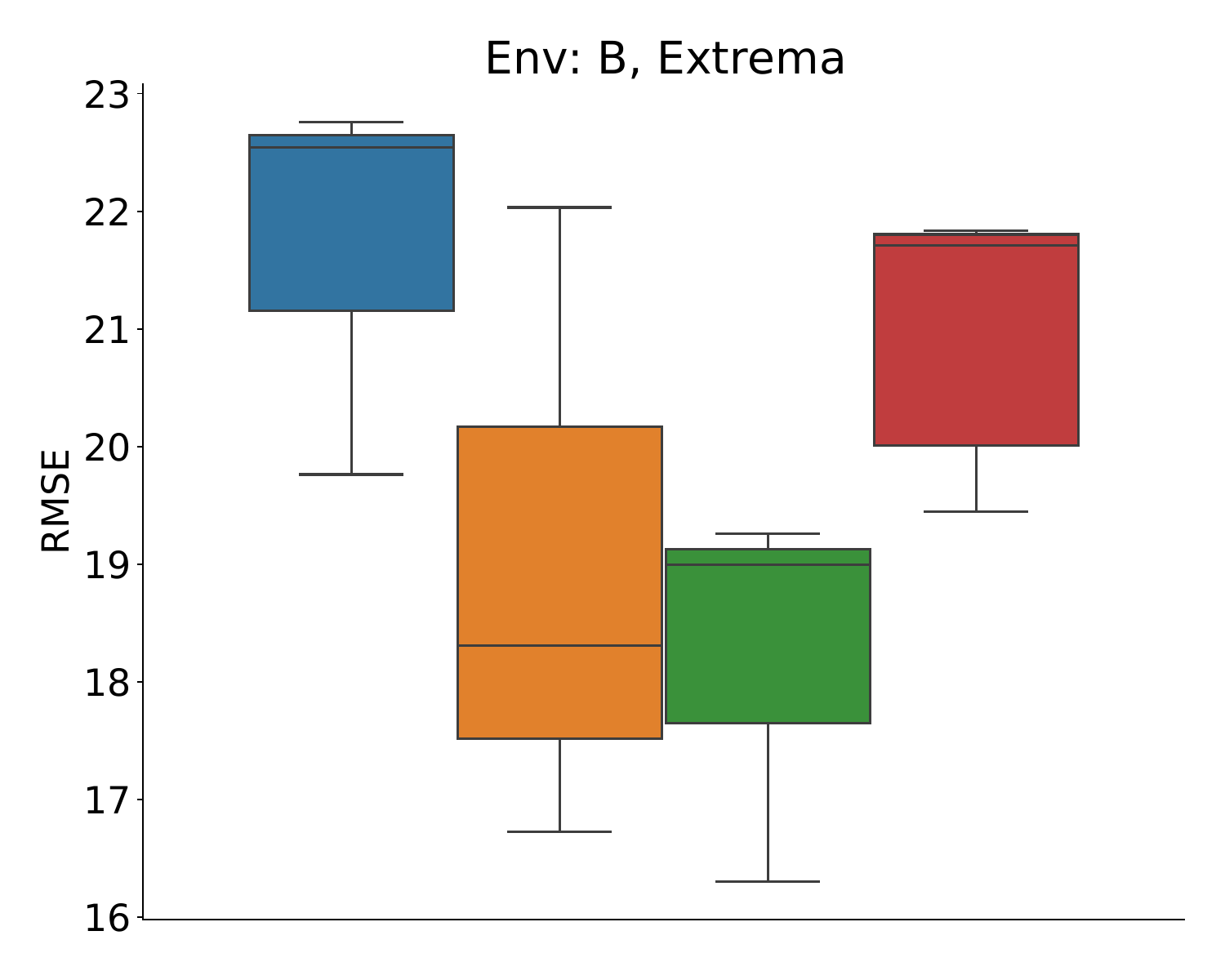}
        \end{subfigure}
        \caption{Simulated drone planning experiments with real data. Error between ground truth quantile values and estimated quantile values. 
        Datasets A and B collected using a hyperspectral camera in Clearlake, California.
        Units are 400nm channel pixel intensity ($0-255$).  Each dataset and objective function pairing is run three times.
        }\label{fig:clearlake_planning_results}
\end{figure}

\begin{figure}[b!]
    \centering
    \includegraphics[trim={0.2cm .2cm 0.3cm 0.3cm}, clip, width=.45\columnwidth]{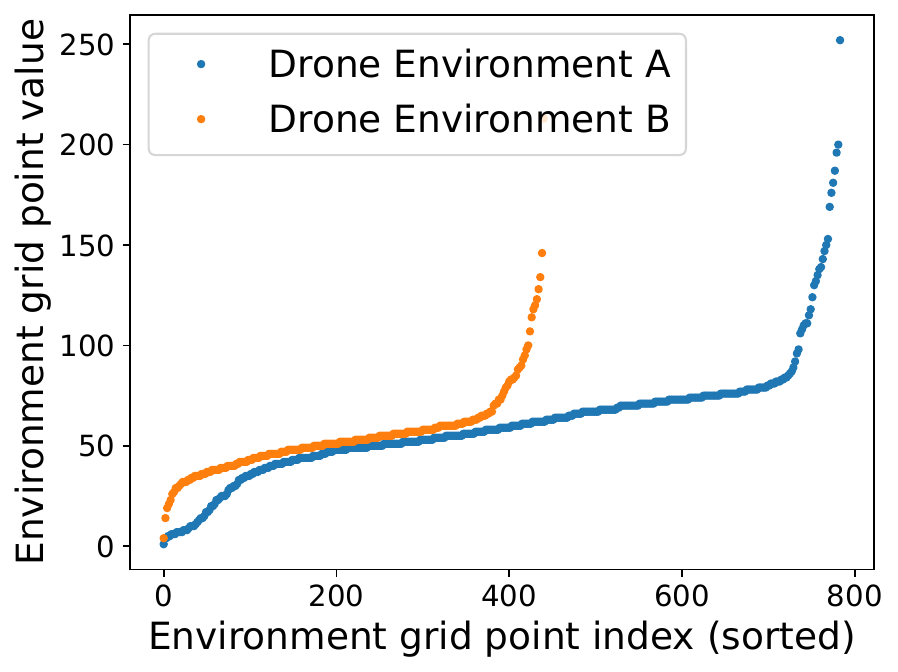}
    \includegraphics[trim={0.2cm .2cm 0.25cm 0.3cm}, clip, width=.45\columnwidth]{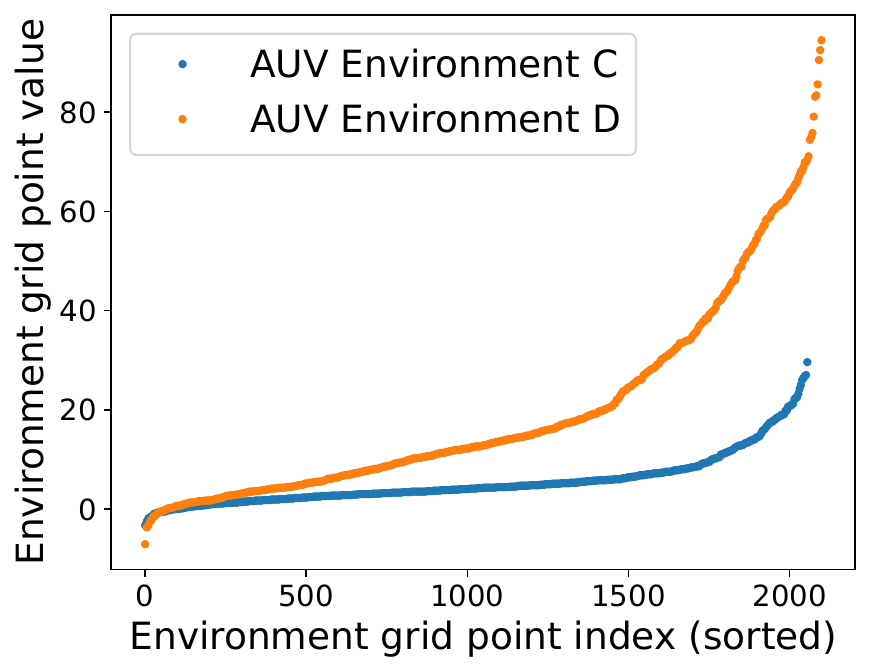}
    \caption{
    Experimental dataset distributions. Drone data is measured in pixel intensity; AUV data in $\mu g/L$ chlorophyll.
    }
    \label{fig:environment_distributions}
\end{figure}


\section{Experiments}
To evaluate our approaches for planning and for location selection, we compare against baselines in two different informative path planning tasks in simulation using four datasets collected in the real world.

In the first task, a simulated drone with a virtual camera gathers data from orthomosaics (a single image produced by combining many smaller images, called orthophotos) collected of a lake using a hyperspectral sensor. 
The orthomosaics, A and B, are taken in the same location but on different days and times.
The drone collects many measurements from one location, where each is a pixel in a downsampled image. 
As a proxy for chlorophyll concentration, we measure the 400nm channel pixel intensity ($0-255$).
The drone maintains a constant altitude and moves in a 2D plane with a north-fixed yaw, moving in either the $x$ or $y$ direction per step.

In the second task, an autonomous underwater vehicle (AUV) explores an environment.
Two AUV surveys, C and D, which were taken in the same reservoir but at different times and different areas, were conducted in a 3D lawnmower pattern using a chlorophyll sensor
and are interpolated using a GP to $\gtsensedlocations$.
At each step, the AUV moves in one $x$, $y$, or $z$ direction and takes five evenly spaced measurements when moving between locations.

We evaluate each task on their two respective datasets (A/B, C/D) and three different quantiles: deciles $(0.1, 0.2, \dots , 0.8, 0.9)$, quartiles $(0.25,0.5,0.75)$, and upper extrema $(0.9,0.95,0.99)$. See \Cref{fig:environment_distributions} for a summary of the dataset distributions.

\begin{figure}[t]
    \centering
                \begin{subfigure}{\threeboxplot}
                \centering
                \includegraphics[width=\textwidth,height=\boxplotheight,trim={.8cm 0 .5cm 0},clip]{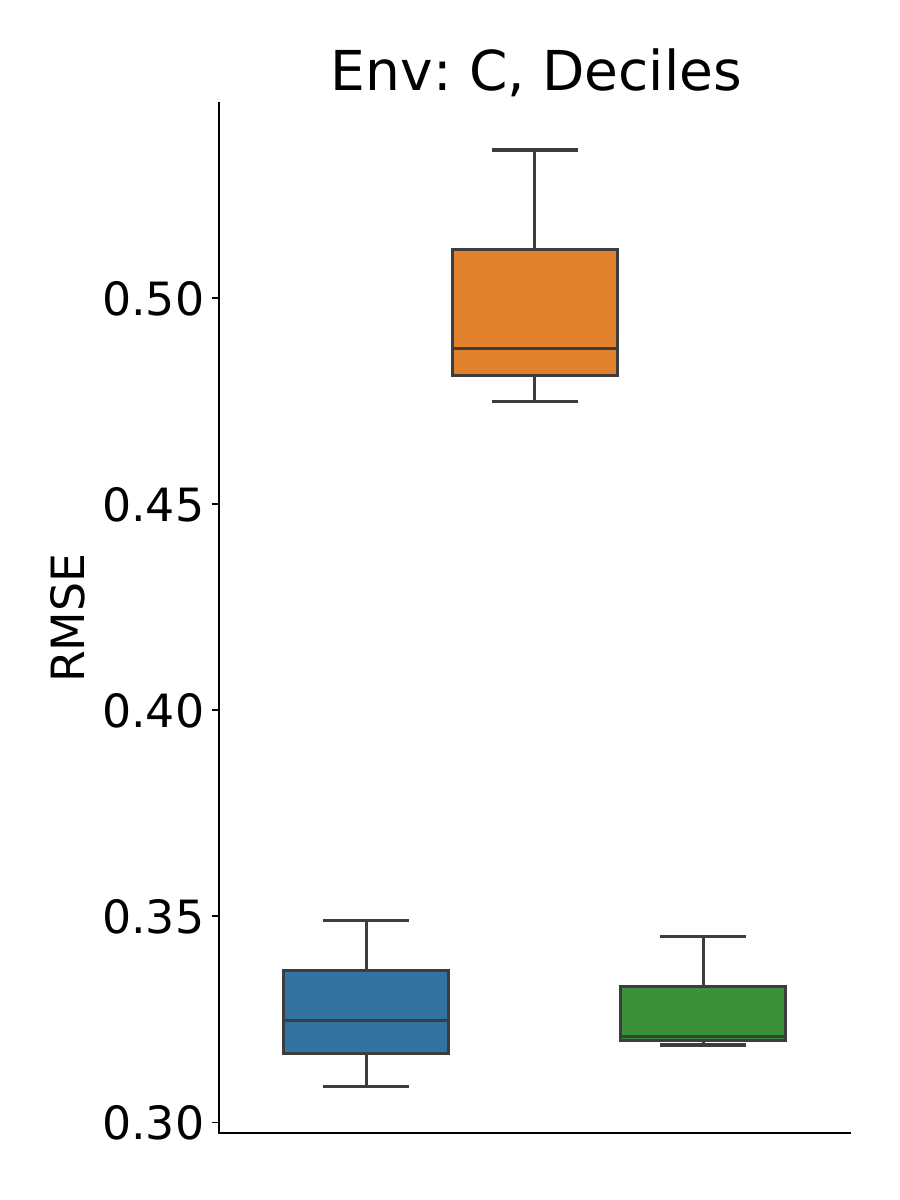}
        \end{subfigure}
        \begin{subfigure}{\threeboxplot}
                \centering
                \includegraphics[width=\textwidth,height=\boxplotheight,trim={.8cm 0 .5cm 0},clip]{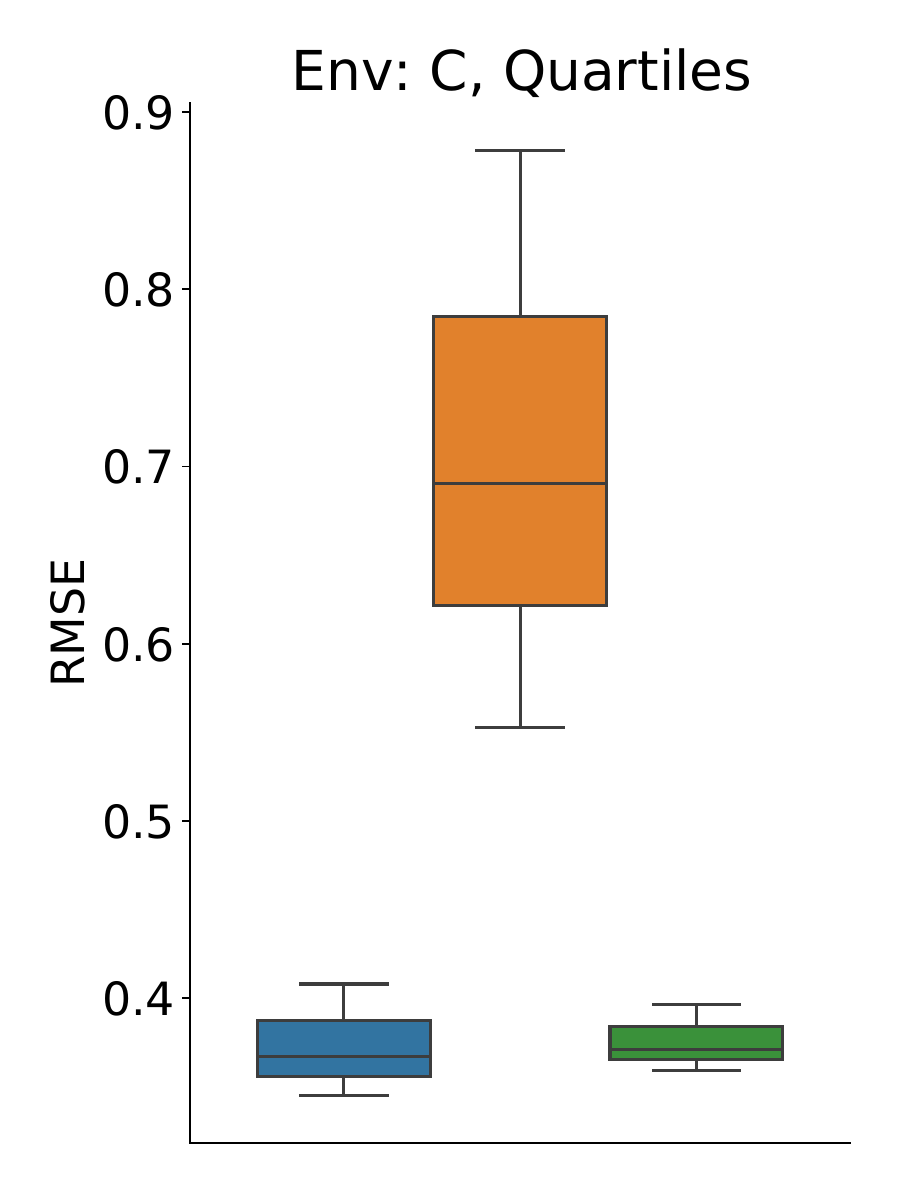}
        \end{subfigure}
        \begin{subfigure}{\fourboxplot}
                \centering
                \includegraphics[width=\textwidth,height=\boxplotheight,trim={.8cm 0 .5cm 0},clip]{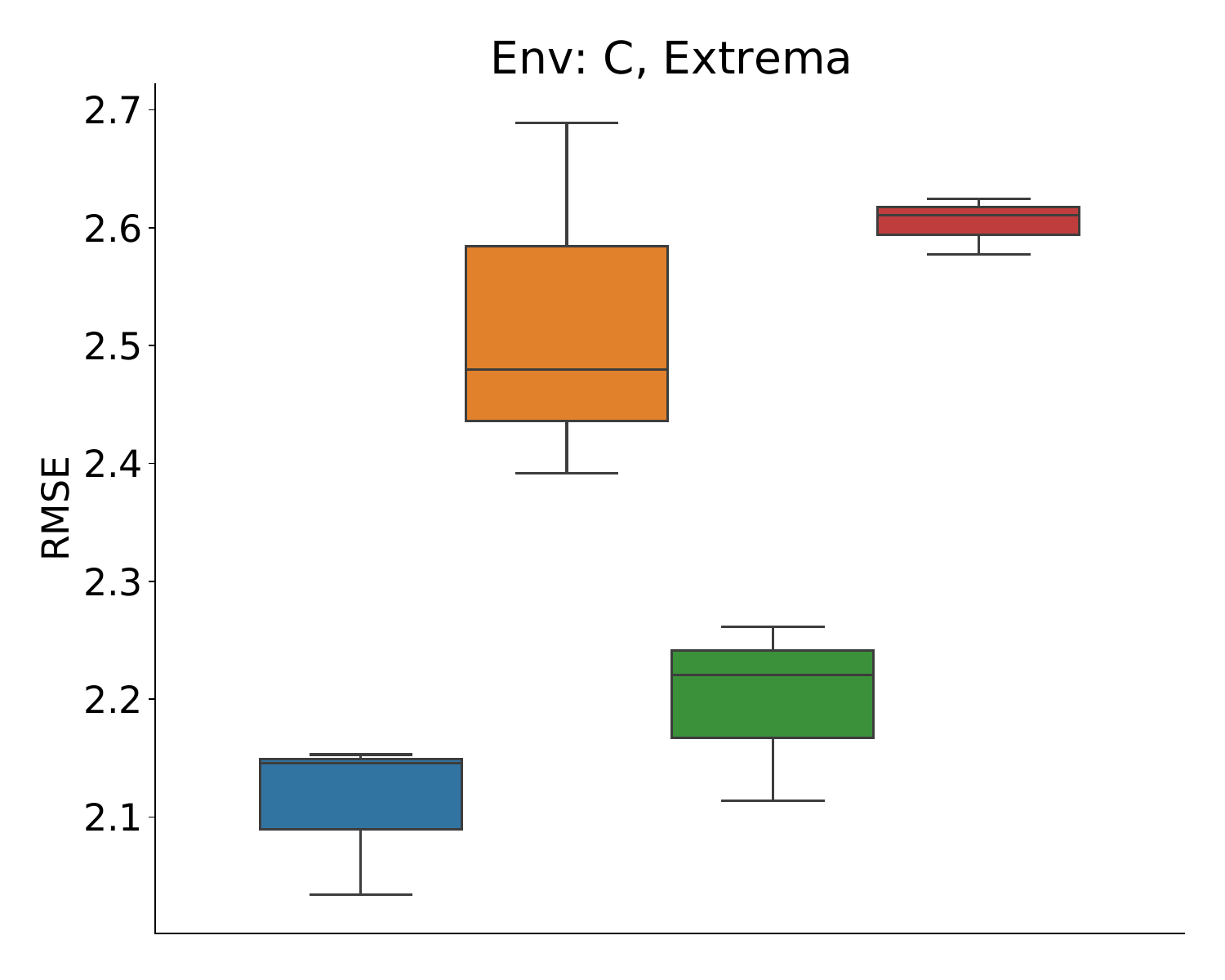}
        \end{subfigure}
        \begin{subfigure}{\threeboxplot}
                \centering
                \includegraphics[width=\textwidth,height=\boxplotheight,trim={.8cm 0 .5cm 0},clip]{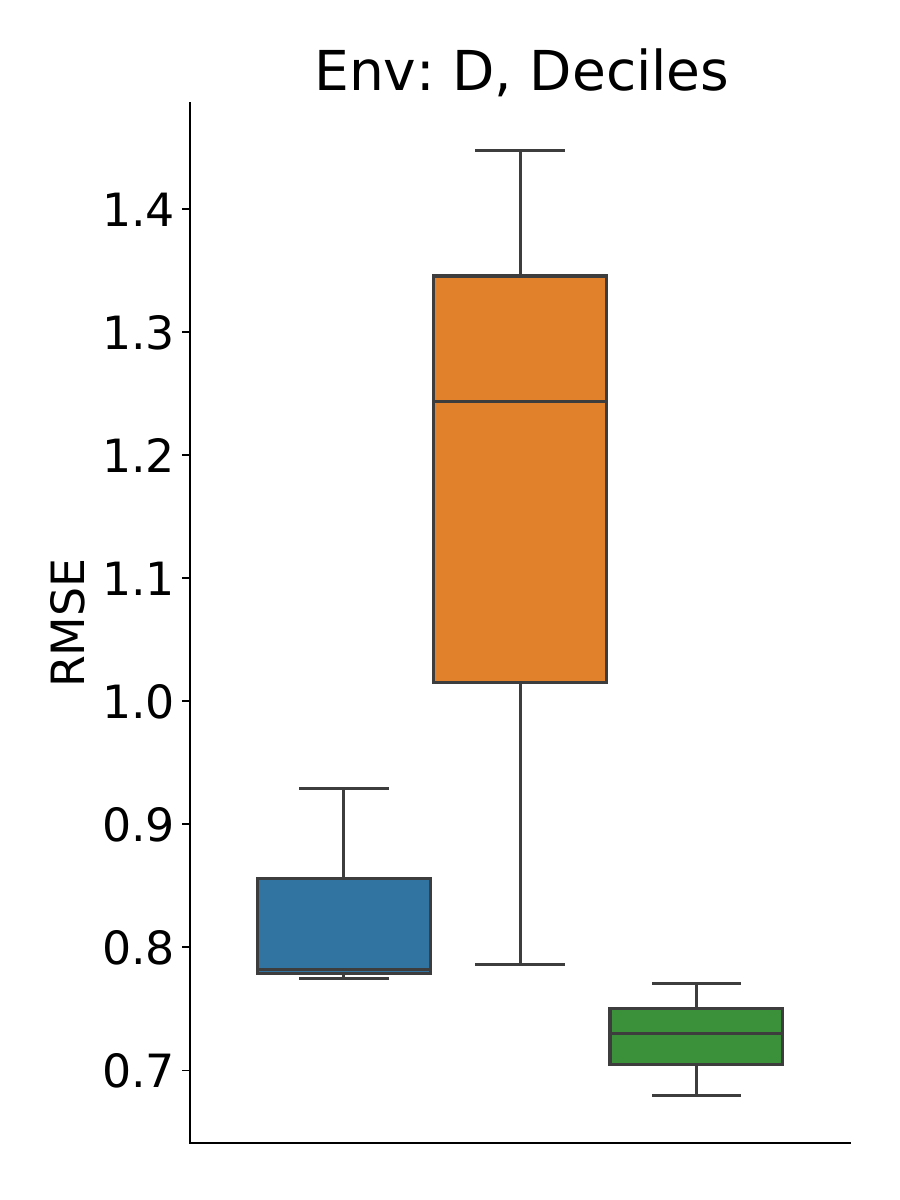}
                
        \end{subfigure}
        \begin{subfigure}{\threeboxplot}
                \centering
                \includegraphics[width=\textwidth,height=\boxplotheight,trim={.8cm 0 .5cm 0},clip]{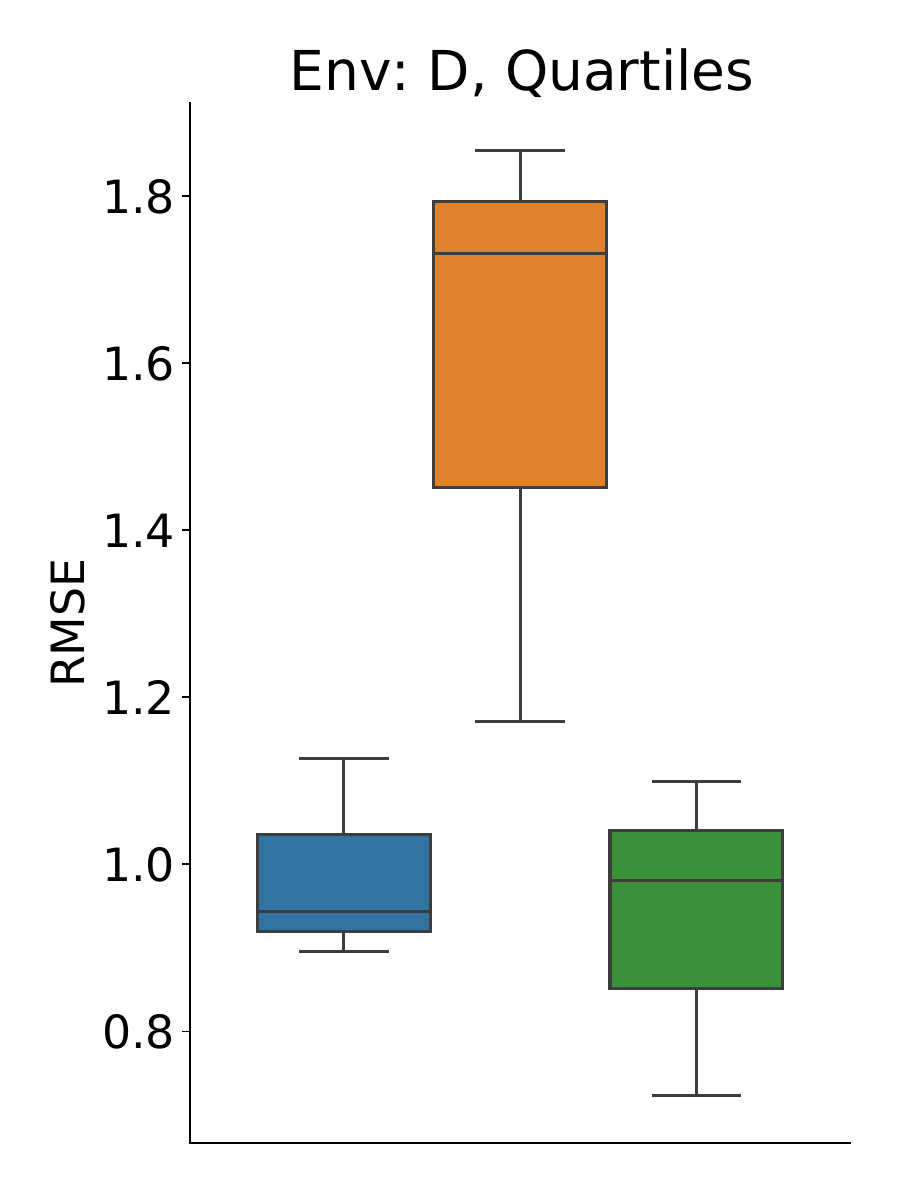}
        \end{subfigure}
        \begin{subfigure}{\fourboxplot}
                \centering \includegraphics[width=\textwidth,height=\boxplotheight,trim={.8cm 0 .5cm 0},clip]{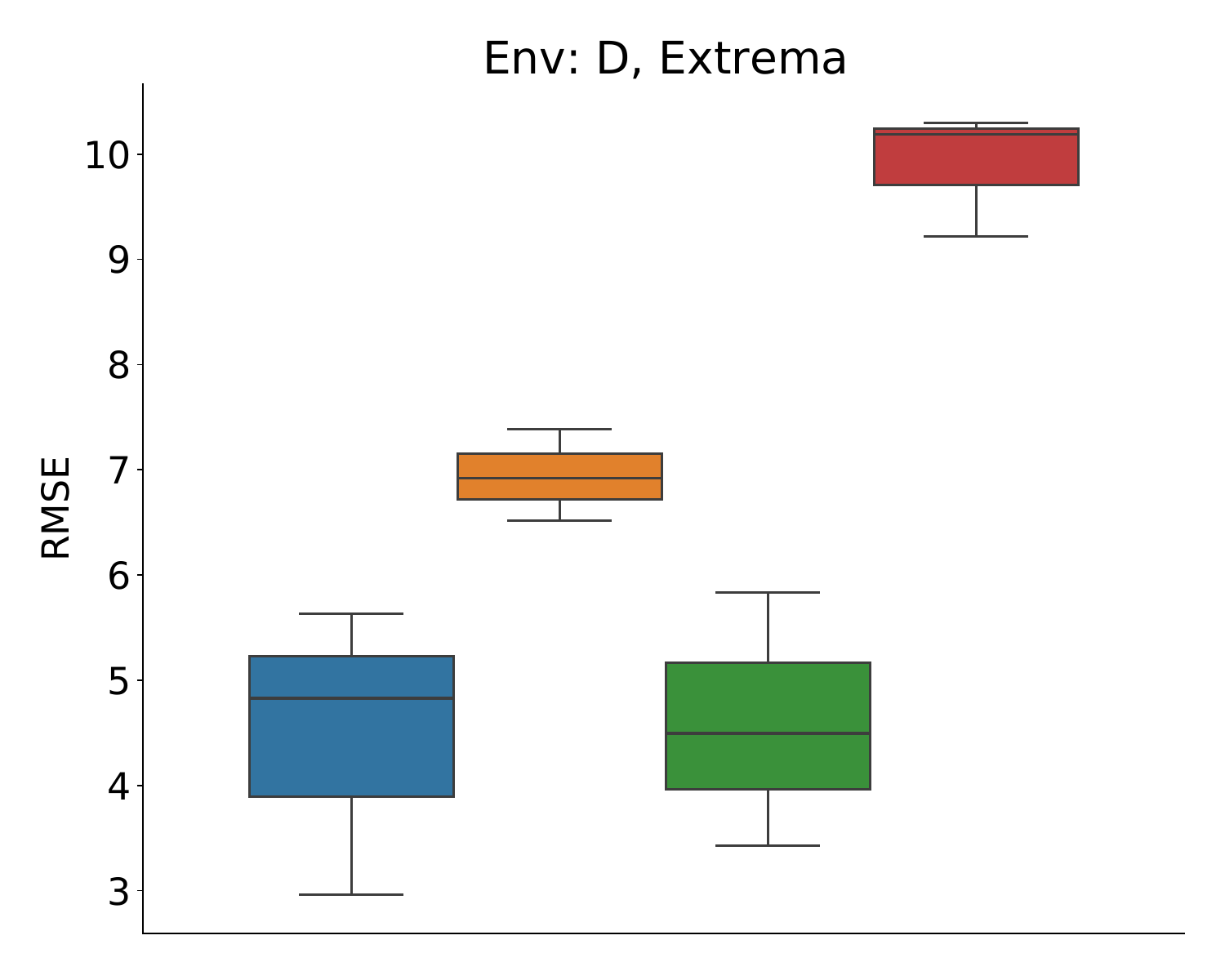}
        \end{subfigure}
        \caption{Simulated AUV planning experiments with real data. 
        Error between ground truth quantile values and estimated quantile values. 
        Datasets C and D collected from a reservoir in California using an underwater robot with a chlorophyll fluorescence sensor.
        Units are $\mu g/L$ chlorophyll.
        Each dataset and objective function pairing is run three times.}\label{fig:ecomapper_planning_results}
\end{figure}

\subsection{Informative Path Planning: Objective Functions}
\label{ssec:ipp_experiments}
To evaluate how well our proposed IPP objective functions estimate the ground truth quantile values in real environments, we compare quantile change (\Cref{eq:quantile_change}) and quantile standard error (\Cref{eq:standard_error}) against a baseline entropy objective function (\Cref{eq:entropy}).
For the upper extrema quantiles, we also compare against expected improvement (\Cref{eq:ei}), as it is similar to a sequential Bayesian optimization based IPP task.

\subsubsection{Setup}
In the planner, we use $\gamma = 0.9$, and each trial is run over 3 seeds.
The objective $c_{plan}$ parameter is set to the approximate magnitude of the rewards seen for each environment, which we found experimentally to be an adequate value.
The GP mean was set to 0 and the datasets were normalized, while the lengthscale was set to 12. The exploration constant $c_{plan}$ was set to $1E-6$ for quantile change and $1E-2$ for quantile standard error.
To compare the performance of the approaches, we use the RMSE between the ground truth quantiles, $\quantilevalues$, and estimated quantile values after performing a survey, $\estimatedquantilevalues$. 
We choose this metric rather than, e.g., comparing the error between the estimated concentration $\mu(\gtsensedlocations)$ and the ground truth concentration  $GT(\gtsensedlocations)$ because the objective of the task specifically focuses on getting the highest accuracy for the quantile values, and does not optimize for overall concentration accuracy.

\textbf{Drone with Camera}
The drone is allowed to take 30 simulated pictures out of a grid with about 300 positions. 
Each picture is downsampled to $8\times 5$ pixels ($40$ measurements) with $37.1^{\circ}$ by $27.6^{\circ}$ field of view, similar to the drone used in the field trial reported in \Cref{sec:field_trial}.
While this is downsampled from the true image, it captures the coarse trends of the concentration, which is both still scientifically useful and performant for GP evaluations.
For each trial, the GP is seeded with 100 evenly spaced measurements across the workspace, as prior data. 
The planner uses $300$ rollouts per step, and the maximum planning depth is 7.

\textbf{AUV with Chlorophyll Sensor}
The AUV is simulated for 200 steps in a $12 \times 14 \times 2$ grid.
The planner uses $130$ rollouts per step and a maximum depth of $10$.
The GP is seeded with measurements from 50 locations.

\subsubsection{Discussion}
Overall, our results are robust across multiple experimental environments as well as robot sensor types.
We find that planning with the quantile standard error objective function has a 10.2\% mean reduction in median error across all environments when compared to using entropy. 
This shows that an objective function tailored to estimating a set of quantiles will outperform a coverage planner, such as entropy, that would typically achieve low overall error in reconstructing the entire environment model.
The focus of our approach is to estimate the quantiles in a
targeted way rather than to achieve uniform coverage.

\begin{figure}[t]
    \centering
        \begin{subfigure}{\thirdcolumn}
                \centering
                \includegraphics[width=\textwidth,height=\boxplotheight,trim={.8cm 0 .5cm 0},clip]{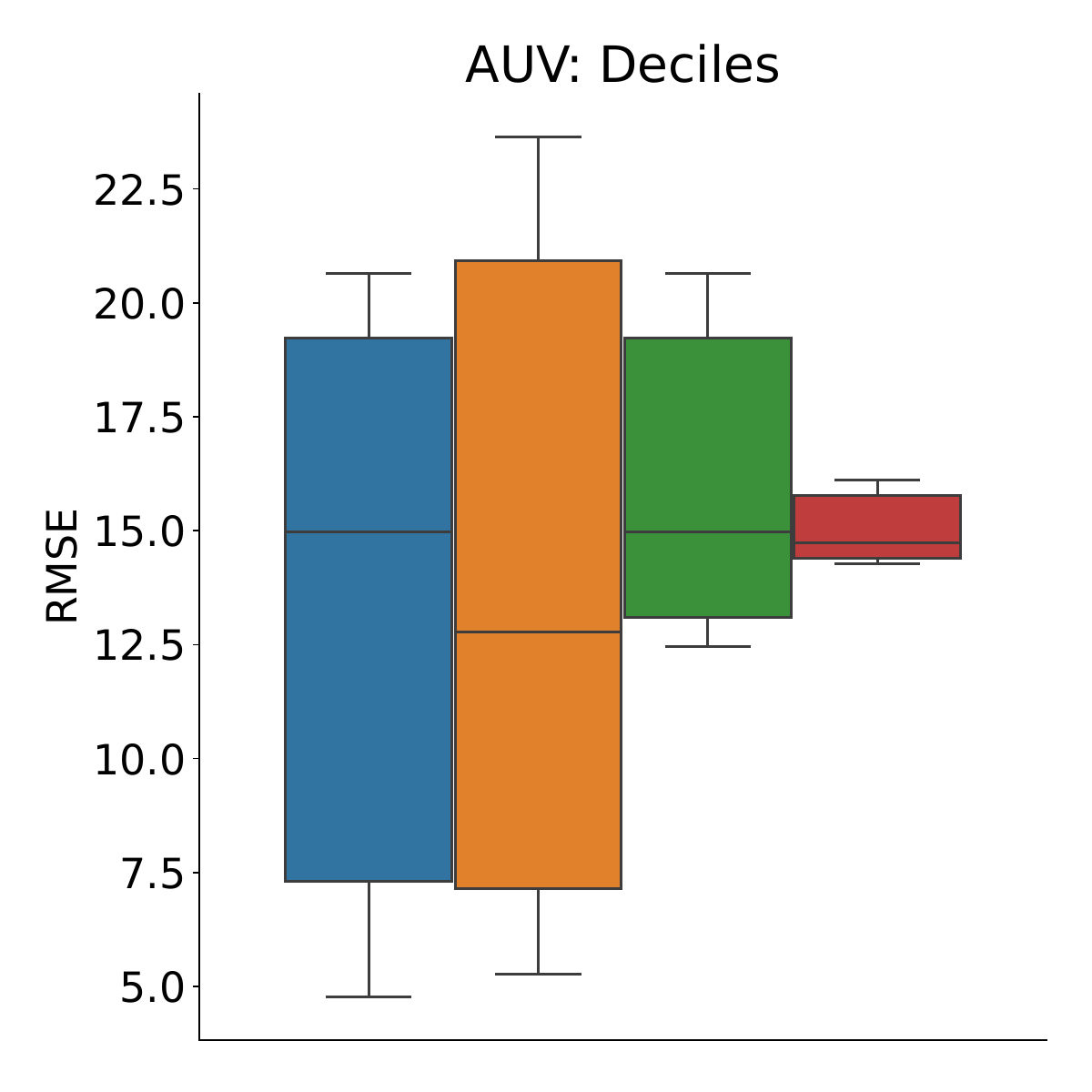}
        \end{subfigure}
        \begin{subfigure}{\thirdcolumn}
                \centering
                \includegraphics[width=\textwidth,height=\boxplotheight,trim={.8cm 0 .5cm 0},clip]{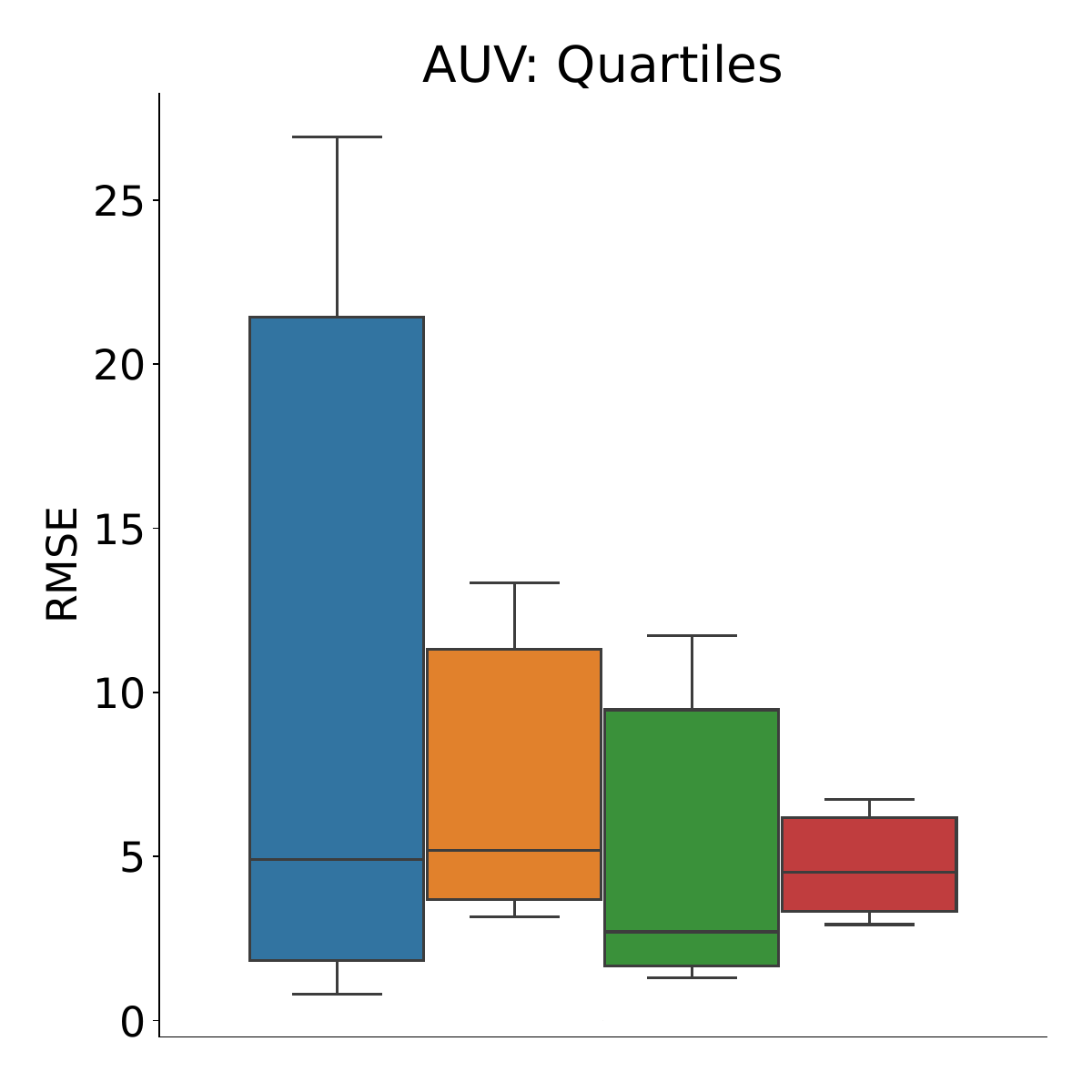}
        \end{subfigure}
        \begin{subfigure}{\thirdcolumn}
                \centering
                \includegraphics[width=\textwidth,height=\boxplotheight,trim={.8cm 0 .5cm 0},clip]{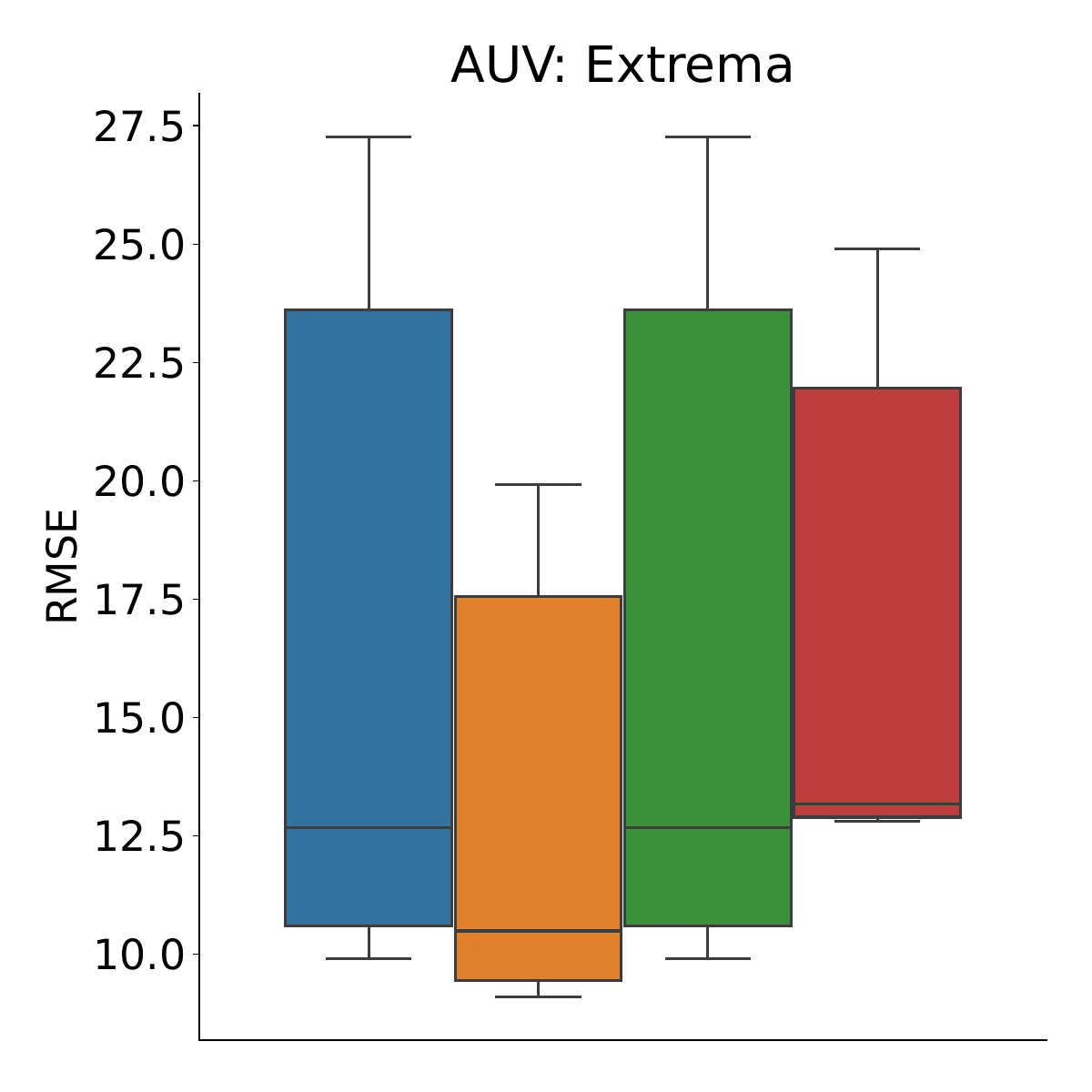}
        \end{subfigure}
                \begin{subfigure}{\thirdcolumn}
                \centering
                \includegraphics[width=\textwidth,height=\boxplotheight]{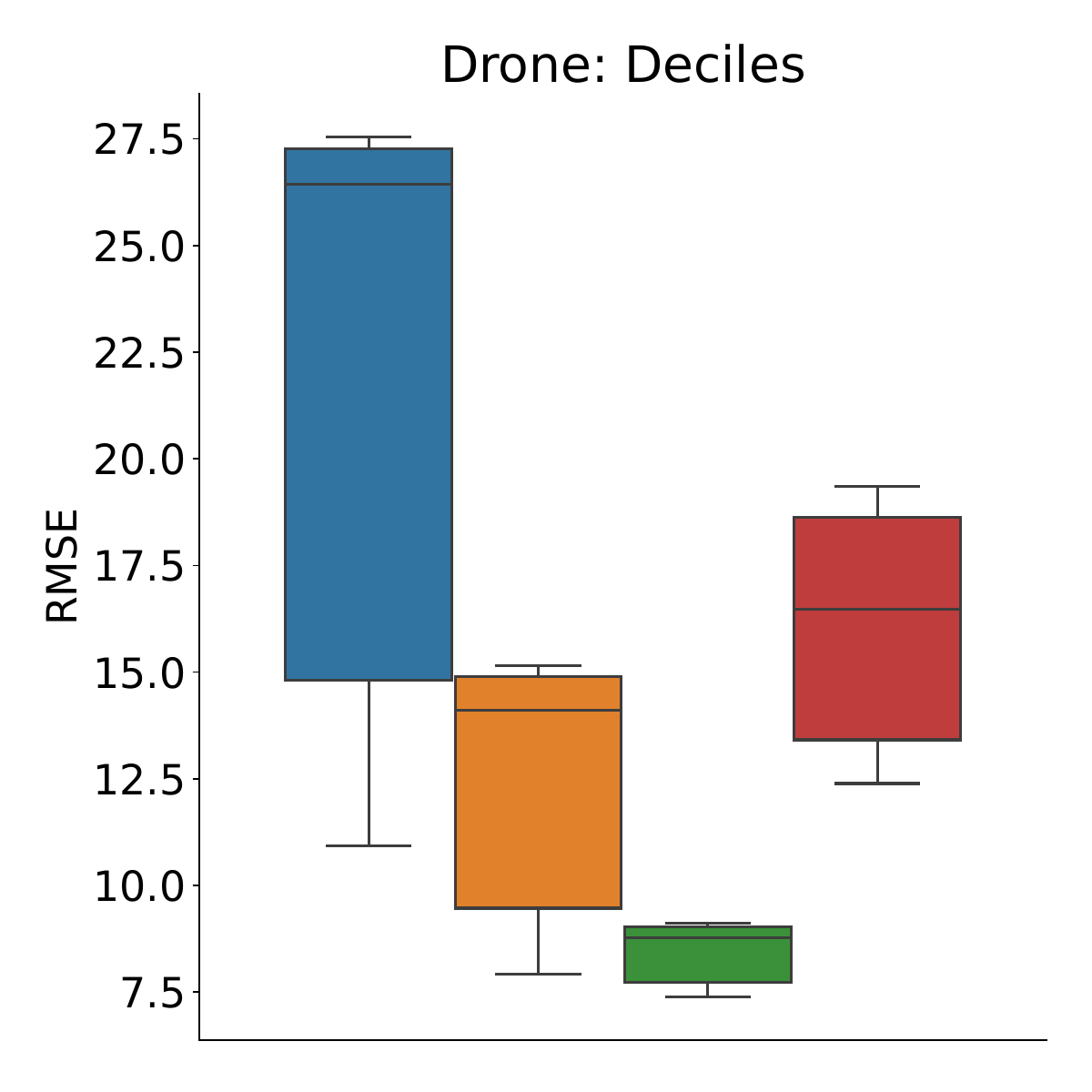}
        \end{subfigure}
        \begin{subfigure}{\thirdcolumn}
                \centering
                \includegraphics[width=\textwidth,height=\boxplotheight]{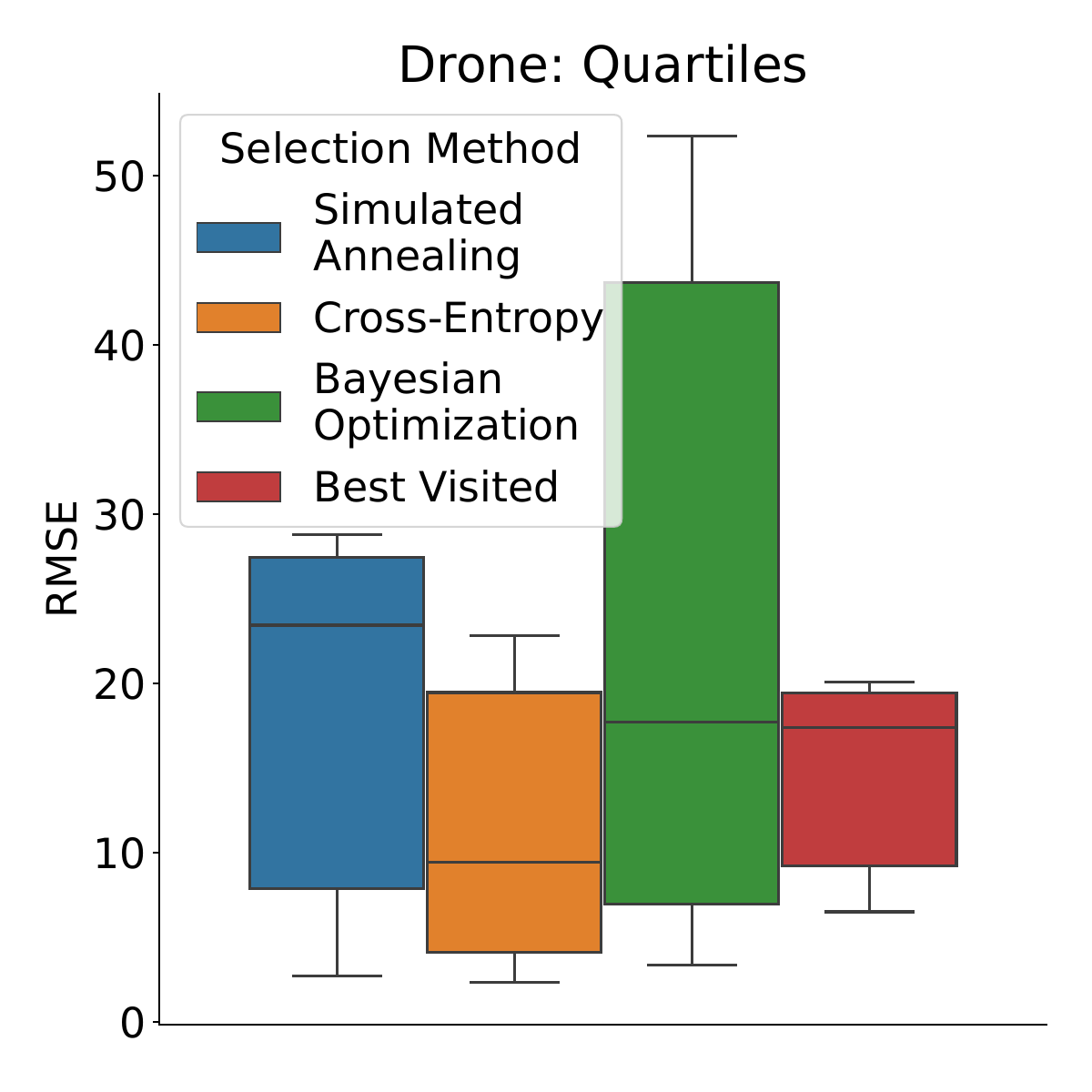}
        \end{subfigure}
        \begin{subfigure}{\thirdcolumn}
                \centering
                \includegraphics[width=\textwidth,height=\boxplotheight]{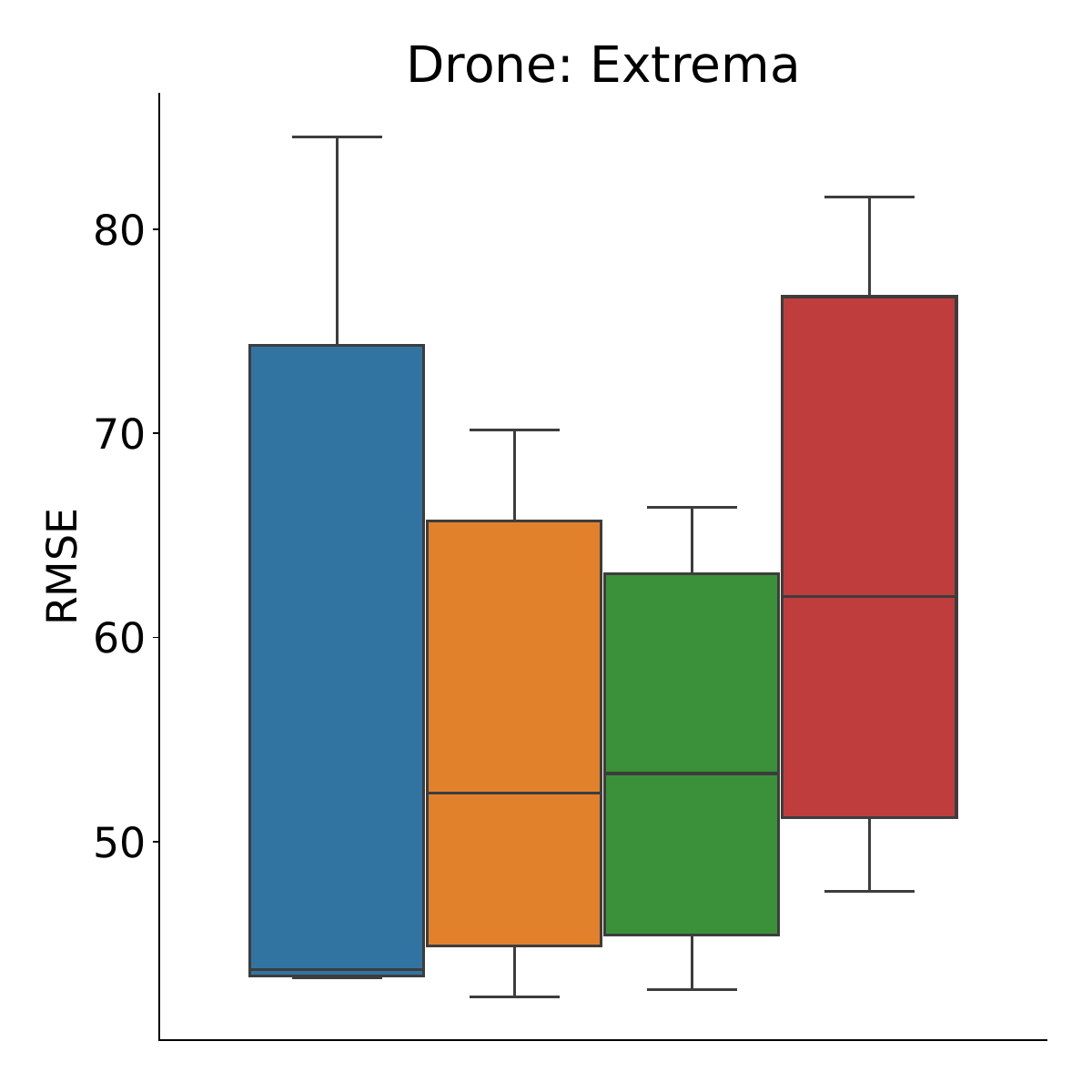}
        \end{subfigure}
        \caption{Quantile location selection results.
        RMSE between the ground truth values at the selected locations and the corresponding true quantile values.
        Units are $\mu g/L$ chlorophyll for AUV, 400nm channel pixel intensity ($0-255$) for Drone.
     }
     \label{fig:ecomapper_ps_boxplots}
\end{figure}

\Cref{fig:clearlake_planning_results,fig:ecomapper_planning_results} show the results of planning with the proposed objective functions for the drone and AUV experiments, respectively. 
In the drone tasks (A/B),
quantile change and quantile standard error outperform the baseline entropy in estimating the deciles and upper extrema.
In quartiles in environment A, both proposed methods perform well, but in environment B entropy outperforms our methods. 
For the extrema, both methods perform better than expected improvement, and entropy outperforms expected improvement in environment A. 
We believe this is because expected improvement focuses explicitly on improving the (single) maximal value and does not do a good job of localizing high concentration areas, thus overestimating the quantile values.

For the AUV tasks (C/D),
quantile standard error outperforms entropy in environment C when estimating deciles and quartiles, and performs equally well as entropy in all other tasks besides estimating the extrema in environment D.
Expected improvement performs poorly in estimating the extrema due to similar issues as with the drone.
Quantile change performs poorly in most AUV tasks in contrast to the drone tasks, where it performs comparably. 
This may indicate that quantile change performance decreases with point sensors, or it is more sensitive to variation in the absolute measurement values.
Thus, we recommend using quantile standard error unless there are computational constraints, since its performance is more consistent.

\begin{figure}[b]
    \centering
    \includegraphics[width=0.8\columnwidth,trim={.2cm .5cm .5cm .5cm},clip]{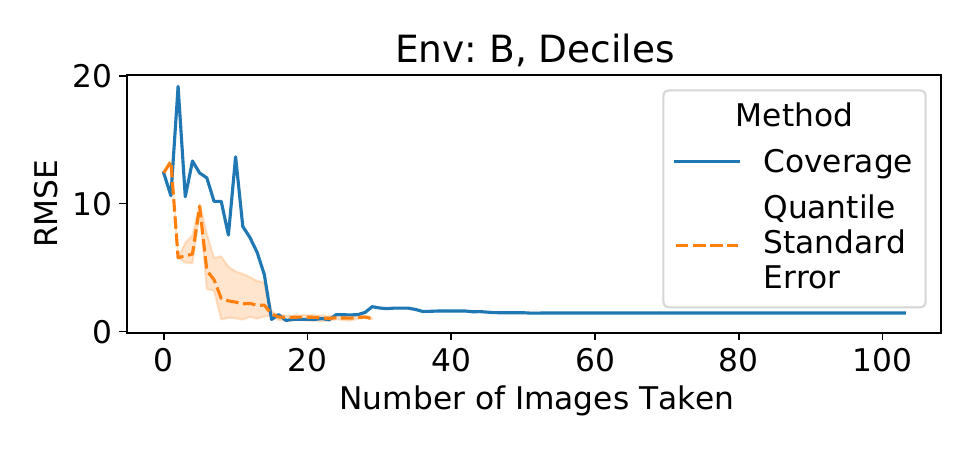}
    \caption{Comparison between quantile standard error with POMDP solver and Spiral-STC coverage planner~\cite{coverage} on Drone Environment B estimating deciles, each over 3 trials.}
    \label{fig:coverage_expr}
\end{figure}

We also compare our approach using the proposed quantile standard error objective function to the Spiral-STC coverage planner~\cite{coverage}. 
\Cref{fig:coverage_expr} shows that the proposed adaptive solution is able to reduce the error between the estimated quantiles and the ground truth quantiles more rapidly than the non-adaptive coverage planner.

\subsection{Location Selection}
To evaluate how close the values at locations suggested by our proposed loss function are to the estimated quantile values, we use the optimization algorithms simulated annealing (SA), cross-entropy (CE), and Bayesian optimization (BO) 
with the results from the quantile estimation task (\Cref{ssec:ipp_experiments}) using the quantile standard error objective function.
\subsubsection{Setup}
We compute the error by $RMSE(\quantilevalues,\estimatedquantilespatiallocations) = \|(\quantilevalues - GT(\estimatedquantilespatiallocations))\|_2$.
We set $c_{select}$ to 15 for deciles, 200 for quartiles, and 30 for extrema.
We compare the results against a Best Visited (BV) baseline which selects the location the robot took a measurement at that is closest to each quantile value, i.e. solve $\estimatedquantilespatiallocations^* = \argmin_{\estimatedquantilespatiallocations_{bv} \subset \sensedlocations} \|\estimatedquantilevalues - \mu(\estimatedquantilespatiallocations_{bv})\|_2$.

For SA, we use $T_\t{max} = 5$, $T_\t{min} = 0.001$, and cooling rate $\t{cr} = 0.005$ which leads to approximately 1000 optimization steps, and reset to the best solution every 100 steps. 
We start the optimization using the solution found by the BV baseline.
For CE, we use $\alpha = 0.9$, $\eta = 0.9$, 50 samples per iteration, and 100 iterations.
$\alpha$ is a weighting factor on new samples, which is used to prevent premature convergence.
For BO, we use the expected improvement acquisition function and initialize the GP with 50 randomly selected $\quantilespatiallocations$s as well as the solution found by the Best Visited baseline.
We use 100 iterations and report the best found solution.

\subsubsection{Discussion}
The error between the values at the estimated quantile locations and the true quantile values for the drone and the AUV experiments can be seen in \Cref{fig:ecomapper_ps_boxplots}.
Overall, we find that cross-entropy and Bayesian optimization produced the locations with values closest to the true quantile values.
Both these methods perform a global search in the space of possible locations, indicating that global search optimizes \Cref{eq:ps_loss} more effectively.
Simulated annealing had greater variability in performance.
We believe this is because it is a local search method and may fail to escape local optima.
Best Visited produced a good initialization for the other methods, but was easily outperformed.
We find that CE performs the best in three out of six possible scenarios.
In particular, we see a 15.7\% mean reduction in median error using CE with our proposed loss function compared to the BV baseline across all environments when using our proposed quantile standard error optimization function for quantile estimation during exploration.

In general, methods can find the best points when selecting for quartiles or deciles,
while the upper extrema are more difficult.
Because the robot is limited in the amount of environment it can explore, the upper extrema  are less likely to be measured during exploration.
This leads to these quantiles being more challenging to select representative points for.

\begin{figure}[t]
    \centering
    \includegraphics[trim={0.5cm .5cm 0.7cm 0.9cm}, clip, width=.45\columnwidth]{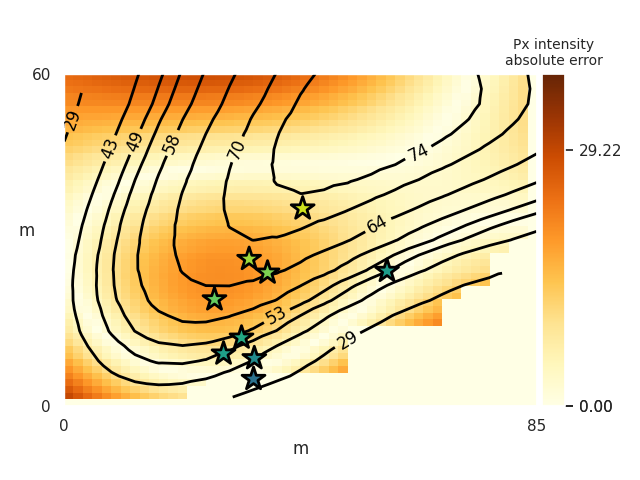}
    \includegraphics[trim={0.5cm .5cm 0.5cm 0.9cm}, clip, width=.45\columnwidth]{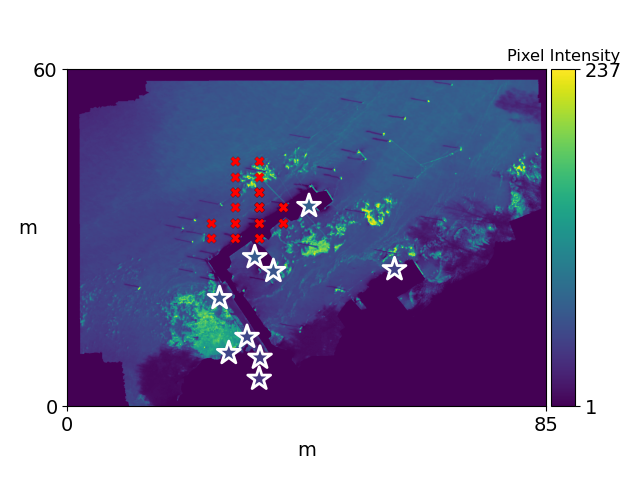}
    
    \caption{
    Physical locations (stars) selected by the cross-entropy optimizer for deciles on the drone experiment.
    [Left] Black lines: true quantile value contours, overlaid on the absolute error between $\mu(\gtsensedlocations)$ and $GT(\gtsensedlocations)$. Note that the lowest error tends to follow the quantile contour lines.
    [Right] Red crosses: locations the robot visited, overlaid on the ground truth image.
    }
    \label{fig:contours}
\end{figure}


\Cref{fig:contours} shows results for one seed of the drone task when monitoring deciles.
The suggested locations, shown as stars, align relatively closely with the true quantile values $\quantilevalues$, shown by the contours on the left image.
This demonstrates the ability of the optimizers to produce good location suggestions to guide environmental analysis.

The right part of \Cref{fig:contours} shows the same locations on top of the orthomosaic of what the drone could measure during exploration.
This part of the figure highlights the difficulty of the problem of IPP for quantiles. 
The robot could only explore 15\% of the total environment.
With only partial knowledge of the distribution, the robot's model of the phenomenon will vary based on the particular points it visited, which in turn affects the estimates of the quantiles.
%
%
Similarly, \Cref{fig:3d_points} shows the suggested locations for the upper extrema for one seed of the AUV task.
For both tasks, all the optimization methods, with the exception of Best Visited, suggest points that may be spatially far from locations the robot has been able to visit.
This allows for points with values potentially closer to the true quantile values to be selected for scientific analysis.

\begin{figure}
    \centering
    \includegraphics[trim={1cm .5cm 0.8cm 1.5cm}, clip, width=.65\columnwidth]{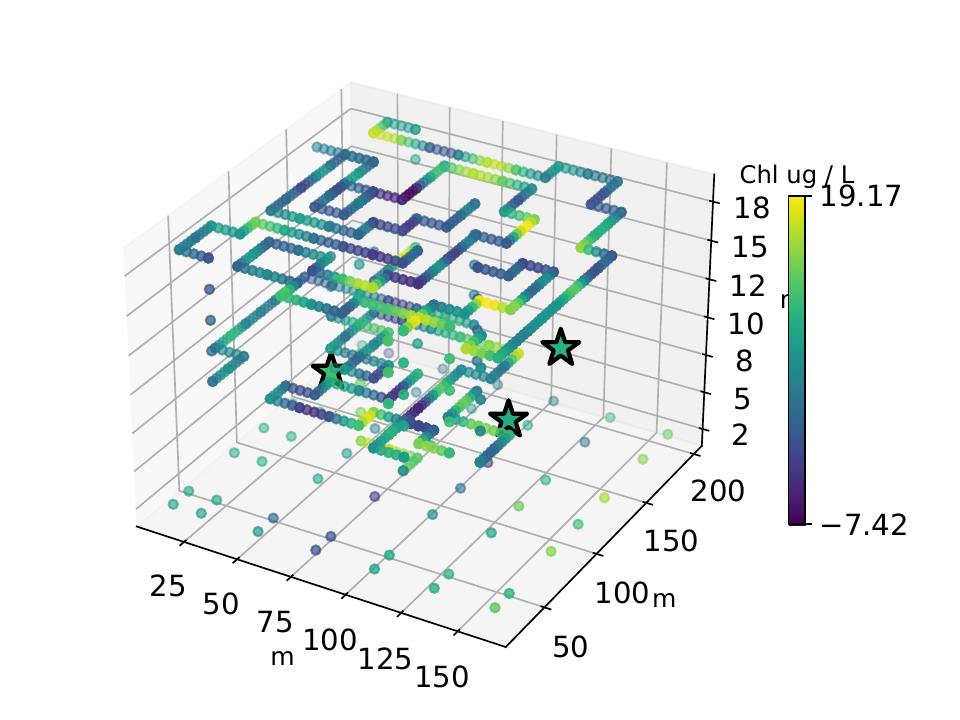}
    \vspace{1px}
    \caption{Physical locations (stars) selected by the cross-entropy optimizer for the upper extrema quantiles with an AUV using a chlorophyll point sensor.
    Blue/green points are the measured locations $\sensedlocations$.
    }
    \label{fig:3d_points}
\end{figure}

\begin{figure}[b]
    \centering
    \includegraphics[trim={4.3cm 2cm 5.1cm 2.3cm}, clip, width=.75\columnwidth, align=c]{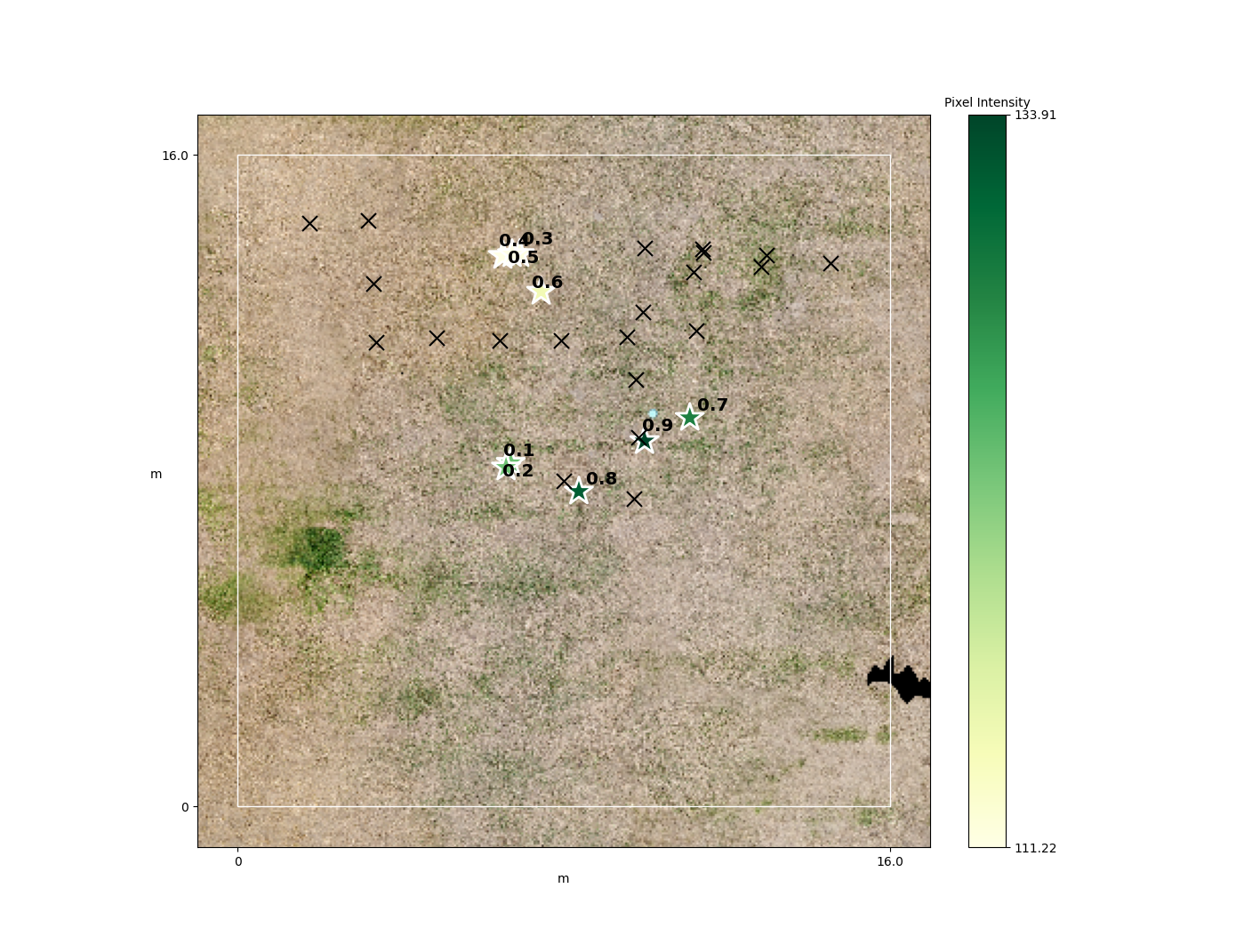}
    \includegraphics[trim={14cm 0cm 4.5cm 0cm}, clip, width=.23\columnwidth, align=c]{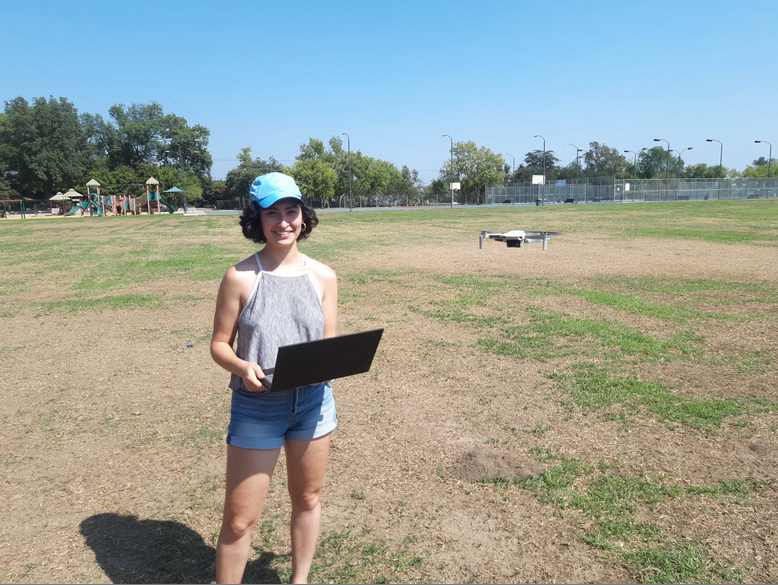}
    \caption{Visualization of a field trial modeling a crop health task.
    \textit{Left:} Crosses are locations where the drone took images. The drone is limited to only visit 20\% of the possible locations to take images.
    The quantiles of interest are the deciles and the locations are chosen by cross-entropy.
    The 9 stars show locations suggested to collect physical specimen. 
    The measurement of interest is the amount of green in each pixel. 
    \textit{Right:} The drone used in the field trial, in flight.
    }
    \label{fig:field-results}
\end{figure}

\subsection{Field trial}\label{sec:field_trial}
The goal of our final experiment is to demonstrate our method on real hardware for a crop health monitoring task in an open, grassy field, where the objective is to estimate the deciles of the green present in each pixel of the images (we use green as a proxy for plant health).

\subsubsection{Setup}
We use a commercial, off-the-the shelf drone with a standard camera to take measurements of the field at a constant height of 3m. 
Similar to the simulated drone, the drone in this experiment moves in a 2D plane with a north-fixed yaw.
The drone has a limited budget of 20 pictures (planning steps) in a $16\times16$ m square grid, where $|\gtlocations| = 10\times10$ locations, and each picture is downsampled to $8 \times 5$ pixels.
For planning, we use the quantile change objective function because it performs well for camera sensors and is faster than quantile standard error.
We operate the drone with a custom app and use onboard GPS and IMU for localization.

\subsubsection{Discussion}
\Cref{fig:field-results} shows the resulting suggested locations using CE based on the 20 steps the robot took. 
We find that, although the robot cannot explore the entire workspace due to its limited budget, and in fact does not visit a large green area, the system is able to suggest varied locations for specimen collection.
Some selected locations (e.g., those representing the $0.3, 0.4, 0.5$ quantiles) are spatially close to each other, which suggests this area contains a large gradient, or those quantile values are similar.
As the goal is to suggest locations that contain the estimated quantile values, such locations may be near one another.

\section{Conclusion}
Scientists traditionally collect physical specimens at locations selected using heuristics.
They later analyze these specimens in a laboratory to characterize a phenomenon of interest (e.g., the distribution of algae in the water).
We propose to, instead, choose these specimen collection locations by first performing an informative path planning survey with a robot and then proposing locations which correspond to the quantiles of interest of the distribution.

To accomplish this, we propose two novel objective functions, quantile change and quantile standard error, for improving the estimates of quantile values through informative path planning. 
We test these in three settings: a drone with a camera sensor over lake imagery, an underwater vehicle taking chlorophyll measurements, and a field trial using a drone for a crop health task.
Our objective functions perform well and outperform information-theoretic and Bayesian optimization baselines. 
In our experiments, our proposed quantile standard error objective function has a 10.2\% mean reduction in median error when compared to entropy as the objective function. 

We also show that using our proposed loss function with black-box optimization methods can select environment locations for analysis that are representative of a set of quantiles of interests. 
We find that a cross-entropy optimizer using our loss function outperforms a baseline of using the best measured points,
with a 15.7\% mean reduction in median error in values across all environments.

Our approach can be used to guide physical specimen collection in real field scenarios, as demonstrated by our field trial.


\chapter{Multirobot Quantile Estimation in Natural Environments}
\label{ch:multirobot}
\Cref{ch:quantiles} has given us a framework in which interpretable scientific goals can be expressed in task specification for a robot in a field exploration deployment.
However, in this context, task specifications may include multiple robots, and with that, different multirobot architectures.
This leads us to characterize performance on the exploration task in this chapter, with the intention of understanding effects of task specifications abstracted away from single-robot deployments.
This chapter is based on work published in \cite{rayas2023study}.

\section{Introduction}


Scientists who study natural environments have used robots to assist in surveying or exploring regions of interest, for example to monitor harmful algal blooms \cite{kemna_pilot_2018,harmful2022}. To describe such phenomena both flexibly and in an interpretable manner, we can specify quantiles of interest that robots can target during exploration \cite{rayas2022informative}. 
Previously, work in this context has focused on single-robot adaptive surveys; in this chapter, we study multi-robot surveys, motivated by groups that may have more than one robot available, and by collaborative surveys between groups that would like to pool their robot resources to maximize scientific output from a survey. 
Though the naive assumption may be that more robots will always be better, in this work, we aim to investigate, in a principled manner,  under what conditions this is true.
Deploying a robot in the environment can be expensive since more resources are typically required with more robots deployed, and it is unclear how having more robots in such a use case will scale.
This study assesses the impact of team size, starting location, planning budget, and communication between robots for quantile estimation tasks in field environments. We believe it is an important step toward principled decisions regarding robot deployments for field robotics in aquatic biology.

Our contributions are as follows:
\begin{itemize}
    \item We present the first study that investigates multirobot quantile estimation;
    \item we present quantitative results on the effect of team size on performance on real-world aquatic datasets;
    \item we present quantitative results on the effect of parameters including initial location spread, exploration budget, and inter-robot communication on performance, giving insight into what matters for a multirobot study;
    \item we discuss the results in the context of field applications and how they may transfer to different experimental setups.
\end{itemize}


\begin{figure}
    \centering
    \includegraphics[width=0.8\columnwidth]{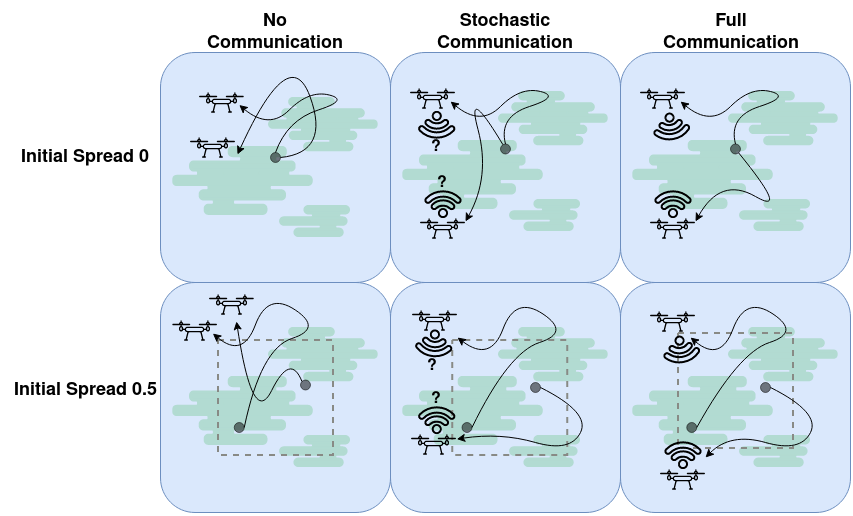}
    \caption{Some variations of multirobot planning approaches.
    Top row shows initial spread $\spread = 0$: All robots start in the center.
    Bottom row shows initial spread $\spread = 0.5$: Robots start spread in an area 50\% length and width of entire workspace.
    Left column shows no communication: Robots have no knowledge of the others; equivalent to each robot planning as if it were the only one.
    Middle column shows stochastic communication: Robots attempt to share observations but messages fail stochastically based on inter-robot distance.
    Right column shows full communication: Robots share the same environment model; equivalent to centralized planning.
    }
    \vspace{-10pt}
    \label{fig:multirobot-hero}
\end{figure}

\section{Background}
In this chapter, we are interested in the problem of multirobot quantile estimation, and specifically, we ask whether and under what conditions a multirobot approach to this problem is effective. 
We place our contributions in the context of previous work in these areas, introduced in \Cref{ch:ipp-background}.
Here, we seek to investigate the effect of using multiple homogeneous robots specifically on the task of quantile estimation in natural environments, with the aim of understanding how to most effectively use resources for challenging and resource-constrained field work problems.
Our work differs from previous work in multirobot IPP in that it explicitly addresses the question of different team sizes and other factors' impact on performance. Further, the task we study is to improve the estimates of an arbitrary set of quantiles, rather than coverage or hotspots which has typically been the focus of other work.

The work presented in this chapter builds upon \Cref{ch:quantiles}, but there are several differences.
First, the methods in \Cref{ch:quantiles} seek to produce, as a final output, a set of locations at which the desired quantile values can be found, while the goal here is to accurately estimate the desired quantile values.
Second, we now generalize the problem to the multirobot domain, where previously only one robot was considered.
Finally, where a main contribution previously was to propose new objective functions for planning that are tailored to the quantile estimation problem, our focus is now on investigating the impact of team size and other factors on quantile estimation accuracy, and thus we adopt just one of the previously proposed objective functions in all our experiments.


\section{Methods}
The problem considered in this paper is that of accurately estimating a set of quantiles of a distribution of interest, such as algae concentration in a lake, 
by exploring a finite 2D workspace with $\nrobots$ robots that take measurements at the locations they visit.



The problem setup follows from \Cref{ch:quantiles}, with slight modifications and additions to account for the multirobot aspect.
There are $\nrobots$ homogeneous robots collaborating on the estimation task
with a planning step budget of $B_T$. The locations at which a robot has taken a
measurement up to time $t$ and the corresponding values measured are represented
by $\sensedlocations_{0:t}$ and $\sensedvalues_{0:t}$, respectively.  Each robot
$n$ maintains an estimate of the quantile values which is given by
\begin{equation}
    \estimatedquantilevalues_n = quantiles(\mu^{(n)}_{GP_t}(\gtsensedlocations),\quantiles)
\end{equation}
where $\mu^{(n)}_{GP_t}(\gtsensedlocations)$ is the estimate of all possible locations
using the GP conditioned on $\sensedlocations_{0:t}$ and $\sensedvalues_{0:t}$. 
Note that, for clarity, we will omit the $(n)$ superscript in the remainder of the section.

For a given $\nrobots$ and $B_T$, our problem is to select a complete path
$\boldsymbol{P}^*$ composed of paths $P_n$ for each robot $n$ within a planning
budget constraint: 

\begin{equation}
\boldsymbol{P}^* = \argmax_{\boldsymbol{P}} f(\boldsymbol{P})
\end{equation}

where
\begin{equation}
   \boldsymbol{P} = \bigcup_{n \in \nrobots} P_n 
\end{equation}

\begin{equation}
   \qquad \forall n \in N :~\textit{length}(P_n) \leq B_n
\end{equation}
and $f$ is the objective function used during planning. 
In this work, we are interested in estimating the value of the underlying
concentration at quantiles $Q$. Note that our goal is not finding the highest
value in the environment nor optimizing model accuracy at every location.

To this end, we use the quantile standard error objective function to evaluate a
proposed location \cite{rayas2022informative}, which measures the difference in
standard error of the estimated quantile values using a robot's current environment
model $\mu_{GP_{i-1}}$, compared to using its current model updated with the
expected new measurements at the proposed location $\mu_{GP_{i}}$.
It additionally includes a term that encourages exploration of locations with
high model variance weighted by the parameter $c$, called the
\textit{exploration bias}.

\begin{equation}
\label{eq:objective}
    \objectivefunction(\sensedlocations_i) =
\frac{d}{\numtiles} +
\sum_{\sensedlocation_j \in \sensedlocations_i} c\sigma^2(\sensedlocation_j)
\end{equation}

\begin{equation}
    d =
\|se(\mu_{GP_{i-1}}(\gtsensedlocations),\quantiles) -
se(\mu_{GP_{i}}(\gtsensedlocations),\quantiles) \|_{1}
\end{equation}

We formulate the planning problem as a POMDP, as defined in \Cref{tab:BayesianSearchAndPOMDPs}, and solve the planning problem using a POMCPOW solver
\cite{sunberg2018online}.  


At each planing step, the robot selects the next
location to travel to based on the measurements it has collected and its
corresponding environment model. After moving to the next location, the robot
takes an image, collecting a set of (noisy) point measurements, which it feeds back into
its GP model.
In this work, we assume a straight-line low-level motion planner for generating
trajectories from one location to the next.
We additionally enable a best-effort communication system between
robots and compare to two other communication formulations, which are further
described in \Cref{sssec:comms}. The planning process continues with each of the
$\nrobots$ robots until the allotted budgets have been reached, at which point
planning is complete and the robots can return to the base.

At this time, all measurements from all robots are compiled and used to produce $\estimatedquantilevaluesfinal$ as follows:
\begin{equation}
    \estimatedquantilevaluesfinal = quantiles(\sensedvalues_\textrm{aggregate},\quantiles)
\end{equation}

We now introduce several variations on the multirobot approach which we compare
in our experiments, as illustrated in \Cref{fig:multirobot-hero}.
In the single robot
case, the formulation is the same, simply with $\nrobots = 1$; however, the
communication and initial spread variations have no effect on the implementation or
execution of the planning.

\subsection{Initial Location Spread}
\label{ssec:alpha-initial-location-spread}
During informative path planning, teams of robots will not cover the entire area
possible and are naturally affected by their starting area.  To this end, we
define an initial location spread parameter $\spread$ which varies from 0 (all robots
start at the same location) to 1 (robots are spread around the
entire workspace).  At spread 0.5, for example, the robots would be started evenly
spaced in a rectangle half the size of the entire workspace. We choose initial locations for robots following a variation of Lloyd's algorithm \cite{lloyd1982least} which iteratively computes the centroids of an approximate Voronoi tessellation of the space and reassigns the locations to those centroids.
We perform the process for $100$ iterations using $100 \nrobots$ randomly sampled points. The result is $\nrobots$ locations approximately uniformly spread in the allotted workspace.

\subsection{Budget}
\label{sssec:budget}
To control for the fact that a team with more robots will in practice simply
have more planning steps, and for a fairer comparison to a single robot baseline,
we implement a variation on the budget constraint which we call
\textit{shared budget}: $\forall n \in \nrobots : B_n = B_T / \nrobots$.  
When $(B_T \hspace{-3px}
\mod{\nrobots}) \neq 0$, the remainder is split evenly among the most robots
possible.
The alternative is the default of a \textit{complete budget}, which gives each robot
the full budget, i.e. $\forall n \in \nrobots : B_n = B_T$.

\subsection{Communication}
\label{sssec:comms}
To investigate the impact of communication on the effectiveness of multirobot
planning, we distinguish between three versions of information sharing.

\textit{Full communication} assumes perfect, instantaneous communication,
which translates to every robot using the same shared GP model of the
environment. This is a centralized formulation of the planning problem. After
every measurement that any robot takes, the shared GP is updated and is available immediately to other robots for planning.

At the other extreme is \textit{no communication}. In this case, each robot has its own environment model.
Essentially, this implementation is equivalent to
$\nrobots$ robots planning in the environment independently, with no knowledge of the
others; when new measurements are taken, they are only used to update that
robot's GP and motion is not coordinated at all.
We believe a comparison to no communication provides a valuable baseline to understand how much improvement arises due solely to information sharing. 
In a practical sense, if separate research groups come together to pool their robots, implementing coordinating mechanisms can demand significant cost of time and effort. 
In that case, deploying the available robots with no behavioral changes may be simplest.

The third variation, \textit{stochastic communication}, lies in the middle.
After taking measurements, a robot attempts to transmit the information to
every other robot. We model each attempt as a Bernoulli trial and
assume data is successfully transmitted with sigmoidal probability $p_{\textrm{success}}$,
following \cite{clark2021queue} and where $distance$ is the distance between the
two robots, $\eta$ is a constant defining the sigmoid steepness,
and $r$ is the distance at which communication quality degrades past a threshold:
\begin{equation}
\label{eq:comms-probability}
    p_{\textrm{success}} = \frac{1}{1 + e^{\eta (distance - r)}}
\end{equation}
We assume that if communication is successful between robots $n$ and $m$ at time
step $i$, all $\sensedlocations_i$ are received by $m$ (i.e., there are no partial
or corrupted measurements transmitted).
The receiving robot's GP is updated with the received data.

We additionally compare to splitting the area into regions, or the \textit{partitioned} case.
This uses a Voronoi partition of the workspace based on the initial locations and restricts each robot to stay within its assigned region at all times.
No communication is enabled between partitioned robots.

\begin{figure}[t!]
  \centering
  \includegraphics[width=\linewidth, height=0.25\linewidth]{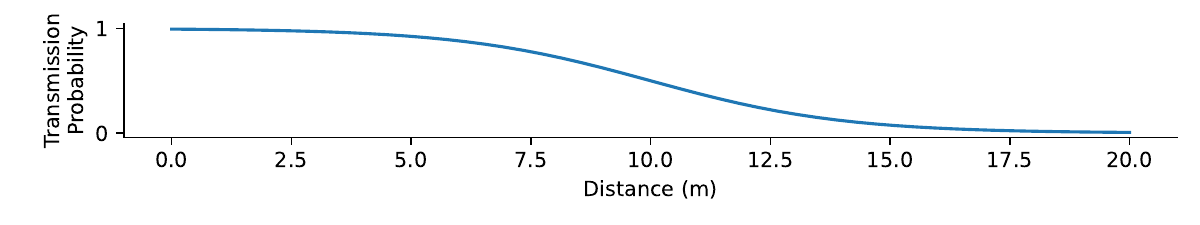}
  \caption{
  Visualization of the message transmission success probability based on inter-robot distance (\Cref{eq:comms-probability}).
}
  \label{fig:comms-probability}
\end{figure}

\section{Experimental Setup}
In our experiments, we evaluate performance on two different real-world datasets. 
The datasets
are collected using a hyperspectral camera mounted on a drone that flies over a freshwater lake. 
Each dataset is collected in the same area of the lake but on different days, leading to different algae distributions.
The robots are bounded in an area of approximately $80 \times 60$ meters.
We use the 400nm channel from the datasets as a proxy for algae and normalize the measured pixel intensity ($0-255$) to $[0, 1]$.
We test two sets of quantiles of interest in our experiments: quartiles $\quantiles = (0.25, 0.5, 0.75)$ and extrema $\quantiles = (0.9, 0.95, 0.99)$.
For each combination of parameters in all experiments, we run 2 seeds.
The available workspace is discretized into a 25x25 grid $\gtlocations$. The environment is unknown ahead of time to the robots.
We assume a communication model as described in \Cref{sssec:comms} with $\eta = 0.5$ and $r = 10$, and a visualization of \Cref{eq:comms-probability} with these parameters is shown in \Cref{fig:comms-probability}. 
We additionally add zero-mean Gaussian noise with $\sigma_\textrm{noise} = 0.05$ to each measurement taken by a robot.
In the POMCPOW solver, we set number of rollouts per step to 100,
rollout depth to 4, and the planner discount factor to 0.8. 
The lengthscale for the GP is 12. We set the drone altitude to 7m and each image
is 25 pixels (measurements).

In order to better understand under what conditions multirobot teams are beneficial for quantile estimation,
we select several parameters to vary in order to investigate their impact on
performance. 
These are listed in \Cref{tab:params}.
As a measure of performance, we report the root mean squared error (RMSE) between the ground
truth quantile values $\quantilevalues$ and the estimated quantile values
$\estimatedquantilevaluesfinal$ at the end of the robot surveys.

\begin{table}[t]
\caption{Parameters studied in our experiments.}
\centering
  \begin{tabular}{lc} 
    \textit{Parameter} & \textit{Tested values} \\ 
    \hline 
    Initial location spread ($\spread$) & 0.0, 0.33, 0.66, 1.0 \\ 
    Total budget ($B_T$) & 10, 15, 30 \\
    Budget type  & \textit{complete, shared} \\
    Communication type &  \textit{none, stochastic, full, partitioned} \\ 
  \end{tabular}
  \label{tab:params}
\end{table}

\subsection{Initial Location Spread}
\label{ssec:alpha-expt}
In this experiment, we vary $\spread = \{0.0, 0.33, 0.66, 1.0\}$.
Each variation is tested on team sizes $\nrobots = \{2, 4, 8\}$, and 
we keep the budget constant at $ B_T = 15$. 
We additionally compare each $\spread$ with no communication and with stochastic communication, but do not consider the single robot case, as $\spread$ and communication have no effect with only one robot.
To quantify the significance of the performance differences observed, we report results from the Wilcoxon signed-rank test \cite{wilcoxon1992individual}, which is a nonparametric version of the t-test and tests whether two paired samples originate from different distributions.

\subsection{Planning Budget}
\label{ssec:budget-expt}
Next, we investigate the effect of the planning budget on performance.
We vary $B_T = \{10, 15, 30\}$ in the complete budget case, setting $\spread$ to the constant value of $0.66$. 
We consider $\nrobots = 1$ as well as the multirobot teams.
In addition, we compare shared to complete budgets, setting $B_T = 15$.
In all cases, we enable stochastic communication between robots.

\subsection{Communication}
\label{ssec:comms-expt}
Finally, we compare performance with different $\nrobots$
across different levels of communication: full, stochastic, and none, and compare these to partitioning the space.
We hold previous parameters constant: $\spread = 0.66; B_T = 15$, and use $\nrobots = \{1, 2, 4, 8\}$.
Note that in the partitioned case, $\spread = 1.0$.


\section{Results}
We now present results from our experiments and include a discussion on their implications for real-world multirobot field work.
Boxplots show final RMSE between the estimated quantile values $\estimatedquantilevaluesfinal$ and the ground truth quantile values $\quantilevalues$ on the Y axis for an experiment, both aggregated across and refined by team sizes. 
Bars above indicate the Wilcoxon signed-rank significance levels for paired experiment groups \cite{florian_charlier_2022_7213391}.

\subsection{Initial Location Spread}
\label{ssec:alpha-results}
\begin{figure}[t!]
  \centering
 \includegraphics[width=\boxwidth,height=\boxheight]{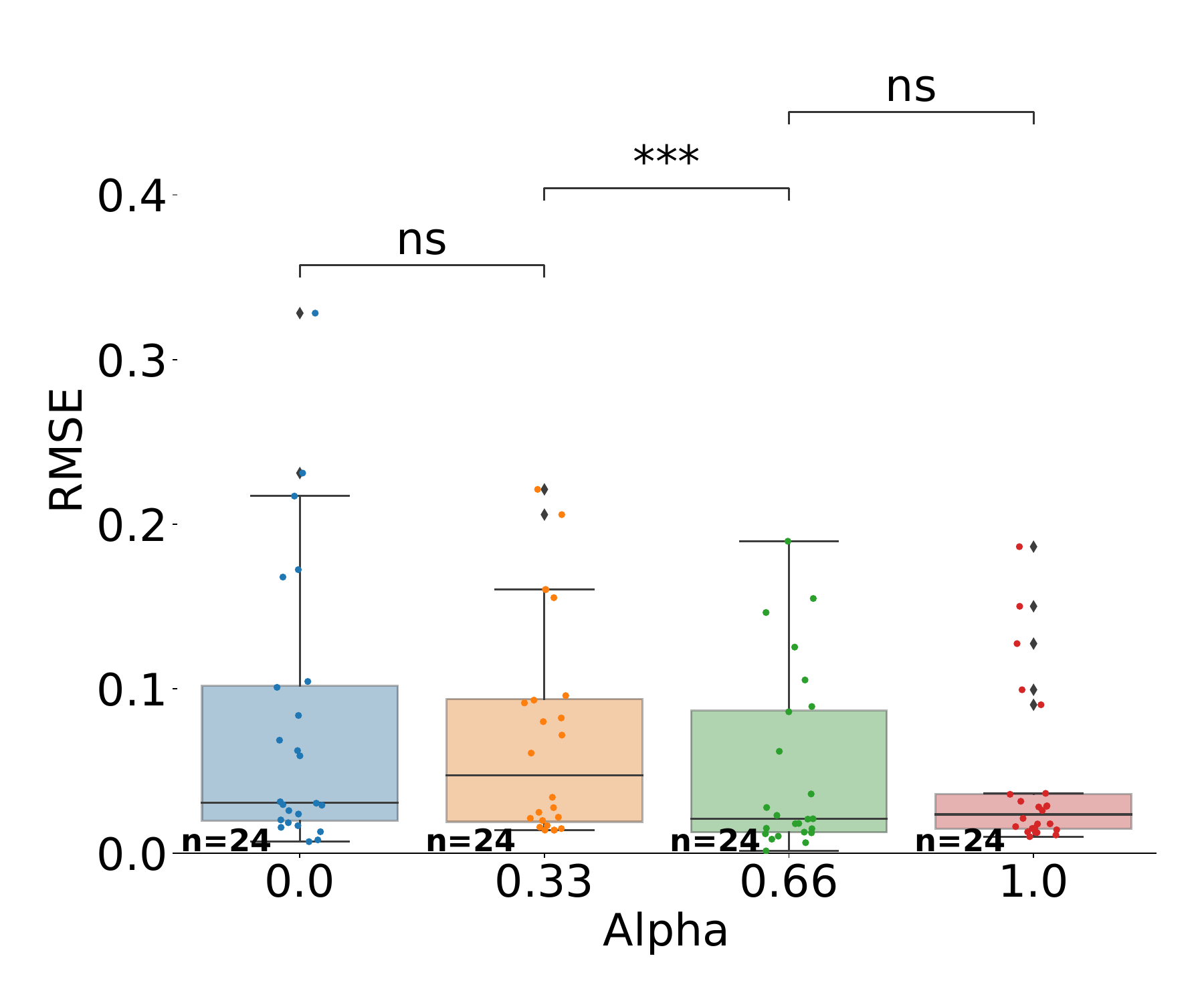}
   %
 \includegraphics[width=\boxwidth,height=\boxheight]{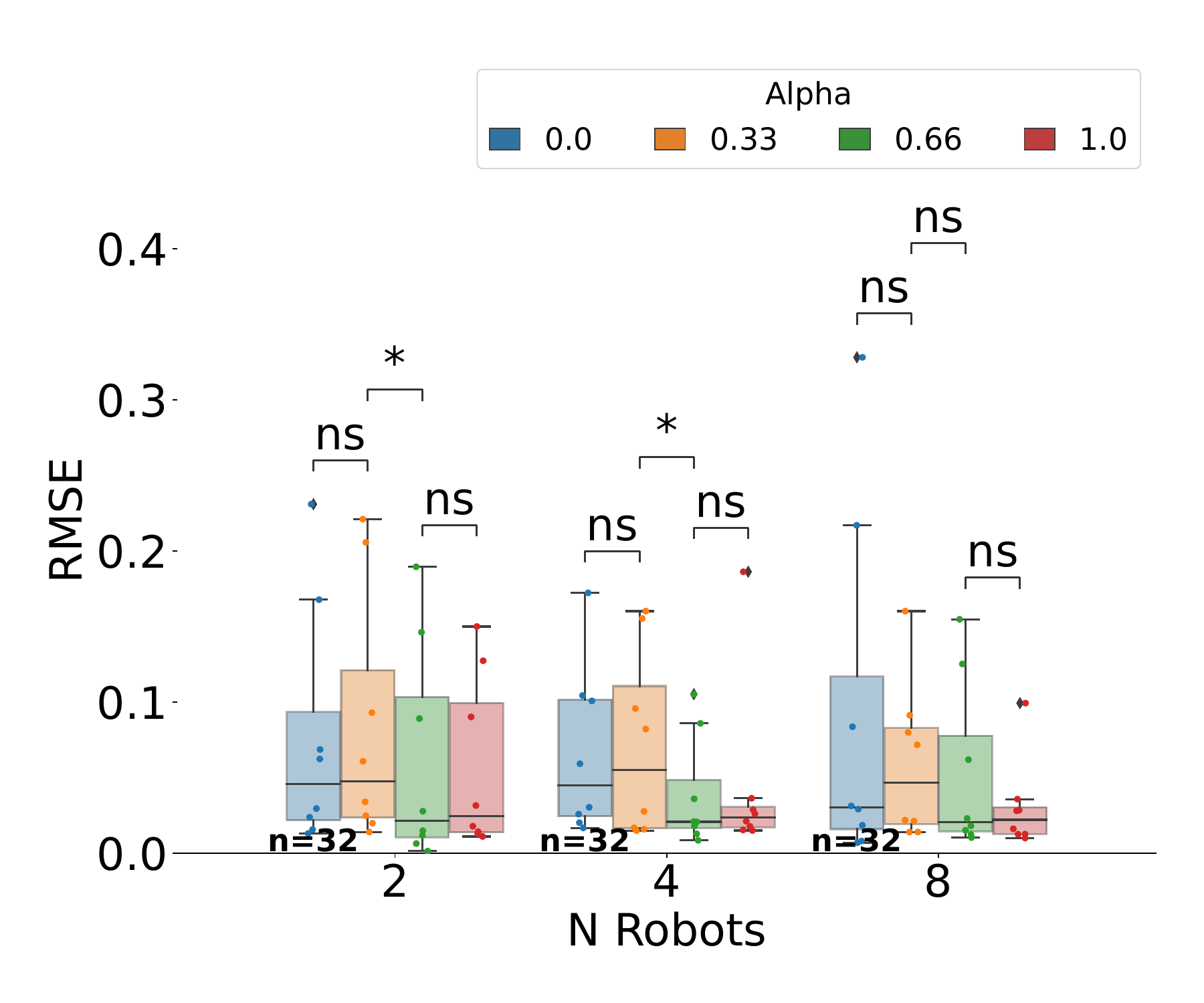}
  %
 \includegraphics[width=\boxwidth,height=\boxheight]{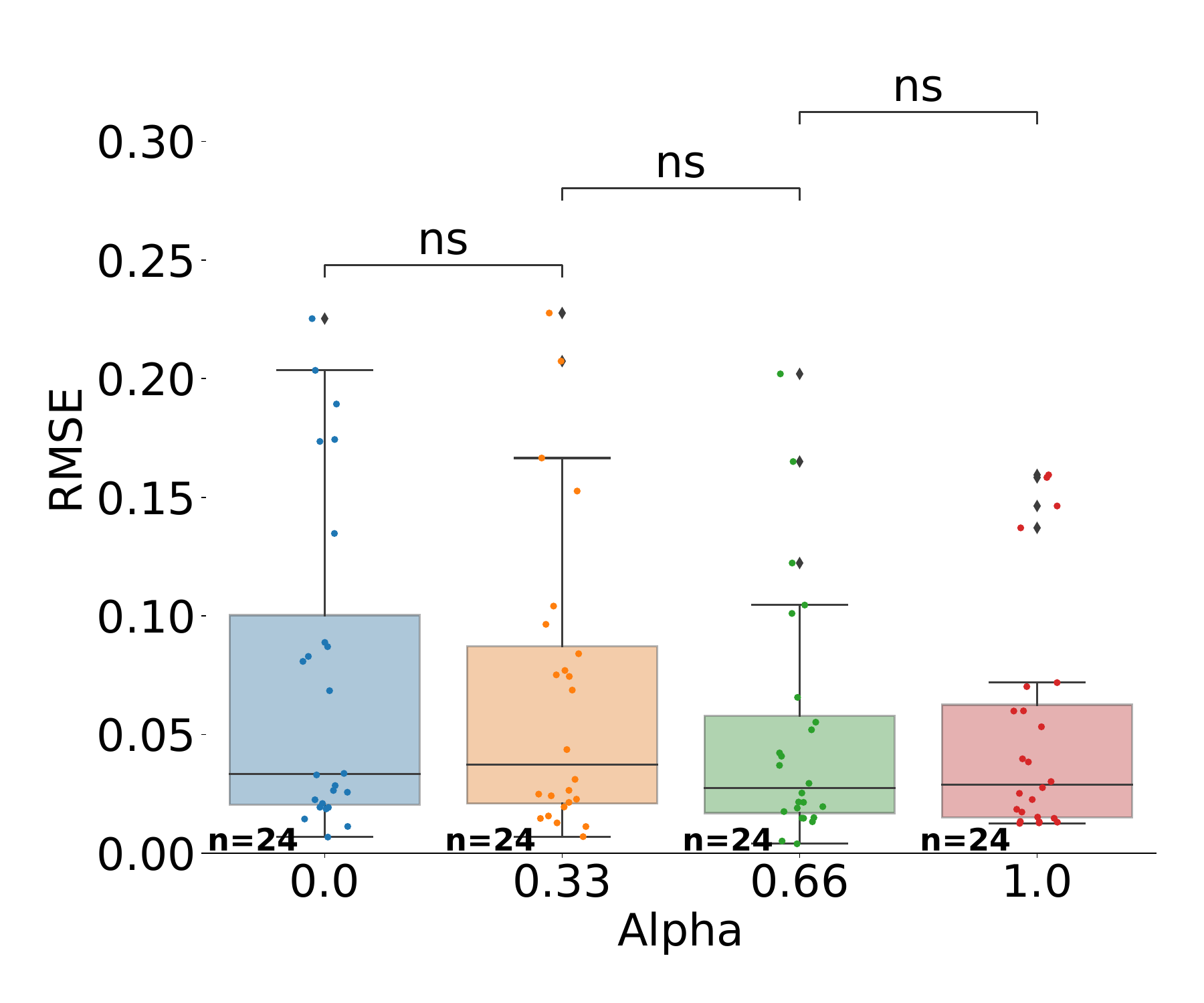}
  %
 \includegraphics[width=\boxwidth,height=\boxheight]{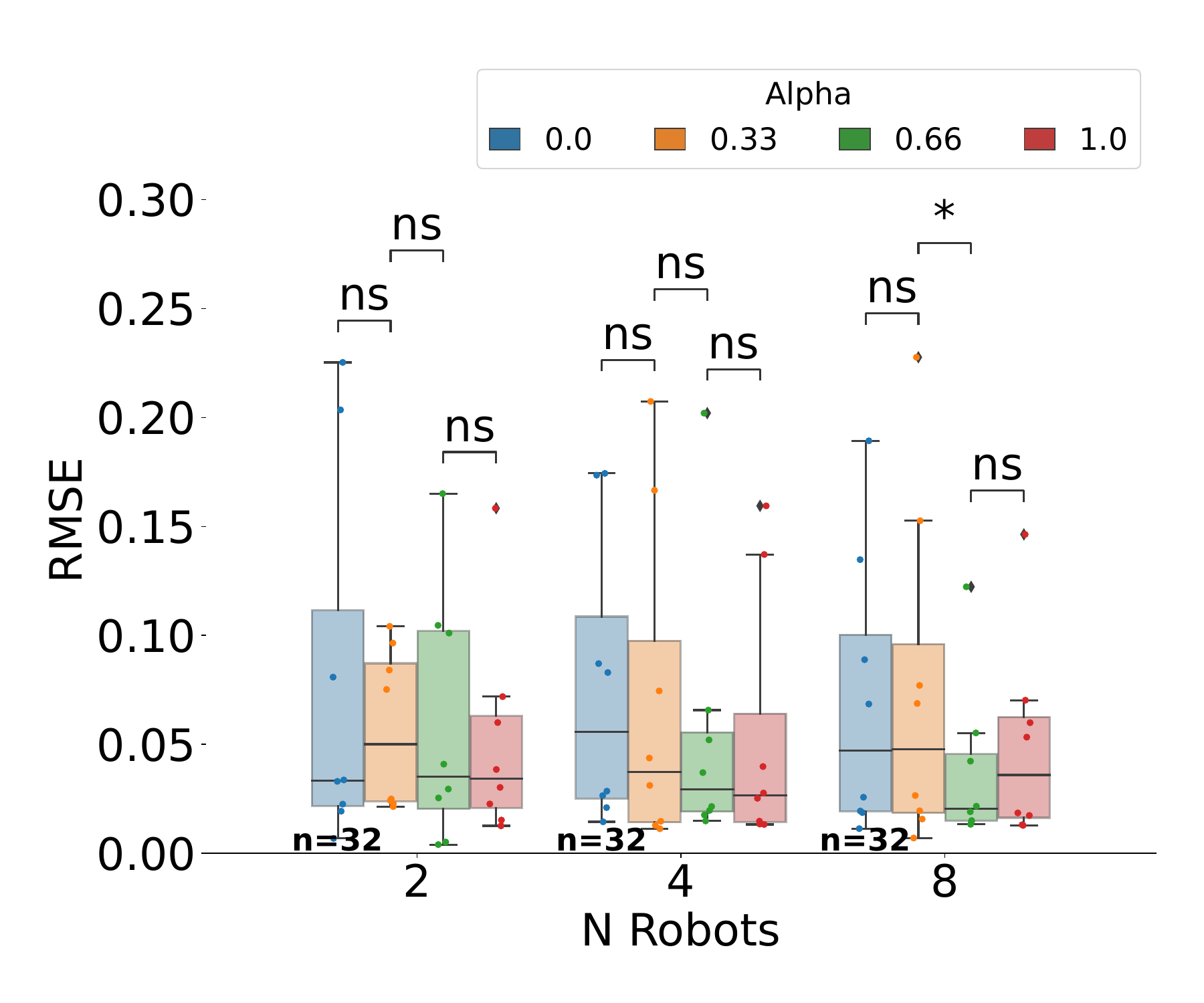}
  \caption{
  Error vs. Initial location spread ($\spread$).
  Top: No communication.
  Bottom: stochastic communication. 
  Right column shows results further separated by team size. 
  Bars above indicate significance under the one-sided Wilcoxon signed-rank test; ***: $\pppsig$, *: $\psig$, ns: no significance.
}
  \label{fig:alpha}
\end{figure}

\Cref{fig:alpha} shows the results of the experiments on $\spread$ using a one-sided Wilcoxon signed-rank test.
We observe that, in general, performance tends to improve (error decreases) with increasing $\spread$.
In particular, in the absence of stochastic communication as seen on the top row of \Cref{fig:alpha}, 
we see a statistically signficant performance improvement between $\spread = 0.33$ and $0.66$ ($\wpres{266}{0.001}$).
On the top right, further separated by $\nrobots$, we similarly observe 
a statistically significant improvement in performance between $\spread=0.33$ and $\spread=0.66$ for both $\nrobots=2$ and $\nrobots=4$ ($\wpres{33}{0.05}$, $\wpres{32}{0.05}$, respectively), 
but we do not see such an improvement between $\spread=0.66$ and $\spread=1.0$ in those cases.
In general, we also see that $\spread$ has a more drastic effect on performance the more robots there are, indicated by the decreasing maximum errors as $\nrobots$ increases.
However, in general, there appear to be diminishing returns with bigger $\spread$.

The bottom row of \Cref{fig:alpha} shows results when robots communicate stochastically with each other at different initial location spreads.
Although we see the same general trend on left, in this case, we do not see a statistically significant drop in error. 
When controlled for team size on the right, we again see increasing $\spread$ resulting in improved performance 
as indicated by lower median and maximum errors.
Compared to the case where there is no communication, we do not notice the same dependency on $\nrobots$ for the effect of $\spread$; in other words, the decrease in error due to $\spread$ appears relatively similar across different values of $\nrobots$ (shown by relative median and maximum errors).
We additionally note that for $\nrobots=8$, there is a statistically significant drop in error from $\spread = 0.33$ to $0.66$ ($\wpres{33}{0.05}$).
Based on these results, 
if robots must be deployed near each other but cannot communicate, then a small robot team may do just as well as a larger team, and vice versa; a small team that cannot communicate does not need to be spread very far to be effective.
On the other hand, a larger spread combined with a large team size will be beneficial in the absence of inter-robot communication, but there will be marginal returns with larger spread.
If communication is enabled, a larger spread will generally always be beneficial (although again with marginal returns), regardless of the team size.

\subsection{Planning Budget}
\begin{figure}
  \centering
 \includegraphics[width=\boxwidth,height=\boxheight]{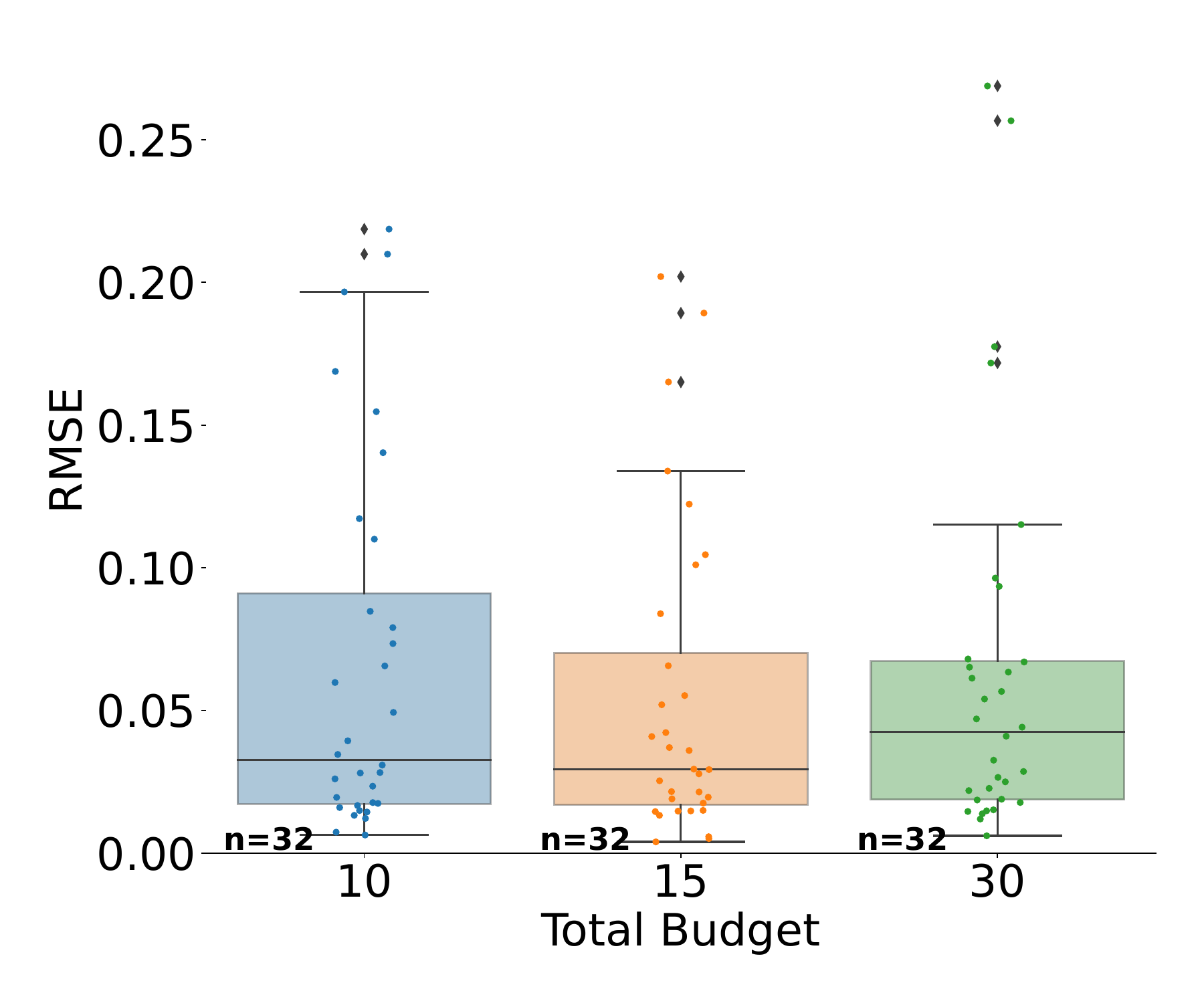}
  %
 \includegraphics[width=\boxwidth,height=\boxheight]{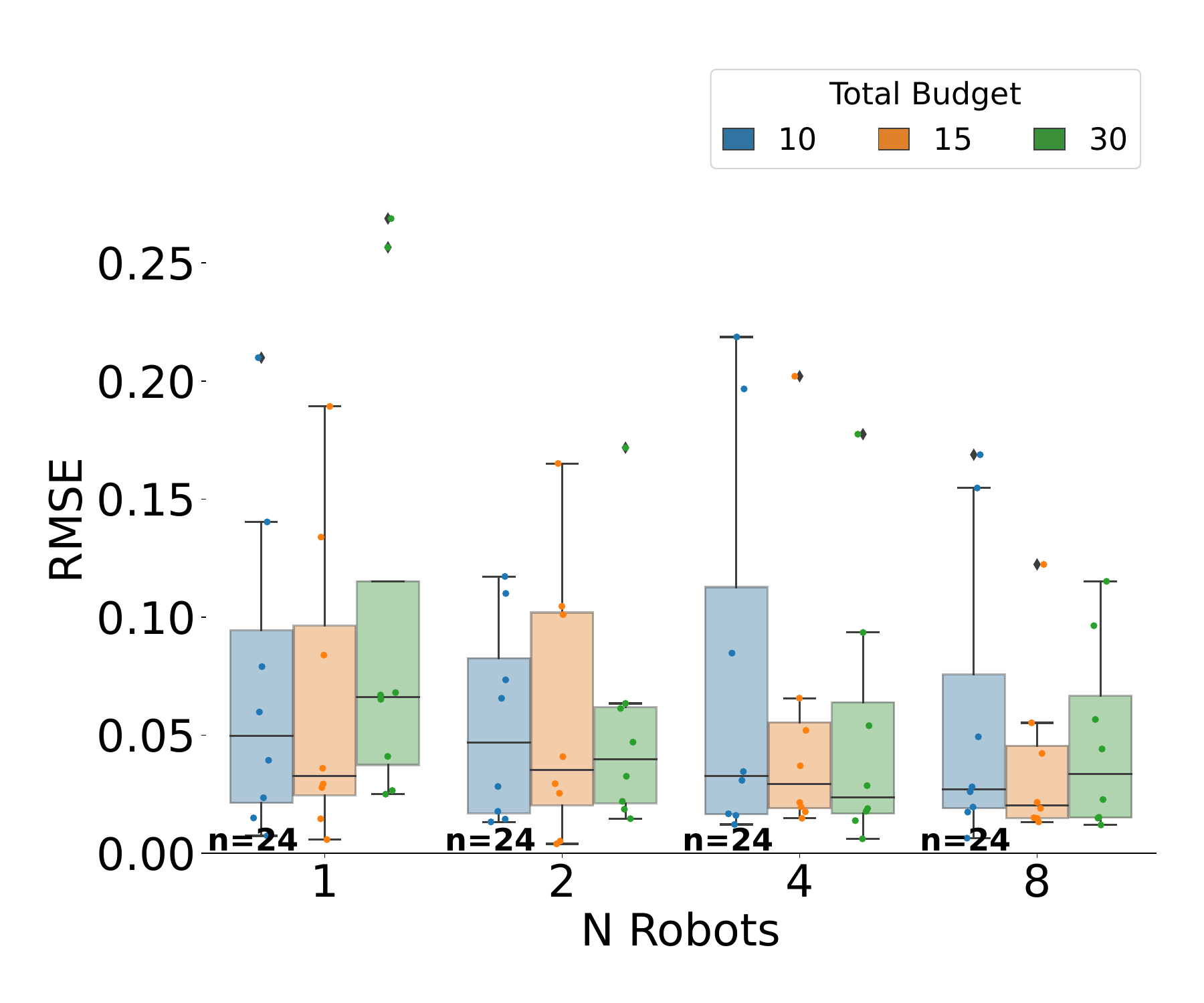}
  \caption{
  Error vs. Budget (number of planning steps $B_T$).
  Right shows results further separated by team size. 
  }
  \label{fig:budget}
\end{figure}

\begin{figure}
  \centering
 \includegraphics[width=\boxwidth,height=\boxheight]{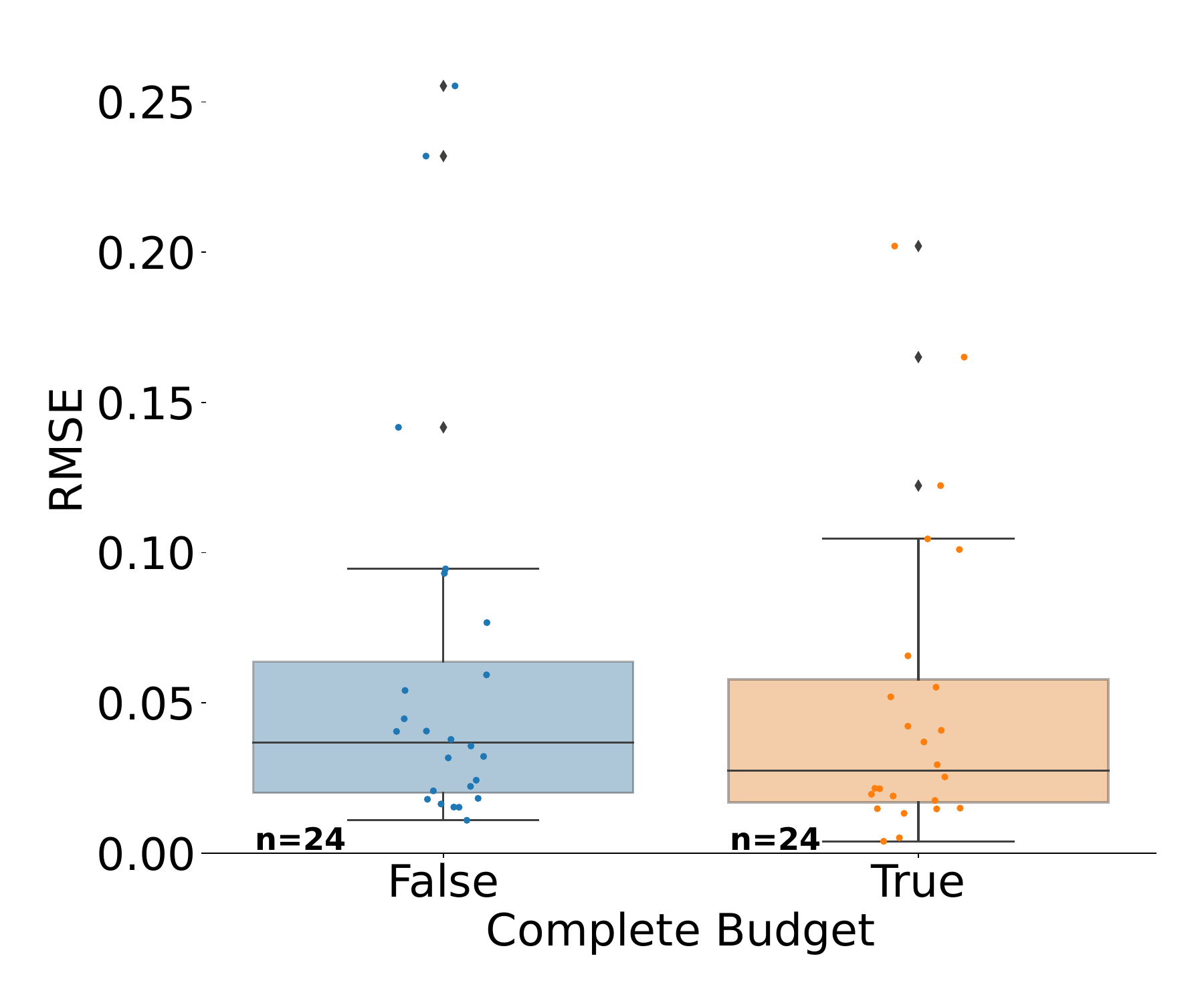}
 \includegraphics[width=\boxwidth,height=\boxheight]{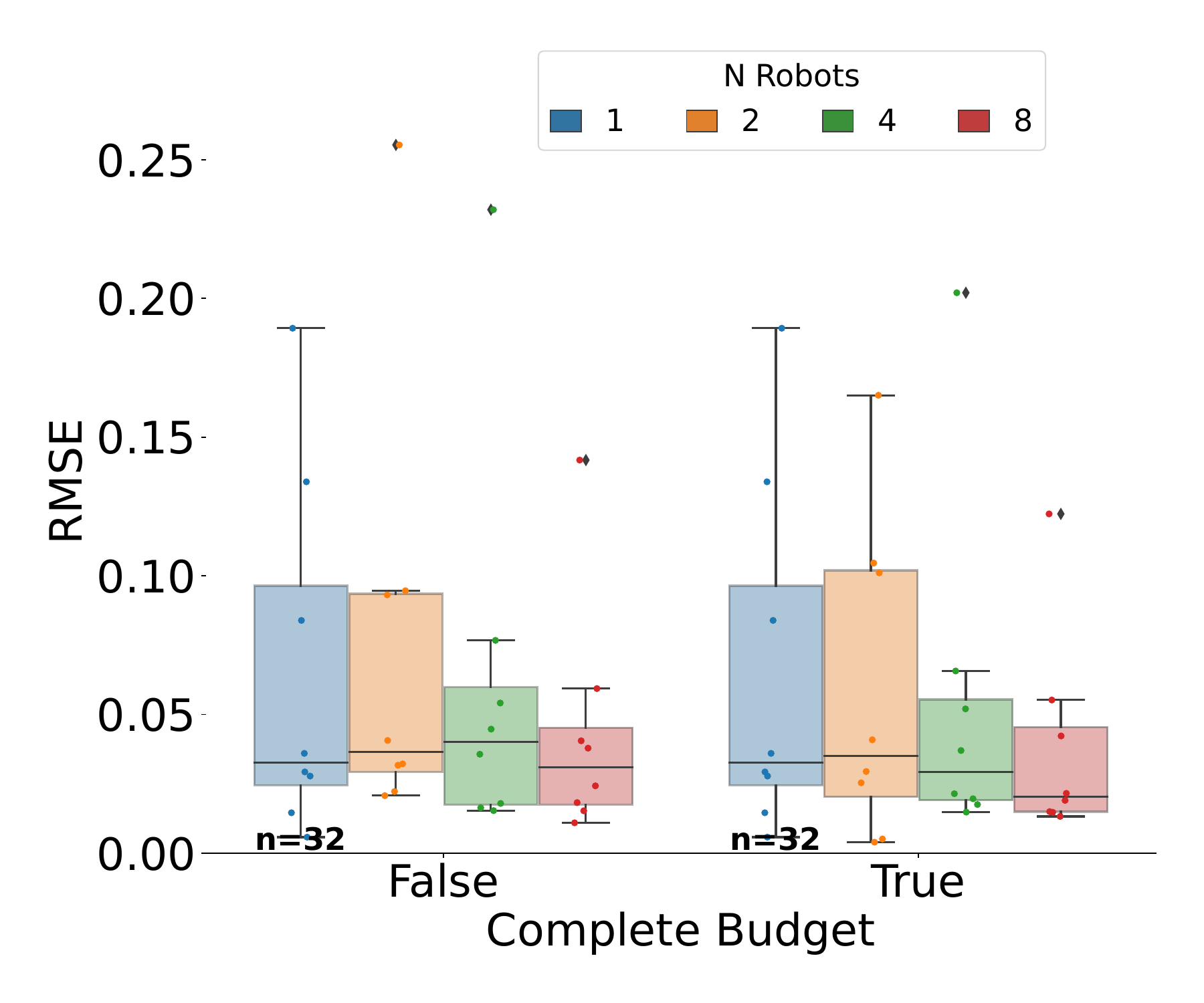}
  \caption{Error vs. Budget Type (Shared and Complete).
  Here, $B_T=15$.
  Left: Shared budget means the combined budget is $B_T$ for any $\nrobots$, while complete budget means it is $\nrobots \cdot B_T$. Note $\nrobots=1$ data not included here since budget type has no effect.
  Right: Results further separated by team size. 
  $\nrobots=1$ (blue) shown for comparison.
  }
  \label{fig:complete-budget}
\end{figure}

Observing both the median and maximum errors on the left side of \Cref{fig:budget} suggests that in general,
a larger budget results in better performance. 
This is expected, as with a larger budgets, robots have more time to explore the environment and gather information, 
and we see a similar trend on the right when controlled for team size.
The improvements are not entirely consistent, however, as we see that $B_T = 30$ is not always better than $B_T = 15$.
This is somewhat surprising, as we would expect improvement with more steps given the results on the left.
What we observe on the right, rather, is that the performance improvement comes from an increase in $\nrobots$, indicated by lower median error for higher $\nrobots$.

To investigate this further, we control for total path length $\sum_n B_n$ and show results for shared and complete budgets in \Cref{fig:complete-budget}.
Recall that with a shared budget, $\forall n \in \nrobots : B_n = B_T / \nrobots$,
and for a complete budget, $\forall n \in \nrobots : B_n = B_T$
(see \Cref{ssec:budget-expt}).
With $B_T = 15$ and $\spread = 0.66$ in this experiment, we see similar performance between shared and complete budgets, which is again counterintuitive, as we would generally expect improvement given more steps.
However, we uncover the real effect when we further control for $\nrobots$, shown on the right. 
With a complete budget, more robots (i.e., more steps) leads to better performance as well, as expected. 
However, even when the budget is shared, there is a notable improvement in maximum error when using more than $1$ robot.
What is notable is that, although we see similar median errors for all cases, the improvement in maximum error is also observed when the total path length is held constant.
This is particularly remarkable because in the $\nrobots = 8$ case, this equates to each robot taking at most 2 steps.
Given these results, it appears that having more robots, even with restricted budgets (e.g., time, path length), is preferable to having one robot with a larger budget.

\subsection{Team Size and Communication}
\begin{figure}
  \centering
  \includegraphics[width=\boxwidth,height=\boxheight]{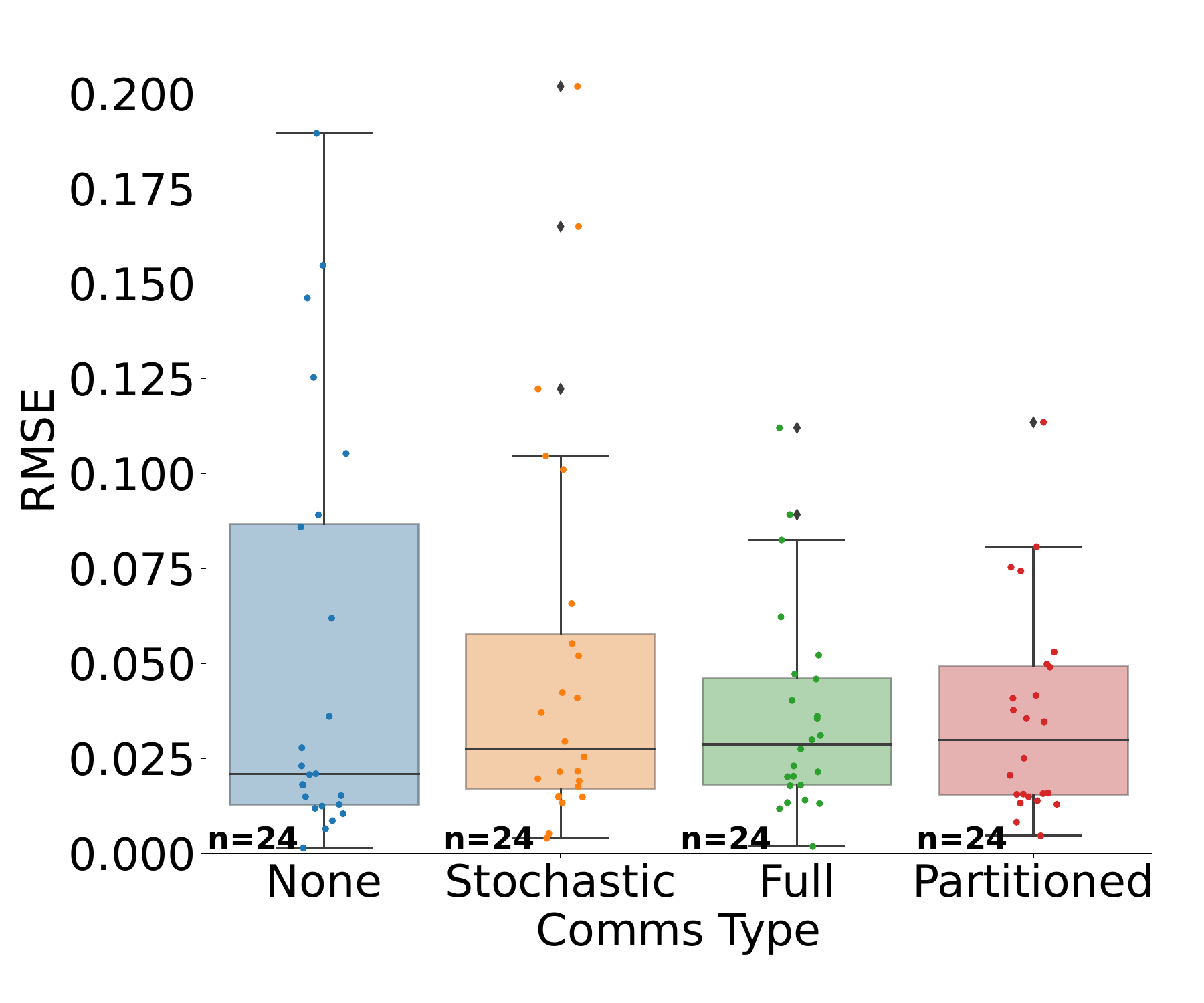}
  \includegraphics[width=\boxwidth,height=\boxheight]{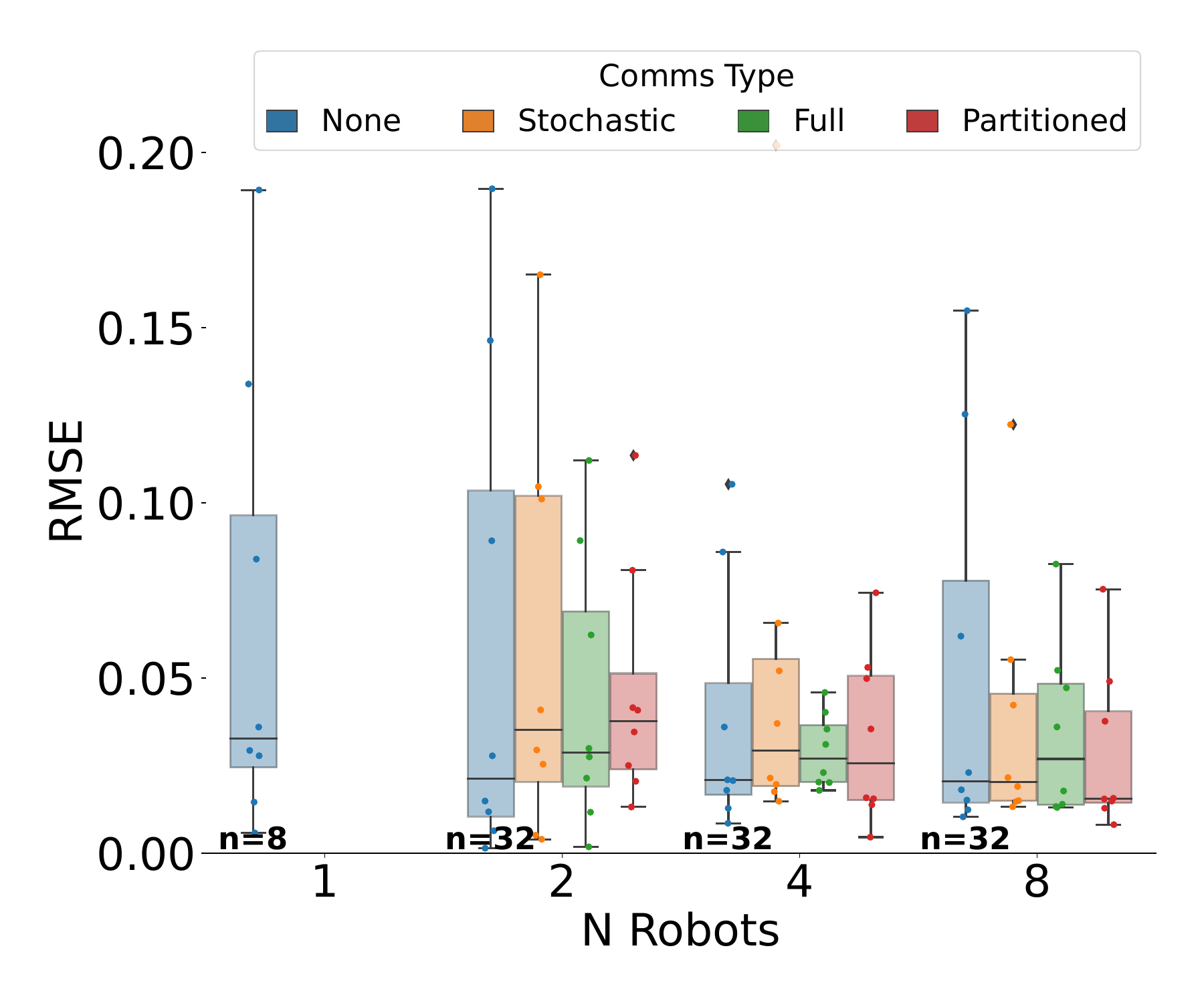} 
  \caption{
  Error vs. Communication type.
  Left: Aggregated across different $\nrobots$.
  Note that results for $\nrobots = 1$ are not included, since they are not affected by the variations.
  Right: Same results, further separated by team size. $\nrobots=1$ included for comparison.
  }
  \label{fig:multirobot-error-vs-comms}
\end{figure}

\Cref{fig:multirobot-error-vs-comms} shows the quantile estimation error results when the robot team uses different communication styles. 
Note that on the right, we include the single robot version as a baseline.

On the left, we see that overall, 
all communication types have a median performance around the same as no communication.
We see that more communication does, however, lead to improved performance, indicated by smaller upper quartile errors. 
This is likely due to the fact the robots can avoid multiple measurements of locations that are not of interest.

On the right, we see that stochastic communication does not have much of an effect compared to no communication when $\nrobots= 2$ or $4$, as seen in the similar median and upper quartile errors.
At $\nrobots = 8$, we see that although the median performance between no, stochastic, and full communication is similar,  both the upper quartile error and the maximum error for no communication is much higher. 
This suggests that enabling communication is more beneficial for larger teams for reducing error and improving performance.
We also see that in general, partitioning achieves similar median error as when there is stochastic or full communication, as well as maximum error particularly for larger teams. 
Given that in this experiment, the partitioned case used $\spread = 1.0$ rather than $0.66$, this suggests that partitioning the space is an effective strategy if the entire area is accessible for deploying robots but communication is not possible.
Based on these results, performance improvement can be achieved by enabling inter-robot communication, even when the communication is imperfect, and it is particularly beneficial with large teams.

Finally, we show some example plans produced during these experiments as a visualization of the task in \Cref{fig:raster-increase-budget,,fig:raster-shared-budget,,fig:raster-partition}.
The robots are shown in different colors and the paths they produce are shown as connected crosses.
The initial locations are marked in black. If the robots were partitioned, the assigned regions are designated with white lines.
The background of the figures is the orthomosaic of the hyperspectral images
previously collected by drones in real aquatic environments.

\begin{figure}
  \centering
  \includegraphics[trim={0.5cm 3cm 0.5cm 3.5cm}, clip, width=\boxwidth]{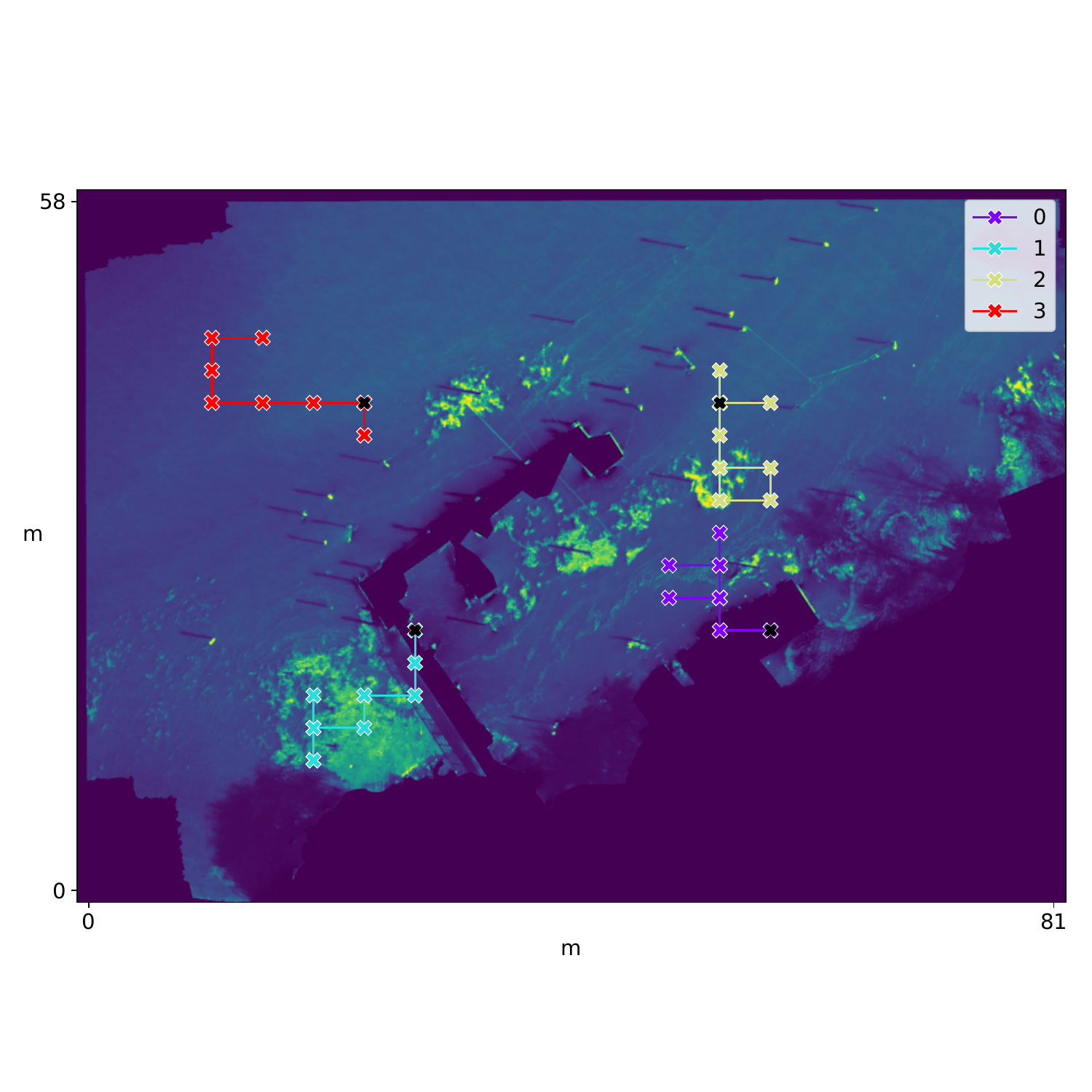}
  \includegraphics[trim={0.5cm 3cm 0.5cm 3.5cm}, clip, width=\boxwidth]{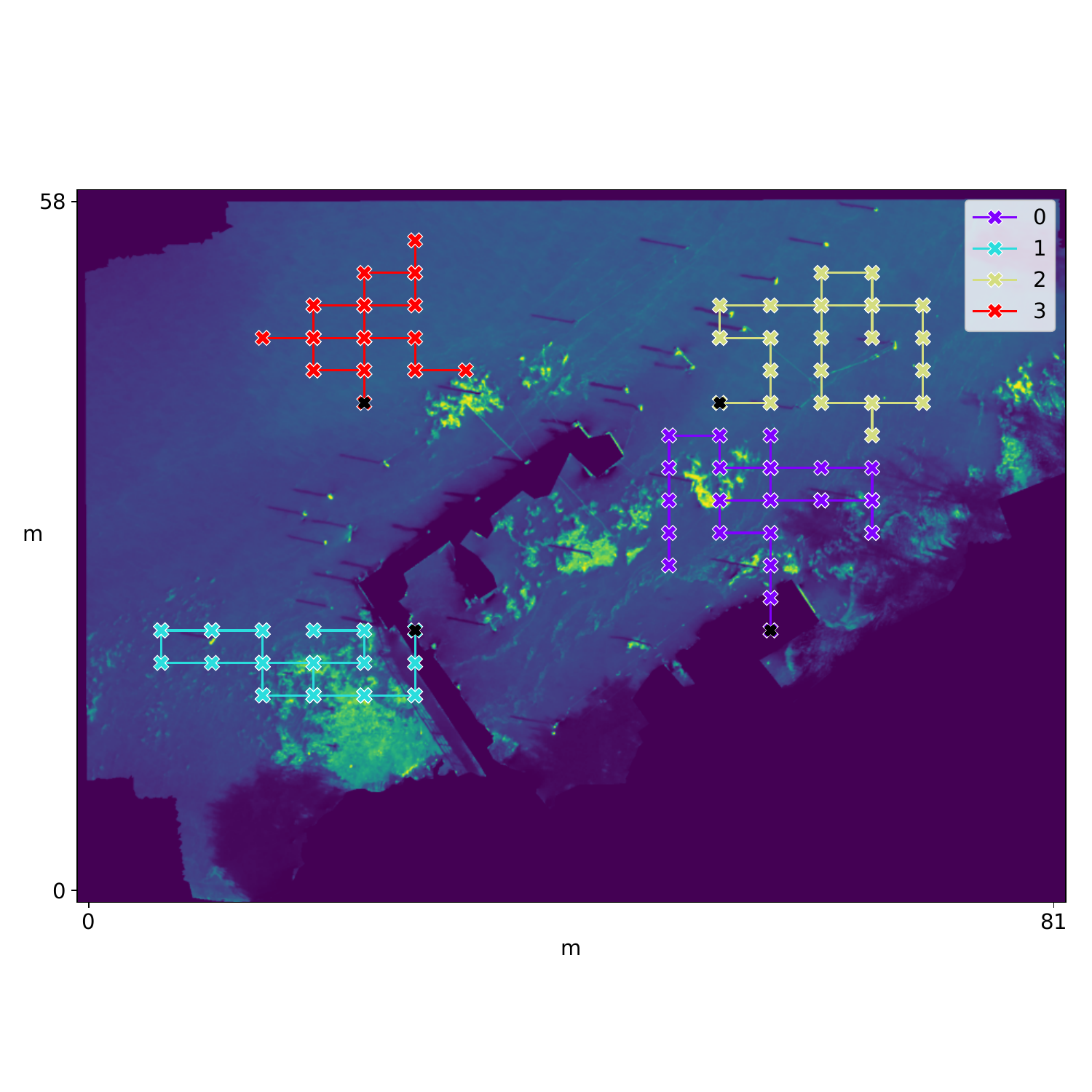}
  \caption{
  Example paths of a four-robot team with $B_T=10$ (left) and $B_T=30$ (right).
  }
  \label{fig:raster-increase-budget}
\end{figure}

\begin{figure}
  \centering
  \includegraphics[trim={0.5cm 3cm 0.5cm 3.5cm}, clip, width=\boxwidth]{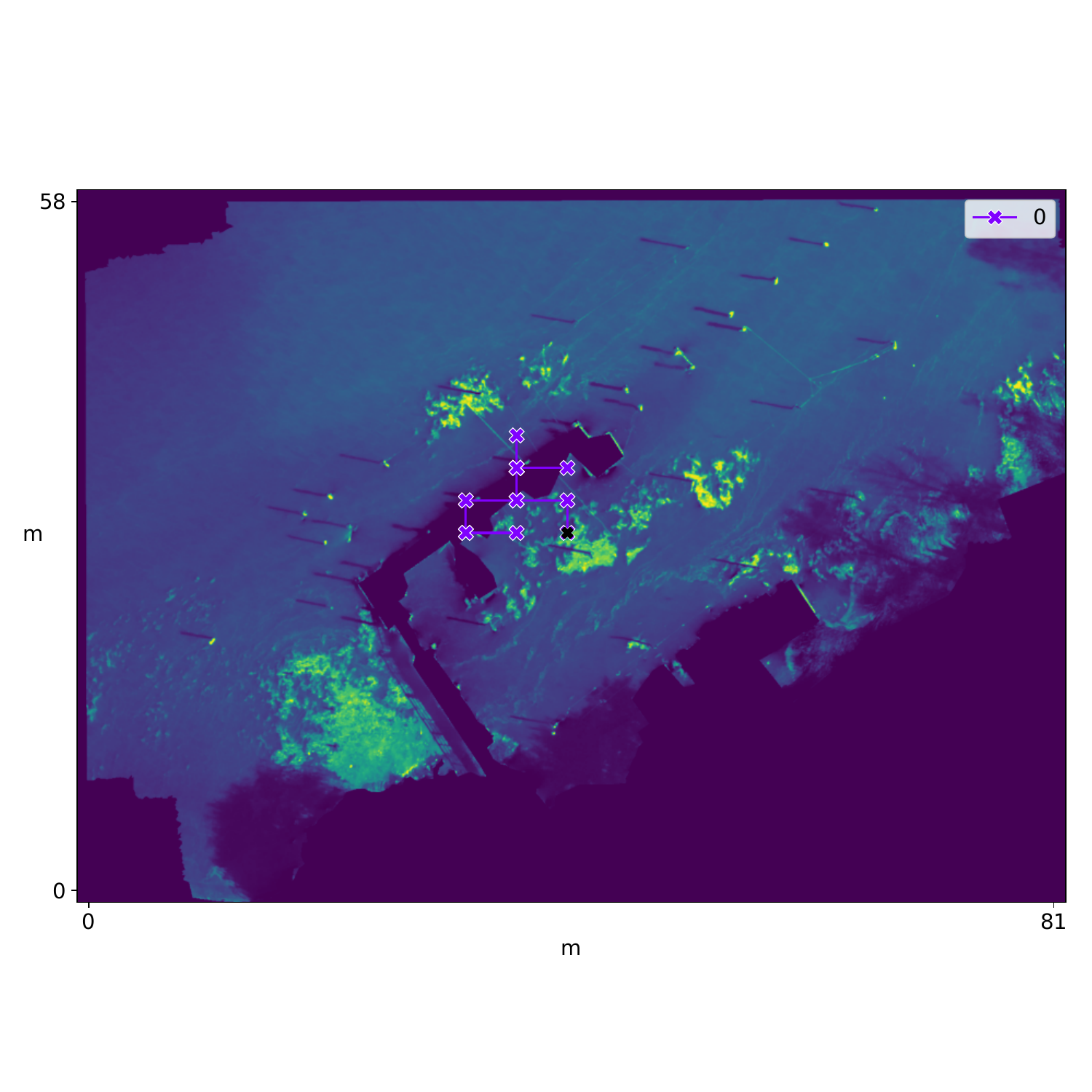}
  \includegraphics[trim={0.5cm 3cm 0.5cm 3.5cm}, clip, width=\boxwidth]{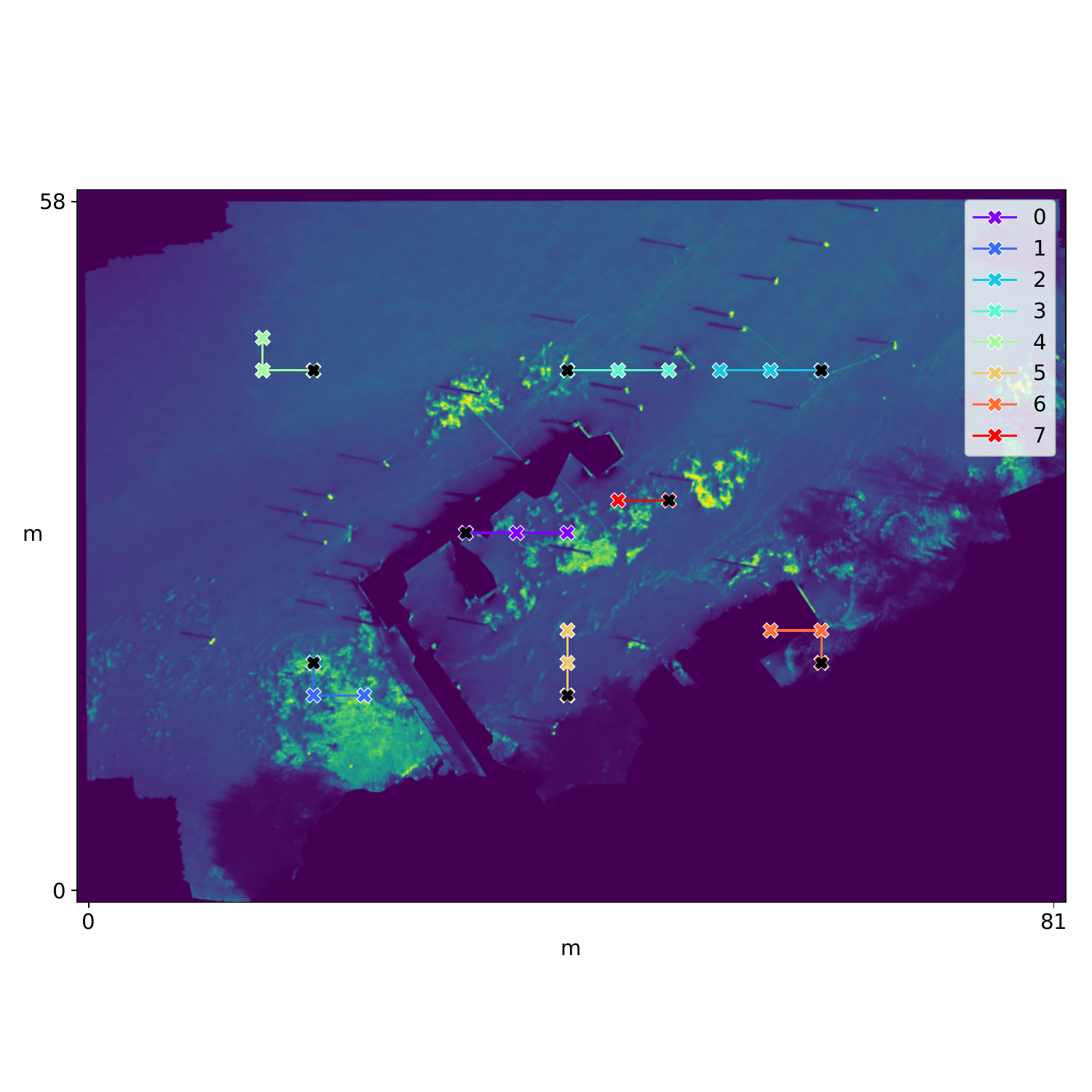}
  \caption{
  Example paths of a single robot with $B_T=15$ (left) and an 8-robot team with a shared budget of $B_T=15$ (right).
  }
  \label{fig:raster-shared-budget}
\end{figure}

\begin{figure}
  \centering
  \includegraphics[trim={0.5cm 3cm 0.5cm 3.5cm}, clip, width=\boxwidth]{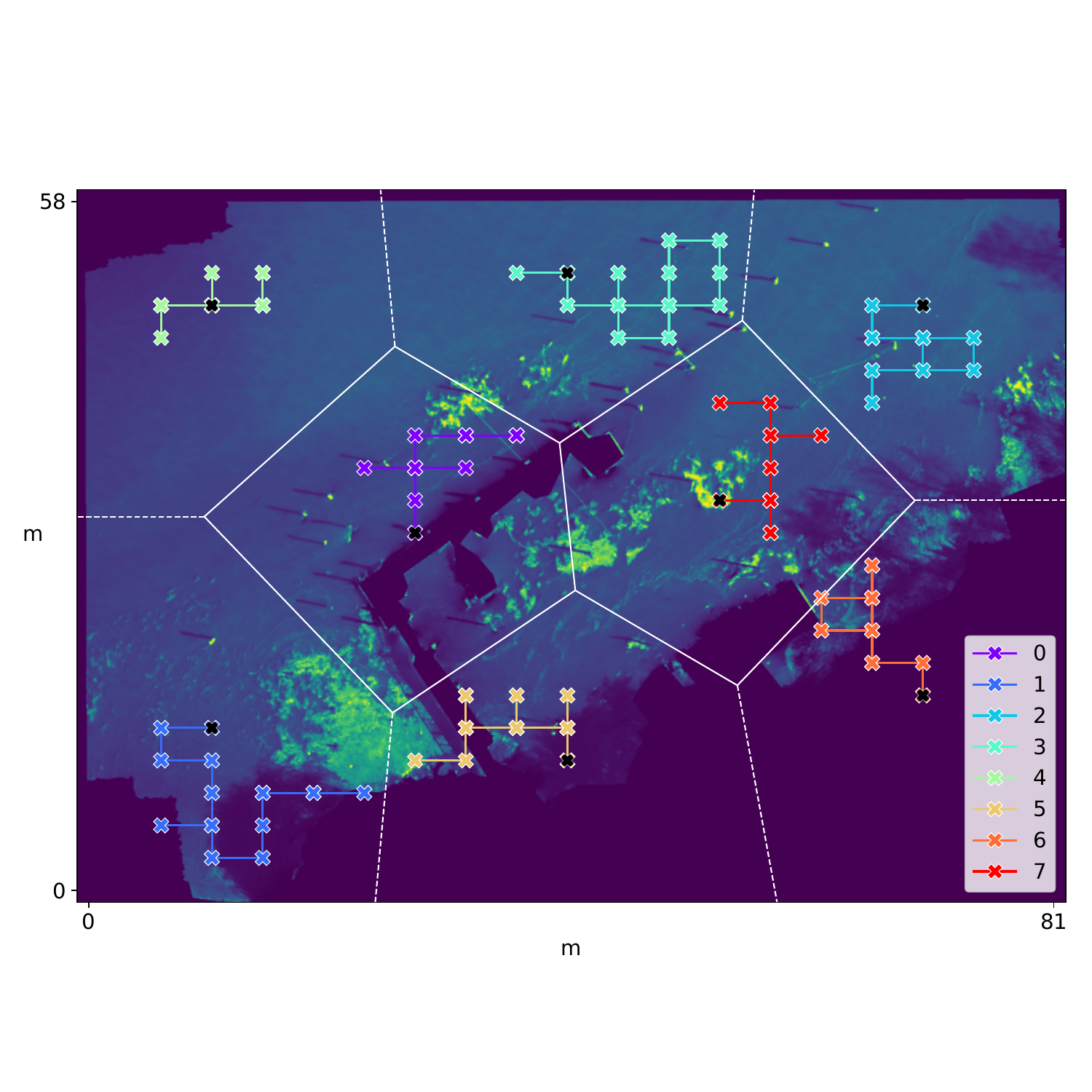}
  \includegraphics[trim={0.5cm 3cm 0.5cm 3.5cm}, clip, width=\boxwidth]{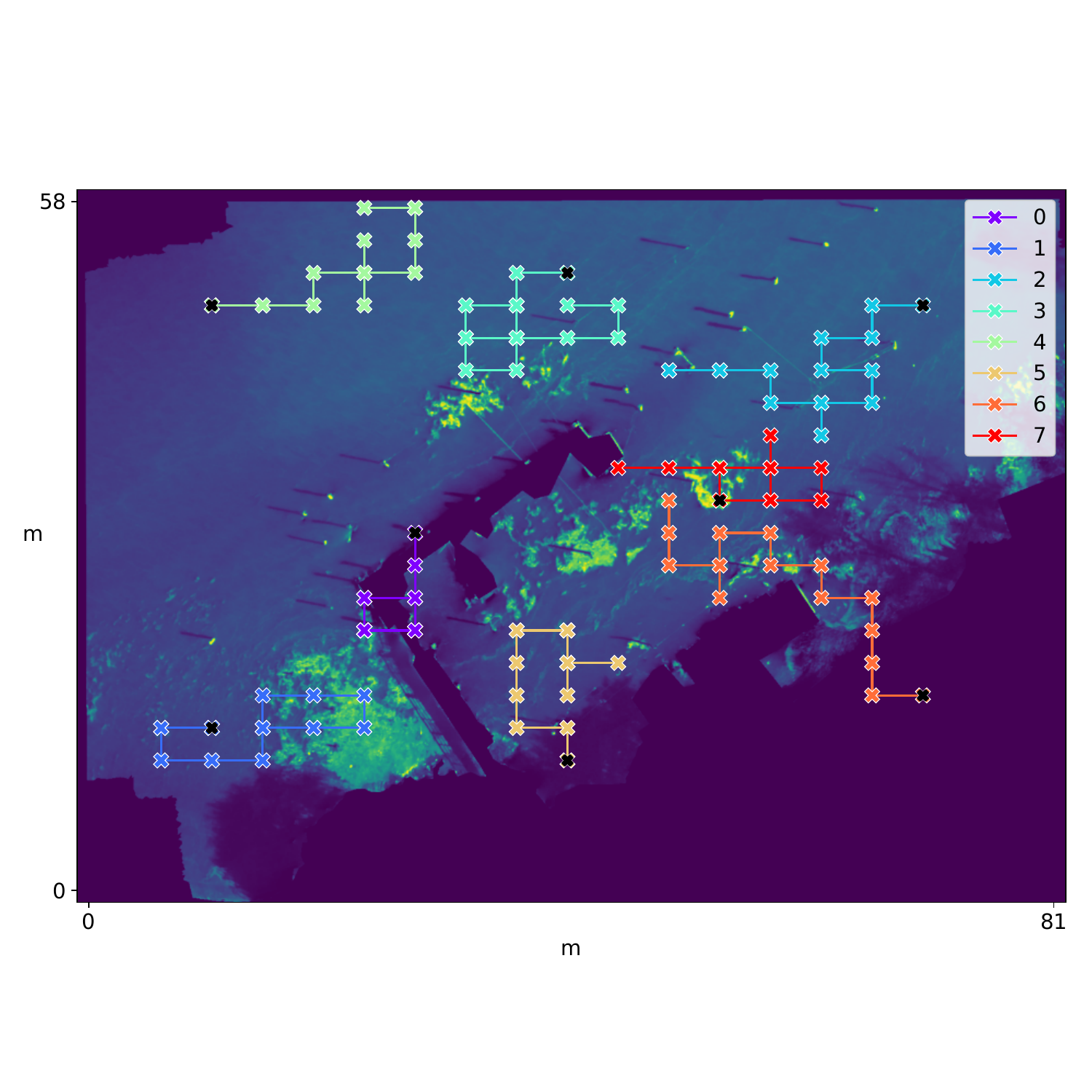}
  \caption{
  Example paths of an 8-robot team under partitions (left) and no communication (right).
  }
  \label{fig:raster-partition}
\end{figure}

\section{Discussion}
Here we summarize the main conclusions drawn from the presented results and discuss further areas of interest based on them.

With regards to the initial location spread, a large spread appears to be beneficial when $\nrobots$ is large, but not when $\nrobots$ is small.
Additionally, past a certain point between $0.33$ and $0.66$, a larger spread results in marginal improvement in error.
This may be particularly valuable information when planning real-world field deployments, when there are time and monetary costs to getting many robots to spatially diverse areas; avoiding unnecessary costs is crucial.
In addition, if $\spread$ is restricted to very low values by the survey resources or setup, then using more robots may help regain some performance.

Notably, in terms of budget $B_T$, a larger budget in and of itself does not significantly improve performance for a fixed $\nrobots$ (assuming a moderate initial spread). What we observe instead is a stronger dependency on $\nrobots$ itself; therefore, a larger number of robots, each with a restricted budget, will likely produce better quantile estimates than a single robot with a $B_T$ equal to the total budget of all of them.

Our results also indicate that enabling communication is generally beneficial to the quantile estimation problem, and in particular for larger groups. 
We note that partitioning the space into disjoint areas and restricting each robot to one area can be an effective alternative strategy,
particularly if inter-robot communication is not an option.

A natural question that arises from these results is to determine what role communication plays in this task, and we address this in \Cref{ch:comms}.
For the purposes of this study, we restricted the communication frameworks to somewhat basic implementations; results may change given a more intelligent communication protocol that takes into account the potential utility of a message, or with more nuanced communication models based on particular hardware capabilities.

It would also be interesting to investigate how this scales to team sizes beyond 8 robots.
In the context of lake monitoring for algae and other environmental monitoring tasks, swarms are not typically deployed;
however, for large-scale studies, it may be useful to understand the impact of teams larger than 8.
Based on our results in this study, we would expect to see lower error with more robots up to a certain point. 
Our results support what previous work has found in terms of diminishing returns with more robots.
We generally observe a larger decrease in error between $\nrobots = 2$ and $4$ than between $4$ and $8$.
Extrapolating from this, doubling $\nrobots$ from $8$ to $16$ robots or more would likely not be able to achieve 0 error, but would be significantly more resource-intensive.

In practice, scientist teams will need to weigh the benefit of lower estimation error with the practicality of acquiring, coordinating, and deploying more robots for their specific application.

\section{Conclusion}
In this work, we have presented the first study on multirobot quantile estimation
for environmental analysis. We investigate the impact of multiple robots on
quantile estimation accuracy, as well as the effect of initial location spread, planning budget, and communication.
We measure how well different combinations of parameters perform in terms of error in quantile value estimates, 
and we further provide statistical results quantifying the significance of the differences observed.
This work is an important step toward characterizing the elements of an effective
multirobot system for environmental analysis and has potential to help
scientists interested in larger scale or collaborative survey projects to better
understand the benefits and drawbacks of different experiment design choices. In
turn, this may improve environmental monitoring and conservation
efforts.


\chapter{Reducing Network Load via Message Utility Estimation for Decentralized Multirobot Teams}
\label{ch:comms}
Our final contribution follows from the analysis of results presented in \Cref{ch:multirobot}.
With multirobot teams in an environmental exploration task, efficient communication can be crucial to successful task completion.
The work described in this chapter considers the problem of abstracting task specification from the angle of autonomously selecting which messages to send to a bandwidth-constrained robot communication network, and follows from \cite{rayas2023reducing}. 

\section{Introduction}

\label{sec:intro}
Quantile estimation is useful for gaining an understanding of natural environments~\cite{rayas2022informative}, and the use of multirobot teams has shown promise for effectively planning exploration and accurately estimating the quantile values~\cite{rayas2023study}.
However, realistic multirobot deployments must consider limitations in communication capabilities; here, we focus on reducing the network load as measured by the number of attempted transmitted messages with the goal of sacrificing minimal quantile estimation accuracy. 
We address the question of how to decide whether a message is worth sending.
Specifically, we:
\begin{itemize}
    \item Propose methods to reduce network load for bandwidth- and resource-constrained decentralized multirobot teams that leverage utility-based message evaluation, with no guarantees on information from teammates;
    \item Describe a decision-theoretic approach to choosing whether to transmit a message to the network based on the utility-based evaluation and a cost of transmission;
    \item Present results showing significant reductions in network load are achievable with minor or no corresponding drops in performance on the problem task of estimating quantile values, even with probabilistic transmission failure and a distance-based transmission probability, on a dataset from a real aquatic environment;
    \item Discuss results in the context of real-world applications.
\end{itemize}

\begin{figure}[t]
    \centering
    \includegraphics[width=\columnwidth]{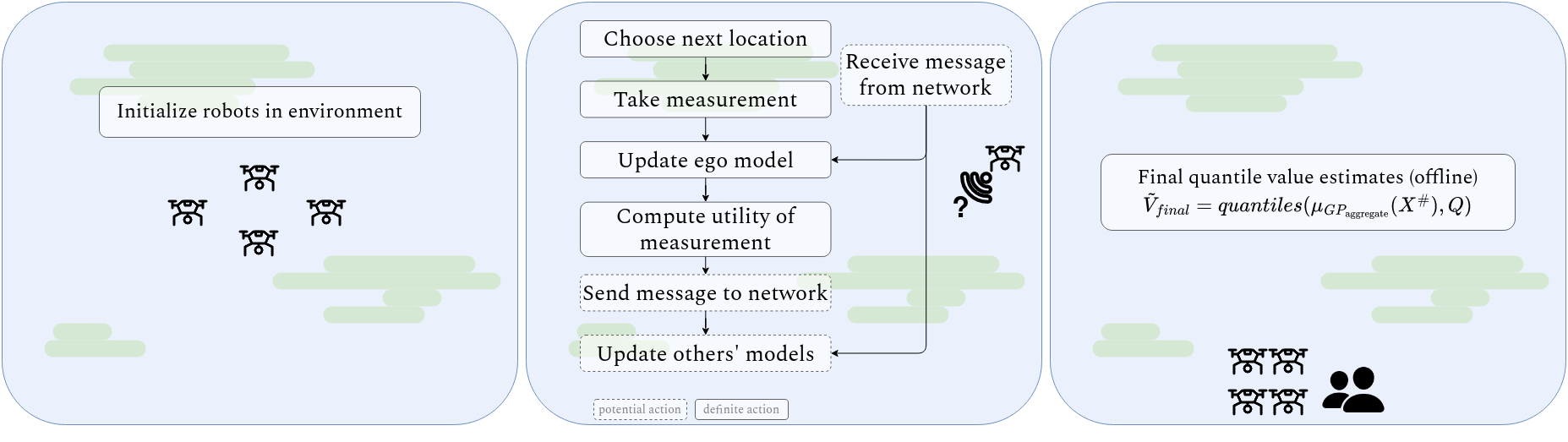}
    \caption{Overview of problem pipeline. The approach consists of 3 phases: Initialization, Exploration, and Aggregation. 
    The bulk of the work happens during the decentralized exploration phase, where robots repeatedly choose the next location to visit, measure the environment, optionally broadcast the measurement in a message, and update their models until their budget $B_T$ is exhausted.}
    \label{fig:comms-hero}
\end{figure}

We assume a decentralized team -- each robot plans independently and is not guaranteed to receive information from  others.
However, we assume a communication protocol such that a robot can, at each step, decide to broadcast its newly acquired measurements to the network. 
The network has limited bandwidth, so robots are incentivized to only send those that will be most beneficial to the others.
We encode the cost of loading the network with a message through a threshold value, where utilities higher than the threshold have a higher benefit than cost.
With these assumptions in place, the question of how to evaluate a potential message for its utility emerges as a key consideration to this problem; in this work, we propose several approaches to solve this problem.
An overview of the problem pipeline is displayed in \Cref{fig:comms-hero}.

\section{Background}
There is extensive literature in utility-based
communication, and our methods share similarities with previously proposed
approaches (see \Cref{ssec:multirobot-comms-background}).
In particular, some other works have
made use of a measure of utility based on expected reward or expected effect on robot actions, and the teams are
decentralized in terms of planning and reward calculation \cite{marcotte2020optimizing,unhelkar2016contact}. 
While our work in this chapter employs these as well, a major difference is
that our end goal is to explore the environment in whatever way results in the
best quantile estimates, while those works define a static goal location for the robots to reach. 
We also exclusively choose to plan using a
greedy algorithm for computational efficiency and tractability as opposed to using MDPs, motivated by
resource-constrained systems, and do not forward-simulate the models of other
robots while~\cite{marcotte2020optimizing} does. 
Significantly, we make the (more realistic)
assumption of stochastic communication on top of the communication decision; just because a robot decides to send a message does not guarantee
reception by all or any of the recipients. We further do not assume perfect measurement sensing. 
Despite our relatively simple planning and modeling methods and more
failure-prone communication assumptions, we are able to achieve a significant
decrease in network loading while maintaining reasonable, and in some cases improving, task performance.

\section{Methods}
\label{sec:methods}

Our formulation extends the setup used in the previous \Cref{ch:quantiles,ch:multirobot}; we briefly mention relevant definitions here. 
The team consists of $\nrobots$ aerial robots, each operating with a fixed planning budget $B_T$ at a fixed height in a discretized environment
$\gtlocations \subset
\mathbb{R}^2$. 
Robots can take one of four possible actions at each step: $\pm x$ or $\pm y$.
Since our focus is on inter-robot communications, we assume a low-level controller, a collision avoidance routine, and perfect localization.
The human operator supplies a set of arbitrary quantiles of interest $\quantiles$, and the end goal is to produce accurate estimates $\estimatedquantilevalues$ of $\quantiles$ where the true values are
defined as $\quantilevalues = \textit{quantiles}(\gtsensedvalues,\quantiles)$, 
the set of all possible locations a robot can measure with its sensor are $\gtsensedlocations$, and 
 the corresponding true values at those locations are $\gtsensedvalues$.
Each robot has a camera with which it can take noisy measurements $\sensedvalues$ representing the concentration of algae at those pixel locations in the environment $\sensedlocations$.
Throughout the deployment, each robot collects measurements $\sensedlocations, \sensedvalues$ both firsthand via its camera and via any messages received from others through the network.
The robots use all their collected measurements to construct and update their internal Gaussian process (GP) model of the environment.
%
Robot $\ego$'s estimate of the quantile values 
is given by
\begin{equation}
    \estimatedquantilevalues_{\ego} = quantiles(\mu_{GP^{\ego}}(\gtsensedlocations),\quantiles)
\end{equation}
where $\mu_{GP^{\ego}}(\gtsensedlocations)$ is the estimate of all possible locations
using the GP conditioned on the measurements $\sensedlocations^{\ego}$ and $\sensedvalues^{\ego}$ that $\ego$ has collected thus far.

In what follows, when viewing the problem from the perspective of a given robot, we will refer to it as $\ego$ or the ego robot and to any of the other $\nrobots$ as $\other$. 
Variables superscripted with $\other$ such as $GP^{\other}$ indicate the GP for $\other$ \textit{from the perspective of $\ego$}.

Motivated by the computational constraints of resource-constrained systems, we opt to use a locally computable, online, greedy planning policy.
At each planning step $t$, $\ego$ maximizes an objective function $\objectivefunction$ over feasible next measurement locations.
We use a modified version of the quantile standard error objective function
to evaluate a proposed location for time $t$
proposed in \Cref{ch:quantiles},
which measures the difference in the standard error of the estimated quantile values using 
two versions of the robot model: the current model $\mu_{GP^{\ego}_{t-1}}$, and the model updated with the hypothetical new values at the proposed locations $\mu_{GP^{\ego}_{t}}$.
We modify $\objectivefunction$ such that it is conditioned not only on the proposed measurement locations $\sensedlocations$, but also the values associated with the locations (denoted by $\newobjectivevalues$ generally) and the GP to be used.
The reasoning is that $\objectivefunction$ is also used later in the process for computing message utilities, and this allows it to generalize to these cases when $\newobjectivevalues$ and $GP$ may take on different values.
When $\objectivefunction$ is used in the planning step by $\ego$ to select $\ego$'s next location, $\newobjectivevalues$ are the expected values at the proposed location $\sensedlocations$: $\mu_{GP^{\ego}_{t-1}}(\sensedlocations)$ (just like in \Cref{ch:quantiles}).
However, when $\objectivefunction$ is used by $\ego$ to determine the utility of $\other$ receiving a real measurement that $\ego$ has already acquired, $\newobjectivevalues$ are those $\sensedvalues$, and the GP may be either $GP^{\ego}$ or $GP^{\other}$ (see \Cref{ssec:modeling}).
The final term encourages exploration of high-variance areas.
\begin{equation}
    \objectivefunction(\sensedlocations_i; \newobjectivevalues, GP) =
\frac{d}{\numtiles} +
\sum_{\sensedlocation_j \in \sensedlocations_i} c\sigma^2(\sensedlocation_j)
\end{equation}
\begin{equation}
    d =
\|se(\mu_{GP_{i-1}}(\gtsensedlocations),\quantiles) -
se(\mu_{GP_{i}}(\gtsensedlocations),\quantiles) \|_{1}
\end{equation}

\subsection{Summary of Process}
\label{ssec:procedure}
The approach consists of 3 phases: Initialization, Exploration, and Aggregation.
These are described at a high level here; further details are in the sections following.

\subsubsection{Initialization}
\label{sssec:procedure-init}
Robots begin distributed in the environment according to the initial location spread $\spread$ (\Cref{ssec:alpha-initial-location-spread} provides details). 
We assume each robot is given the number of other robots and their corresponding starting locations.

\subsubsection{Exploration}
\label{sssec:procedure-explore}
The bulk of the work happens during the decentralized exploration phase, where robots repeatedly choose the next location to visit, measure the environment, optionally broadcast the measurement in a message, and update their models until their budget $B_T$ is exhausted. The next measurement locations at $t$ are selected according to
\begin{equation}
\label{eq:greedy-planning}
    \sensedlocations^{\ego}_t = \argmax_{\sensedlocations \subset \sensedlocations^{\ego}_{\textrm{next}}} \objectivefunction (\sensedlocations;  \mu_{GP^{\ego}}(\sensedlocations), GP^{\ego})
\end{equation}
where $\sensedlocations_{\textrm{next}}$ contains the sets of locations that could be sensed one step from the current location, $\location^{\ego}_t$.
On top of this planning cycle, $\ego$ may receive a message from another $\other$ at any point, leading to further model updates (\Cref{ssec:modeling}).

\subsubsection{Aggregation}
\label{sssec:procedure-aggregate}
Once the robots have reached the allotted budgets, their data is collected. The operator can produce an aggregate GP model of the environment offline using it and, with that, obtain final estimates of quantile values:
\begin{equation}
\label{eq:final-qest}
    \estimatedquantilevaluesfinal = quantiles(\mu_{GP_\textrm{aggregate}}(\gtsensedlocations),\quantiles)
\end{equation}

\subsection{Communication Framework}
\label{ssec:comms}
The communication network consists of $\nrobots$ robots, each with the ability to broadcast a message at each step in its path. 
A message $\msg_t^{\ego} = (\ego, \location^{\ego}_t, \sensedlocations^{\ego}_t, \sensedvalues^{\ego}_t)$ sent by $\ego$ at time $t$ is composed of a header containing the sending robot's ID number and current location, and a body containing the current set of measurements. 
Messages are transmitted successfully to every other robot $\other$ in the network subject to the probability 
depending on distance and parameterized by the dropoff rate $\eta$ and communication radius $r$ \cite{rayas2023study}:
\begin{equation}
    \successprob(distance) = \frac{1}{1 + e^{\eta (distance - r)}}
\end{equation}
and a successfully transmitted message will be received whole, uncorrupted, and without delay. 
In general, we assume $\ego$ does not receive a handshake or confirmation of whether its message was received successfully or not by $\other$.
We assume that robots can self-localize and that, within the radius $r$, they can accurately sense $\location^{\other}$, which they do at each step. 

\subsection{Modeling Others and Updating Models}
\label{ssec:modeling}
A central consideration is how to model other robots using incomplete information since such a model is necessary to adequately assess the utility any given message will have for others. 
The ego robot $\ego$ maintains a separate model $(\location^{\other},~ GP^{\other})$ for each $\other$. 
$\location^{\other}$ is the most up-to-date location $\ego$ has, and is updated in two cases.
First, by $\ego$'s sensing capabilities at each step if $\other$ is within radius $r$, 
and second, anytime $\ego$ receives a message from $\other$, via the header.
$GP^{\other}$ is a GP based on the measurements $\ego$ believes $\other$ has.
We denote this set of measurements as $\othermeasurements := (\sensedlocations^{\other}, \sensedvalues^{\other})$.
%
$\othermeasurements$, and $GP^{\other}$, which is defined by it, can also be updated in two scenarios.
The first is when $\ego$ receives a message $\msg$ from $\other$, since it is clear that $\other$ must possess the data contained.
The second is more uncertain. If $\ego$ broadcasts $\msg$ to the network, it does not know whether it was successfully transmitted to any $\other$; however, recall that $\ego$ can sense the location of $\other$ within $r$. We use this to determine a proxy for reasonable confidence of transmission success: If $\other$ is within $r$, $\ego$ believes $\msg$ was received and thus updates $\othermeasurements$.%

\subsection{Computing Message Utility}
\label{ssec:utility}
We introduce several methods for computing the utility $\utility^{\other}$ of $\msg^{\ego} = (\ego, \location^{\ego}, \sensedlocations, \sensedvalues)$ (dropping the super- and subscripts for readability). 
The first,
\textit{Reward}, defines $\utility_{rw}$ as the reward $\ego$ believes $\other$ would receive, using the objective function $\objectivefunction$:
\begin{equation}
    \utility_{rw}(\sensedlocations, \sensedvalues, \other) = 
    \begin{cases}
      \infty, & \text{if} \quad | \sensedlocations^{\other} | = 0\\
      \objectivefunction(\sensedlocations ; \sensedvalues, GP^{\other}), & \text{otherwise}
    \end{cases}
\end{equation}

We call the second method \textit{Action}. Intuitively, it considers a message utile if $\ego$ believes that adding it to $GP^{\other}$ would result in $\other$ taking a different action than it would without those measurements. 
Then $\utility_{ac}$ is defined as $\infty$ if the actions $\sensedlocations_{\textrm{next,w/o}}$ and $\sensedlocations_{\textrm{next,w/}}$ differ as computed using \Cref{eq:greedy-planning}, and if they do not, $\utility_{ac}$ takes the value of $\utility_{rw}$:
\begin{equation}
    \utility_{ac}(\sensedlocations, \sensedvalues, \other) = 
    \begin{cases}
      \infty, & \text{if} ~ | \sensedlocations^{\other} | = 0\\
      \infty, & \text{if} ~ \sensedlocations_{\textrm{next,w/o}} \neq \sensedlocations_{\textrm{next,w/}} \\
      \utility_{rw}(\sensedlocations, \sensedvalues , \other), & \text{otherwise}
    \end{cases}
\end{equation}
Both $\utility_{rw}$ and $\utility_{ac}$ include the caveat that, if $\ego$ believes $\other$ has collected no measurements yet, the message should be sent.

The last three methods presented are simpler in that they do not make use of the model of $\other$.
The third method, \textit{Ego-reward}, uses the reward $\ego$ received from its measurement:
\begin{equation}
    \utility_{ego}(\sensedlocations, \sensedvalues, \other) = 
     \objectivefunction(\sensedlocations, \sensedvalues ; GP^{\ego})
\end{equation}
Finally, \textit{Always} and \textit{Never} are constant baselines, which always assign either $\infty$ or $0$:
\begin{equation}
    \utility_{al}(\sensedlocations, \sensedvalues, \other) = 
     \infty
\end{equation}
\begin{equation}
    \utility_{nv}(\sensedlocations, \sensedvalues, \other) = 
     0
\end{equation}

\subsection{Deciding to Communicate}
With the utilities of a message $\{\utility^{\other}\}$ computed for each $\other$, $\ego$ now faces the problem of deciding whether the utility overcomes the cost of loading the network with a message. 
Since the communication is broadcast-based rather than point-to-point, $\ego$ must first compute a single utility value $\aggutility$ for the team and then determine if it crosses a utility threshold $\utilthresh$ representing the cost.
For this, we describe a method 
for aggregating a set of utilities 
which computes 
\textit{expected utility}, weighting utilities by the estimated transmission success probability:
\begin{equation}
\label{eq:expected-utility}
    \aggutility_{EU}(\{\utility^{\other}\}) = \frac{1}{\nrobots - 1} \sum_{\other} (p_{est}(\other) \cdot \utility^{\other})
\end{equation}
\begin{equation}
\label{eq:estimated-success-probability}
    p_{est}(\other) = 
    \begin{cases}
      \successprob(d^{\other}), & \text{if} \quad d^{\other} \leq r \\
      \successprob(r), & \text{otherwise}
    \end{cases}
\end{equation}
where $d^{\other} = ||\location^{\ego} - \location^{\other}||_2$ and $r$ is as defined in \Cref{ssec:comms}.
The estimated success probability, computed in \Cref{eq:estimated-success-probability}, directly uses the probability based on the distance to $\other$ if it is known (i.e., if $\other$ is within the sensing radius $r$), but defaults to an optimistic probability estimate for any $\other$ that is beyond the sensing radius by using $r$ as the assumed distance.
This optimistic heuristic prevents overly restrictive utility values in the face of unknown teammate positions.
Note that 
if any $\utility^{\other}$ is $\infty$, then $\aggutility$ is considered past the threshold $\utilthresh$.
If $\aggutility \geq \utilthresh$, $\ego$ decides to transmit $\msg$.%

\section{Experiments and Discussion}

\begin{table*}
\centering
\caption{
  Number of attempted and successful transmitted messages per utility method. 
  Restricting transmissions with Action results in more than $40\%$ reduction in network load (attempted) while improving mean final quantile estimation error  by $1.84\%$ and only resulting in a less than $25\%$ decrease in successful transmissions.
  Results reported from $60$ experiments total of length 40 steps.
  All experiments use $\nrobots=4$, so 
  the most messages possibly sent total is $40\times(\nrobots-1)=120$.
  }
  \begin{tabular}{c||cc|cc|c}
    \textit{Method} & \textit{Attempted} & \textit{Median} & \textit{Successful} & \textit{Median} & \textit{Median} \\ 
     & \textit{($\mu, \sigma$)} & \textit{Decrease} & \textit{($\mu, \sigma$)} &  \textit{Decrease} &  \textit{RMSE Increase} \\ 
    \hline 
    %
%
%
Action & (66.6, 27.95) & 42.5\% & (38.2, 23.85) & 24.32\% & -1.84\% \\
Always & (120.0, 0.0) & -- & (44.93, 16.78) & -- & -- \\
Ego-reward & (76.2, 28.33) & 27.5\% & (34.4, 20.27) & 32.2\% & 5.96\% \\
Never & (0.0, 0.0) & 100.0\% & (0.0, 0.0) & 100.0 & 73.98\% \\
Reward & (68.2, 32.1) & 52.5\% & (36.0, 22.77) & 35.96\% & 16.03\% \\
  \end{tabular}
  \label{tab:comms}
\end{table*}

\begin{figure}
  \centering
 \includegraphics[width=\columnwidth]{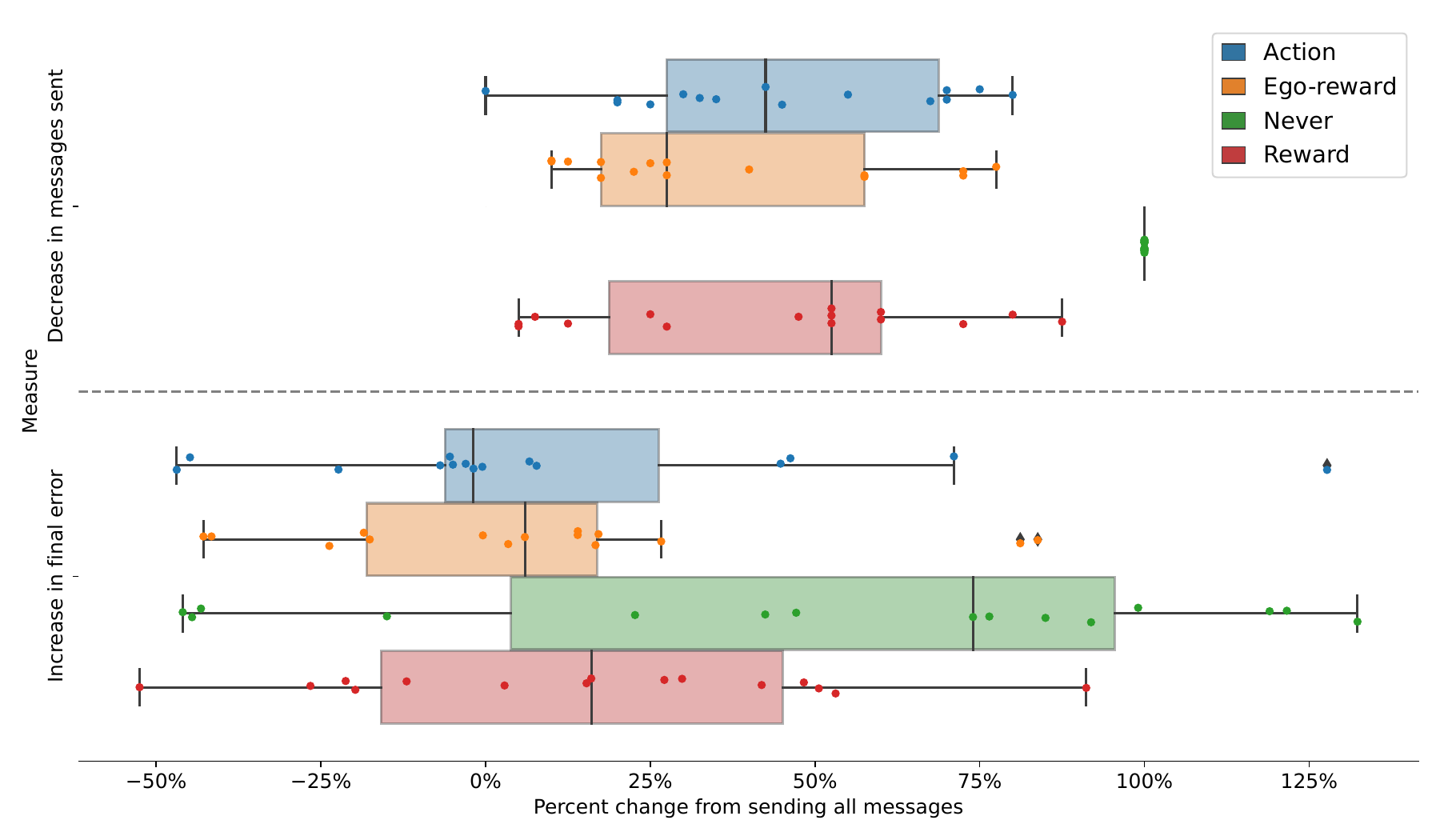}
  \caption{
  Tradeoff in network load vs. task performance for different utility methods, compared to always sending messages.
  Top: Percent decrease in messages sent (network load). Higher is better.
  Bottom: Percent increase in final quantile estimation error. Lower is better. 
  }
  \label{fig:comms-tradeoff}
\end{figure}

\begin{figure}
  \centering
 \includegraphics[width=\columnwidth]{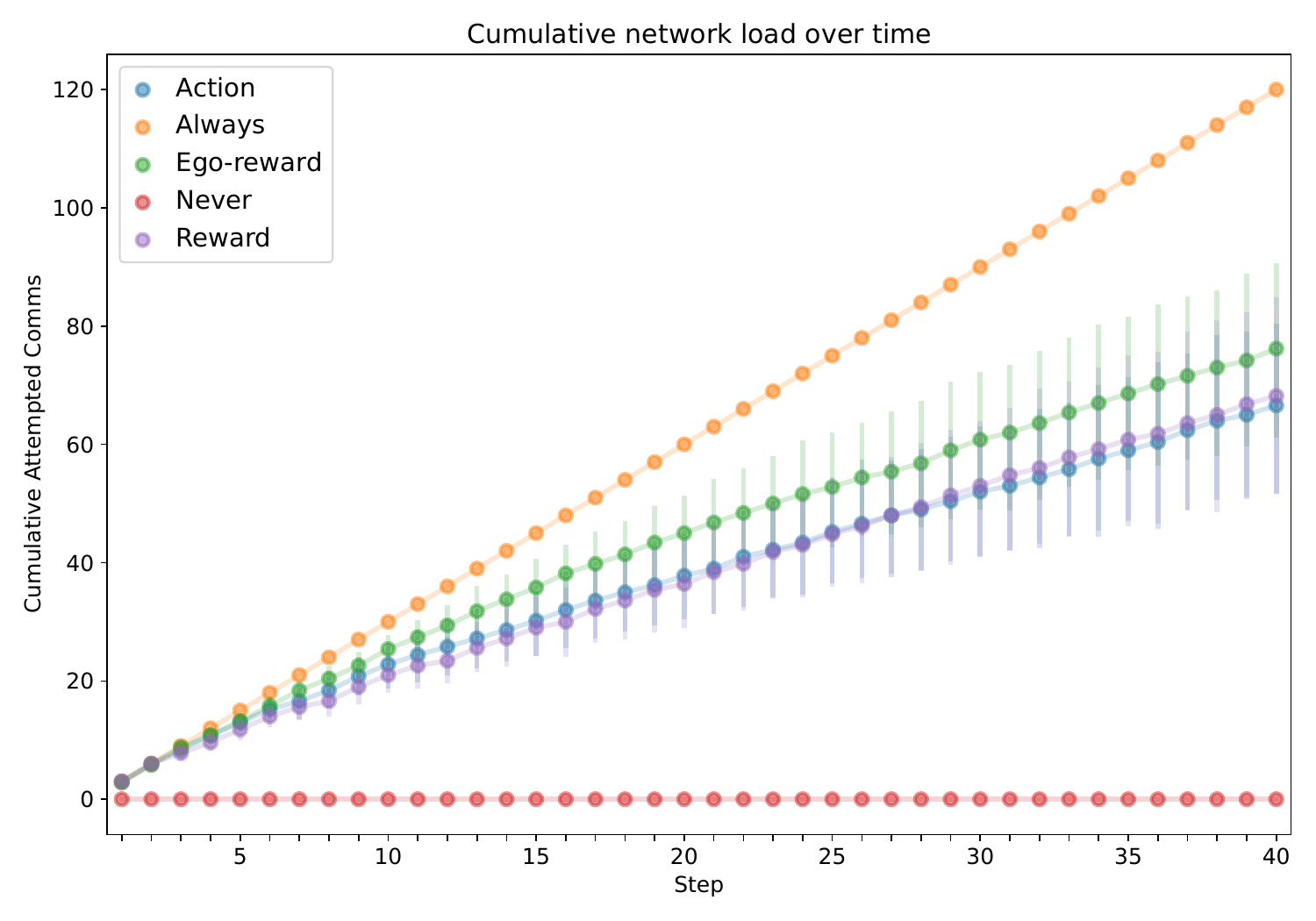}
  \caption{
  Cumulative network load, as measured by messages sent, over time. 
  Each step on the x-axis represents a robot taking a planning step. 
  }
  \label{fig:comms-cumulative}
\end{figure}

\begin{figure}
  \centering
 \includegraphics[width=0.6\columnwidth]{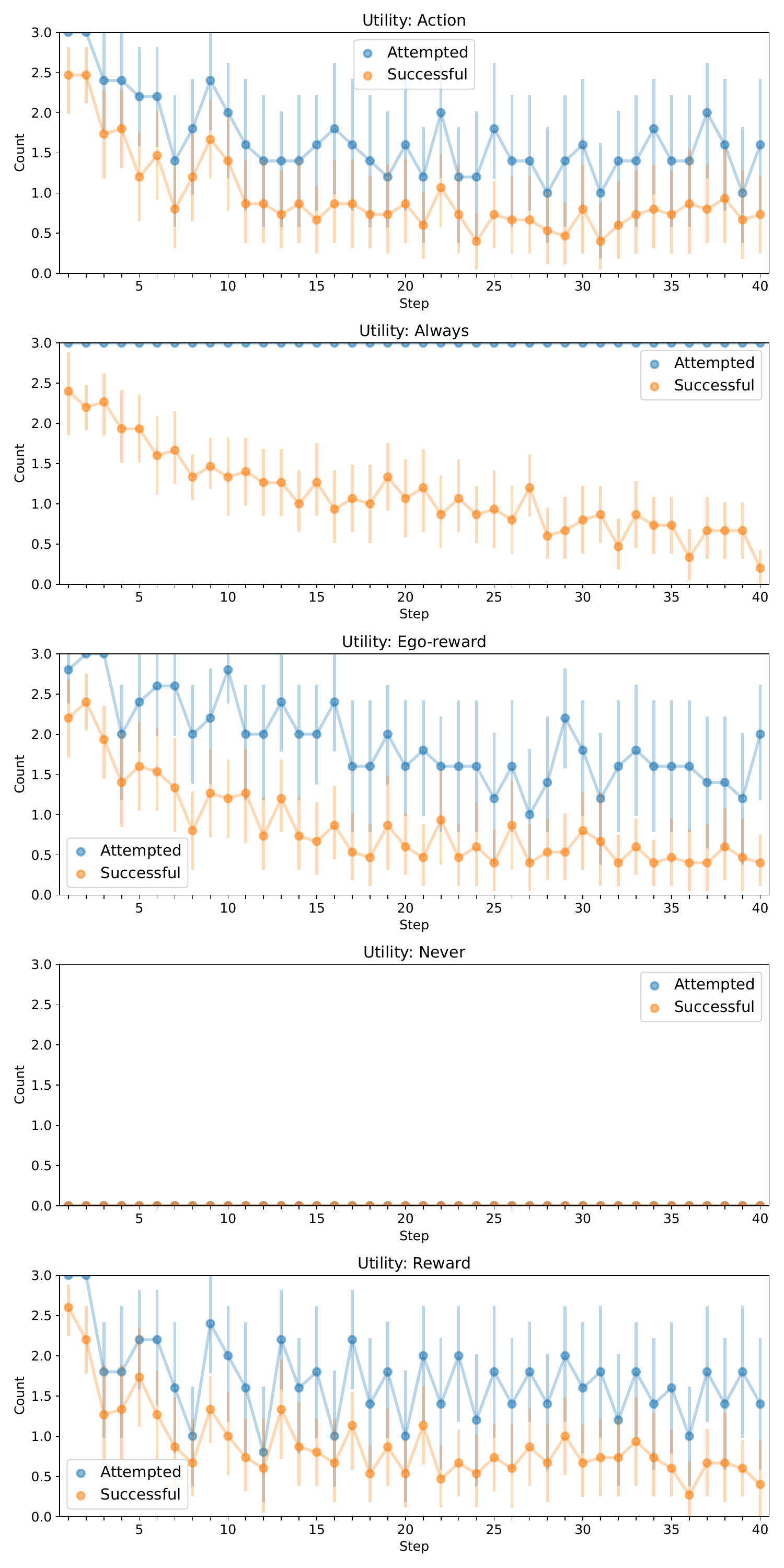}
  \caption{
  Amount of attempted and successful transmissions over time for different utility methods.
  Each step on the x-axis represents a robot taking a planning step. 
  }
  \label{fig:comms}
\end{figure}

We now present experimental results for the methods described in \Cref{sec:methods}.
We test performance on 3 sets of quantiles $\quantiles$: Quartiles (0.25, 0.5, 0.75), Median-Extrema (0.5, 0.9, 0.99), and Extrema (0.9, 0.95, 0.99).
For each combination of parameters, we run 5 seeds, resulting in 60 total instances.
We fix the following parameters: $\nrobots = 4$; budget $B_T = 10$ ($40$ steps total); initial spread $\spread = 0.2$; $\eta = 0.4$; communication radius $r = 15.0$; thresholds $\utilthresh_{rw} = 2.8 \times 10^{-4}$ and $\utilthresh_{ego} = 8.3 \times 10^{-5}$.
$\utilthresh_{rw}$ is used for Reward and Action, while $\utilthresh_{ego}$ is used for Ego-reward.
We selected these $\utilthresh$ experimentally by observing the received $\utility$ values during testing and setting $\utilthresh$ to the approximate 25th percentile of observed rewards. 
This can be interpreted as only messages that are in the top 75\% of utility being worth the cost of sending to the network, while those in the lower quarter are not.

The robots operate in simulation using a real-world dataset collected with a hyperspectral camera mounted on a drone from a lake in California. The environment is roughly $80 \times 60$ meters, discretized to $\gtlocations = 25 \times 25$ locations. Each set of measurements is an image of $5 \times 5$ pixel intensities for the 400nm channel of the hyperspectral image, normalized to $[0,1]$. To each measurement, we add zero-mean Gaussian noise with a standard deviation $0.05$.%

\subsection{Utility Methods Comparison}
We investigate the effectiveness of the different methods on performance in terms of both the network load as well as the final error (RMSE) between the estimated quantile values $\estimatedquantilevaluesfinal$ and $\quantilevalues$.
\Cref{tab:comms} displays numerical results, including a comparison of each method to Always in terms of the median percent change of each reported metric. 
These results are shown graphically in \Cref{fig:comms-tradeoff}, illustrating the tradeoff between network load and task performance.
\Cref{fig:comms-cumulative} shows the cumulative network load as measured by the attempted transmissions over time, while 
\Cref{fig:comms} shows the number of attempted and successful message transmissions plotted over time.
\Cref{fig:rmse-utility} shows boxplots of the final RMSE for different utility methods.
\Cref{fig:rasters} illustrates example paths taken 
superimposed on the hyperspectral image of the lake.

We observe in \Cref{fig:comms} that 
across all methods, successful transmissions generally decrease over time due to robots spreading out.
Comparing Never to the other methods in \Cref{fig:rmse-utility}, we can confirm our previous findings that inter-robot communication improves performance.
However, we are mainly interested in the effect of decreasing network load.

When we analyze results in \Cref{tab:comms} and \Cref{fig:comms-tradeoff}, we see that severe reductions in communication do not necessarily result in severe decreases in performance. 
We see that Action results in the best performance: 
Compared to Always, there is more than $40\%$ decrease in network load as measured by attempted message transmissions, but this translates to less than $25\%$ decrease in actual (successful) message transmissions. 
Most notably, in terms of final error, there is a $1.84\%$ decrease in mean RMSE, indicating that performance, on average, improves slightly even with restricted communication.

Other reward-based methods, Reward and Ego-reward, also perform favorably.
\Cref{fig:comms,fig:comms-cumulative} suggest that all three proposed methods have message transmission rates that generally decrease over time.
We speculate that this may be due to two reasons.
First, as robots explore more of the environment and receive more information from the others, their beliefs of their teammates will change such that $|\sensedlocations^{\other}|$ increases, leading to more accurate environment models and thus less value in receiving an additional set of measurements.
Second, with more exploration time, robots are more likely to be situated at further distances from each other. With distance, the expected utility will decrease according to the transmission probability, and even messages believed to be useful may not be deemed worthy because of their unlikelihood of successful delivery. 
We note, however, that the optimistic assumption in \Cref{eq:estimated-success-probability} will prevent a highly useful message from not being sent due to distance.

As shown qualitatively in \Cref{fig:comms,fig:comms-cumulative} and quantitatively in \Cref{tab:comms}, Action, Reward, and Ego-reward all have lower network load than Always, but their relative loads vary as do their final estimation error.
In particular, Reward results in the highest decrease in network load (excluding Never) at $52.5\%$, but also the largest increase in final error of more than $16\%$. 
Thus, such a metric may be valuable when network bandwidth is severely limited.
Ego-reward, on the other hand, has the smallest decrease in network load at $27.5\%$, but a relatively small increase in error at less than $6\%$. 
Ego-reward has the benefit of not requiring any onboard modeling of teammate beliefs; thus, this method could be powerful when onboard compute is limited. 
Action is the overall best performer, with a decrease in network load of $42.5\%$ and a decrease in final error of $1.84\%$. 
We suspect Action performs better because it directly considers the impact of a message on future behavior, leading to more informed and targeted measurement acquisition at the next step. 
We note that the comparison of relative network loads is more valuable here than the absolute value, as the threshold values used in \Cref{eq:expected-utility} are tunable parameters that should reflect actual network constraints in a real deployment.

\begin{figure}
  \centering
         \includegraphics[width=0.7\columnwidth]{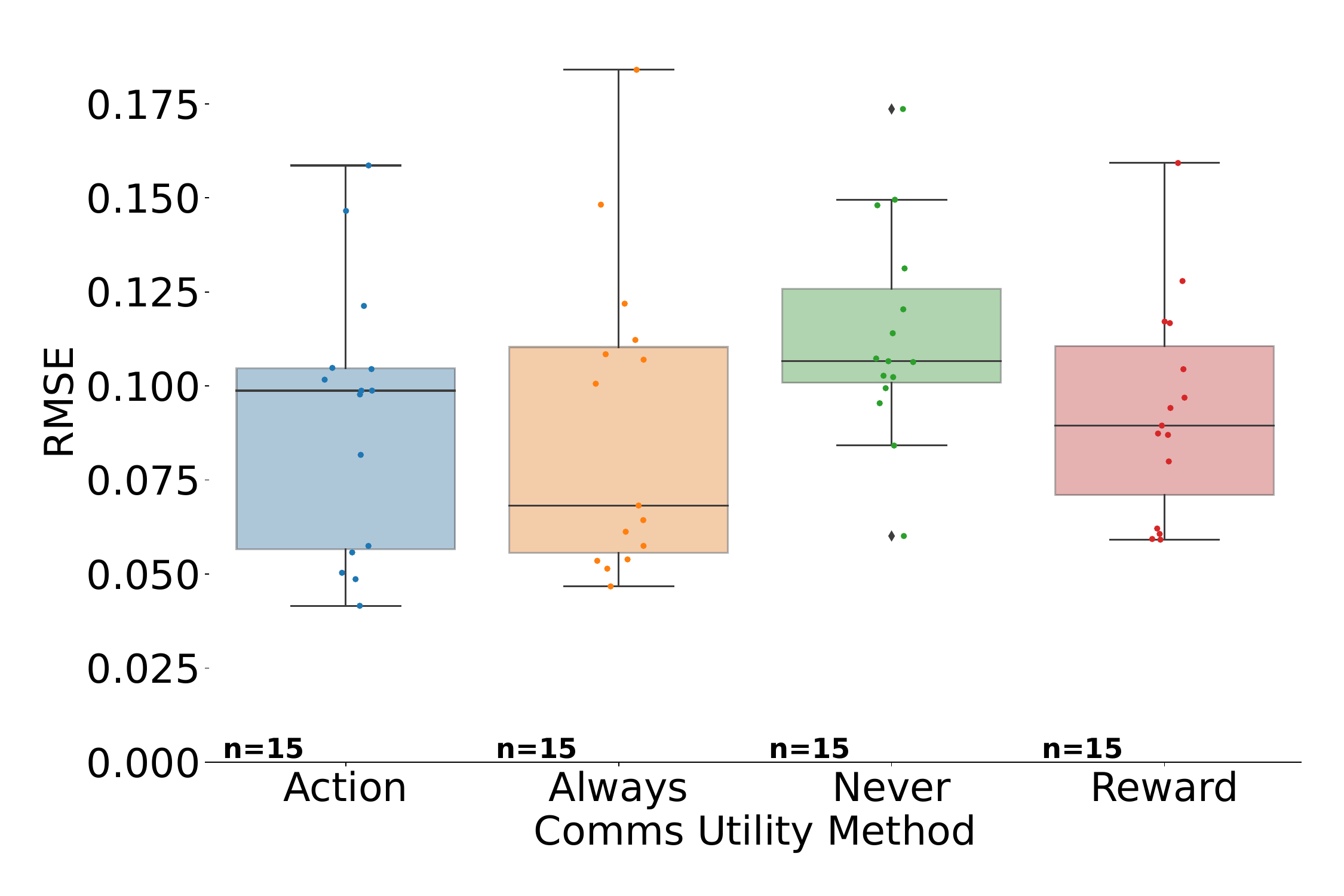}
  \caption{
  Final estimation error for different utility methods. 
  }
  \label{fig:rmse-utility}
  \end{figure}

\begin{figure}
  \centering
  %
 \includegraphics[trim={0.5cm 3cm 0.5cm 3.5cm}, clip, width=0.7\columnwidth]{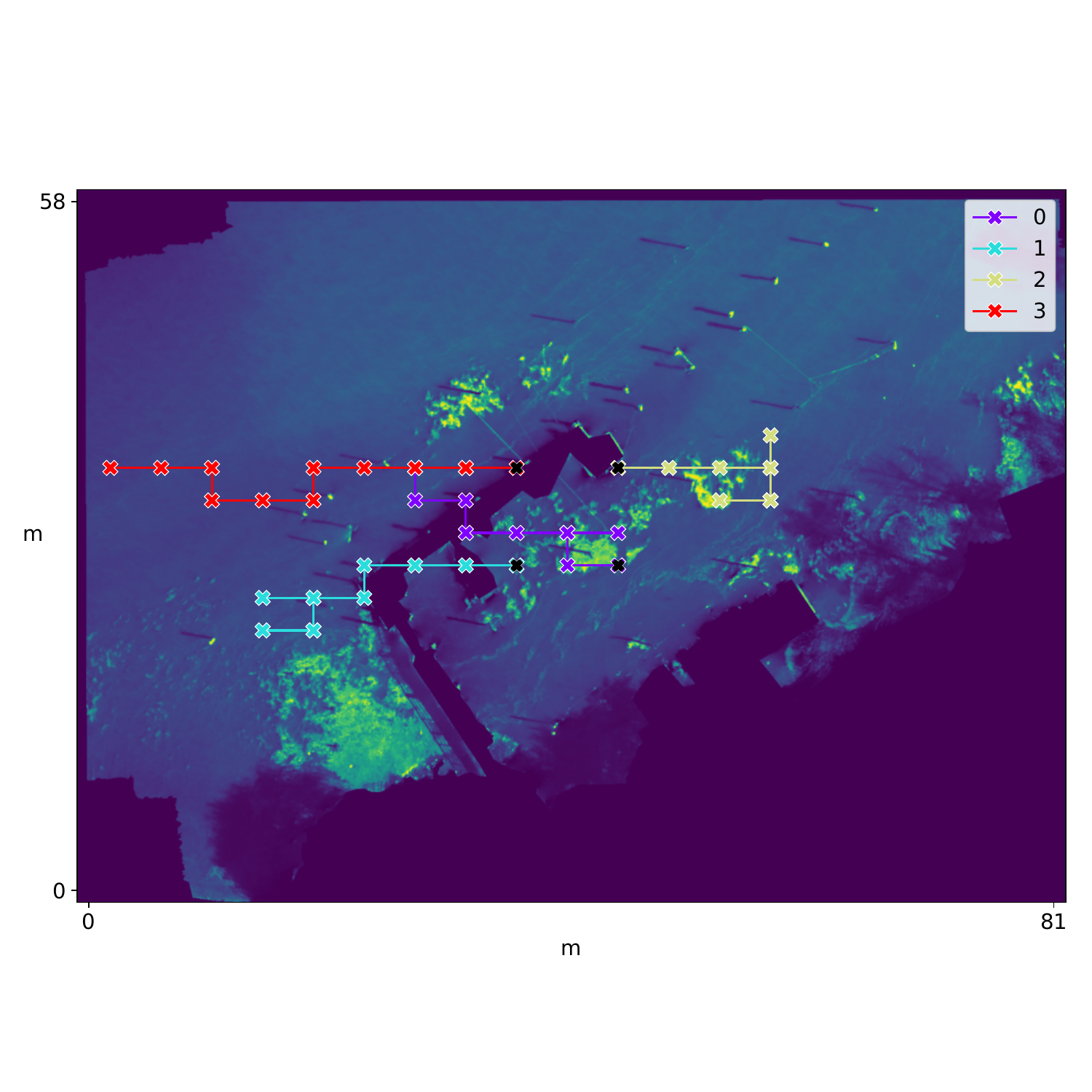}
  %
 \includegraphics[trim={0.5cm 3cm 0.5cm 3.5cm}, clip, width=0.7\columnwidth]{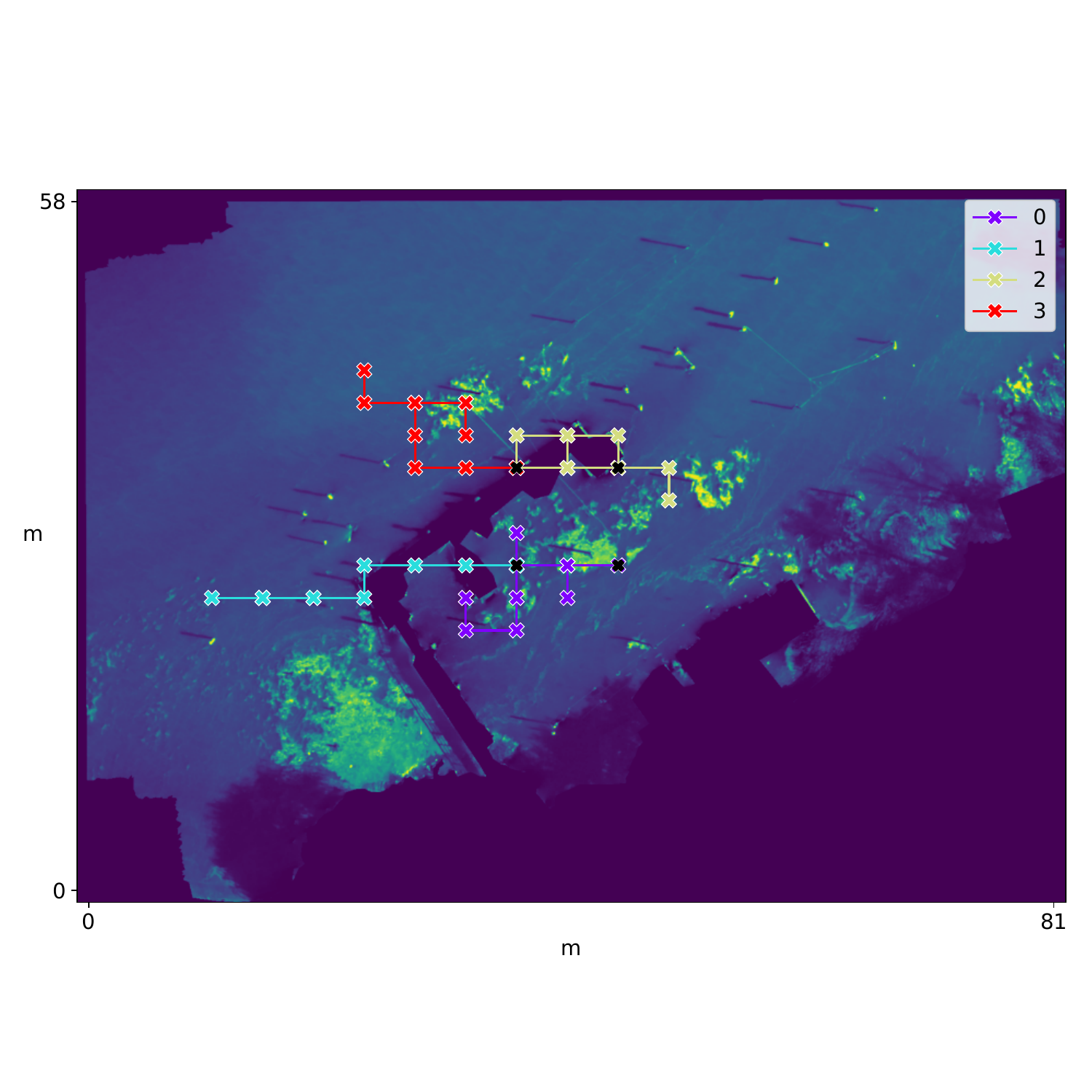}
  \caption{
  Example paths using two different utility methods with $\quantiles$ = $(0.9, 0.95, 0.99)$.
  Top: Action. Bottom: Reward.
}
  \label{fig:rasters}
\end{figure}

\subsection{Oracle Handshaking}
We implement a modification to the communication protocol where we assume an oracle handshaking routine, which gives $\ego$ an accurate confirmation handshake from each $\other$ indicating whether the message was successfully received; $\ego$ then only updates $M^{\other}$ when successful, leading to likely more correct models.
\Cref{fig:handshake} shows the final RMSE comparing the two; we see that there is a significant improvement with oracle handshaking (Wilcoxon signed-rank test $\wpres{1143}{0.05}$).
Though practically impossible, this suggests that implementing such a (probabilistic) routine in the network could improve performance further; however, the resulting burden of these messages must also be taken into account.
     
\begin{figure}
         \centering
         \includegraphics[width=0.7\columnwidth]{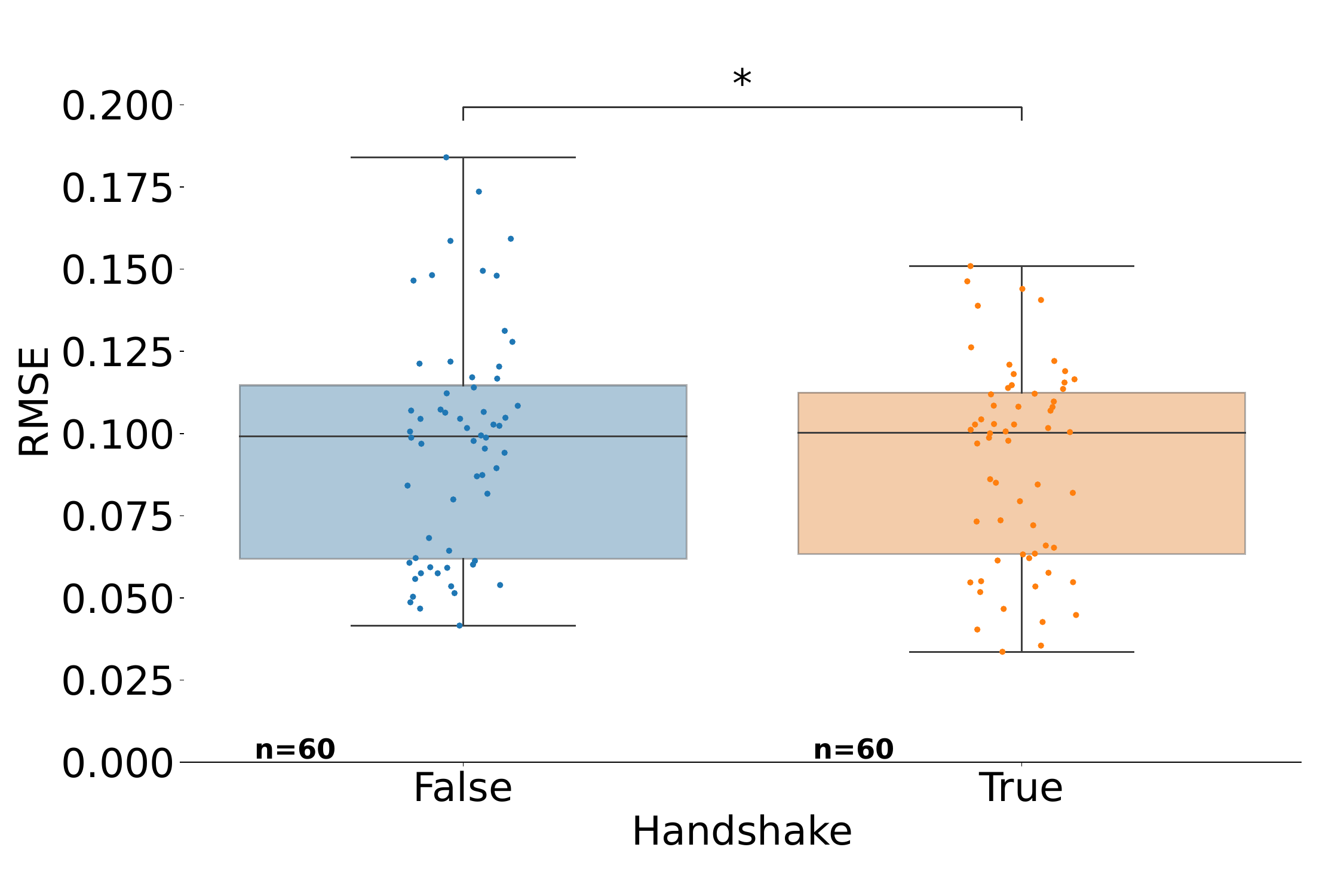}
  \caption{
  Right: the sender receives an oracle handshake
  when a message is successfully transmitted.
  Bar indicates significance under one-sided Wilcoxon signed-rank test; *: $\psig$.
  }
  \label{fig:handshake}
\end{figure}

\section{Conclusion}
In this chapter, 
we presented a study on restricted communication in the context of quantile estimation for environmental analysis.
We show that restricting message transmission based on whether the actions that the sender believes other robots will take at the next state would change with the message results in a $42\%$ decrease in network load while also decreasing quantile estimation error by $1.84\%$, when compared to transmitting every message possible.

Our analysis indicates that targeted communication during quantile estimation can have real benefits to real-world multirobot deployments when robots and their network are resource-constrained.
Importantly, these results allow users to relegate the task of deciding which messages to send to a network in an intelligent manner to the robots.

In addition to receipt handshaking, there are several interesting directions of future work.
First, we have chosen to implement robot planning using greedy action selection, but it would be valuable to extend this method to use a non-myopic planning strategy.
We also believe allowing the sender to choose the most valuable measurements from its collected data rather than restricting it to the most recent measurements may improve performance. 
However, this would lead to challenges in scalability as the environment is explored, and could be addressed by applying a rolling window over the previous measurements to mitigate this or by keeping and updating a limited set of high-utility measurements from which messages are selected.
Further, we are interested in methods for better approximating the models that a robot maintains for each of its teammates, which could be achieved by augmenting message headers with metadata such as that robot's current dataset size or current quantile estimates.
Finally, choosing the threshold value $\utilthresh$ is an important hyperparameter choice. It would be interesting and valuable to investigate ways in which this value could be automatically extracted or learned either offline or during exploration.

\chapter{Conclusion}
  \label{ch:conclusion}

In this thesis, we have proposed solutions which have advanced robot autonomy for long-horizon tasks in multiple different ways, by considering how task specification can be abstracted to interpretable and flexible descriptions.
To this end, we presented work in the domain of constrained manipulation and mobility, 
where we were able to avoid hand-specification of transition points between subtasks of a sequenced task by automatically discovering optimal constraint manifold intersection points.
We further proposed a way to learn the constraint manifolds themselves, bypassing the need to manually specify each constraint analytically.
We also presented work in the domain of robotic environmental monitoring. 
We proposed a flexible task specification framework allowing direct solutions to scientifically-grounded goals, as opposed to a generic solution.
This led us to systematically study how task specification with multirobot teams can be effectively used in these types of problems, finding that among various factors, inter-robot communication is key.
Finally, we developed a method which exploited this fact while respecting realistic network bandwidth restrictions, by autonomously selecting useful but limited information to share with the robot network.
These works, together, present advances along several axes for task specification abstraction, which we believe has the potential to lead to more interpretable, useful, and autonomous robot deployments for real-world, long-horizon tasks.

%
%
%
%
%
%
\clearpage
\chapter*{Bibliography}
\addcontentsline{toc}{chapter}{Bibliography}

\begin{singlespace}
  \patchcmd{\bibsetup}{\interlinepenalty=5000}{\interlinepenalty=10000}{}{}

  \setlength\bibitemsep{0.9\baselineskip}

  \raggedright
  \printbibliography[
    heading = none
  ]
\end{singlespace}


\chapter*{Appendices}
\addcontentsline{toc}{chapter}{Appendices}

\section*{Appendix A:~\smp}
\addcontentsline{toc}{section}{Appendix A:~\smp}
\label{appendix:psm}
As a reference, we provide in \Cref{sec:smp_single_tree} a variant of \smp~that grows a single tree over the manifold sequence without splitting it into individual subtrees, which we use as a baseline in the experiments. Afterwards, we summarize in \Cref{sec:rrt_extend} the extension step of RRT${}^*$ \cite{karaman2011sampling} that is called in the inner loop of the \smp~algorithm. 
In \Cref{sec:Probabilistic completeness} and \Cref{sec:Asymptotic optimality}, we provide a theoretical analysis regarding the probabilistic completeness and asymptotic optimality of the proposed algorithm.

\subsection*{\smp~(Single Tree)}
\addcontentsline{toc}{subsection}{\smp~(Single Tree)}
\label{sec:smp_single_tree}

\Cref{alg:smp_single_tree} describes a variant of \smp~that grows a single tree over the manifold sequence without splitting it into individual subtrees. For each node, we store the corresponding current and target manifold as well as the free configuration space, which is extracted by the functions \textit{GetCurrentManifold}, \textit{GetTargetManifold}, and \textit{GetFreeSpace} in \cref{algl:get_curr_manifold} -- \cref{algl:get_freespace}. Given this information, we call the same steering function that \smp~uses and the RRT$^*$ extend routine described in \Cref{alg:rrt_extend}. The main difference between \smp~(Single Tree) algorithm and \smp~is that \smp~(Single Tree) grows the tree on all manifolds while \smp~grows a subtree for every manifold in $\mathcal{M}$ in a sequential manner. 

\subsection*{RRT$^*$ Extension Step}
\addcontentsline{toc}{subsection}{RRT$^*$ Extension Step}
\label{sec:rrt_extend}
The RRT$^*$\_EXTEND routine (\Cref{alg:rrt_extend}) is a straightforward adaptation of the extension step in RRT$^*$ \cite{karaman2011sampling} to the \smp~problem.
\new{This routine checks the newly projected $q_\t{new}$ configuration for collision with the current free configuration space $C_{\t{free},i}$, performs rewiring steps, and eventually adds it as a node to the tree. 
The rewiring step is equivalent to the one in RRT$^*$ \cite{karaman2011sampling}, which checks if any nearby points can be reached with a shorter distance and updates its parent nodes accordingly.
It uses a distance function $c(q_0, q_1)\in \R_{\geq 0}$ between two nearby configurations and a function \emph{Cost}$(q)$ that stores the path costs from the root of the tree to a node $q$ in order to reconnect a new node $q$ to its neighbor that results in the shortest path from root $q_\t{start}$ to $q$.
Note that $c(q_0, q_1) \triangleq J(\overline{q_0q_1})$, where $\overline{q_0q_1}$ denotes the geodesic \cite{Boothby:107707} joining points $q_0$ and $q_1$ on the current manifold.
}

\begin{algorithm}[t]
\begin{algorithmic}[1]
	\algrenewcommand\algorithmicindent{1.5em}%
	\State $V = \{q_{\t{start}}\}$; $E = \emptyset$; $n=\text{len}(\mathcal{M})-1$
        \For{$k = 1$ to $n m$}
			\State $q_\t{rand} \leftarrow \t{Sample}(C)$ \label{alg:st_steer_start}
			\State $q_\t{near} \leftarrow \t{Nearest}(V, q_\t{rand})$
			\State $M_i \leftarrow \t{GetCurrentManifold}(q_\t{near})$ \label{algl:get_curr_manifold}
			\State $M_{i+1} \leftarrow \t{GetTargetManifold}(q_\t{near})$
			\State $C_{\t{free},i} \leftarrow \t{GetFreeSpace}(q_\t{near})$ \label{algl:get_freespace}
			\State{$q_\t{new} \leftarrow \t{\smp\_STEER} (\alpha, \beta, r, q_\t{near}, M_i, M_{i+1}$)}
	\If{\new{$||h_{M_{i}} (q_\t{new})|| > \ge$}}
	    \State continue
	\EndIf
 		    \State RRT$^*$\_EXTEND$(V_i, E_i, q_\t{near}, q_\t{new},$ \new{$C_{\t{free},i})$}
		\EndFor 
		\State \textbf{return} $\mathrm{OptimalPath}(V, E, q_\t{start}, M_{n+1})$
\end{algorithmic}
\caption{\smp~(Single Tree) $(\mathcal{M}, q_{\t{start}}, \ga, \gb, \ge, \rho, r, m)$}
\label{alg:smp_single_tree}
\end{algorithm}

\begin{algorithm}[t]
\begin{algorithmic}[1]
	\algrenewcommand\algorithmicindent{1.0em}%
			\If{$\t{CollisionFree}(q_\t{near}, q_\t{new}, $ \new{$C_{\t{free},i})$}}
				\State $Q_\t{near} = \t{Near}\Big(V, q_\t{new}, \min{} \Big\{\gg_\t{RRT*} \left(\tfrac{\log(|V|)}{ |V|}\right)^{1/k}\!, \ga\Big\}\Big)$
				\State $V \leftarrow V \cup \{q_{\t{new}} \}$
				\State $q_\t{min} = q_\t{near}$;  $c_\t{min} = \t{Cost}(q_\t{near}) + c(q_\t{near}, q_\t{new})$
				\For{$q_\t{near}\in Q_\t{near}$}
				\If{$\t{CollisionFree}(q_\t{near}, q_\t{new}, $ \new{$C_{\t{free},i})$} \textbf{and}\\$\quad\qquad\t{Cost}(q_\t{near}) + c(q_\t{near}, q_\t{new}) <c_\t{min}$}
					\State $q_\t{min} = q_\t{near}$; $c_\t{min} = \t{Cost}(q_\t{near}) + c(q_\t{near}, q_\t{new})$
				\EndIf
				\EndFor
				\State $E \leftarrow E \cup \{(q_\t{min}, q_\t{new})\}$
				\For{$q_\t{near}\in Q_\t{near}$}
				\If{$\t{CollisionFree}(q_\t{new}, q_\t{near}, $ \new{$C_{\t{free},i})$} \textbf{and}\\  $\qquad\quad\t{Cost}(q_\t{new}) + c(q_\t{new}, q_\t{near}) <\t{Cost}(q_\t{near})$}
					\State $q_\t{parent} = \t{Parent}(q_\t{near})$
					\State $E \leftarrow E ~\backslash~ \{(q_\t{parent}, q_\t{near})\}$
					\State $E \leftarrow E \cup \{(q_\t{new}, q_\t{near})\}$
				\EndIf
				\EndFor
				\State \textbf{return} True
			\Else
				\State \textbf{return} False
			\EndIf	
\end{algorithmic}
\caption{RRT$^*$\_EXTEND$~(V, E, q_\t{near}, q_\t{new}, $ \new{$C_{\t{free},i})$}}
\label{alg:rrt_extend}
\end{algorithm}

\subsection*{Probabilistic Completeness}
\addcontentsline{toc}{subsection}{Probabilistic Completeness}
\label{sec:Probabilistic completeness}
In this section, we prove the probabilistic completeness of \Cref{alg:smp_single_tree}. 
Note that, for ease of analysis, the analysis presented in this section and in the subsequent section assumes that $\rho = 0$. 
Moreover, in this analysis we refer to the collision-free region of a manifold when we allude to $M_i$.

Although in this section, we prove probabilistic completeness and asymptotic analysis for \Cref{alg:smp_single_tree}, they extend to \Cref{alg:smp} as well. We make this claim based on the following arguments. 
\Cref{alg:smp} is essentially a sequence of $n$ infinite loops in which, at the termination of the $i^{\text{th}}$ loop, the algorithm will have computed the optimal path from the start location to the intersection of manifold $M_i$ and $M_{i+1}$. On the contrary, \Cref{alg:smp_single_tree} consists of a single infinite loop which upon termination outputs the optimal path from start point through the sequence of manifolds to the goal manifold. 
It is important to note that each loop in the sequence of infinite loops in \Cref{alg:smp} can be initiated prior to terminating the previous loop. This argument stems from the fact that a search tree in a manifold can be initiated immediately when the nodes in the search tree associated with the previous manifold reach the intersection of the manifold and the previous manifold. Therefore, each infinite loop in \Cref{alg:smp} can be viewed as been executed in parallel and can be assumed to terminate simultaneously.  
Additionally, it is straightforward to understand that at the end of each infinite loop, \Cref{alg:smp} solves a subproblem of the problem that \Cref{alg:smp_single_tree} solves. 
Therefore, by the principle of optimality \cite[Chapter 5]{CoV_Liberzon}, both \Cref{alg:smp_single_tree} and \Cref{alg:smp} solve the same problem and are identical for our analysis. 

\begin{definition}
A collision-free path is said to have \emph{strong} $\delta$-\emph{clearance}  if the path lies entirely inside the $\delta$-interior of $\cup\mathcal{M}$, where $\cup \mathcal{M} \triangleq \cup_{i=1}^{n+1} M_i$ \cite{karaman2011sampling}. 
\end{definition}

We start by assuming that there exists a path $\Hat{\v\tau}$ with \emph{strong} $\delta$-\emph{clearance} connecting the goal manifold $M_{n+1}$ with the start configuration $q_\t{start}$ embedded on the sequence of manifolds under consideration. Let $L$ be the total length of the path, computed based on the pullback metric \cite{lee2006riemannian} of the manifolds due to their embedding in $\R^{k}$. Let $\xi>0$ be the minimum over the reach \cite{aamari2019estimating} of all manifolds in the sequence and the manifolds resulting from the pairwise intersection of adjacent manifolds in the sequence. Informally, the reach of a manifold is the size of an envelope around the manifold such that any point within the envelope and the manifold has an unique projection onto the manifold.  For the analysis presented here, we pick the steering parameter $\alpha$ such that $\xi \geq \alpha$. We use the notation $\textit{Tube}(M_i,\xi)$ to denote the set $\{x\in \R^k ~|~ d(x,M_i)<\xi\}$, where 
\begin{align}
    d(x,M_i) = \inf \{\|x-y\|_{\R^k}|\ y \in M_i\}
\end{align}
{is the minimum distance of the point $x$ to the manifold}. Now, if we define 
\begin{align}
\zeta_i = \sup\limits_{q \in \textit{Tube}(M_i,\xi)} \|h_{M_i}(q)\|,
\end{align}
then for the sake of analysis we assume that $r = \max\ \{\zeta_1,\cdots, \zeta_{n+1}\}$. If $\nu = \min\ (\delta, \alpha)$, then we define a sequence of points
\begin{align}
    \{&[q^1_0=q_\t{start},q^1_1, \cdots, q^1_{m_1}],[q^2_0,q^2_1, \cdots, q^2_{m_2}],\cdots,\nonumber\\  &[q^n_0,q^n_1, \cdots, q^n_{m_1}], [q^{n+1}_0]\}
\end{align} 
on $\Hat{\v\tau}$, such that $[q^i_0,q^i_1, \cdots, q^i_{m_i}] \in M_i$ and $\sum_{i=1}^{n} m_i = m$, where $m=\frac{5L}{\nu (n+1)}$ is the total number of points in the path. Without loss of generality, we assume that for every $M_i$ with $1<i<n+1$ there exists a non-negative integer $j<m_i$ such that  $q^i_0,q^i_1, \cdots, q^i_j \in M_{i-1} \cap M_{i}$ and $q^i_{j+1},q^i_{j+2}, \cdots, q^i_{m_i} \in M_{i} \setminus M_{i+1}$. In other words, there exist some points at the beginning of $[q^i_0,q^i_1, \cdots, q^i_{m_i}]$ that belong to $M_{i-1} \cap M_{i}$ and the rest of the points on the manifold belong exclusively to {$M_i$}.  For ease of analysis, the sequence of points on $\Hat{\v\tau}$ is chosen such that 
\begin{align}
      \|q^i_j - q^i_{j+1}\|_{\R^k} \leq \|q^i_j - q^i_{j+1}\|_{M_i} \leq \nu/5
\end{align}
\new{where $\|\cdot\|_{\R^k}$ and $\|\cdot\|_{M_i}$ are the distances between the points according to the metrics on the   ambient space and manifold respectively.}
We use $B(q^i_j,\nu) \subset \R^k$ to denote a ball of radius $\nu$ around $q^i_j$ under the standard Euclidean norm on $\R^k$. We denote the tree that is grown with RRT$^*$ as $T$.

We prove the probabilistic completeness of our strategy in two parts. The first part proves the probabilistic completeness of RRT$^*$ on a single manifold. In the second part, we prove that with probability one, the tree $T$ grown on a manifold can be expanded onto the next manifold as the number of samples tends to infinity. For the first part, as suggested in \cite[Section 5.3]{kingston2019ijrr}, the Lemma 1 in \cite{Kleinbort2019} can be shown to hold for the single manifold case and probabilistic completeness of RRT/RRT$^*$ on a manifold can be easily proven using \cite[Theorem 1]{Kleinbort2019}. We now focus on proving the second part that shows the probabilistic completeness of our strategy. 
We start by proving \Cref{lemma} which  enables us to prove that a tree grown with RRT$^*$ on a manifold can be extended to the next manifold. 

\begin{lemma}
\label{lemma}
Suppose that ${T}$ has reached $M_i$ and contains a vertex $\Tilde{q}^i_{m_i}$ such that $\Tilde{q}^i_{m_i} \in B(q^i_{m_i},\nu / 5)$. If a random sample $q_\t{rand}^{i+1}$ is drawn such that $q_\t{rand}^{i+1} \in B(q^{i+1}_{0},\nu / 5)$, then the straight path between $\textit{Project}(q_\t{rand}^{i+1}, M_{i} \cap M_{i+1})$ and the nearest neighbor $q_\t{near}$ of $q_\t{rand}^{i+1}$ in $T$ lies entirely in $C_{\t{free},i+1}$.
\end{lemma}
\begin{proof}
By definition $\|q_\t{near}-q_\t{rand}^{i+1}\| \leq \|\Tilde{q}^i_{m_i}-q_\t{rand}^{i+1}\|$, then using the triangle inequality and some algebraic manipulation similar to that used in the proof of \cite[Lemma 1]{Kleinbort2019}, we can show that
\begin{align}
    \|q_\t{near}-q^i_{m_i}\| ~\leq~ &\|\Tilde{q}^i_{m_i}-q^i_{m_i} \| + 2 \|x^{i+1}_{0} - q_\t{rand}^{i+1} \| \nonumber\\ &+2\|q^i_{m_i}-q^{i+1}_{0}\|
\end{align}
which leads to $\|q_\t{near}-q^i_{m_i}\| \leq 5 \frac{\nu}{5}\leq \nu$ and \mbox{$q_\t{near} \in B(q^i_{m_i},\nu)$.} Again, by the triangle inequality, we can show that \mbox{$\|q_\t{near}-q_\t{rand}^{i+1}\| \leq \nu$.} As the sample $q_\t{rand}^{i+1}$ is taken from within the reach of $M_i$ there exists a unique nearest point of $q_\t{rand}^{i+1}$ on $M_i$ \cite{aamari2019estimating}. In other words, the operation $\textit{Project}(q_\t{rand}^{i+1}, M_i \cap M_{i+1})$ is well-defined. Therefore, as
\begin{align}
    \|q_\t{near}-\textit{Project}(q_\t{rand}^{i+1}, M_i \cap M_{i+1})\| \leq \|q_\t{near}-q_\t{rand}^{i+1}\| \leq \nu,
\end{align} the straight path between $\textit{Project}(q_\t{rand}^{i+1},  M_i \cap M_{i+1})$ and   $q_\t{near}$ lies entirely in $C_{\t{free},i+1}$.
\end{proof}

Note that the above lemma is an extension of \cite[Lemma 1]{Kleinbort2019}. The next theorem will prove that with probability one \smp~will yield a path as the number of samples goes to infinity. Since we are only concerned about the transition of $T$ from one manifold to the next, we focus on the iterations in \smp~after $T$ reaches a neighborhood of $q^i_{m_i} \in M_i$. We refer to such an iteration as a \textit{boundary iteration}. 

\begin{theorem}
The probability that \smp~fails to reach the final manifold $M_{n+1}$ from an initial configuration after $t$ boundary iterations is bounded from above by
$a \exp{(-b t)}$, for some positive real numbers $a$ and $b$.
\end{theorem}
\begin{proof}
Let $B(q^i_{m_i},\nu / 5)$ contain a vertex of $T$. Let $p$ be the probability that in a boundary iteration a vertex contained in $B(q^{i+1}_{0},\nu / 5)$ is added to $T$. From \Cref{lemma}, if we obtain a sample $q_\t{rand}^{i+1} \in B(q^{i+1}_{0},\nu / 5)$, then $T$ can reach $\textit{Project}(q_\t{rand}^{i+1}, M_i \cap M_{i+1})$. The value $p$ can be computed as a product of the probabilities of two events: 1) a sample is drawn from $B(q^{i+1}_{0},\nu / 5)$, and 2) $T$ is extended to include $\textit{Project}(q_\t{rand}^{i+1}, M_i \cap M_{i+1})$. 
The probability that a sample is drawn from $B(q^{i+1}_{0},\nu / 5)$ is given by $|B(q^{i+1}_{0},\nu / 5)|/|C|$, where $|B(q^{i+1}_{0},\nu / 5)|$ and $|C|$ are the volumes of $B(q^{i+1}_{0},\nu / 5)$ and $C$ respectively. From the proof of \Cref{lemma}, we infer that the line joining $q_\t{near}$ and $q_\t{rand}^{i+1}$ is collision-free. 
Thus $T$ will be augmented with a new vertex contained in $M_i \cap M_{i+1}$ if line 9 and 14 in \Cref{alg:smp} are executed. The probability of execution of line 9 and 14 is $(1-\beta) \frac{r-\|h_{M_{i+1}}(q_\t{new})\|}{r}$, which results in the joint probability
\begin{align}
    p= \frac{|B(x^{i+1}_{0},\nu / 5)|}{|C|}(1-\beta)\frac{r-\|h_{M_{i+1}}(q_\t{new})\|}{r}.
\end{align}
Further,  $\|h_{M_{i+1}}(q_\t{new})\| \approx 0$ as $q_\t{new}$ is very close to $M_{i+1}$ and $ |B(x^{i+1}_{0},\nu / 5)| \ll |C|$, thus $\beta$ can be picked such that $p < 0.5$. For \smp~to reach $M_{i+1}$ from the initial point, the boundary iteration should successfully extend $T$ for at least $n$ times (there are $n$ intersections between the $n+1$ sequential manifolds). The process can be viewed as $t > n$ Bernoulli trials with success probability $p$. Let $\Pi_t$ denote the number of successes in $t$ trials,  then 
\begin{align}
    \mathbb{P}\left[\Pi_t < n\right] = \sum_{i=0}^{n-1}{t\choose i}p^i(1-p)^{t-i}\comma
\end{align}
where $\mathbb{P}[\cdot]$ denotes the probability of occurrence of an event. By using the fact that $n \ll t$, this can be upper bounded as \begin{align}
    \mathbb{P}\left[\Pi_t < n\right] &\leq \sum_{i=0}^{n-1}{t\choose n-1}p^i(1-p)^{t-i},
\end{align}
as $p < 0.5$,
\begin{align}
    \mathbb{P}\left[\Pi_t < n\right] \leq {t\choose n-1} \sum_{i=0}^{n-1} (1-p)^{t}~.
\end{align}
Applying $(1-p)\leq \exp{(-p)}$ yields
\begin{align}
    \mathbb{P}\left[\Pi_t < n\right] \leq n {t\choose n-1}  (\exp{(-pt)}).
\end{align}
Through further algebraic simplifications, we can show that 
\begin{align}
    \mathbb{P}\left[\Pi_t < n\right] \leq \frac{n}{(n-1)!}t^n\exp{(-p t)}~.
\end{align}
\end{proof}
As the failure probability of \smp~exponentially goes to zero as $t \to \infty$, \smp~is probabilistically complete. 

\subsection*{Asymptotic Optimality}
\addcontentsline{toc}{subsection}{Asymptotic Optimality}
\label{sec:Asymptotic optimality}
For ease of reference, we begin by giving some definitions and stating some lemmas initially introduced in \cite{karaman2011sampling}, which are required for proving the asymptotic optimality of \smp. 

\begin{definition}
A path $\tau_1$ is said to be homotopic to $\tau_2$ if there exists a continuous function $H : [0,1] \times [0,1] \rightarrow \cup \mathcal{M}$, called the \emph{homotopy} \cite{hatcher2002algebraic}, such that $H(t,0) = \tau_1(t)$, $H(t,1) = \tau_2(t)$, and $H(\cdot, \alpha)$ is a collision-free path for all $\alpha \in [0,1]$. 
\end{definition}

\begin{definition}
A collision-free path $\tau : [0,1] \rightarrow \cup\mathcal{M}$ is said to have \emph{weak} $\delta$-\emph{clearance} \cite{karaman2011sampling} if there exists a path $\tau'$ that has strong $\delta$-clearance and there exists a homotopy $H : [0,1] \times [0,1] \rightarrow \cup\mathcal{M}$ with $H(t,0)=\tau(t)$, $H(t,1)=\tau'(t)$, and for all $\alpha \in (0,1]$ there exists $\delta_{\alpha}>0$ such that $H(t, \alpha)$ has strong $\delta_{\alpha}$-clearance.
\end{definition}

\begin{lemma}
\label{lemma: strong to weak}
\cite[Lemma 50]{karaman2011sampling}
Let $\tau^*$ be a path with weak $\delta$-clearance. Let $\{\delta_n\}_{n\in \mathbb{N}}$ be a sequence of real numbers such that $\lim{n \rightarrow \infty} \delta_n=0$ and $0\leq\delta_n\leq\delta$ for all $n \in \mathbb{N}$. Then, there exists a sequence $\{\tau_n\}_{n\in \mathbb{N}}$ of paths such that $\lim{n \rightarrow \infty} \tau_n=\tau^*$ and $\tau_n$ has strong $\delta_n$-clearance for all $n \in \mathbb{N}$.
\end{lemma}
The above lemma establishes the relationship between the weak and strong $\delta$-clearance paths.

If the configuration space admits a vector space structure, then it can be shown that the space of paths on the configuration space above also admits a vector space structure. Moreover, the space of path becomes a normed space if it is endowed with the \emph{bounded variation} norm \cite{Stein1385521}
\begin{align}
    \|\tau\|_{\text{BV}} \triangleq \int_0^1 |\tau(t)| dt~+~ \text{TV}(\tau).
\end{align}
$\text{TV}(\tau)$ denotes the \emph{total variation} norm \cite{Stein1385521} defined as 
\begin{align}
    \text{TV}(\tau)=\sup_{\{n\in \mathbb{N}, 0=t_1<t_2<...<t_n=1\}} \sum_{i=1}^{n} |\tau(t_i)-\tau(t_{i-1})|.
\end{align}
Using the norm in the space of paths, the distance between the paths $\tau_1$ and $\tau_2$ can defined as 
\begin{align}
\label{eqn: path distance}
    \|\tau_1 - \tau_2\|_{\text{BV}} = \int_0^1 |\tau_1(t)-\tau_2(t)| dt~+~ \text{TV}(\tau_1-\tau_2).
\end{align}
The normed vector space of paths enables us to mathematically formulate the notion of the convergence of a sequence of paths to a path. Formally, given a sequence of paths $\{\tau_n\}, n \in \mathbb{N}$, the sequence converges to a path $\Bar{\tau}$, denoted as $\lim{n \rightarrow \infty} \tau_n = \Bar{\tau}$, if $\lim{n \rightarrow \infty} \|\tau_n - \Bar{\tau}\|_{\text{BV}}=0$.

Let $\mathcal{P}$ denote the set of weak $\delta$-clearance paths which satisfies the constraints in Equation 1. Let $\tau^* \in \mathcal{P}$ be a path with the minimal cost. Due to the continuity of the cost function, any sequence of paths $\{\tau_n \in \mathcal{P}\}, n \in \mathbb{N}$ such that $\lim{n \rightarrow \infty} \tau_n = \tau^*$ also satisfies $\lim{n \rightarrow \infty} J(\tau_n) = J(\tau^*)$. For brevity, we identify $J(\tau^*)$ with $J^*$ and $J^{\text{\smp}}_n$ denotes the random variable modeling the cost of the minimum-cost solution returned by \smp~after $n$ iterations. The \smp~algorithm is asymptotically optimal if
\begin{align}
    \label{eqn: asymp opt def}
    \mathbb{P}\left[\lim{n\rightarrow  \infty} J^{\text{\smp}}_n = J^* \right] = 1.
\end{align}
A weaker condition  than \Cref{eqn: asymp opt def} is 
\begin{align}
    \mathbb{P}\left[\limsup_{n\rightarrow  \infty} J^{\text{\smp}}_n = J^* \right] = 1.
\end{align} 
Note that from \cite[Lemma 25]{karaman2011sampling}, we infer that the probability that $\limsup_{n\rightarrow  \infty} J^{\text{\smp}}_n = J^*$ is either zero or one. Under the assumption that the set of points traversed by an optimal path has measure zero, \cite[Lemma 28]{karaman2011sampling} proves that the probability that \smp~returns a tree containing an optimal path in finite number of iterations is zero. 

Since \smp~is based on RRT$^*$, we focus on how our technique affects the proofs of asymptotic optimality for RRT$^*$. Also, the work in \cite{kingston2018sampling} has shown that RRT$^*$ is optimal when applied on a single manifold. Furthermore, it is shown in \cite{kingston2018sampling} that the steering parameter $\gamma$ in the single manifold case can be bounded from below by $\left(2 \left(1 + \frac{1}{k}\right)\frac{\mu \left( \cup_{q \in C_\t{free}} D_q\right)}{\zeta_{M}(1)}\right)^{\frac{1}{k}}$, where $D_q$ is the set of points in $\R^k$ which are projected on $q$ and $\zeta_{M}(1)$ is the set points in $\R^k$ which are projected onto a unit open ball contained in the manifold $M$. \new{Also,  $\mu(\cdot)$ denotes a \emph{measure} \cite{Stein1385521} on the \emph{measurable space}  \cite{Stein1385521} $\R^k$. Intuitively, a measure maps the subset of a measurable space to its volume.} In this section, we show that under the assumption $\rho = 0 $ in \Cref{alg:smp}, with probability one \smp~eventually returns the optimal path. 

Let $\{Q_1, Q_2,..., Q_n\}$ be a set of independent uniformly distributed points drawn from $C_\t{free}$ and let $\{I_1, I_2,..., I_n\}$ be their associated labels that outlines the order of the points  with support $[0,1]$. In other words, a point $Q_i$ is assumed to be drawn after another point $Q_{\Bar{i}}$ if $I_i < I_{\Bar{i}}$. 
Let $\{\Hat{Q}_1, \Hat{Q}_2,..., \Hat{Q}_n\}$ be the set of points resulting from projecting the point set onto the manifolds as delineated in \Cref{alg:smp}. Similar to \cite{karaman2011sampling}, we consider the graph formed by adding an edge $(\Hat{Q}_i, \Hat{Q}_{\Bar{i}})$, whenever (i) $I_i < I_{\Bar{i}}$ and \mbox{(ii) $\|\Hat{Q}_i-\Hat{Q}_{\Bar{i}}\| \leq r_n = \gamma(\frac{\log(|V_n|)}{|V_n|})^{\frac{1}{k}}$,} where $V_n$ is the vertex set of the graph and $\gamma$ is the steering parameter used in \Cref{alg:smp}. 
Let this graph be denoted by $\mathcal{G}_n=(V_n, E_n)$, \new{where $V_n$ and $E_n$ are the set of vertices and edges of $\mathcal{G}_n$ respectively}. With slight abuse of notation, $J_n^{\text{\smp}}(\Hat{Q}_i)$ denotes the cost of the best path starting from $q_\t{start}$ to the vertex $\Hat{Q}_i$ in the graph $\mathcal{G}$. 
Consider the tree $\Bar{\mathcal{G}}_n$ which is a subgraph of $\mathcal{G}_n$ where the cost of reaching the vertex $\Hat{Q}_i$ equals $J_n^{\text{\smp}}(\Hat{Q}_i)$. 
Since \smp~uses RRT$^*$ for graph construction, it is easy to see that the tree returned by \smp~at the $n$-th iteration is equivalent to $\Bar{\mathcal{G}}_n$. 
Therefore, if $\limsup_{n \rightarrow \infty} J_n^{\text{\smp}}(M_{n+1})$ converges to $J^*$ with probability one with respect to $\mathcal{G}_n$, then it implies that with probability one \smp~will eventually  return a tree that contains the optimal path connecting $q_\t{start}$ and the goal manifold $M_{n+1}$. 
Hence, our next step is focused on showing that the optimal path in $\mathcal{G}_n$ converges to $\tau^*$.

According to \Cref{lemma: strong to weak}, there exists a sequence of strong $\delta-$clearance paths $\{\tau_m\}_{m \in \mathbb{N}}$ that converges to an optimal path $\tau^*$. Let $B_m\triangleq\{B_{m,1}, B_{m,2},...,B_{m,p}\}$ be a set of open balls of radius $r_n$ whose centers lie on the path $\tau_m$ such that adjacent balls are placed $2r_n$ distance apart. The number of balls $p$ is assumed to be large enough to cover $\tau_m$, i.e. $\tau_m \setminus \left(\cup_{i =1 }^{p} B_{m,i} \cap \tau_m\right)$ is a set of measure zero. Moreover, we denote $\Tilde{B}_{m,i}$ as the region obtained as the intersection of the open ball $B_{m,i}$ with the manifold containing its center. Let $\Theta_{m,i}$ denote the event that there exists vertices $\Hat{Q}_i$ and $\Hat{Q}_{\Bar{i}}$ such that $\Hat{Q}_i \in \Tilde{B}_{m,i}$, $\Hat{Q}_{\Bar{i}} \in \Tilde{B}_{m,i+1}$ and $I_{i} \geq I_{\Bar{i}}$. Recall that, $I_{i}$ and $ I_{\Bar{i}}$ are the labels associated with projected points $\Hat{Q}_i$ and $\Hat{Q}_{\Bar{i}}$ respectively. Also note that the edge $(\Hat{Q}_i, \Hat{Q}_{\Bar{i}})$ is included in $\mathcal{G}_n$. 
$D_q$ denotes the set of points that can be projected on the point $q \in  \cup \mathcal{M}$. Additionally, $\zeta(1)$ is defined as 
\begin{align}
\zeta(1) \triangleq \max{M \in \mathcal{M}} \min{q \in M} \mu \left(\cup_{q \in \Tilde{B}_{M}(q',1)} D_q\right),
\end{align}
where $\Tilde{B}_{M_i}(q',1)$ is formed by the intersection of a open unit ball centered at point $q'\in M$ with $M$.
If $\Theta_m = \cap_{i=1}^{p} \Theta_{m,i}$, then the following lemma proves that with probability one, the event $\Theta_{m,i}$ for all $i \in \{1,2,...,p\}$ occurs for large $m$.
\begin{lemma}
If 
\begin{align}
    \gamma > \left(2 \left(1 + \frac{1}{k}\right) \left(\frac{\mu( \cup_{q \in \cup \mathcal{M}} D_q)}{\zeta(1)}\right)\right)^{\frac{1}{k}},
\end{align}
then $\Theta_m$ occurs for all large $m$, with probability one, i.e., $\mathbb{P}(\liminf_{n \rightarrow \infty} \Theta_m) = 1$.
\end{lemma}

The proof of the above lemma follows from the proof of \cite[Lemma 71]{karaman2011sampling} if we replace $\mu(C_\t{free})$ with $\mu( \cup_{q \in \cup \mathcal{M}} D_q)$ and infer that the probability of finding a vertex of the graph in $\Tilde{B}_{m,i}$ is $\frac{\zeta(1)}{\mu( \cup_{q \in \cup \mathcal{M}} D_q)}$.
If $\mathcal{L}_n$ denotes the set of all paths that satisfy the constraints in \Cref{eq:smp_problem} and are contained in the tree returned by \smp~after $n$ iterations such that $\Bar{\tau}^{\text{\smp}}_n \triangleq \min{\tau^{\text{\smp}} \in \mathcal{L}_n} \|\tau^{\text{\smp}} - \tau_m  \|_{\text{BV}}$, then the following lemma can be proven.

\begin{lemma}
\cite[Lemma 72]{karaman2011sampling} The random variable \mbox{$\|\Bar{\tau}^{\text{\smp}}_n - \tau_m\|_{\text{BV}}$} converges to zero with probability one: 
\begin{align}
\label{eqn:grph pth 2 seq pth}
    \mathbb{P}\left[\lim{n \rightarrow \infty} \|\Bar{\tau}^{\text{\smp}}_n - \tau_m\|_{\text{BV}} = 0\right] =1.
\end{align}
\end{lemma}

Recall that by construction $\lim{m \rightarrow \infty} \tau_m = \tau^*$. Expressing \Cref{eqn:grph pth 2 seq pth} as
\begin{align}
    \mathbb{P}\left[\lim{n \rightarrow \infty} \|\Bar{\tau}^{\text{\smp}}_n - \tau^* -(\tau_m -\tau^*)\|_{\text{BV}} = 0\right] =1
\end{align}
and applying the triangle inequality yields
\begin{align*}
    \|\Bar{\tau}^{\text{\smp}}_n\!-\!\tau^* \!-\!(\tau_m \!-\!\tau^*)\|_{\text{BV}} \geq \|\Bar{\tau}^{\text{\smp}}_n - \tau^*\|_{\text{BV}} - \|\tau_m -\tau^*\|_{\text{BV}}.
\end{align*}
From \Cref{eqn:grph pth 2 seq pth} and since $\mathbb{P}\left[\lim{m \rightarrow \infty} \|\tau_m -\tau^*\|_{\text{BV}} = 0 \right]=1$ yields the following result:
\begin{align}
    \mathbb{P}\left[\lim{n \rightarrow \infty} \|\Bar{\tau}^{\text{\smp}}_n - \tau^*\|_{\text{BV}} = 0 \right]=1.
\end{align}
From the continuity of the cost function and due to the fact that $J^{\text{\smp}}_{i+1} \leq J^{\text{\smp}}_{i}$, $i \in \mathbb{N}$ and $J^{\text{\smp}}_{i} \geq J^*$ we obtain the required result (\Cref{eqn: asymp opt def}).

\newpage
\section*{Appendix B:~\ecmnn~Orthogonal Subspace Alignment}
\addcontentsline{toc}{section}{Appendix B:~\ecmnn~Orthogonal Subspace Alignment}

\label{sec:osa}
In previous work \cite{thopalli2019subspacealignment, he2020quantum}, subspace alignment techniques -- without orthogonality constraints -- have been introduced to improve domain adaptation. 
For our purposes here, we require a subspace alignment algorithm that preserves orthogonality of the subspaces being aligned, which we present in this section.

Given a set of orthonormal vectors $\coordframe = \{\orthobasisvec_1, \orthobasisvec_2, \dots, \orthobasisvec_\dimambient\}$ which spans a \emph{space}, the matrix  $\orthobasismat = \begin{bmatrix}
    \orthobasisvec_1 & \orthobasisvec_2 & \dots & \orthobasisvec_\dimambient
\end{bmatrix} \in \Re^{\dimambient \times \dimambient}$ belongs to the Orthogonal Group $O(\dimambient)$.
The Orthogonal Group has two connected components, where one connected component called the Special Orthogonal Group $SO(\dimambient)$ is characterized by determinant $1$, and the other is characterized by determinant $-1$.  
However, if $\orthobasismat = 
\begin{bmatrix}
    \orthobasisvec_1 & \orthobasisvec_2 & \dots & \orthobasisvec_\dimambient
\end{bmatrix}$ has determinant 1 (i.e. if $\orthobasismat \in SO(\dimambient)$), 
then substituting $\orthobasisvec_1$ with its additive inverse ($-\orthobasisvec_1$)  will result in $\flippedorthobasismat = 
\begin{bmatrix}
    -\orthobasisvec_1 & \orthobasisvec_2 & \dots & \orthobasisvec_\dimambient
\end{bmatrix}$ with determinant $-1$. Aligning two coordinate frames  $\coordframe^a$ and $\coordframe^c$ to have a common origin and associated basis matrices $\orthobasismat_a$ and $\orthobasismat_c$, respectively, is equivalent to finding an $\rotmat \in SO(\dimambient)$ such that $\orthobasismat_a \rotmat = \orthobasismat_c$. The solution to this problem exists if and only if $\orthobasismat_a$ and $\orthobasismat_c$ come from the same connected component of $O(\dimambient)$, i.e. if either both $\orthobasismat_a, \orthobasismat_c \in SO(\dimambient)$ or both determinants of $\orthobasismat_a$ and $\orthobasismat_c$ are $-1$.

For a \emph{subspace} such as the normal space $\normalspaceatjointposition$ associated with an on-manifold data point $\jointposition$ on $M$ spanned by the eigenvectors $\lnormalcoordframe = \{\coveigvec_{\dimambient-\dimconstraint+1}, \dots, \coveigvec_\dimambient\}$, the concept of a determinant does not apply to $\covnullspaceeigmat = 
\begin{bmatrix}
    \coveigvec_{\dimambient-\dimconstraint+1} & \dots & \coveigvec_\dimambient
\end{bmatrix} \in \Re^{\dimambient \times \dimconstraint}$, as it is not a square matrix.
However, the normal space $\normalspaceatjointposition$ can be described with infinitely-many orthonormal bases ${\covnullspaceeigmat}_0$, ${\covnullspaceeigmat}_1$, ${\covnullspaceeigmat}_2$, ... ${\covnullspaceeigmat}_\infty$ where the set of column vectors of each is an orthonormal basis of $\normalspaceatjointposition$. Each of these is a member of $\Re^{\dimambient \times \dimconstraint}$. 
Moreover, we can pick the transpose of one of them, for example ${\covnullspaceeigmat}_0\T$, as a projection matrix, and ${\covnullspaceeigmat}_0$ as the inverse projection matrix. 
Applying the projection operation to each of the orthonormal bases, we get
${\covnullspaceolmat}_0 = {\covnullspaceeigmat}_0\T{\covnullspaceeigmat}_0 = \eye_{\dimconstraint \times \dimconstraint}$, ${\covnullspaceolmat}_1 = {\covnullspaceeigmat}_0\T{\covnullspaceeigmat}_1$, ${\covnullspaceolmat}_2 = {\covnullspaceeigmat}_0\T{\covnullspaceeigmat}_2$, ... ${\covnullspaceolmat}_\infty = {\covnullspaceeigmat}_0\T{\covnullspaceeigmat}_\infty$, and we will show that ${\covnullspaceolmat}_0$, ${\covnullspaceolmat}_1$, ${\covnullspaceolmat}_2$, ..., ${\covnullspaceolmat}_\infty$ are members of $O(\dimconstraint)$, which also has two connected components like $O(\dimambient)$. 
To show this, first note that although ${\covnullspaceeigmat}_0\T {\covnullspaceeigmat}_0 = \eye_{\dimconstraint \times \dimconstraint}$, the matrix ${\covnullspaceeigmat}_0 {\covnullspaceeigmat}_0\T \neq \eye_{\dimambient \times \dimambient}$. Hence, for any matrix $\mathbf{A} \in \Re^{\dimambient \times \dimambient}$ in general, ${\covnullspaceeigmat}_0 {\covnullspaceeigmat}_0\T \mathbf{A} \neq \mathbf{A}$.
However, we will show that ${\covnullspaceeigmat}_0 {\covnullspaceeigmat}_0\T \coveigvec = \coveigvec$ for any vector $\coveigvec$ in the vector space $\normalspaceatjointposition$. 
Suppose ${\covnullspaceeigmat}_0 = \begin{bmatrix}
    \orthobasisvec_1 & \orthobasisvec_2 & \dots & \orthobasisvec_\dimconstraint
\end{bmatrix} \in \Re^{\dimambient \times \dimconstraint}$, then we can write ${\covnullspaceeigmat}_0 {\covnullspaceeigmat}_0\T = \sum_{i=1}^{\dimconstraint} \orthobasisvec_i \orthobasisvec_i\T$.
Since the collection $\{\orthobasisvec_1, \orthobasisvec_2, \dots, \orthobasisvec_\dimconstraint\}$ spans the vector space $\normalspaceatjointposition$, any vector $\coveigvec$ in this vector space can be expressed as $\coveigvec = \sum_{i=1}^{\dimconstraint} \alpha_i \orthobasisvec_i$.
Moreover, $\orthobasisvec_i\T \coveigvec = \orthobasisvec_i\T \sum_{j=1}^{\dimconstraint} \alpha_j \orthobasisvec_j = \alpha_i$ for any $i = 1, 2, ..., \dimconstraint$, because by definition of orthonormality $\orthobasisvec_i\T \orthobasisvec_j = 1$ for $i = j$ and $\orthobasisvec_i\T \orthobasisvec_j = 0$ for $i \neq j$.
Hence, ${\covnullspaceeigmat}_0 {\covnullspaceeigmat}_0\T \coveigvec = (\sum_{i=1}^{\dimconstraint} \orthobasisvec_i \orthobasisvec_i\T) \coveigvec = \sum_{i=1}^{\dimconstraint} (\orthobasisvec_i\T \coveigvec) \orthobasisvec_i = \sum_{i=1}^{\dimconstraint} \alpha_i \orthobasisvec_i = \coveigvec$.
Similarly, because the column vectors of ${\covnullspaceeigmat}_0$, ${\covnullspaceeigmat}_1$, ${\covnullspaceeigmat}_2$, ..., ${\covnullspaceeigmat}_\infty$ are all inside the vector space $\normalspaceatjointposition$, it follows that ${\covnullspaceeigmat}_0 {\covnullspaceeigmat}_0\T {\covnullspaceeigmat}_0 = {\covnullspaceeigmat}_0$, ${\covnullspaceeigmat}_0 {\covnullspaceeigmat}_0\T {\covnullspaceeigmat}_1 = {\covnullspaceeigmat}_1$, ${\covnullspaceeigmat}_0 {\covnullspaceeigmat}_0\T {\covnullspaceeigmat}_2 = {\covnullspaceeigmat}_2$, ..., ${\covnullspaceeigmat}_0 {\covnullspaceeigmat}_0\T {\covnullspaceeigmat}_\infty = {\covnullspaceeigmat}_\infty$. 
Similarly, it can be shown that ${\covnullspaceeigmat}_i {\covnullspaceeigmat}_i\T \coveigvec = \coveigvec$ for any vector $\coveigvec$ in the vector space $\normalspaceatjointposition$ for any $i = 0, 1, 2, ..., \infty$. 
Furthermore, ${\covnullspaceolmat}_0\T {\covnullspaceolmat}_0 = {\covnullspaceeigmat}_0\T({\covnullspaceeigmat}_0{\covnullspaceeigmat}_0\T{\covnullspaceeigmat}_0) = {\covnullspaceeigmat}_0\T{\covnullspaceeigmat}_0 = \eye_{\dimconstraint \times \dimconstraint}$, ${\covnullspaceolmat}_1\T{\covnullspaceolmat}_1 = {\covnullspaceeigmat}_1\T({\covnullspaceeigmat}_0{\covnullspaceeigmat}_0\T{\covnullspaceeigmat}_1) = {\covnullspaceeigmat}_1\T{\covnullspaceeigmat}_1 = \eye_{\dimconstraint \times \dimconstraint}$, ... ${\covnullspaceolmat}_\infty\T{\covnullspaceolmat}_\infty = {\covnullspaceeigmat}_\infty\T({\covnullspaceeigmat}_0{\covnullspaceeigmat}_0\T{\covnullspaceeigmat}_\infty) = {\covnullspaceeigmat}_\infty\T {\covnullspaceeigmat}_\infty = \eye_{\dimconstraint \times \dimconstraint}$, and ${\covnullspaceolmat}_0 {\covnullspaceolmat}_0\T = {\covnullspaceeigmat}_0\T({\covnullspaceeigmat}_0{\covnullspaceeigmat}_0\T{\covnullspaceeigmat}_0) = {\covnullspaceeigmat}_0\T{\covnullspaceeigmat}_0 = \eye_{\dimconstraint \times \dimconstraint}$, ${\covnullspaceolmat}_1{\covnullspaceolmat}_1\T = {\covnullspaceeigmat}_0\T({\covnullspaceeigmat}_1{\covnullspaceeigmat}_1\T{\covnullspaceeigmat}_0) = {\covnullspaceeigmat}_0\T{\covnullspaceeigmat}_0 = \eye_{\dimconstraint \times \dimconstraint}$, ... ${\covnullspaceolmat}_\infty{\covnullspaceolmat}_\infty\T = {\covnullspaceeigmat}_0\T({\covnullspaceeigmat}_\infty{\covnullspaceeigmat}_\infty\T{\covnullspaceeigmat}_0) = {\covnullspaceeigmat}_0\T {\covnullspaceeigmat}_0 = \eye_{\dimconstraint \times \dimconstraint}$. 
All these show that ${\covnullspaceolmat}_0$, ${\covnullspaceolmat}_1$, ${\covnullspaceolmat}_2$, ..., ${\covnullspaceolmat}_\infty \in O(\dimconstraint)$.
Moreover, using ${\covnullspaceeigmat}_0$ as the inverse projection matrix, we get ${\covnullspaceeigmat}_0 = {\covnullspaceeigmat}_0 {\covnullspaceolmat}_0$, ${\covnullspaceeigmat}_1 = {\covnullspaceeigmat}_0 {\covnullspaceolmat}_1$, ${\covnullspaceeigmat}_2 = {\covnullspaceeigmat}_0 {\covnullspaceolmat}_2$, ... ${\covnullspaceeigmat}_\infty = {\covnullspaceeigmat}_0 {\covnullspaceolmat}_\infty$.
Therefore, there is a one-to-one mapping between ${\covnullspaceeigmat}_0$, ${\covnullspaceeigmat}_1$, ${\covnullspaceeigmat}_2$, ..., ${\covnullspaceeigmat}_\infty$ and ${\covnullspaceolmat}_0$, ${\covnullspaceolmat}_1$, ${\covnullspaceolmat}_2$, ..., ${\covnullspaceolmat}_\infty$.
Furthermore, between any two of 
${\covnullspaceeigmat}_0$, ${\covnullspaceeigmat}_1$, ${\covnullspaceeigmat}_2$, ..., ${\covnullspaceeigmat}_\infty$, e.g. ${\covnullspaceeigmat}_i$ and ${\covnullspaceeigmat}_j$, there exists $\diffson \in SO(\dimconstraint)$ such that ${\covnullspaceeigmat}_i \diffson = {\covnullspaceeigmat}_j$ if their $SO(\dimconstraint)$ projections ${\covnullspaceolmat}_i$ and ${\covnullspaceolmat}_j$ both are members of the same connected component of $O(\dimconstraint)$.

Now, suppose for nearby on-manifold data points $\jointposition_a$ and $\jointposition_c$, their approximate normal spaces $\normalspaceid_{\jointposition_a}\constraintmanifold$ and $\normalspaceid_{\jointposition_c}\constraintmanifold$ are spanned by eigenvector bases $\lnormalcoordframe^a = \{\coveigvec^a_{\dimambient-\dimconstraint+1}, \dots, \coveigvec^a_\dimambient\}$ and $\lnormalcoordframe^c = \{\coveigvec^c_{\dimambient-\dimconstraint+1}, \dots, \coveigvec^c_\dimambient\}$, respectively. Due to the curvature on the manifold $\constraintmanifold$, the normal spaces $\normalspaceid_{\jointposition_a}\constraintmanifold$ and $\normalspaceid_{\jointposition_c}\constraintmanifold$ may intersect, but in general are different subspaces of $\Re^{\dimambient \times \dimambient}$. 
For the purpose of aligning the basis of $\normalspaceid_{\jointposition_a}\constraintmanifold$ to the basis of $\normalspaceid_{\jointposition_c}\constraintmanifold$, one may think to do projection of the basis vectors of $\normalspaceid_{\jointposition_a}\constraintmanifold$ into $\normalspaceid_{\jointposition_c}\constraintmanifold$. Problematically, this projection may result in a non-orthogonal basis of $\normalspaceid_{\jointposition_c}\constraintmanifold$. 
Hence, we resort to an iterative method using a differentiable Special Orthogonal Group $SO(\dimconstraint)$. In particular, we form an $\dimconstraint \times \dimconstraint$ skew-symmetric matrix $\skewsymmmat \in so(\dimconstraint)$ with $\dimconstraint (\dimconstraint - 1) / 2$ differentiable parameters, where $so(\dimconstraint)$ is the Lie algebra of $SO(\dimconstraint)$, i.e, the set of all skew-symmetric $\dimconstraint \times \dimconstraint$ matrices, 
and transform it through a differentiable exponential mapping (or matrix exponential) to get $\diffson = \exp(\skewsymmmat)$ with $\exp: so(\dimconstraint) \rightarrow SO(\dimconstraint)$. With $\covnullspaceeigmat^a = 
\begin{bmatrix}
    \coveigvec^a_{\dimambient-\dimconstraint+1} & \dots & \coveigvec^a_\dimambient
\end{bmatrix}$ and $\covnullspaceeigmat^c = 
\begin{bmatrix}
    \coveigvec^c_{\dimambient-\dimconstraint+1} & \dots & \coveigvec^c_\dimambient
\end{bmatrix}$, we can do an iterative training process to minimize the alignment error between $\covnullspaceeigmat^a \diffson$ and $\covnullspaceeigmat^c$, that is $\osaloss = \norm{\eye_{\dimconstraint \times \dimconstraint} - (\covnullspaceeigmat^a \diffson)\T \covnullspaceeigmat^c}_2^2$. Depending on whether both $\covnullspaceolmat^a$ and $\covnullspaceolmat^c$ (which are the projections of $\covnullspaceeigmat^a$ and $\covnullspaceeigmat^c$, respectively, to $O(\dimconstraint)$) are members of the same connected component of $O(\dimconstraint)$ or not, this alignment process may succeed or fail. 
However, if we define $\flippedcovnullspaceeigmat^a = 
\begin{bmatrix}
    -\coveigvec^a_{\dimambient-\dimconstraint+1} & \coveigvec^a_{\dimambient-\dimconstraint+2} & \dots & \coveigvec^a_\dimambient
\end{bmatrix}$ and $\flippedcovnullspaceeigmat^c = 
\begin{bmatrix}
    -\coveigvec^c_{\dimambient-\dimconstraint+1} & \coveigvec^c_{\dimambient-\dimconstraint+2} & \dots & \coveigvec^c_\dimambient
\end{bmatrix}$, two out of the four pairs $(\covnullspaceeigmat^a, \covnullspaceeigmat^c)$, $(\flippedcovnullspaceeigmat^a, \covnullspaceeigmat^c)$, $(\covnullspaceeigmat^a, \flippedcovnullspaceeigmat^c)$, and $(\flippedcovnullspaceeigmat^a, \flippedcovnullspaceeigmat^c)$ will be pairs in the same connected component. Thus, two of these pairs will achieve minimum alignment errors after training the differentiable Special Orthogonal Groups $SO(\dimconstraint)$ on these pairs, indicating successful alignment. These are the main insights for our \emph{local} alignment of neighboring normal spaces of on-manifold data points.

For the \emph{global} alignment of the normal spaces, we represent the on-manifold data points as a graph. 
Our Orthogonal Subspace Alignment (OSA) is outlined in \Cref{alg:osa}. 
We begin by constructing a sparse graph of nearest neighbor connections of each on-manifold data point, followed by the construction of this graph into an (un-directed) minimum spanning tree (MST), and eventually the conversion of the MST to a directed acyclic graph (DAG). This graph construction is detailed in \cref{algl:startgraphconstruction} -- \cref{algl:endgraphconstruction} of \Cref{alg:osa}.

Each directed edge in the DAG represents a pair of on-manifold data points whose normal spaces are to be aligned locally.
Our insights for the local alignment of neighboring normal spaces are implemented in \cref{algl:startlocalalignment} -- \cref{algl:endlocalalignment} of \Cref{alg:osa}.
In the actual implementation, these local alignment computations are done as a vectorized computation which is faster than doing it in a for-loop as presented in \Cref{alg:osa}; this for-loop presentation is made only for the sake of clarity. 
We initialize the $\dimconstraint (\dimconstraint - 1) / 2$ differentiable parameters of the $\dimconstraint \times \dimconstraint$ skew-symmetric matrix $\skewsymmmat$ with near zero random numbers, which essentially will map to a near identity matrix $\eye_{\dimconstraint \times \dimconstraint}$ of $\diffson$ via the $\exp()$ mapping, as stated in \cref{algl:diffsoninit} of \Cref{alg:osa}\footnote{Although most of our {\ecmnn} implementation is done in PyTorch \cite{Paszke_PyTorch_AutoDiff_2017}, the OSA algorithm is implemented in TensorFlow \cite{TensorFlowBib}, because at the time of implementation of the OSA algorithm, PyTorch did not support the differentiable matrix exponential (i.e. the exponential mapping) computation yet while TensorFlow did.}. This is reasonable because we assume that most of the neighboring normal spaces are already/close to being aligned initially. We optimize the alignment of the four pairs $(\covnullspaceeigmat^a, \covnullspaceeigmat^c)$, $(\flippedcovnullspaceeigmat^a, \covnullspaceeigmat^c)$, $(\covnullspaceeigmat^a, \flippedcovnullspaceeigmat^c)$, and $(\flippedcovnullspaceeigmat^a, \flippedcovnullspaceeigmat^c)$ in \cref{algl:firstiterativealignmenterrorminimization} -- \cref{algl:fourthiterativealignmenterrorminimization} of \Cref{alg:osa}.

Once the local alignments are done, the algorithm then traverses the DAG in breadth-first order, starting from the root $\rootid$, where the orientation of the root is already chosen and committed to. During the breadth-first traversal of the DAG, three things are done: First, the orientation of each point is chosen based on the minimum alignment loss; second, the local alignment transforms are compounded/integrated along the path from root to the point; and finally, the (globally) aligned orthogonal basis of each point is computed and returned as the result of the algorithm. These steps are represented by \cref{algl:startglobalalignment} -- \cref{algl:endglobalalignment} of \Cref{alg:osa}.

\begin{singlespace*}
\begin{algorithm}[H]
	\caption{Orthogonal Subspace Alignment (OSA)}
	\label{alg:osa}
	\begin{algorithmic}[1]
	    \Function{OSA}{$\{(\jointposition \in \onconstraintconfigspace, \text{orthogonal basis stacked as matrix } \covnullspaceeigmat \text{ associated with } \normalspaceatjointposition)\}$}
    		\State \label{algl:startgraphconstruction} \# construct a sparse graph between each data point $\jointposition \in \onconstraintconfigspace$ with its $\mstnumnearestneighbor$ nearest neighbors, 
    		\State \# followed by minimum spanning tree and directed acyclic graph computations 
    		\State \# to obtain directed edges $\dagedges$; $\mstnumnearestneighbor$ needs to be chosen to be a value as small as possible that 
    		\State \# still results in all non-root points $\{\jointposition \in \onconstraintconfigspace \backslash \{\rootjointposition\} \}$ being reachable from the root point $\rootjointposition$:
    		\State $\nnsparsegraph \gets \textit{computeNearestNeighborsSparseGraph}(\{\jointposition \in \onconstraintconfigspace\}, \mstnumnearestneighbor)$
    		\State $\mst \gets \textit{computeMinimumSpanningTree}(\nnsparsegraph)$
    		\State \label{algl:endgraphconstruction} $\dagedges \gets \textit{computeDirectedAcyclicGraphEdgesByBreadthFirstTree}(\mst)$
    		\For{ each directed edge $\directededge = (\jointposition_c, \jointposition_a) \in \dagedges$} \label{algl:startlocalalignment}
    		    \State Obtain $\covnullspaceeigmat^a = 
                \begin{bmatrix}
                    \coveigvec^a_{\dimambient-\dimconstraint+1} & \dots & \coveigvec^a_\dimambient
                \end{bmatrix} \in \Re^{\dimambient \times \dimconstraint}$ associated with the source subspace $\normalspaceid_{\jointposition_a}\constraintmanifold$ 
    		    \State Obtain $\covnullspaceeigmat^c = 
                \begin{bmatrix}
                    \coveigvec^c_{\dimambient-\dimconstraint+1} & \dots & \coveigvec^c_\dimambient
                \end{bmatrix} \in \Re^{\dimambient \times \dimconstraint}$ associated with the target subspace $\normalspaceid_{\jointposition_c}\constraintmanifold$
                \State Define $\flippedcovnullspaceeigmat^a = 
                \begin{bmatrix}
                    -\coveigvec^a_{\dimambient-\dimconstraint+1} & \coveigvec^a_{\dimambient-\dimconstraint+2} & \dots & \coveigvec^a_\dimambient
                \end{bmatrix} \in \Re^{\dimambient \times \dimconstraint}$
                \State Define $\flippedcovnullspaceeigmat^c = 
                \begin{bmatrix}
                    -\coveigvec^c_{\dimambient-\dimconstraint+1} & \coveigvec^c_{\dimambient-\dimconstraint+2} & \dots & \coveigvec^c_\dimambient
                \end{bmatrix} \in \Re^{\dimambient \times \dimconstraint}$
                \State \label{algl:diffsoninit} Define differentiable $SO(\dimconstraint)$ $\diffson^{\overrightarrow{a}\overrightarrow{c}}$, $\diffson^{\overrightarrow{a}\overleftarrow{c}}$, $\diffson^{\overleftarrow{a}\overrightarrow{c}}$, and $\diffson^{\overleftarrow{a}\overleftarrow{c}}$, initialized near identity
                \State \# try optimizing the alignment of the 4 possible pairs:
                \State \label{algl:firstiterativealignmenterrorminimization} $(\diffson^{\overrightarrow{a}\overrightarrow{c}}, \loss_{\overrightarrow{a}\overrightarrow{c}}) \gets \textit{iterativelyMinimizeAlignmentError}(\covnullspaceeigmat^a \diffson^{\overrightarrow{a}\overrightarrow{c}}, \covnullspaceeigmat^c)$
                \State $(\diffson^{\overrightarrow{a}\overleftarrow{c}}, \loss_{\overrightarrow{a}\overleftarrow{c}}) \gets \textit{iterativelyMinimizeAlignmentError}(\covnullspaceeigmat^a \diffson^{\overrightarrow{a}\overleftarrow{c}}, \flippedcovnullspaceeigmat^{c})$
                \State $(\diffson^{\overleftarrow{a}\overrightarrow{c}}, \loss_{\overleftarrow{a}\overrightarrow{c}}) \gets \textit{iterativelyMinimizeAlignmentError}(\flippedcovnullspaceeigmat^{a} \diffson^{\overleftarrow{a}\overrightarrow{c}}, \covnullspaceeigmat^c)$
                \State \label{algl:fourthiterativealignmenterrorminimization} $(\diffson^{\overleftarrow{a}\overleftarrow{c}}, \loss_{\overleftarrow{a}\overleftarrow{c}}) \gets \textit{iterativelyMinimizeAlignmentError}(\flippedcovnullspaceeigmat^{a} \diffson^{\overleftarrow{a}\overleftarrow{c}}, \flippedcovnullspaceeigmat^{c})$
                \State \# record optimized local alignment rotation matrices and its associated loss w/ the edge: 
                \State \label{algl:endlocalalignment} Associate $(\diffson^{\overrightarrow{a}\overrightarrow{c}}, \loss_{\overrightarrow{a}\overrightarrow{c}})$, $(\diffson^{\overrightarrow{a}\overleftarrow{c}}, \loss_{\overrightarrow{a}\overleftarrow{c}})$, $(\diffson^{\overleftarrow{a}\overrightarrow{c}}, \loss_{\overleftarrow{a}\overrightarrow{c}})$, $(\diffson^{\overleftarrow{a}\overleftarrow{c}}, \loss_{\overleftarrow{a}\overleftarrow{c}})$ with $\directededge$
    		\EndFor
    		\State \label{algl:startglobalalignment} \# commit on the orientation of the root point as un-flipped ($\overrightarrow{\rootid}$) instead of flipped ($\overleftarrow{\rootid}$):
    		\State $ori(\rootid) = \overrightarrow{\rootid}$
    		\State \# define the compound/global alignment rotation matrix of the root as an identity matrix:
    		\State $\globalalignmentrotmat^{\rootid} = \eye_{\dimconstraint \times \dimconstraint}$
    		\State \# aligned orthogonal basis of the root is:
    		\State ${\covnullspaceeigmat}^{\text{aligned}, \rootid} = \covnullspaceeigmat^{\rootid}$
    		\State \# do breadth-first traversal from root to:
    		\State \# (1) select the orientation $ori()$ of each point based on the minimum alignment loss, 
    		\State \# (2) compound/integrate the local alignment transforms $\globalalignmentrotmat$ along the path to the point,  
    		\State \# (3) and finally compute the aligned orthogonal basis $\alignedcovnullspaceeigmat$:
    		\State $Q = \textit{Queue}()$
    		\State $Q.\textit{enqueue}(\textit{childrenOfNodeInGraph}(\rootid, \dagedges))$
    		\While{$\textit{size}(Q) > 0$}
    		    \State $\currentnodeid = Q.\textit{dequeue}()$
    		    \State $Q.\textit{enqueue}(\textit{childrenOfNodeInGraph}(\currentnodeid, \dagedges))$
    		    \State $\parentnodeid = \textit{parentOfNodeInGraph}(\currentnodeid, \dagedges)$
    		    \State \# select the local alignment rotation matrix based on the minimum alignment loss 
    		    \State \# among the two possibilities:
    		    \If{$\loss_{\overrightarrow{\currentnodeid}ori(\parentnodeid)} < \loss_{\overleftarrow{\currentnodeid}ori(\parentnodeid)}$}
    		        \State $ori(\currentnodeid) = \overrightarrow{\currentnodeid}$
    		        \State $\globalalignmentrotmat^{\currentnodeid} = \diffson^{\overrightarrow{\currentnodeid}ori(\parentnodeid)} \globalalignmentrotmat^{\parentnodeid}$
    		        \State ${\covnullspaceeigmat}^{\text{aligned}, \currentnodeid} = \covnullspaceeigmat^{\currentnodeid} \globalalignmentrotmat^{\currentnodeid}$
    		    \Else
    		        \State $ori(\currentnodeid) = \overleftarrow{\currentnodeid}$
    		        \State $\globalalignmentrotmat^{\currentnodeid} = \diffson^{\overleftarrow{\currentnodeid}ori(\parentnodeid)} \globalalignmentrotmat^{\parentnodeid}$
    		        \State ${\covnullspaceeigmat}^{\text{aligned}, \currentnodeid} = \flippedcovnullspaceeigmat^{\currentnodeid} \globalalignmentrotmat^{\currentnodeid}$
    		    \EndIf
    		\EndWhile
    		\State \label{algl:endglobalalignment} \Return $\{\alignedcovnullspaceeigmat \text{ associated with } \normalspaceatjointposition \text{ for each } \jointposition \in \onconstraintconfigspace\}$
    	\EndFunction
	\end{algorithmic}
\end{algorithm}
\end{singlespace*}

\end{document}